\documentclass{article}

 \pdfoutput=1

\usepackage{etoolbox}
\usepackage{comment} 
\newtoggle{colt}
\togglefalse{colt}

\newcommand{\arxiv}[1]{\iftoggle{colt}{}{#1}}
\newcommand{\loose}{\looseness=-1}

\arxiv{
\usepackage[letterpaper, left=1in, right=1in, top=1in,
bottom=1in]{geometry}
  \usepackage{parskip}
}

\PassOptionsToPackage{hypertexnames=false}{hyperref}  %

\usepackage{amsmath}

\arxiv{
\usepackage[dvipsnames]{xcolor}
\usepackage[colorlinks=true, linkcolor=blue!70!black, citecolor=blue!70!black,urlcolor=black,breaklinks=true]{hyperref}
}

\usepackage{microtype}
\usepackage{hhline}

\usepackage{amsthm}
\usepackage{mathtools}
\usepackage{bbm}
\usepackage{amsfonts}
\usepackage{amssymb}

\usepackage{algorithm}
\usepackage{verbatim}
\usepackage[noend]{algpseudocode} %
\newcommand{\multiline}[1]{\parbox[t]{\dimexpr\linewidth-\algorithmicindent}{#1}}

\usepackage[nameinlink,capitalize]{cleveref}

\makeatletter
\newcommand{\neutralize}[1]{\expandafter\let\csname c@#1\endcsname\count@}
\makeatother

\usepackage{algorithm}

\arxiv{
\usepackage{natbib}
\bibliographystyle{plainnat}
\bibpunct{(}{)}{;}{a}{,}{,}
}

\usepackage{xpatch}

\usepackage{thm-restate}
\declaretheorem[name=Theorem,parent=section]{theorem}
\declaretheorem[name=Lemma,parent=section]{lemma}
\declaretheorem[name=Assumption, parent=section]{assumption}
\declaretheorem[name=Definition, parent=section]{definition}

\declaretheorem[name=Corollary, parent=section]{corollary}

\declaretheorem[name=Remark, parent=section]{remark}
\declaretheorem[name=Proposition, parent=section]{proposition}

\declaretheorem[name=Conjecture, parent=section]{conjecture}

\usepackage{crossreftools}
  \newcommand{\creftitle}[1]{\crtcref{#1}}

\makeatletter
  \renewenvironment{proof}[1][Proof]%
  {%
   \par\noindent{\bfseries\upshape {#1.}\ }%
  }%
  {\qed\newline}
  \makeatother

\newcommand{\sssref}[1]{\texorpdfstring{\hyperref[#1]{\mbox{Section \ref*{#1}}}}{Section \ref*{#1}}}

\xpatchcmd{\proof}{\itshape}{\normalfont\proofnameformat}{}{}
\newcommand{\proofnameformat}{\bfseries}

\newcommand{\pfref}[1]{Proof of \cref{#1}}

\renewcommand{\eqref}[1]{\texorpdfstring{\hyperref[#1]{(\ref*{#1})}}{(\ref*{#1})}}
\crefformat{equation}{#2Eq.\,(#1)#3}
\Crefformat{equation}{#2Eq.\,(#1)#3}

\Crefformat{figure}{#2Figure~#1#3}
\Crefformat{line}{#2Line~#1#3}
\crefformat{line}{#2Line~#1#3}
\Crefformat{assumption}{#2Assumption~#1#3}

\Crefname{assumption}{Assumption}{Assumptions}

\crefname{fact}{Fact}{Facts}

\Crefformat{figure}{#2Figure #1#3}
\Crefformat{assumption}{#2Assumption #1#3}

\usepackage{xparse}

\ExplSyntaxOn
\DeclareDocumentCommand{\XDeclarePairedDelimiter}{mm}
 {
  \__egreg_delimiter_clear_keys: %
  \keys_set:nn { egreg/delimiters } { #2 }
  \use:x %
   {
    \exp_not:n {\NewDocumentCommand{#1}{sO{}m} }
     {
      \exp_not:n { \IfBooleanTF{##1} }
       {
        \exp_not:N \egreg_paired_delimiter_expand:nnnn
         { \exp_not:V \l_egreg_delimiter_left_tl }
         { \exp_not:V \l_egreg_delimiter_right_tl }
         { \exp_not:n { ##3 } }
         { \exp_not:V \l_egreg_delimiter_subscript_tl }
       }
       {
        \exp_not:N \egreg_paired_delimiter_fixed:nnnnn 
         { \exp_not:n { ##2 } }
         { \exp_not:V \l_egreg_delimiter_left_tl }
         { \exp_not:V \l_egreg_delimiter_right_tl }
         { \exp_not:n { ##3 } }
         { \exp_not:V \l_egreg_delimiter_subscript_tl }
       }
     }
   }
 }

\keys_define:nn { egreg/delimiters }
 {
  left      .tl_set:N = \l_egreg_delimiter_left_tl,
  right     .tl_set:N = \l_egreg_delimiter_right_tl,
  subscript .tl_set:N = \l_egreg_delimiter_subscript_tl,
 }

\cs_new_protected:Npn \__egreg_delimiter_clear_keys:
 {
  \keys_set:nn { egreg/delimiters } { left=.,right=.,subscript={} }
 }

\cs_new_protected:Npn \egreg_paired_delimiter_expand:nnnn #1 #2 #3 #4
 {%
  \mathopen{}
  \mathclose\c_group_begin_token
   \left#1
   #3
   \group_insert_after:N \c_group_end_token
   \right#2
   \tl_if_empty:nF {#4} { \c_math_subscript_token {#4} }
 }
\cs_new_protected:Npn \egreg_paired_delimiter_fixed:nnnnn #1 #2 #3 #4 #5
 {
  \mathopen{#1#2}#4\mathclose{#1#3}
  \tl_if_empty:nF {#5} { \c_math_subscript_token {#5} }
 }
\ExplSyntaxOff

\XDeclarePairedDelimiter{\supnorm}{
  left=\lVert,
  right=\rVert,
  subscript=\infty
  }
 
\usepackage[utf8]{inputenc} %
\usepackage[T1]{fontenc}    %
\usepackage{url}            %
\usepackage{booktabs}       %
\usepackage{amsfonts}       %
\usepackage{nicefrac}       %
\usepackage{microtype}      %
\usepackage{float}
\usepackage{placeins}

\usepackage{tocloft}            %

\usepackage{makecell}
\usepackage{amsmath}

\usepackage{enumitem}

\usepackage{breakcites}

\usepackage{mathrsfs}

\usepackage{multicol}
  
\usepackage{colortbl}
\usepackage{caption}
\usepackage{setspace}

\usepackage{transparent}

\usepackage{inconsolata}
\usepackage[scaled=.90]{helvet}
\usepackage{xspace}

\usepackage{bm}
\newcommand{\x}{\bm{x}}

\usepackage{tablefootnote}

\DeclareFontFamily{U}{jkpmia}{}
\DeclareFontShape{U}{jkpmia}{m}{it}{<->s*jkpmia}{}
\DeclareFontShape{U}{jkpmia}{bx}{it}{<->s*jkpbmia}{}
\DeclareMathAlphabet{\mathfrak}{U}{jkpmia}{m}{it}
\SetMathAlphabet{\mathfrak}{bold}{U}{jkpmia}{bx}{it}

\DeclarePairedDelimiter{\abs}{\lvert}{\rvert} %
\DeclarePairedDelimiter{\brk}{[}{]}
\DeclarePairedDelimiter{\crl}{\{}{\}}
\DeclarePairedDelimiter{\prn}{(}{)}
\DeclarePairedDelimiter{\nrm}{\|}{\|}
\DeclarePairedDelimiter{\tri}{\langle}{\rangle}

\DeclarePairedDelimiter{\ceil}{\lceil}{\rceil}

\let\Pr\undefined
\let\P\undefined
\DeclareMathOperator{\En}{\mathbb{E}}

\DeclareMathOperator{\P}{P}
\DeclareMathOperator{\Pr}{Pr}

\DeclareMathOperator*{\argmin}{arg\,min} %
\DeclareMathOperator*{\argmax}{arg\,max}

\newcommand{\mb}[1]{\boldsymbol{#1}}
\newcommand{\wt}[1]{\widetilde{#1}}
\newcommand{\wh}[1]{\widehat{#1}}
\newcommand{\wb}[1]{\widebar{#1}}

\def\ddefloop#1{\ifx\ddefloop#1\else\ddef{#1}\expandafter\ddefloop\fi}
\def\ddef#1{\expandafter\def\csname bb#1\endcsname{\ensuremath{\mathbb{#1}}}}
\ddefloop ABCDEFGHIJKLMNOPQRSTUVWXYZ\ddefloop
\def\ddefloop#1{\ifx\ddefloop#1\else\ddef{#1}\expandafter\ddefloop\fi}
\def\ddef#1{\expandafter\def\csname b#1\endcsname{\ensuremath{\mathbf{#1}}}}
\ddefloop ABCDEFGHIJKLMNOPQRSTUVWXYZ\ddefloop
\def\ddef#1{\expandafter\def\csname sf#1\endcsname{\ensuremath{\mathsf{#1}}}}
\ddefloop ABCDEFGHIJKLMNOPQRSTUVWXYZ\ddefloop
\def\ddef#1{\expandafter\def\csname c#1\endcsname{\ensuremath{\mathcal{#1}}}}
\ddefloop ABCDEFGHIJKLMNOPQRSTUVWXYZ\ddefloop
\def\ddef#1{\expandafter\def\csname h#1\endcsname{\ensuremath{\widehat{#1}}}}
\ddefloop ABCDEFGHIJKLMNOPQRSTUVWXYZ\ddefloop
\def\ddef#1{\expandafter\def\csname hc#1\endcsname{\ensuremath{\widehat{\mathcal{#1}}}}}
\ddefloop ABCDEFGHIJKLMNOPQRSTUVWXYZ\ddefloop
\def\ddef#1{\expandafter\def\csname t#1\endcsname{\ensuremath{\widetilde{#1}}}}
\ddefloop ABCDEFGHIJKLMNOPQRSTUVWXYZ\ddefloop
\def\ddef#1{\expandafter\def\csname tc#1\endcsname{\ensuremath{\widetilde{\mathcal{#1}}}}}
\ddefloop ABCDEFGHIJKLMNOPQRSTUVWXYZ\ddefloop
\def\ddefloop#1{\ifx\ddefloop#1\else\ddef{#1}\expandafter\ddefloop\fi}
\def\ddef#1{\expandafter\def\csname scr#1\endcsname{\ensuremath{\mathscr{#1}}}}
\ddefloop ABCDEFGHIJKLMNOPQRSTUVWXYZ\ddefloop

\newcommand{\ind}{\mathbbm{1}}    %

\newcommand{\veps}{\varepsilon}

\newcommand{\ldef}{\vcentcolon=}
\newcommand{\rdef}{=\vcentcolon}

\newcommand{\cXspan}{\cX_{\texttt{span}}}

\newcommand{\Trounds}{T_{\mathsf{prompt}}}

\newcommand{\pistarb}{\pistar_{\beta}}
\newcommand{\pistarbh}{\pistar_{h,\beta}}

\newcommand{\pirefh}[1][h]{\pi_{h,\texttt{ref}}}

\newcommand{\Sigmafull}{\Sigma_{\texttt{full}}}
\newcommand{\Sigmaexp}{\Sigma_{\texttt{exp}}}
\newcommand{\pibt}{\pibar_{\theta\ind{t}}}

\newcommand{\Tdatafull}{\Tdata(\veps,\delta)}
\newcommand{\Tdatafullsq}{\Tdata^2(\veps,\delta)}
\newcommand{\Tsamplefull}{\Tsample(\veps,\delta)}
\newcommand{\Tcompfull}{\Tsamplefull}
\newcommand{\Tdata}{\ndata}
\newcommand{\Tsample}{\nsample}
\newcommand{\Tcomp}{T_{\texttt{comp}}}

\newcommand{\Tprompt}{T_{\texttt{prompt}}}
\newcommand{\nexp}{T_{\texttt{exp}}}
\newcommand{\nspan}{T_{\texttt{span}}}

\newcommand{\Tspan}{\nspan}
\newcommand{\Texp}{\nexp}

\newcommand{\Tcompa}{T^{\texttt{auto}}_{\texttt{comp}}}
\newcommand{\Tcompafull}{\Tcompa(\veps,\delta)}

\newcommand{\mtalg}{\texttt{MTSS}\xspace}%
\newcommand{\uncertsa}{\texttt{UncertainStateAction}\xspace}
\newcommand{\vepsspan}{\veps_{\texttt{span}}}
\newcommand{\vepsprompt}{\veps_{\texttt{prompt}}}

\newcommand{\pibtheta}[1][\theta]{\wb{\pi}_{#1}}
\newcommand{\pibst}[1][\thetastar]{\wb{\pi}_{#1}}
\newcommand{\phidel}{\varphi}
\newcommand{\phidelbar}{\wb{\phidel}}
\newcommand{\phidb}{\phidelbar}

\newcommand{\piref}{\pi_{\texttt{ref}}}

\newcommand{\termi}{\mathrm{I}}
\newcommand{\termii}{\mathrm{II}}
\newcommand{\termiii}{\mathrm{III}}

\newcommand{\Zgood}{\cZ_{\texttt{good}}}
\newcommand{\cZgood}{\cZ_{\texttt{good}}}

\newcommand{\Jbbeta}{\wb{J}_{\beta}}

\newcommand{\Psispan}{\Psi_{\texttt{span}}}
\newcommand{\Sigmaspan}{\Sigma_{\texttt{span}}}
\newcommand{\Sigmas}{\Sigmaspan}

\newcommand{\cDspan}{\cD_{\texttt{span}}}
\newcommand{\cDexp}{\cD_{\texttt{exp}}}

\newcommand{\pihatf}{\pihat_f}

\newcommand{\cEconc}{\cE_{\texttt{conc}}}
\newcommand{\cEspan}{\cE_{\texttt{span}}}

\newcommand{\Mrej}{M_{\texttt{rej}}}
\newcommand{\deltarej}{\delta_{\texttt{rej}}}

\newcommand{\rbar}{\wb{r}}

\newcommand{\Id}{I_{d}}

\newcommand{\pistar}{\pi^{\star}}
\newcommand{\pihat}{\wh{\pi}}

\newcommand{\spanalg}{\texttt{SpannerSampling}\xspace}

\newcommand{\rejection}{\texttt{SoftmaxSampler}\xspace}
\newcommand{\rejectionargs}{\rejection_{\beta,M,\delta}(f\midsem x, \piref)}
\newcommand{\cEaccept}{\cE_{\texttt{accept}}}
\newcommand{\Zhat}{\wh{Z}}

\newcommand{\pif}{\pi_{f}}
\newcommand{\pifstar}{\pi_{\fstar}}

\newcommand{\Cinf}{C_{\infty}}

\newcommand{\Qstarb}{Q^{\star}_{\beta}}

\newcommand{\Ccov}{C_{\texttt{cov}}}

\newcommand{\framework}{sampling oracle\xspace}

\newcommand{\Jbeta}{J_{\beta}}
\newcommand{\ndata}{T_{\texttt{data}}}
\newcommand{\nsample}{T_{\texttt{comp}}}

\newcommand{\pitheta}{\pi_{\theta}}
\newcommand{\pist}{\pi_{\theta^{\star}}}
\newcommand{\thetast}{\theta^{\star}}
\newcommand{\ball}{\bbB_2(1)}

\newcommand{\yp}{y_{+}}
\newcommand{\ym}{y_{-}}

\newcommand{\Rmax}{R_{\texttt{max}}}

\newcommand{\rstar}{r^{\star}}

\newcommand{\pref}{\texttt{pref}}

\newcommand{\dpo}{\texttt{DPO}\xspace}
\newcommand{\onlinedpo}{\texttt{OnlineDPO}\xspace}
\newcommand{\xpo}{\texttt{XPO}\xspace}

\newcommand{\Qstar}{Q^\star}

\newcommand{\Unif}{\mathsf{Unif}}

\newcommand{\pibar}{\wb{\pi}}

  \newcommand{\afrak}{\mathfrak{a}}

  \newcommand{\ba}{\mathbf{a}}
  \newcommand{\bx}{\mathbf{x}}
  \newcommand{\br}{\mathbf{r}}

\newcommand{\vepsstat}{\veps_{\texttt{stat}}}

\renewcommand{\emptyset}{\varnothing}

\newcommand{\filt}{\mathscr{F}}

\newcommand{\mainalg}{\spanalg}

\newcommand{\M}[1]{^{{\scriptscriptstyle M}}}  %

\newcommand{\fstar}{f^{\star}}

\newcommand{\thetahat}{\wh{\theta}}
\newcommand{\thetastar}{\theta^{\star}}

\newcommand{\algcommentlight}[1]{\textcolor{blue!70!black}{\transparent{0.5}\footnotesize{\texttt{\textbf{//\hspace{2pt}#1}}}}}
\newcommand{\algcommentbig}[1]{\textcolor{blue!70!black}{\transparent{0.5}\footnotesize{\texttt{\textbf{/*
          #1~*/}}}}}

\newcommand{\midsem}{\,;}

\newcommand{\trn}{\top}

\newcommand{\psdgt}{\succ}
\newcommand{\approxleq}{\lesssim}
\newcommand{\approxgeq}{\gtrsim}

\newcommand{\fhat}{\wh{f}}

\renewcommand{\ind}[1]{^{{\scriptscriptstyle#1}}}

\newcommand{\bigoh}{O}
\newcommand{\bigoht}{\wt{O}}
\newcommand{\bigom}{\Omega}

\newcommand{\indic}{\mathbb{I}}

\renewcommand{\Pr}{\bbP}

\newcommand{\poly}{\mathrm{poly}}
\newcommand{\polylog}{\mathrm{polylog}}

\newcommand{\kl}[2]{D_{\mathsf{KL}}\prn*{#1\,\|\,#2}}
\newcommand{\Dkl}[2]{D_{\mathsf{KL}}\prn*{#1\,\|\,#2}}

\newcommand{\Dhel}[2]{D_{\mathsf{H}}\prn*{#1,#2}}

\newcommand{\Dhels}[2]{D^{2}_{\mathsf{H}}\prn*{#1,#2}}

\newcommand{\Dtv}[2]{D_{\mathsf{TV}}\prn*{#1,#2}}

\newcommand{\Ber}{\mathrm{Ber}}

\newcommand{\unif}{\mathrm{unif}}

\newcommand{\mathand}{\quad\text{and}\quad}

\def\multiset#1#2{\ensuremath{\left(\kern-.3em\left(\genfrac{}{}{0pt}{}{#1}{#2}\right)\kern-.3em\right)}}

\newcommand{\iid}{i.i.d.\xspace}

\renewcommand{\emptyset}{\varnothing}

\newcommand{\MA}{\mathcal{A}}
\newcommand{\norm}[1]{\left \lVert #1 \right \rVert}
\DeclareMathOperator*{\EE}{\mathbb{E}}
\newcommand{\RR}{\mathbb{R}}
\newcommand{\NN}{\mathbb{N}}
\newcommand{\st}{\star}

\newcommand{\Alg}{\mathtt{Alg}}
\newcommand{\MX}{\mathcal{X}}
\newcommand{\MB}{\mathcal{B}}
\newcommand{\MC}{\mathcal{C}}
\DeclareMathOperator{\val}{val}

\newcommand{\valDNF}{\val_{\mathsf{DNF}}}
\newcommand{\Algtil}{\widetilde \Alg}
\newcommand{\Algbar}{\overline \Alg}
\newcommand{\epref}{\varepsilon_{\mathsf{ref}}}
\newcommand{\thetafinal}{\theta_{\mathsf{final}}}
\DeclareMathOperator{\sgn}{sgn}

\newcommand{\ystar}{{y^\st}}
\newcommand{\Cstar}{{C^\st}}

\newcommand{\Ccond}{C_{\texttt{cond}}}
\newcommand{\fraka}{\mathfrak{a}}
\newcommand{\refe}{\texttt{ref}}

\newcommand{\fitval}{\texttt{FitValue}}

\newcommand{\tnorm}[1]{\lvert\!\lvert\!\lvert #1 \rvert\!\rvert\!\rvert}
\newcommand{\Lhat}{\widehat{L}}
\newcommand{\spann}{\texttt{span}}

\newcommand{\varphibar}{\widebar{\varphi}}
\newcommand{\algcommentbiglight}[1]{\textcolor{blue!70!black}{\transparent{0.5}\footnotesize{\texttt{\textbf{/* #1~*/}}}}}
\newcommand{\ldotst}{%
	\mathinner{{\ldotp}{\ldotp}}%
}

\renewcommand{\a}{\bm{a}}
\renewcommand{\ba}{\bm{a}}
\renewcommand{\bx}{\bm{x}}
\renewcommand{\x}{\bm{x}}

\newcommand{\z}{\bm{z}}
\renewcommand{\br}{\bm{r}}
\newcommand{\inner}[2]{\langle #1,#2\rangle}

\newcommand{\reg}{\texttt{reg}}
\newcommand{\E}{\mathbb{E}}
\newcommand{\wtilde}{\widetilde}
\newcommand{\reals}{\mathbb{R}}

\newcommand{\tauind}{j}

\renewcommand{\P}{\mathbb{P}}

\newcommand{\Nspanb}{\widebar{N}_\spann}
\newcommand{\Nspann}{N_\spann}

\newcommand{\Nbar}{\widebar{N}}
\newcommand{\Nreg}{N_\texttt{reg}}
\newcommand{\softmaxsample}{\texttt{SoftmaxSamplerDensity}\xspace}
\newcommand{\uveps}{\underline{\veps}_\texttt{span}}
\newcommand{\vepslip}{\veps_\texttt{span}}
\newcommand{\opsdp}{\texttt{Optimistic}-\texttt{PSDP}}
\DeclareMathOperator{\Law}{Law}
\newcommand{\by}{\bm{y}}
\newcommand{\brho}{\bm{\rho}}

\newcommand{\tauindl}{j}

 \let\underbar\undefined

\makeatletter
\let\save@mathaccent\mathaccent
\newcommand*\if@single[3]{%
  \setbox0\hbox{${\mathaccent"0362{#1}}^H$}%
  \setbox2\hbox{${\mathaccent"0362{\kern0pt#1}}^H$}%
  \ifdim\ht0=\ht2 #3\else #2\fi
  }
\newcommand*\rel@kern[1]{\kern#1\dimexpr\macc@kerna}
\newcommand*\widebar[1]{\@ifnextchar^{{\wide@bar{#1}{0}}}{\wide@bar{#1}{1}}}
\newcommand*\underbar[1]{\@ifnextchar_{{\under@bar{#1}{0}}}{\under@bar{#1}{1}}}
\newcommand*\wide@bar[2]{\if@single{#1}{\wide@bar@{#1}{#2}{1}}{\wide@bar@{#1}{#2}{2}}}
\newcommand*\under@bar[2]{\if@single{#1}{\under@bar@{#1}{#2}{1}}{\under@bar@{#1}{#2}{2}}}
\newcommand*\wide@bar@[3]{%
  \begingroup
  \def\mathaccent##1##2{%
    \let\mathaccent\save@mathaccent
    \if#32 \let\macc@nucleus\first@char \fi
    \setbox\z@\hbox{$\macc@style{\macc@nucleus}_{}$}%
    \setbox\tw@\hbox{$\macc@style{\macc@nucleus}{}_{}$}%
    \dimen@\wd\tw@
    \advance\dimen@-\wd\z@
    \divide\dimen@ 3
    \@tempdima\wd\tw@
    \advance\@tempdima-\scriptspace
    \divide\@tempdima 10
    \advance\dimen@-\@tempdima
    \ifdim\dimen@>\z@ \dimen@0pt\fi
    \rel@kern{0.6}\kern-\dimen@
    \if#31
      \overline{\rel@kern{-0.6}\kern\dimen@\macc@nucleus\rel@kern{0.4}\kern\dimen@}%
      \advance\dimen@0.4\dimexpr\macc@kerna
      \let\final@kern#2%
      \ifdim\dimen@<\z@ \let\final@kern1\fi
      \if\final@kern1 \kern-\dimen@\fi
    \else
      \overline{\rel@kern{-0.6}\kern\dimen@#1}%
    \fi
  }%
  \macc@depth\@ne
  \let\math@bgroup\@empty \let\math@egroup\macc@set@skewchar
  \mathsurround\z@ \frozen@everymath{\mathgroup\macc@group\relax}%
  \macc@set@skewchar\relax
  \let\mathaccentV\macc@nested@a
  \if#31
    \macc@nested@a\relax111{#1}%
  \else
    \def\gobble@till@marker##1\endmarker{}%
    \futurelet\first@char\gobble@till@marker#1\endmarker
    \ifcat\noexpand\first@char A\else
      \def\first@char{}%
    \fi
    \macc@nested@a\relax111{\first@char}%
  \fi
  \endgroup
}
\newcommand*\under@bar@[3]{%
  \begingroup
  \def\mathaccent##1##2{%
    \let\mathaccent\save@mathaccent
    \if#32 \let\macc@nucleus\first@char \fi
    \setbox\z@\hbox{$\macc@style{\macc@nucleus}_{}$}%
    \setbox\tw@\hbox{$\macc@style{\macc@nucleus}{}_{}$}%
    \dimen@\wd\tw@
    \advance\dimen@-\wd\z@
    \divide\dimen@ 3
    \@tempdima\wd\tw@
    \advance\@tempdima-\scriptspace
    \divide\@tempdima 10
    \advance\dimen@-\@tempdima
    \ifdim\dimen@>\z@ \dimen@0pt\fi
    \rel@kern{0.6}\kern-\dimen@
    \if#31
      \underline{\rel@kern{-0.6}\kern\dimen@\macc@nucleus\rel@kern{0.4}\kern\dimen@}%
      \advance\dimen@0.4\dimexpr\macc@kerna
      \let\final@kern#2%
      \ifdim\dimen@<\z@ \let\final@kern1\fi
      \if\final@kern1 \kern-\dimen@\fi
    \else
      \underline{\rel@kern{-0.6}\kern\dimen@#1}%
    \fi
  }%
  \macc@depth\@ne
  \let\math@bgroup\@empty \let\math@egroup\macc@set@skewchar
  \mathsurround\z@ \frozen@everymath{\mathgroup\macc@group\relax}%
  \macc@set@skewchar\relax
  \let\mathaccentV\macc@nested@a
  \if#31
    \macc@nested@a\relax111{#1}%
  \else
    \def\gobble@till@marker##1\endmarker{}%
    \futurelet\first@char\gobble@till@marker#1\endmarker
    \ifcat\noexpand\first@char A\else
      \def\first@char{}%
    \fi
    \macc@nested@a\relax111{\first@char}%
  \fi
  \endgroup
}
\makeatother

\usepackage[suppress]{color-edits}
 \addauthor{df}{ForestGreen}
 \addauthor{dr}{purple}
 \addauthor{zm}{red}
 \newcommand{\dfc}[1]{}
 \newcommand{\dhruv}[1]{}
\newcommand{\zm}[1]{}

\makeatletter
\let\OldStatex\Statex
\renewcommand{\Statex}[1][3]{%
  \setlength\@tempdima{\algorithmicindent}%
  \OldStatex\hskip\dimexpr#1\@tempdima\relax}
\makeatother

 \usepackage{accents}

\usepackage{relsize} 

\let\oldparagraph\paragraph
\newcommand{\paragraphi}[1]{\oldparagraph{\emph{#1}.}}
\renewcommand{\paragraph}[1]{\oldparagraph{#1.}}

\newcommand{\nn}{\nonumber}

\arxiv{

        \title{\scalebox{.935}{Is a Good Foundation Necessary for Efficient Reinforcement Learning?}\\
          \scalebox{.935}{The Computational Role of the Base Model in Exploration}}

\arxiv{
\author{
Dylan J. Foster\thanks{Email: \texttt{dylanfoster@microsoft.com}.} \\
Microsoft Research
\and
Zakaria Mhammedi\thanks{Email: \texttt{mhammedi@google.com}.} \\
Google Research
\and
Dhruv Rohatgi\thanks{Email: \texttt{drohatgi@mit.edu}. This research was partially conducted during the author's internship at Microsoft Research.} \\ MIT  
}
}

}

\allowdisplaybreaks

\date{}

\addtocontents{toc}{\protect\setcounter{tocdepth}{0}}

\begin{document}
\maketitle
\begin{abstract}

Language model alignment (or, reinforcement learning) techniques that leverage \emph{active exploration}---deliberately encouraging the model to produce diverse, informative responses---offer the promise of super-human capabilities. However, current understanding of algorithm design primitives for \emph{computationally efficient} exploration with language models is limited. To better understand how to leverage access to powerful pre-trained generative models to improve the efficiency of exploration, we introduce a new computational framework for RL with language models, in which the learner interacts with the model through a \emph{sampling oracle}. Focusing on the \emph{linear softmax} model parameterization, we provide new results that reveal the computational-statistical tradeoffs of efficient exploration:\loose
\arxiv{\begin{enumerate}}
\item \emph{Necessity of coverage.} Coverage refers to the extent to which the pre-trained model covers near-optimal responses---a form of hidden knowledge. We show that coverage, while not necessary for data efficiency, lower bounds the \emph{runtime} of any algorithm in our framework.
\item \emph{Inference-time exploration.} We introduce a new algorithm, \emph{\spanalg}, which obtains optimal data efficiency and is computationally efficient whenever the pre-trained model enjoys sufficient coverage, matching our lower bound. \spanalg leverages inference-time computation with the pre-trained model to reduce the effective search space for exploration.
\item \emph{Insufficiency of training-time interventions.} We
  contrast\arxiv{ the result above} by showing that
  \emph{training-time} interventions\arxiv{ (e.g., exploratory
    modifications to DPO)} that produce \emph{proper} policies cannot
  achieve similar guarantees in polynomial time.\loose
\arxiv{\item \emph{Computational benefits of multi-turn exploration.} Finally, we show that
under additional representational assumptions, one can achieve improved runtime (replacing sequence-level
coverage with token-level coverage) through \emph{multi-turn}
exploration. En route, we show that any MDP where the optimal
KL-regularized value function is linear (linear-$\Qstarb$) is learnable in the
reset access model.\loose
}
\end{enumerate}
\arxiv{We view these results as a step toward a computational theory
of decision making with
generative models.\loose}

 \end{abstract}

\section{Introduction}
\label{sec:intro}
Language models are rapidly approaching human performance on a vast array
of natural language tasks
\citep{brown2020language,ouyang2022training,touvron2023llama,achiam2023gpt,anil2023palm},
but current models are constrained by the limitations of passively generated
human training data.
Domains where high-quality feedback is available (e.g., math and code)
offer the tantalizing possibility of overcoming these limitations:
By iteratively generating new proposals
and refining them with human or super-human feedback (e.g., from a
formal proof checker), a language model could eventually discover
novel, potentially super-human behaviors and capabilities.\loose

The central hurdles to achieving novel capabilities with this template are (1)
the amount of feedback---that is, the \emph{data
  efficiency}---required by alignment/post-training, and (2) the \emph{computational efficiency}. Both metrics are important, but since gathering feedback is often costly or slow (e.g., due to cost of
gathering human labels, or high computational overhead of formal proof
checkers), data is often more tightly constrained than
computation. Unfortunately, the most popular alignment techniques,
like PPO \citep{schulman2017proximal} and Online DPO
\citep{xu2023some,guo2024direct}, are data-inefficient due to their
reliance on passive exploration. These techniques treat the
pre-trained model as a \emph{policy} and iteratively update it with
reinforcement learning, but since there
is no explicit mechanism to promote novelty, they are unlikely to generate positive responses (e.g., novel and correct proofs) by chance
\citep{xie2024exploratory}.  In principle, this issue could be mitigated through \emph{active exploration}
techniques developed in the theory of reinforcement learning, which deliberately generate diverse,
informative responses \citep{jiang2017contextual,agarwal2019reinforcement,jin2021bellman,foster2021statistical,foster2023foundations}. %
However,
these techniques---while satisfactorily \emph{data-efficient}---cannot be
implemented in a \emph{computationally efficient} fashion in their most general
form
\citep{dann2018oracle,kane2022computational,golowich2024exploration}. Recent attempts to specialize active exploration to language model alignment face the same issue: such methods require either (1) enumeration over the (exponentially large) space of responses \citep{chen2022human,wang2023rlhf,ye2024theoretical,xiong2024iterative}; or (2) non-convex training objectives that are not known to be efficiently implementable in even the simplest settings \citep{xie2024exploratory,cen2024value}.\loose

\paragraph{The role of the base model} Language model alignment features unique structure not present in
general reinforcement learning---most prominently, access to a powerful pre-trained base model
(the starting point from which alignment proceeds) that encodes substantial prior knowledge (e.g, whether
  proofs or programs are at least syntactically valid, if not useful).\footnote{\arxiv{Other, more technical, features
    include (i) deterministic, known transition dynamics (rendering
    the problem statistically equivalent to contextual bandits), and (ii)
    the presence of regularization to the base model.}}
Yet, there is little understanding of what properties of the base model
  are necessary for novel behaviors to emerge through RL \citep{openai2024o1,deepseek2025r1}, or whether
  this process can be accelerated through algorithmic interventions
  (e.g., the idea of directly using the base model to reduce the effective search
  space has appeared in many empirical
  works
  \citep{liu2023statistical,hao2023reasoning,tran2023iterative,yao2024tree,xiong2024iterative,yan2024efficient}).
  Meanwhile, the previously-mentioned theoretical works (based on active
  exploration) only make superficial use of the base model, rendering the lack of computational efficiency perhaps unsurprising. 
This
  motivates the central question we explore:\loose
  \begin{center}
    \emph{How can we best leverage access to powerful pre-trained generative
        models to improve computational efficiency of exploration,
        and how should we evaluate algorithms that
        do so?}
    \end{center}
    To address this question, we introduce a new computational
    framework for language reinforcement learning in which access to the model is abstracted away
    through a \emph{sampling oracle}, and provide new algorithms and
    fundamental limits which elucidate essential properties---in
    particular, the notion of \emph{coverage}---of the
    pre-trained model for computationally efficient learning. In the process, we bring
    clarity to computational benefits of algorithmic interventions that have been
    explored empirically but are not yet fully understood,
    including (i) benefits of \emph{inference-time computation}
    \citep{brown2024large,snell2024scaling,wu2024empirical}; and (ii)
    benefits of \emph{multi-turn} techniques that explore at the per-step\arxiv{ (e.g., token or sub-sequence)}
level
\citep{lightman2023lets,qu2024recursive,kumar2024training,setlur2024rewarding,setlur2024rl,xiong2024building,kazemnejad2024vineppo}.\loose

\subsection{Background: Online Alignment from Reward-Based Feedback}\label{sec:background}
To motivate our computational framework, we begin by formally introducing the statistical problem of language
model alignment.
We adopt
 a contextual
bandit formalism \citep{rafailov2024direct,xiong2024iterative}
where the language model is
a \emph{policy} $\pi:\cX\to\Delta(\cY)$ that maps a prompt (context)
$x\in\cX$ to a response (action) $y\in\cY$ by sampling
$y\sim{}\pi(\cdot\mid{}x)$. \arxiv{We use $\rho\in\Delta(\cX)$ to denote the
distribution over prompts. }We begin with a reference policy $\piref$, which is typically
obtained through pre-training or supervised fine-tuning. From here, our alignment
protocol proceeds as follows: We receive $\Trounds$ \iid prompts
$x\ind{1},\ldots,x\ind{T}\sim\rho\in\Delta(\cX)$. For each prompt $x\ind{t}$,
we can select up to $N$ responses
$y\ind{t}_1,\ldots,y\ind{t}_N\in\cY$ (the responses may be sampled
from $\piref$ or from some alternative sampling procedure
\citep{liu2023statistical,khaki2024rs,shi2024crucial}), with which we
query a \emph{\textbf{reward oracle}} for a reward
$r_i\ind{t}\in\brk{0,\Rmax}$.
We assume that $\En\brk*{r\mid{}x,y}=\rstar(x,y)$, where
$\rstar:\cX\times\cY\to\brk*{0,\Rmax}$ is the underlying \emph{reward
  function}, which represents the feedback source (e.g, verifier or
human labeler) that the algorithm interacts with. %
All responses can be chosen adaptively based on prior feedback---this stands in contrast to traditional offline
alignment \citep{ye2024theoretical,liu2024provably,huang2024correcting}, which is the special case of our
formulation in which $y_i\ind{t} \sim \piref(x\ind{t})$ for all
$t$.\footnote{Responses need not be chosen according to the
  order $t$; the algorithm can sample $N'<N$ responses for $x\ind{t}$,
  then sample responses for another $x\ind{t'}$ before to returning
  $x\ind{t}$ and sampling more. This generality makes our lower bounds
  stronger; our algorithms use $N=2$ and proceed in order however.
}
Once data collection concludes, we produce a
final policy $\pihat$ with the aim of achieving high reward. We
define $\Tdata\leq{}N\cdot{}\Trounds$ as the total number of reward
queries used by the algorithm. Note that in general, we can
  have $\Tdata\ll N\cdot{}\Trounds$, as the algorithm can potentially
  abstain from querying the reward oracle for a given prompt. \loose %

As in prior work on
alignment \citep{xiong2024iterative,ye2024theoretical,xie2024exploratory}, we focus maximizing \emph{KL-regularized} reward. Letting
$J(\pi)\ldef{}\En_{x\sim\rho,y\sim\pi(x)}\brk{\rstar(x,y)}$ denote the
average reward and 
$\Dkl{\pi}{\piref}\ldef{}\En_{x\sim\rho}\brk*{\Dkl{\pi(x)}{\piref(x)}}$
denote KL-divergence, we define for regularization parameter $\beta>0$:\loose
\begin{align}
\label{eq:kl_reward}
J_\beta(\pi) \coloneqq &~ J(\pi) -
\beta\cdot\Dkl{\pi}{\piref}.
\end{align}
We measure the quality of the policy $\pihat$ via regret to the optimal
KL-regularized policy: we desire that\arxiv{\loose\[
J_\beta(\pistarb) - J_\beta(\pihat) \leq \veps,
\]}
where $\pistarb := \argmax_{\pi:\cX\to\Delta(\cY)} J_\beta(\pi)$ is the optimal policy, and $\veps>0$ is small. \arxiv{A bound on the regularized regret ensures
that} $\pihat$ achieves near-optimal reward, but does not drift too far
from the base policy $\piref$. We view $\beta$ as a fixed (but
potentially small) problem-dependent parameter, so as to allow 
novel responses that deviate non-trivially from $\piref$. We abbreviate
$\En_{\pi}\brk*{\cdot}\ldef{}\En_{x\sim\rho,y\sim\pi(\cdot\mid{}x)}\brk{\cdot}$.
\begin{remark}[Autoregressive models]
  \label{rem:auto}
  We focus on the abstract setting above, but our motivating example
  is autoregressive sequence models of length $H$, where $\cY=\cA^{H}$ represents the space of token
  sequences over a vocabulary $\cA$. We will return to this specific
  setting in \cref{sec:multi}.
\end{remark}

\oldparagraph{Statistical lens: How much reward data do we need?} Since the underlying reward function $\rstar$ is unknown to the algorithm designer, the total number of reward queries $\Tdata$ used by an algorithm reflects its
\emph{data efficiency}, i.e. how much data needs to be collected from
the reward oracle to learn a good policy. Collecting high-quality
reward signals can be costly or time-consuming (e.g., when
human-generated, or when reward evaluation requires computationally
intensive code execution or formal verification), so data efficiency
is critical. To give provable data efficiency guarantees, typical
alignment algorithms \citep{xiong2024iterative,xie2024exploratory}
take as input a user-specified policy class
$\Pi=\crl*{\pitheta\mid{}\theta\in\Theta}$ for a \emph{parameter space} $\Theta$, and invoke the standard statistical assumption that the optimal policy lies in $\Pi$.

\begin{assumption}[Policy realizability]\label{ass:realizable}
The policy class $\Pi$ satisfies $\pistarb \in \Pi$.
\end{assumption}

\begin{remark}[Preference-based feedback]
Our absolute reward formulation for language model alignment has been used in
prior work empirically
\citep{wang2023helpsteer,wang2024helpsteer2,wang2024interpretable,xiong2024building}
and in theory
\citep{zhao2024sharp,wang2024arithmetic,xiong2024building}. This
formulation is closely related to the widely-used theoretical model for \arxiv{reinforcement
learning with human feedback (RLHF)} where the learner receives
\emph{preference-based feedback}. Our main algorithms use $N=2$ and
readily extend to the preference-based setting, while our lower
bounds allow for general $N$. \arxiv{We discuss this connection further in} \cref{app:preference}.\loose
\end{remark}

\subsection{A Computational Framework for Online Alignment}
The response space
$\cY$ in the online alignment framework can be exponentially
large %
  (e.g., \cref{rem:auto}).
Without \arxiv{further }assumptions, there is no
hope of learning a near-optimal policy without enumerating over $\cY$\arxiv{,
rendering discussion of computational efficiency moot}. To address
this, we assume that the learning algorithm has access to a certain \emph{sampling oracle}.\loose %

Informally, we consider two different settings. \textbf{(1)} In the \textbf{\emph{strong oracle}} setting, the learner can draw conditional samples from $\pi_\theta(\cdot\mid{}x)$ for any prompt $x\in\cX$ and parameter $\theta\in\Theta$ (with the convention that $\mb{0}\in\Theta$ and $\pi_{\mb{0}}=\piref$). \textbf{(2)} In the \textbf{\emph{weak oracle}} setting, the learner can draw conditional samples from $\piref(\cdot\mid{}x)$ for any prompt $x\in\cX$. We let $\Tsample$
  denote the total number of sampling oracle queries used by the
  algorithm throughout the learning process. See \cref{sec:coverage} for formal details.\loose

Our algorithms only need the weak oracle, but our lower bounds apply even to the strong oracle. We
view access to the weak oracle as a
minimal assumption: efficient conditional sampling is arguably \emph{the}
defining property of autoregressive language models.
We use the sampling oracle complexity $\Tsample$ as an
information-theoretic proxy for the computational efficiency of
an alignment algorithm, one that parallels the role of oracle/query
complexity \citep{nemirovski1983problem,kearns1998efficient}, and is amenable to
upper and lower bounds.\arxiv{ A similar abstraction was used by
\citet{huang2024self} for the complementary problem of language
model self-improvement.}\loose

As an example, (reward-based) \onlinedpo \citep{guo2024direct}, is
\arxiv{perhaps the
simplest} online alignment algorithm: For each
$t\in\brk{\Trounds}$, the algorithm computes a parameter
$\theta\ind{t}\in\Theta$ by optimizing a \dpo objective
(\cref{eq:dpo_reward}) with its current dataset
$\cD\ind{t}$, then samples a pair of responses
$y_1\ind{t},y_2\ind{t}\sim{}\pi_{\theta\ind{t}}(\cdot\mid{}x\ind{t})$,
observes corresponding rewards $(r_1\ind{t},r_2\ind{t})$, and updates
$\cD\ind{t+1}\gets\cD\ind{t}\cup\crl{(x\ind{t},y_1\ind{t},y_2\ind{t},r_1\ind{t},r_2\ind{t})}$. This
algorithm uses two (strong) sampling oracle queries\arxiv{ to gather reward feedback}
per round, so the computational cost is no worse than the cost of
gathering feedback: $\Tcomp=\Tdata$.\footnote{Technically, $\onlinedpo$ also requires observing the log-densities of the observed responses, though this requirement simplifies for the linear softmax policy class that we consider in the sequel. See \cref{sec:coverage} for details.
}
Unfortunately, since \onlinedpo engages in purely passive exploration,
the algorithm's data efficiency itself is unsatisfactory. We make this distinction quantitative below.
\loose

\subsection{Linear Softmax Policy Parameterization}

To understand when we can hope to achieve favorable data efficiency
$\Tdata$ (e.g., through active exploration) without entirely sacrificing computational efficiency $\Tcomp$, we focus on
perhaps the simplest concrete choice of policy class: \emph{linearly parametrized softmax
  policies} \citep{xiong2024iterative,cen2024value}.\loose

\begin{definition}
  \label{def:softmax}
Let $\Theta \subset \RR^d$ be a convex parameter set and let $\phi: \cX \times\cY \to \RR^d$ be a feature embedding. The \arxiv{associated }linear-softmax policy class is $\Pi = \{\pi_\theta: \theta \in \Theta\}$, where $\pi_\theta: \cX \to \Delta(\cY)$ is defined by\loose
\begin{align}
  \label{eq:softmax}
\pi_{\theta}(y\mid{}x)\propto{}\piref(y\mid{}x)\cdot\exp\prn*{\beta^{-1}\tri*{\theta, \phi(x,y)}}.
\end{align}
\end{definition}
\noindent With this policy class, \cref{ass:realizable} becomes a
natural assumption about the expressivity of the feature embedding
$\phi$: for example, if the reward function is linear in the features,
i.e. %
\arxiv{\begin{align}
  \label{eq:linear_reward}
\rstar(x,y) = \tri*{\thetast, \phi(x,y)}
\end{align}}
for some $\thetast\in\bbR^{d}$, then the optimal \arxiv{KL-}regularized policy
$\pistarb$ is exactly $\pi_{\theta^\st}$ \citep{xie2024exploratory},
so that \cref{ass:realizable} is satisfied so long as $\theta^\st \in
\Theta$. %

\arxiv{In spite of the simplicity of the parameterization, \cref{def:softmax}
is rich enough to capture autoregressive sequence models in
which weights for all but the last layer are frozen, and there is some evidence
\citep{malladi2023kernel} that post-training methods with deep models operate in this
``lazy/kernel'' regime. We hope that by developing a sharp understanding of
computational-statistical tradeoffs for this simple setting, our
    work can serve as a useful starting point toward understanding the
    general nonlinear setting.}

  For sequence modeling (\cref{rem:auto}),
  the \emph{strong} sampling oracle
can be at odds with \cref{def:softmax}: while the feature dimension
$d$ is bounded, $\cY$ is exponentially large, and even if
$\piref$ is an autoregressive sequence model, $\pi_\theta$ may not
admit an explicit autoregressive factorization for all
$\theta$. However, the \emph{weak} sampling oracle is entirely natural
for autoregressive sequence modeling; see \cref{sec:multi}\arxiv{ for discussion}.\loose

\paragraph{Tradeoffs between data \arxiv{efficiency }and computational
  efficiency}
Let $\Tdatafull$ and
$\Tsamplefull$ denote the reward and sampling oracle queries required for an algorithm to ensure $\Jbeta(\pistarb) -
\Jbeta(\pihat) \leq \veps$ with probability at least $1-\delta$.
  Even for
linear softmax policies, all existing
algorithms are unsatisfactory with respect to $\Tdatafull$ or
$\Tsamplefull$. On one hand, \citet{xie2024exploratory} show that if we define\loose%
  \begin{equation}
    \Ccov(\pistarb)\ldef{}\sup_{x\in\cX, y\in\cY}\frac{\pistarb(y\mid{}x)}{\piref(y\mid{}x)}
    \label{eq:coverage}
  \end{equation}
  as the \emph{coverage coefficient} for $\pistarb$, then the \onlinedpo method
in the prequel, while implementable in polynomial time\arxiv{ per
iteration}, must suffer %
\arxiv{\begin{align}
  \label{eq:dpo_bad}
  \Tdatafull \approxgeq \min\crl*{\Ccov(\pistarb), \exp\prn*{\frac{\Rmax}{\beta}}},
\end{align}}
when $d=\bigoh(1)$ and $\veps,\delta=\bigom(1)$. Informally,
$\Ccov(\pistarb)$ is represents the number of responses
one must draw from $\piref$ before high reward is observed by chance
\citep{brown2024large,snell2024scaling,wu2024empirical}. This is a
form of hidden knowledge, but its presence in $\Tdata$ reflects
passive exploration.
On the other hand, \citet{xie2024exploratory} introduced a variant of
\onlinedpo called \xpo (see \cref{app:related}), which augments\arxiv{ the
\dpo training objective with a bonus \arxiv{designed }to
encourage}%
active exploration. This allows \xpo to achieve
polynomial data efficiency, irrespective of whether\arxiv{ the base policy}
$\piref$ has favorable coverage:\loose
  \begin{align}
    \label{eq:xpo_data}
    \Tdatafull \approxleq \frac{d^{2}\log(\delta^{-1})}{\veps^2}.
  \end{align}
  Note that $\Ccov(\pistarb)\gg\poly(d)$ in general, representing a benefit over passive exploration.
  Like \onlinedpo, \xpo uses the sampling oracle to generate two
  responses
  $(y\ind{t}_1,y\ind{t}_2)\sim\pi_{\theta\ind{t}}(\cdot\mid{}x\ind{t})$
  at each iteration.
Yet, while the objective \xpo uses to update the policy $\pi_{\theta\ind{t+1}}$ is amenable to gradient-based methods, \arxiv{the bonus term} introduces
non-convexity not present in the \dpo objective, and it is not
known whether it can be minimized in polynomial time (nor with
$\Tsamplefull$ polynomial) for linear softmax
policies, even when $\abs*{\cY}$
is small. Other active exploration
algorithms are similarly unsatisfactory \citep{chen2022human,ye2024theoretical,xiong2024iterative,cen2024value}.\loose

\subsection{Contributions}
We develop a sharp understanding of computational-statistical
tradeoffs for online alignment with linear softmax policies, highlighting the central role of the
base model (policy) $\piref$ in enabling computational efficiency, along with
benefits of inference-time computation and multi-step exploration.\loose

\paragraph{The (computational) necessity of coverage
  (\cref{sec:coverage})}
The coverage coefficient $\Ccov(\pistarb)$ captures the extent to which $\piref$ covers near-optimal responses---a form of
knowledge encoded in the pre-trained model
\citep{brown2024large,snell2024scaling,wu2024empirical}. While coverage is not necessary for data efficiency (e.g., \cref{eq:xpo_data}), we show that it \emph{is}
required for computational efficiency. Formally 
(\cref{thm:coverage}), for any
algorithm in the \framework framework, the number of sampling oracle
calls (and runtime) is lower bounded as\arxiv{\loose
\begin{align}
  \label{eq:coverage_lower_intro}
  \Tsamplefull \approxgeq \min\crl*{\Ccov(\pistarb), \exp\prn*{\frac{\Rmax}{\beta}}}.
\end{align}}
This serves as a skyline for algorithm design, and contributes to a
growing body of work that highlights the computational benefits of
coverage \citep{huang2024self}.\loose

\paragraph{Efficient inference-time exploration
  (\cref{sec:algorithms})}
We give a new algorithm, \spanalg, which (i) achieves near-optimal
  data efficiency $\Tdatafull \approxleq
  \poly(d,\beta^{-1},\veps^{-1},\log(\delta^{-1}))$ for both rewards
  and prompts, and (ii) runs in
  polynomial time, achieving minimal oracle efficiency as governed by
  the lower bound \arxiv{in \cref{eq:coverage_lower_intro}:\loose
  \[\Tsamplefull \approxleq
    \poly(\Ccov(\pistarb),\Tdatafull).
  \]}
\spanalg leverages 
  inference-time computation to tilt learned policies toward an
  exploratory distribution, using the base
  policy $\piref$ to reduce the effective search space for exploration
  to a manageable size.\loose

\paragraph{Insufficiency of training-time interventions
  (\cref{sec:computational})}
Active exploration algorithms based on ``training-time'' interventions
(e.g., modifications to the \dpo objective, as in \xpo) are typically
\emph{proper} in the sense that
they explore using a sequence
$\pi_{\theta\ind{1}},\ldots,\pi_{\theta\ind{T}}$ of iteratively
computed linear softmax policies and ultimately output such a policy;
meanwhile \spanalg, by invoking extra inference-time computation,
engages in \emph{improper} exploration. We show (\cref{thm:training}) that \emph{no data-efficient
  proper exploration algorithm can run in polynomial time} (\arxiv{including}%
polynomial dependence on $\Ccov(\pistarb)$ and $\exp(\Rmax/\beta)$).
This gives a separation between algorithms based
on training-time interventions and algorithms like \spanalg that
explore improperly through inference-time computation.\loose

\paragraph{Computational benefits of multi-turn exploration (\cref{sec:multi})}
The preceding results, when specialized to autoregressive modeling, engage
in exploration at the sequence-level. As a final result (\cref{thm:multi}), we show that
under the additional
representational condition that $\pistarb$ can be represented
as an autoregressive policy,
it is
possible to achieve substantially improved runtime and oracle
complexity $\Tsample$ (replacing the coverage coefficient $\Ccov(\pistarb)$
with an \emph{token-level} counterpart) by appealing to \emph{multi-turn} exploration at
the per-step (token or sub-sequence) level
\citep{lightman2023lets,qu2024recursive,kumar2024training,setlur2024rewarding,setlur2024rl,xiong2024building,kazemnejad2024vineppo}. This
is achieved as a special case of a more general result, which may be
of independent interest: any MDP where the optimal
KL-regularized value function $\Qstarb$ is linear can be efficiently learned in the
reset access model.

\loose

\arxiv{We view our results as an initial step toward a computational foundation
for language model exploration (and more broadly, efficient decision making with
generative models). To this end, we highlight several open problems
and directions for future research (\cref{sec:discussion}).
}

\arxiv{
\subsection{Notation}

  We adopt
    standard big-oh notation, and write $f=\bigoht(g)$ to denote that
    $f = \bigoh(g\cdot{}\max\crl*{1,\mathrm{polylog}(g)})$ and
    $a\approxleq{}b$ as shorthand for $a=\bigoh(b)$. We use
    $\bbB_{p}(r)$ to denote the $\ell_p$-ball of radius $r$, and
    define $\nrm*{x}_{\Sigma}^2=\tri*{x,\Sigma{}x}$ for a matrix
    $\Sigma\psdgt 0$. We use $\Id$ to denote the identity matrix in
    $d$ dimensions.

    }

\section{Sampling Oracle Framework and Necessity of Coverage}
\label{sec:coverage}

In this section, we formally introduce our sampling oracle framework\arxiv{
for linear softmax policies}, then prove that
coverage for the base policy $\piref$ is necessary for computational
efficiency in this framework.\loose\arxiv{
}
\arxiv{\paragraph{Preliminaries}}
Henceforth (until \cref{sec:multi}), we focus on the linear softmax
parameterization in \cref{def:softmax} and make
\cref{ass:realizable}. For statistical tractability, we make a (standard) norm bound assumption.\loose
\begin{assumption}
  \label{ass:norm}
  We assume all $\theta\in\Theta$ satisfy $\nrm*{\theta}\leq{}B$
  for a parameter $B>0$, and that $\nrm*{\phi(x,y)}\leq{}1$ and $\tri*{\thetastar,\phi(x,y)-\phi(x,y')}\in\brk*{-\Rmax,\Rmax}$ for all
  $x\in\cX$, $y,y'\in\cY$. Furthermore, we assume that
  $\mb{0}\in\Theta$.\loose
\end{assumption}
We assume that $\beta\leq\Rmax\leq{}B$ without loss of
generality.\footnote{If $\Rmax<\beta$, $\onlinedpo$ itself is
  statistically efficient. Our main upper bounds depend on\arxiv{ the
    parameter} $B$ only logarithmically.\loose} We do not explicitly assume that rewards are
linear (i.e., \arxiv{\cref{eq:linear_reward}}), but under
\cref{ass:realizable} we have (\cref{lem:reward_difference}):
\begin{align}
  \label{eq:reward_difference}
\rstar(x,y) - \rstar(x,y') = \tri*{\thetastar,\phi(x,y)-\phi(x,y')} \quad
     \forall{}x\in\cX, y,y'\in\cY.
\end{align}

\subsection{Sampling Oracle Framework}
We now formally define our computational framework for the linear
softmax policy parameterization. We assume the prompt space $\cX$, response space
$\cY$, and parameter space $\Theta$ are given to the alignment
protocol, but the feature embedding $\phi$ and the reference policy $\piref$
are specified only implicitly (i.e., are ``unknown'' a-priori), and must be accessed through one of the following computational oracles.\loose
\begin{definition}[Sampling oracles]
  \label{def:oracle}
\textnormal{\textbf{\underline{Setting I (strong oracle)}}:} In one query, the learner proposes a prompt $x\in\cX$ and
  parameter $\theta\in\Theta$, and receives a conditional sample
  $y\sim\pitheta(\cdot\mid{}x)$, as well as the
 corresponding feature
$\phi(x,y)$ for the sampled response (note that $\pi_{\mb{0}}=\piref$).\\
\textnormal{\textbf{\underline{Setting II (weak oracle)}}:} In one query, the learner proposes a prompt $x \in \cX$ and receives a conditional sample $y \sim \piref(\cdot\mid{}x)$, as well as the corresponding feature $\phi(x,y)$.
\end{definition}

\begin{definition}\label{def:alignment}
An \emph{online alignment algorithm} in the (strong/weak) setting is
an algorithm that, given parameters $\veps,\delta>0$, produces a
policy $\pihat$ satisfying $J_\beta(\pistarb) - J_\beta(\pihat) \leq
\veps$ with probability at least $1-\delta$. We write $\Tdatafull$ and
$\Tsamplefull$ to denote the total number of reward oracle queries and
(strong/weak) sampling oracle queries respectively.
\end{definition}

Notice, any algorithm operating in our framework (i) must invoke the sampling oracle if it wishes to query the reward oracle with some
$y\sim\pi_{\theta}(\cdot\mid{}x\ind{t})$, and (ii) only has knowledge of the features $\phi(x,y)$ that have previously
been revealed by the sampling oracle. As an example, given a dataset
$\cD\ind{t}=\crl*{(x\ind{i},y_1\ind{i},y_2\ind{i},r_1\ind{i},r_2\ind{i})}_{i<t}$
of prompt/response/reward tuples, the (reward-based) \onlinedpo update takes the form\loose
\arxiv{    \begin{align}
      \theta\ind{t}&=\argmin_{\theta\in\Theta}\sum_{i<t}
      \prn*{
                     \beta\log\frac{\pitheta(y_1\ind{i}\mid{}x\ind{i})}{\piref(y\ind{i}_1\mid{}x\ind{i})} -
        \beta\log\frac{\pitheta(y\ind{i}_2\mid{}x\ind{i})}{\piref(y\ind{i}_2\mid{}x\ind{i})}
                     -(r\ind{i}_1-r\ind{i}_2)}^2.\label{eq:dpo_reward}
           \end{align}
         }
         Since
         $\beta\log\frac{\pitheta(y_1\ind{i}\mid{}x\ind{i})}{\piref(y\ind{i}_1\mid{}x\ind{i})}
         -
         \beta\log\frac{\pitheta(y\ind{i}_2\mid{}x\ind{i})}{\piref(y\ind{i}_2\mid{}x\ind{i})}=\tri*{\theta,\phi(x\ind{i},y_1\ind{i})-\phi(x\ind{i},y_2\ind{i})}$
         and this objective only evaluates 
    $\phi(x,y)$ for previously drawn responses,
    we see that it can be implemented in the strong setting (\cref{def:alignment}), using the
    strong sampling oracle to draw
    $(y_1\ind{t},y_2\ind{t})\sim\pi_{\theta\ind{t}}(\cdot\mid{}x\ind{t})$.

As we will discuss in
        \cref{sec:multi}, algorithms that use the weak oracle have
        important consequences when we specialize our to
        autoregressive sequence modeling; our main algorithm, \spanalg
        enjoys this property.\loose
    \loose

    \arxiv{
    \begin{remark}[Log-probability queries]
      \label{rem:lob_probability}
      The reader may note that the framework in \cref{def:oracle}
      reveals the features $\phi(x,y)$ for responses $y$ sampled from
      the oracle, but does not reveal 
      the log-probabilities $\log\pitheta(y\mid{}x)$ themselves. As
      highlighted above, the observed features are closely related, as
      they can be used to evaluate $\beta\log\frac{\pitheta(y\mid{}x)}{\piref(y\mid{}x)} -
        \beta\log\frac{\pitheta(y'\mid{}x)}{\piref(y'\mid{}x)}=\tri*{\thetastar,\phi(x,y)-\phi(x,y')}$,
        but cannot be used to compute $\log\pitheta(y\mid{}x)$ itself
        in general. We adopt this formalism because it simplifies the coverage-based
        lower bounds in \cref{sec:coverage_lower}; our algorithmic results only make use of the
        features $\phi(x,y)$, and hence fall into this framework. See
        \cref{app:sampling} for discussion around nuances of
        log-probability queries beyond the linear softmax
        parameterization.
      \end{remark}
      }

\arxiv{\begin{remark}[Connection to optimization oracles]
    \label{rem:optimization}
  There is a large body of work on algorithms for linear contextual
  bandits with large response spaces $\cY$ in which the response space
  is accessed through an \emph{optimization oracle} which can solve
  $\argmax_{y\in\cY}\tri*{\theta,\phi(x,y)}$
  efficiently for any $x\in\cX$ and $\theta\in\Theta$
  \citep{dani2008stochastic,bubeck2012towards,hazan2016volumetric,chen2017nearly,
    cao2019disagreement,katz2020empirical,zhu2022contextual}. Our
  formulation in \cref{def:oracle} can be viewed as an alternative,
  sampling-based computational framework for decision making with
  large response spaces, one which may be of independent
  interest. Note that while there is a sense in which sampling and
  optimization are polynomially equivalent when the set
  $\crl*{\phi(x,y)}_{y\in\cY}$ is convex \citep{lovasz2006fast}, they are not
  equivalent in general.
\end{remark}
}

\subsection{Coverage is Necessary for Computational Efficiency}
\label{sec:coverage_lower}
We now present
the first of our main results, which shows that the coverage
coefficient $\Ccov(\pistarb)$ lower bounds the number of sampling oracle
queries (and hence runtime) of any algorithm in our framework.\loose
\begin{restatable}[Necessity of coverage]{theorem}{coveragelower}\label{thm:coverage}
  Let $\Cstar, Y\geq{}2$ be given. Let $\Alg$ be an online alignment algorithm that uses $\Tdatafull$
  reward oracle queries and $\Tsamplefull$ strong sampling oracle
  queries whenever (i) the parameter space is\arxiv{ the Euclidean
  ball} $\Theta=\ball$, (ii) \cref{ass:norm} is satisfied with
  $\Rmax=B=1$, (iii) $\Ccov(\pistarb) \leq \Cstar$, and (iv) the response space has size\arxiv{ at most} $Y=\abs{\cY}$. Then, either $\Tdatafull \geq Y/8$, or \loose
  \begin{align}
    \label{eq:coverage_lower}
    \Tcompfull \geq \Omega\prn*{\min\crl[\big]{e^{\beta^2 d/2},
    e^{\beta^{-1}/2}, \Cstar}}.
    \end{align}
\end{restatable}
For simplicity, consider the regime where 
$d\geq{}\beta^{-3}$. Then \cref{thm:coverage} shows that any algorithm
needs $\Tcompfull \geq
\Omega\prn[\big]{\min\crl[\big]{e^{\beta^{-1}/2}, \Ccov(\pistarb)}}$ to
achieve non-trivial data-efficiency $\Tdatafull \ll \abs{\cY}$; note
that the presence of $e^{\beta^{-1}}$ in the lower bound is
fundamental, as we always have
$\Ccov(\pistarb)\leq\exp(\nicefrac{\Rmax}{\beta})$. We emphasize that
this construction uses a single prompt.

The intuition for the
construction in \cref{thm:coverage} is as
follows. There is a single ``hidden'' response $\ystar$ that the algorithm
must discover to achieve high reward. Because the base policy places low probability on this
response, we are
unlikely to sample $\ystar\sim\pitheta(\cdot\mid{}x)$ unless $\tri*{\theta,\thetastar}\geq{}1-\beta$. This leaves the
algorithm designer with two options: (i) brute-force search over
$\theta\in\ball$ until we find
$\tri*{\theta,\thetastar}\geq{}1-\beta$, which requires an exponential
number of oracle queries in the dimension $d$, or (ii) eat the cost of
the coverage coefficient by drawing roughly $\Ccov(\pistarb)$
responses $y\sim\piref(\cdot\mid{}x)$ until we observe $\ystar$.

\section{Efficient Online Alignment via Inference-Time Exploration}
\label{sec:algorithms}

\cref{thm:coverage} serves as a skyline, showing that coverage (hidden
knowledge) for the base
policy is essential for computationally
efficient \arxiv{online }alignment; note that various works have shown
that existing pre-trained models exhibit favorable coverage for tasks of interest
\citep{brown2024large,snell2024scaling,wu2024empirical}.
We now present our main algorithm, \spanalg, which achieves the computational skyline in 
\cref{eq:coverage_lower} without sacrificing the polynomial
data-efficiency achieved by (inefficient) active exploration algorithms such as \xpo.\loose

\vspace{-5pt}

\subsection{Algorithm: \spanalg}

\arxiv{%
\newcommand{\optioni}{\textcolor{blue!70!black}{\textsc{Option I}}}
\newcommand{\optionii}{\textcolor{blue!70!black}{\textsc{Option II}}}
\newcommand{\algi}{\hspace{\algorithmicindent}}

\arxiv{\begin{algorithm}[tp]}
\caption{$\spanalg$}
\label{alg:spanner}
\begin{algorithmic}[1]
  \Statex[0]\multiline{ {\bfseries input:}
    Base policy $\piref$, KL-regularization parameter $\beta>0$,
    number of spanner rounds $\nspan\in\bbN$, number of exploration rounds $\nexp\in\bbN$,
    failure probability $\delta\in(0,1)$.
  }
  \State Define $\vepsstat\ldef
  c\cdot\sqrt{d\Rmax^2\log(B\Rmax^{-1}\delta^{-1}\nexp)}$ for abs.
  constant $c>0$.
  \State Set $\lambda\gets{}(\nicefrac{\Rmax}{B})^2$ and $\nu\ldef{}\beta/\vepsstat$.
   \hfill\algcommentlight{Spanner
    params.}
  \label{line:spanner_params}
  \State Set $\Mrej\ldef{}8e^2\cdot\Ccov(\pistarb)$ and
  $\deltarej\ldef{}\nexp^{-1}$.\hfill\algcommentlight{Rejection
    sampling params.} \label{line:rejection_params_spanner}
  \Statex[0] \algcommentbig{Spanner construction phase}
  \State Initialize dataset $\cDspan\gets\crl{\emptyset}$ and $\Psispan\gets\crl{\emptyset}$ and set
  $\Sigmaspan\gets\lambda\Id$.
    \For{iteration $t = 1,2,\dotsc,\Tprompt$}  \label{line:spanner_outer}
    \State Observe prompt $x\ind{t}\sim\rho$.
 \For{iteration $i=1,2,\ldots,\Tspan$} \label{line:spanner}
 \State Sample $(y\ind{t,i}_1,
    y\ind{t,i}_2)\sim\piref(\cdot\mid{}x\ind{t})$.
  \If{$\nrm*{\phidel(x\ind{t},y_1\ind{t,i},y_2\ind{t,i})}_{\Sigmas^{-1}}>\nu$} \hfill\algcommentlight{$\phidel(x,y_1,y_2) := \phi(x,y_1)-\phi(x,y_2).$}\label{line:spanner_if}
  \State Observe rewards $(r\ind{t}_1, r\ind{t}_2)$ for $(x\ind{t}, y_1\ind{t,i},
  y_2\ind{t,i})$.
  \State Update $\cDspan\gets\cDspan\cup\crl*{(x\ind{t}, y_1\ind{t,i},
  y_2\ind{t,i}, r_1\ind{t}, r_2\ind{t})}$ and
$\Psispan\gets\Psispan\cup\crl*{(x\ind{t},y_1\ind{t,i},y_2\ind{t,i})}$.
  \State $\Sigmas\gets\Sigmas+\phidel(x\ind{t},y_1\ind{t,i},y_2\ind{t,i})
  \phidel(x\ind{t},y_1\ind{t,i},y_2\ind{t,i})^{\trn}$.
\State \textbf{break}
\EndIf
\EndFor
  \EndFor
  \Statex[0] \algcommentbig{Exploration phase}
  \State Initialize dataset $\cDexp\ind{1}=\crl{\emptyset}$.
  \For{iteration $t = 1,2,\dotsc,\nexp$}
    \Statex[1] \algcommentbig{Estimate policy and reward model}
    \State\label{line:spanner_dpo} Fit reward model via regression:
  \begin{small}
    \begin{align}
      \label{eq:spanner_dpo}
      \theta\ind{t}\gets\argmin_{\theta\in\Theta}\sum_{(x,y_1,y_2,r_1,r_2)
        \in \cDexp\ind{t}\cup\cDspan} \prn*{
        \tri*{\theta, \phidel(x,y_1,y_2)}
      -(r_1-r_2)}^2.
      \end{align}
  \end{small}
  \Statex[1] \algcommentbig{Sample responses and update dataset}
  \State Define truncated reward function: \label{line:spanner_reward}
  \begin{align}
    r\ind{t}(x,y,y')\ldef{}
    \tri*{\theta\ind{t},\phidel(x,y,y')}\indic\crl[\big]{\nrm*{\phidel(x,y,y')}_{\Sigmas^{-1}}\leq\nu}.
  \end{align}
  \State Observe prompt $x\ind{t}\sim\rho$. Sample $y\ind{t}_2\sim\piref(\cdot\mid{}x\ind{t})$ and observe
  reward $r\ind{t}_2$.
      \Statex[1]\algcommentlight{Defines policy
    $\pihat\ind{t}(\cdot\mid{}x) \sim
    \rejection_{\beta,\Mrej,\deltarej}(r\ind{t}(x,\cdot, y')\midsem
    x, \piref)$ for $y'\sim\piref(\cdot\mid{}x)$.}
  \State\mbox{Sample $y\ind{t}_1\sim
    \rejection_{\beta,\Mrej,\deltarej}(r\ind{t}(x\ind{t},\cdot, y_2\ind{t})\midsem
    x\ind{t}, \piref)$ and observe reward
    $r\ind{t}_1$.}\label{line:spanner_sampling}
  \State Update dataset: $\cDexp\ind{t+1}\gets\cDexp\ind{t}\cup\crl{(x\ind{t}, y_1\ind{t},
    y_2\ind{t}, r_1\ind{t}, r_2\ind{t})}$.
    \EndFor
    \State \textbf{return}
        $\pihat \sim\unif\prn*{\pihat\ind{1},\ldots,\pihat\ind{\nexp}}$.
\end{algorithmic}
\end{algorithm}

}

\arxiv{%
\begin{algorithm}[ht]
\caption{$\rejectionargs$}
\label{alg:rejection}
\begin{algorithmic}[1]
  \Statex[0]\multiline{ {\bfseries input:}
    Function $f$, prompt $x$, base policy $\piref$, parameter
    $\beta>0$, rejection threshold $M>0$, failure probability $\delta\in(0,1)$.}
  \State Let $N\ldef{}4M\log(4\delta^{-1})$.
  \Statex[0] \algcommentbig{Estimate normalization constant}
  \State Sample $y_1,\ldots,y_N\sim\piref(\cdot\mid{}x)$ \iid
  \State Set $\Zhat\ldef{}\frac{1}{N}\sum_{i=1}^{n}\exp\prn*{\beta^{-1}f(x,y_i)}$.\label{line:rejection_est}
  \Statex[0] \algcommentbig{Rejection sampling}
  \For{iteration $i = 1,2,\dotsc,N$}\label{line:rejection_for}
  \State Sample $y\sim \piref(\cdot\mid{}x)$ and $\xi\sim\Ber\prn*{\nicefrac{\exp\prn*{\beta^{-1}f(x,y)}}{\Zhat{}M}}$.
  \State If $\xi=1$, \textbf{return} $y$.
    \EndFor
    \State \textbf{return}
    $y\sim\piref(\cdot\mid{}x)$.\hfill\algcommentlight{Failure event;
      occurs with low probability.}
\end{algorithmic}
\end{algorithm}

}

\arxiv{\spanalg (\cref{alg:spanner}) consists of
two phases, a \emph{spanner computation phase} and an
\emph{exploration phase}.\loose} \vspace{-5pt}
\paragraph{Spanner phase} Define \emph{relative features} via
$\phidel(x,y,y')=\phi(x,y)-\phi(x,y')$; we use these features throughout the algorithm because---per \cref{eq:reward_difference}---the
difference in rewards $\rstar(x,y)-\rstar(x,y')$ is linear under
\cref{ass:realizable}. In the first
phase, the algorithm aims to compute a \emph{spanner}: a small
collection $\Psispan$ of tuples $(x,y,y')$ such that the\arxiv{ second moment}
matrix $\Sigmas=\lambda\Id +
\sum_{(x,y,y')\in\Psispan}\phidel(x,y,y')\phidel(x,y,y')$ covers the
feature space in directions that have high probability under the
optimal KL-regularized policy $\pistarb$.%
To build the spanner, the
algorithm proceeds in $\Tprompt$ rounds, where at each round
$t\in\brk{\Tprompt}$, we sample $x\ind{t}\sim\rho$, then for each $i\in\brk{\Tspan}$
sample an independent pair
$(y_1\ind{t,i},y_2\ind{t,i})\sim\piref(\cdot\mid{}x\ind{t})$ and check if
$\nrm*{\phidel(x\ind{t},y_1\ind{t,i},y_2\ind{t,i})}_{\Sigmaspan^{-1}}\geq\nu$
for an accuracy parameter $\nu$; whenever this occurs, we query the
reward oracle for $y_1\ind{t,i}$ and $y_2\ind{t,i}$ to receive $(r_1\ind{t},r_2\ind{t})$ and add
$(x\ind{t},y_1\ind{t,i},y_2\ind{t,i},r_1\ind{t},r_2\ind{t})$ to a dataset
$\cDspan$ for use in the second phase, then proceed to the next round $t+1$. This process ensures that: (i) the matrix $\Sigmaspan$ covers $\pistarb$ well, in the sense that\loose
\begin{align}
  \label{eq:spanner_body}
\bbP_{x\sim\rho,y\sim\pistarb(\cdot\mid{}x),y'\sim\piref(\cdot\mid{}x)}\brk*{\nrm*{\phidel(x,y,y')}_{\Sigmaspan^{-1}}>\nu}\approxleq
  \frac{\poly(d,\nu^{-1})}{\Tprompt} + \frac{  \Ccov(\pistarb)}{\Tspan},
\end{align}
and (ii) the size of the spanner stays uniformly bounded as
$\abs*{\Psispan}\leq\poly(d,\nu^{-1})$. These properties imply that if
we estimate $\thetastar$ using least squares on any dataset
$\cD\supset\cDspan$, the resulting estimator will have high accuracy
on directions covered by $\pistarb$, up to the error term in
\cref{eq:spanner_body}. Critically, the size of the spanner---and
hence the number of reward queries $\Tdata$---is uniformly bounded by
$\poly(d,\nu^{-1})$, \emph{irrespective} of $\Tspan$. This means
that the second error term in \cref{eq:spanner_body} can be made arbitrarily small
by increasing inference-time computation (i.e. $\Tspan$), without
increasing the number of reward oracle queries or prompts.

\paragraph{Exploration phase} In the exploration phase, \spanalg performs on-policy exploration in
order to ``fill in'' directions that are not well-covered by the
spanner. This phase proceeds for $\Texp$ rounds, and alternates
between (i) computing an estimate $\theta\ind{t}$ in
\cref{line:spanner_dpo} via\footnote{This is
      equivalent to minimizing the \dpo loss: $\sum\arxiv{_{(x,y_1,y_2,r_1,r_2)
        \in \cDexp\ind{t}\cup\cDspan}} \prn[\big]{
        \beta\log\frac{\pitheta(y_1\mid{}x)}{\piref(y_1\mid{}x)} -
        \beta\log\frac{\pitheta(y_2\mid{}x)}{\piref(y_2\mid{}x)}
        -(r_1-r_2)}^2$.}
    \arxiv{\loose
\[
\theta\ind{t} = \argmin_{\theta\in\Theta}\sum_{(x,y_1,y_2,r_1,r_2)\in\cDexp\ind{t}\cup\cDspan}
      \prn*{
        \tri*{\theta,\phi(x,y_1)-\phi(x,y_2)}
        -(r_1-r_2)}^2
    \]}
    and (ii) updating the dataset $\cDexp\ind{t}$, by sampling a pair
    $(y_1\ind{t},y_2\ind{t})$ (given $x\ind{t}$) from a \emph{truncated
      softmax policy} parameterized by $\theta\ind{t}$ and querying the
    reward oracle for $(r_1\ind{t}, r_2\ind{t})$. The truncated
    softmax policy $\pibt$ is a new type of exploratory policy
    which---to our knowledge---is novel to this work, and induces a joint distribution over a pair
    $(y,y')\mid{}x$ via
    $\pibt(y,y'\mid{}x)=\pibt(y\mid{}x,y')\piref(y'\mid{}x)$,
    where
    \begin{align}
      \label{eq:truncated_softmax_body}
\pibt(y\mid{}x,y')\propto\piref(y\mid{}x)\cdot\exp\prn*{\beta^{-1}\tri*{\theta\ind{t},\phidel(x,y,y')}\indic\crl[\big]{\nrm*{\phidel(x,y,y')}_{\Sigmaspan^{-1}}\leq\nu}
      }.
    \end{align}
    Without the indicator in \cref{eq:truncated_softmax_body}, this
    coincides with the standard softmax policy
    $\pi_{\theta\ind{t}}(y\mid{}x)$, but the indicator ``truncates''
    the reward in directions that are uncertain according to the
    spanner. Truncation allows \spanalg to
    proceed using only a \emph{weak} sampling oracle\arxiv{ (\cref{def:oracle})}: whenever
    the spanner phase succeeds, we are guaranteed that
    $\frac{\pibt(y\mid{}x,y')}{\piref(y\mid{}x)}\approxleq\Ccov(\pistarb)$
    for ``most'' pairs $(x,y')$ (\cref{lem:truncated_density}). This means we can use
    \emph{rejection sampling} (\rejection; \cref{alg:rejection}) at
    inference-time to
    transform samples from $\piref$ into samples from
    $\pibt(y\mid{}x,y')$, with computational cost
    $\Tcomp=\bigoht(\Ccov(\pistarb))$ per round.\footnote{For a generic
      parameter $\theta$, we have
      \arxiv{$\frac{\pitheta(y\mid{}x)}{\piref(y\mid{}x)}\leq\Ccov(\pitheta)$,
      but in general, we can have}
      $\Ccov(\pitheta)\gg\Ccov(\pistarb)$. A central insight in
      our analysis is that we can control the density ratio by
      $\Ccov(\pistarb)$ even when $\theta\neq\thetastar$ by building a spanner and using it to truncate.\loose
    }
    We write $\pihat\ind{t}(y,y'\mid{}x) \approx
        \pibt(y,y'\mid{}x)$
to denote the distribution induced by rejection sampling with the
\rejection subroutine, and let
$\pihat\ind{t}(y\mid{}x)\ldef\En_{y'\sim\piref(\cdot\mid{}x)}\brk*{\pihat\ind{t}(y\mid{}x,y')}$. See
\cref{sec:rejection} for detailed background on \rejection.

    \begin{remark}[Average-case vs. uniform spanners]
      \label{rem:spanner}
Our usage of the term ``spanner'' is inspired but
  technically different from the notion of an optimal design or
  barycentric spanner, which has been widely used in the linear bandit
  literature \citep{awerbuch2008online,hazan2016volumetric,lattimore2020learning}. These notions provide a small collection of responses
  for which second moment matrix $\Sigma$ achieves \emph{uniform
    coverage} in the sense that
  $\max_{x,y,y'}\nrm*{\phidel(x,y,y')}_{\Sigma^{-1}}\leq\poly(d)$ or
  similar. For computational reasons, we cannot hope to achieve such a
  uniform guarantee, and instead settle for average-case coverage with
  respect to $\pist$.
    \end{remark}

    \arxiv{
      \begin{remark}[Anchor responses]
        \label{rem:anchor}
              \cref{alg:spanner} can be slightly simplified as follows:
      Instead of sampling $y_2\ind{t}\sim\piref(\cdot\mid{}x\ind{t})$,
      we can set $y_2\ind{t}=\mathfrak{y}\;\;\forall{}t$ for an
      arbitrary fixed ``anchor'' response $\mathfrak{y}$. This leads
      to the same guarantee, but does not fall
        into the sampling oracle framework in \cref{def:oracle}, as it
        requires observing the features $\phi(x\ind{t},\mathfrak{y})$
        for all $t$. However, we use this technique within our
        multi-turn algorithm \mtalg in \cref{sec:multi}.\loose
        
    \end{remark}
  }

\vspace{-5pt}
\subsection{Guarantee for \spanalg}

The main guarantee for \spanalg is as follows.

\begin{restatable}[Guarantee for \spanalg]{theorem}{spannermain}
  \label{thm:spanner}
    For any $\veps>0$ and $\delta\in(0,1)$, by choosing $\Tprompt$,
    $\Tspan$, and $\nexp$
  appropriately, \cref{alg:spanner} learns a policy with
  $\En_{\pihat\sim\unif\prn{\pihat\ind{1},\ldots,\pihat\ind{\nexp}}}\brk[\big]{\Jbeta(\pistarb)-\Jbeta(\pihat)}\leq\veps$ with probability at least
  $1-\delta$, and achieves the following data efficiency and oracle efficiency
  bounds:\loose
      \[
    \Tdatafull =
    \bigoht\prn*{\frac{\Rmax^2}{\beta}}\cdot\frac{d^2\log^2(\delta^{-1})}{\min\crl{\veps,\beta}},
    \mathand     \Tsamplefull  = 
    \bigoht\prn*{\Ccov(\pistarb)\cdot\frac{\Rmax^2}{\beta^2}}\cdot\Tdatafullsq.
  \]
  Moreover, (1) for any $x\in\cX$, one can generate a sample
$y\sim\pihat(\cdot\mid{}x)$ from the returned policy using at most
$\Tsample=\bigoht\prn[\big]{\Ccov(\pistarb)}$ weak sampling
oracle queries; (2) the algorithm uses  at
most $  \bigoht\prn*{\frac{\Rmax^4}{\beta^3}}\cdot
  \frac{d^2 \log^2
  \prn*{\delta^{-1}}}{\veps}$ prompts.
\end{restatable}

On the computational side, we observe that the number of sampling oracle queries $\Tcompfull$ is controlled by the coverage coefficient $\Ccov(\pistarb)\leq\exp\prn*{\Rmax/\beta}$,
achieving the lower bound in \cref{thm:coverage}, and the total runtime of the
algorithm scales as $\poly(d,\Ccov(\pistarb),\veps^{-1},\beta^{-1},\log(\delta^{-1}))$.\footnote{Indeed,
    outside of queries to the sampling oracle, the only runtime
    overhead in \spanalg is (i) minimizing the \dpo loss (linear least
    squares), (ii) inverting the second moment matrix, and (iii)
    evaluating various inner products.} Furthermore, the algorithm only requires a \emph{weak sampling oracle}
  (\cref{def:oracle}), and hence can be
    viewed as performing exploration purely at inference time, with an
    iteratively updated reward model.%
    \arxiv{ Whether the polynomial
dependence on problem parameters for $\Tcompfull$ can be improved is an interesting question.}\loose

On the statistical side, our bound on $\Tdatafull$ matches the
minimax rate for linear bandits in terms of dependence on $d$ and
$\Rmax$ when $\veps\leq\beta$
\citep{lattimore2020bandit}, and is \emph{independent of the coverage
coefficient}, reflecting active exploration. The number of prompts used by the algorithm is also
independent of the coverage coefficient, though it is slightly larger than
the number of reward queries. We observe that \cref{thm:spanner}
achieves a \emph{fast rate} in the sense that
$\Tdatafull\approxleq\frac{1}{\beta\veps}$ when $\veps\leq\beta$,
improving over the $\Tdatafull\approxleq{}\frac{1}{\veps^2}$ rate for
\xpo (\cref{eq:xpo_data}) and other prior
    work \citep{xiong2024iterative,xie2024exploratory,cen2024value}; this is
    a secondary benefit of working with truncated policies (see
    \cref{eq:truncated_teaser}), and is facilitated by the strong
    convexity induced by regularization; we view it as analogous to $\frac{1}{\Delta\veps}$-type rates for bandits with
gap-$\Delta$ \citep{lai1985asymptotically,lattimore2020bandit}. Concurrent work of \citet{zhao2025logarithmic} achieves
a similar fast rate, but their algorithm is not computationally
efficient in our framework.\loose

\noindent\textbf{Analysis techniques.} \arxiv{As alluded
to in the prequel, the} key algorithmic ideas to ensure that
  $\Tcompfull$---but not the number of reward queries and prompts---is controlled by $\Ccov(\pistarb)$ are: (i) even though $\Tspan$ can grow with $\Ccov(\pistarb)$,
  the size of the spanner $\Psispan$ (and number of reward queries in
  the spanner phase) is
  bounded as $\poly(d,\beta^{-1}, \Rmax)$; and (ii) the truncated
  policy construction ensures we can simulate $\pibt$ using
  rejection sampling with $\bigoht(\Ccov(\pistarb))$ draws from
  $\piref$. Our regret analysis makes use of the following decomposition for truncated softmax policies, which may be of
independent interest (see \cref{lem:truncated_regret_simple}\arxiv{ for the
full statement}): Given a parameter $\theta$, define $\vepsstat^2
\ldef \nrm*{\theta-\thetastar}_{\Sigmaspan}^2$, and for $\veps>0$, let
\[
    \cXspan(\veps) \ldef
\crl*{x\in\cX\mid{}\bbP_{(y,y')\sim\piref(x)}\brk*{\nrm*{\phidel(x,y,y')}_{\Sigma^{-1}}>\nu} \leq\veps}.
\]
Then under \cref{ass:realizable,ass:norm}, if
$\nu\leq\beta/\vepsstat$, we have that for all $\veps>0$,
\begin{align}
 \Jbeta(\pist) - \Jbeta(\pibtheta)
  &\leq
\frac{1}{\beta}\En_{(y,y')\sim\pibtheta(\cdot\mid{}x)}\brk*{\tri*{\theta-\thetastar,\phidel(x,y,y')}^2}\label{eq:truncated_teaser}\\
  &~~~~+ \bigoh\prn*{\Rmax\Ccov(\pistarb)}\cdot\veps
    + \bigoh\prn{\Rmax}\cdot \bbP_{x\sim\rho}\brk*{x\notin\cXspan(\veps)}
    .\notag
\end{align}
The first term above is controlled by the (on-policy) exploration
phase; by virtue of the square, this leads to the fast
$\frac{1}{\beta\veps}$ rate. Meanwhile, the last two terms are controlled
by the spanner construction: for
$\vepsspan\approx\frac{1}{\Tspan}$, we have that $\bbP_{x\sim\rho}\brk*{x\notin\cXspan(\vepsspan)}\approxleq\frac{1}{\Tprompt}$.

\section{Training-Time Interventions Cannot Be Computationally Efficient}
\label{sec:computational}

\spanalg (\cref{alg:spanner}) leverages
  inference-time computation with the pre-trained model to reduce the
  effective search space for exploration. As a result, the algorithm is
  \emph{improper} in the sense that it does not use linear softmax
  policies $\pitheta\in\Pi$ to draw
  the responses for which it queries
  the reward oracle; this is true for the
spanner phase (the algorithm samples
$y\sim\piref(\cdot\mid{}x)$ properly, but adaptively chooses whether
or not to query the reward oracle $y$), and for the exploration phase
(due to the use of truncation and rejection sampling). We contrast this with the notion of \emph{proper exploration}.\loose
\begin{definition}[Proper alignment algorithm]\label{def:proper-exploration}
An online alignment algorithm is \emph{proper} if, for each
$t\in\brk{\Trounds}$ and $i\in\brk{N}$, the algorithm queries the
reward oracle with
$y\ind{t}_i\sim\pi_{\theta\ind{t}_i}(\cdot\mid{}x\ind{t})$ for some
$\theta_i\ind{t}\in\Theta$.%
\footnote{The parameter $\theta_i\ind{t}$ may
  be chosen adaptively based on the previously sampled responses and rewards.}
\end{definition}
Proper algorithms are closely related to the notion of
\emph{training-time} interventions for exploration, in the sense that
any algorithm that computes exploratory policies $\pi\ind{t}$ by solving %
\arxiv{\[
\pi\ind{t} = \argmin_{\pi\in\Pi}L_{\cD}\ind{t}(\pi)
\]}
for some loss function $L_{\cD}\ind{t}(\pi)$ that depends on the dataset
$\cD$ collected so far will inevitably be proper in the sense of
\cref{def:proper-exploration}---no matter how clever we are about
designing the loss. This includes \onlinedpo \citep{guo2024direct}, 
\xpo \citep{xie2024exploratory,cen2024value}, and many others \citep{zhang2024self,liu2024provably,gao2024rebel}. We
show that under the Exponential Time Hypothesis (ETH), no such
algorithm can be simultaneously data-efficient and computationally
efficient.

\begin{restatable}[Proper \arxiv{alignment }algorithms cannot be computationally efficient]{theorem}{complower}
  \label{thm:training}
Under the Randomized \arxiv{Exponential Time Hypothesis}%
(\cref{conj:randeth}), there is no proper alignment algorithm, even
with a \emph{strong oracle}\arxiv{ (\cref{def:oracle})} and a Euclidean projection oracle for $\Theta$, that (i)
has $\Tdatafull \leq
\poly(d,\beta^{-1},\veps^{-1},\delta^{-1})$ and $\Tsamplefull \leq \poly(d, \exp\prn{\beta^{-1}},
  \veps^{-1},\delta^{-1})$ under \cref{ass:norm} (with $\Rmax=1$, $B=\sqrt{d}$),\footnote{Concretely, we use the parameter
  set
  $\Theta=\crl[\big]{\theta\in\bbR^{d}\mid{}\nrm*{\theta}_{\infty}\leq{}1}$.}
and (ii) has runtime
$\poly(d,\exp\prn{\beta^{-1}},\veps^{-1},\delta^{-1})$.\loose
\end{restatable}
Note that \cref{def:proper-exploration} does not require the final output policy $\pihat$ to be proper; the restriction is solely on the policies used to \emph{explore}. For contrast, we recall that since $\Ccov(\pistarb)\leq\exp(\beta^{-1})$ when
$\Rmax=1$, \spanalg achieves $\Tdatafull \leq
\poly(d,\beta^{-1},\veps^{-1},\delta^{-1})$ and $\Tsamplefull \leq \poly(d, \exp\prn{\beta^{-1}},
  \veps^{-1},\delta^{-1})$ under the conditions of
  \cref{thm:training}, and does so with time complexity
  $\poly(d,\exp\prn{\beta^{-1}},\veps^{-1},\delta^{-1})$.
  \cref{thm:training} shows that no proper alignment algorithm can
  achieve polynomial runtime in a similar fashion. In particular,
  while \xpo \citep{xie2024exploratory} achieves $\Tdatafull \leq
\poly(d,\beta^{-1},\veps^{-1},\delta^{-1})$, \cref{thm:training}
implies that it cannot be implemented efficiently for linear softmax
policies. This answers a question raised by
\citet{xie2024exploratory}, and gives a separation between
\arxiv{algorithms based on }training-time interventions and algorithms like \spanalg that
use additional inference-time computation to explore 
improperly. We remark that like \cref{thm:coverage}, this lower bound
uses only a single prompt.
\arxiv{\paragraph{Proof sketch}}
To prove \cref{thm:training}, we reduce from the Max-$k$-DNF problem,
embedding a $k$-DNF formula in $\phi(x,y)$ so that responses
correspond to clauses, and embedding a maximally satisfying assignment
in the hidden parameter $\thetastar$. Our construction ensures that
any \emph{proper} exploration policy $\pi_{\theta}$ places all but a vanishing fraction of mass
on a ``null'' response $y_0=\mb{0}$ (which is useless for gathering reward
information) unless $\theta$ corresponds to an
assignment that satisfies a large fraction of clauses. Directly
finding such a $\theta$ requires (approximately) maximizing the
underlying $k$-DNF formula, and the assumptions $\Tdatafull \leq
\poly(d,\beta^{-1},\veps^{-1},\delta^{-1})$ and $\Tsamplefull \leq \poly(d, \exp\prn{\beta^{-1}},
  \veps^{-1},\delta^{-1})$ ensure that we will not sample a
  non-null response $y_0\neq\mb{0}$ by chance (which could reveal information
  about\arxiv{ the assignment} $\thetastar$); from here, the
  hardness follows. Interestingly, \arxiv{the result uses}%
  hardness of approximation for Max-$k$-DNF in a somewhat non-standard
  parameter regime\arxiv{; we establish this in \cref{sec:maxkdnf} by
  reducing from Max-$k$-CSP and appealing to gap
  amplification.\loose}

\arxiv{\section{Computational Benefits of Multi-Turn Exploration}}
\label{sec:multi}

Our motivating example (\cref{rem:auto}) is the autoregressive
setting where $\cY=\cA^{H}$ for a token space $\cA$ and
horizon $H$ (so that responses $y=(a_1,\ldots,a_H)$ correspond to
sequences of tokens), and where $\piref$ is
explicitly represented as an autoregressive policy of the form %
\arxiv{\[
\piref(y\mid{}x) = \piref(a_{1:H}\mid{}x)=\prod_{h=1}^{H}\pirefh(a_h\mid{}x,a_{1:h-1}).
\]}
In what follows, we show how to specialize \mainalg to this
setting, then derive algorithms with \emph{improved} computational efficiency
by alternatively viewing this as a reinforcement learning problem
in a \emph{token-level MDP} where actions correspond to tokens
\citep{rafailov2024r}.

\subsection{Autoregressive Softmax Policies: Representational Issues
  and \mainalg}
\label{sec:autoregressive}
\newcommand{\Piseq}{\Pi_{\texttt{seq}}}
\newcommand{\Piauto}{\Pi_{\texttt{auto}}}
\newcommand{\pithetah}[1][\theta]{\pi^{\texttt{auto}}_{h,#1}}
\newcommand{\piauto}[1][\theta]{\pi^{\texttt{auto}}_{#1}}
\newcommand{\piautoh}[1][h]{\pi^{\texttt{auto}}_{#1,\theta}}
\newcommand{\piseq}[1][\theta]{\pi^{\texttt{seq}}_{#1}}
\newcommand{\phiseq}{\phi^{\texttt{seq}}}
\newcommand{\Thetaseq}{\Theta^{\texttt{seq}}}

When the base policy $\piref$ is autoregressive, it is natural to
learn a policy with the same autoregressive structure. We
consider the class
$\Piauto\ldef{}\crl[\big]{\piauto=(\piautoh[1],\ldots,\piautoh[H])\mid{}\theta_h\in\Theta_h\;\forall{}h}$
of autoregressive linear softmax policies given by
\begin{equation}
\label{eq:auto}
  \pithetah(a_h\mid{}x,a_{1:h-1})
  = \frac{\pirefh(a_h\mid{}x,a_{1:h-1})\exp\prn*{\beta^{-1}\tri*{\theta_h,\phi_h(x,a_{1:h})}}}{\sum_{a'\in\cA}\pirefh(a'\mid{}x,a_{1:h-1})\exp\prn*{\beta^{-1}\tri*{\theta_h,\phi_h(x,a_{1:h-1},a')}}},
\end{equation}
with 
$\piauto(a_{1:H}\mid{}x)\ldef\prod_{h=1}^{H}\pithetah(a_h\mid{}x,a_{1:h-1})$;
we assume
$\theta_h\in\Theta_h\subset\bbR^{d}$ with $\nrm*{\theta_h}\leq{}B$ and
$\nrm*{\phi_h(x,a_{1:h})}\leq{}1$, where each $\Theta_h$ is convex. This
parameterization corresponds to a standard deep
autoregressive model (e.g., GPT-2 architecture) in which the weights
for all but the last layer are frozen \citep{radford2019language}. For this setting, we use
the following computational oracle, which asserts that we can
sample from each conditional policy efficiently.\loose
\begin{definition}[Autoregressive sampling oracle]
  \label{def:oracle_autoregressive} \textnormal{\textbf{\underline{Setting I (strong oracle)}}:} In one query, the learner
  proposes a prompt $x\in\cX$, layer $h\in\brk{H}$, prefix
  $a_{1:h-1}\in\cA^{h-1}$, and parameter $\theta_h\in\Theta_h$, and receives a conditional sample
  $a_h\sim\pithetah(\cdot\mid{}x,a_{1:h-1})$ and the corresponding feature
$\phi_h(x,a_{1:h})$\arxiv{ for the sampled response}.\\
\textnormal{\textbf{\underline{Setting II (weak oracle)}}:} In one query, the learner
  proposes a prompt $x\in\cX$, layer $h\in\brk{H}$, and prefix
  $a_{1:h-1}\in\cA^{h-1}$, and receives a conditional sample
  $a_h\sim\pirefh(\cdot\mid{}x,a_{1:h-1})$ and corresponding feature
  $\phi_h(x,a_{1:h})$. We let $\Tcompa$ denote the total number of autoregressive sampling
  queries\arxiv{ used by the algorithm}.\footnote{Technically, our
    algorithm results also require query access to $\phi_h$ for a
    fixed reference action---see \cref{remark:mtss-query-access}.} \loose

\end{definition}
Efficient conditional sampling is arguably the defining property of
autoregressive models, so we view this as a minimal assumption. As before, our algorithmic results only use the
weak oracle (sampling from $\piref$), but the strong oracle will be
useful for discussion. For rewards, we remain in the
setup of
\cref{sec:background}: For each prompt $x\ind{t}$ for
$t\in\brk{\Trounds}$, the algorithm can query the reward oracle for up
to $N$ responses $y\ind{t}_1,\ldots,y\ind{t}_N\in\cA^{H}$.\loose

\paragraph{Representational issues}
To
make use of the class $\Piauto$, we need to assume that $\pistarb\in\Piauto$
(\cref{ass:realizable}). Perhaps the simplest setting where we might
hope for this is when rewards are linear:\loose
\begin{equation}
  \label{eq:reward_sequence}
\rstar(x,y) = \sum_{h=1}^{H}\tri*{\thetastar_h,\phi_h(x,a_{1:h})}
\end{equation}
for some $\thetastar_h\in\Theta_h$. Here, the
optimal KL-regularized policy $\pistarb$ under \cref{eq:reward_sequence} setting
takes the form\footnote{\mbox{By the chain rule, the
  sequence-level KL-regularizer in \cref{eq:kl_reward} is a
  equivalent to a sum of per-action regularizers (\cref{eq:kl}).}}\loose
\[
\piseq[\thetastar](a_{1:H}\mid{}x)
  \ldef \frac{\piref(a_{1:H}\mid{}x)\exp\prn*{\beta^{-1}\sum_{h=1}^{H}\tri*{\thetastar_h,\phi_h(x,a_{1:h})}}}{\sum_{(a'_1,\ldots,a'_H)\in\cA^{H}}\piref(a'_{1:H}\mid{}x)\exp\prn*{\beta^{-1}\sum_{h=1}^{H}\tri*{\thetastar_h,\phi_h(x,a'_{1:h})}}}.
\]
This corresponds
to the linear softmax policy in \cref{def:softmax} with sequence-level feature map
$\phiseq(x,a_{1:H})\ldef{}(\phi_1(x,a_1),\ldots,\phi_H(x,a_{1:H}))\in\bbR^{dH}$
and parameter space
$\Thetaseq\ldef{}(\Theta_1,\ldots,\Theta_H)\subset\bbR^{dH}$ (the natural policy class is
$\Piseq\ldef{}\crl*{\piseq[\theta]\mid{}\theta\in\Thetaseq}$). Unfortunately, \emph{sequence-level} linear softmax policies of this type cannot be
represented as autoregressive linear softmax policies in general; that is, there may not exist
any $\theta=(\theta_1,\ldots,\theta_H)$ such that
$\piseq[\thetastar]=\piauto$---see \cref{prop:autoregressive}.\footnote{$\pistarb$ can always be
  represented as an autoregressive softmax policy applied to a certain
  \emph{KL-regularized value function} $\Qstar_{h,\beta}$---see
  \cref{eq:softmax_qstar}---but is not necessarily a \emph{linear}
  softmax unless $\Qstar_{h,\beta}$ itself is linear.}\loose

\paragraph{Applying \mainalg}
Even though autoregressive realizability may not hold under
\cref{eq:reward_sequence}, we can still apply \mainalg efficiently
under the sequence-level realizability assumption that
$\pistarb\in\Piseq$ (which is implied by
\cref{eq:reward_sequence}). In particular, a weak autoregressive
oracle (\cref{def:oracle_autoregressive}) immediately gives a weak
sequence-level sampling oracle (\cref{def:oracle}) with $\Tcompa\leq{}H\cdot\Tcomp$.
\begin{corollary}
  \label{prop:ss_autoregressive}
  Suppose \cref{ass:realizable} is satisfied for\arxiv{ the class} $\Piseq$
  and
  $\tri*{\thetastar,\phiseq(x,a_{1:H})}\in\brk*{0,\Rmax}$. \mainalg
learns a policy with $\En_{\pihat\sim\unif\prn{\pihat\ind{1},\ldots,\pihat\ind{\nexp}}}\brk[\big]{\Jbeta(\pistarb)-\Jbeta(\pihat)}\leq\veps$
  with probability at least
  $1-\delta$   when configured appropriately, and does so with:\loose
      \[
    \Tdatafull =
    \bigoht\prn*{\frac{\Rmax^2}{\beta}}\cdot\frac{d^2H^2\log^2(\delta^{-1})}{\min\crl{\veps,\beta}},
    \mathand     \Tcompafull  = 
    \bigoht\prn*{\Ccov(\pistarb)\cdot\frac{H\Rmax^2}{\beta^2}}\cdot\Tdatafullsq.
  \]
\end{corollary}
For this result, the fact that \mainalg only uses a \emph{weak} sequence-level sampling oracle (\cref{def:oracle}) is
crucial: due to aforementioned representational issues, a strong sequence-level sampling oracle (sampling from $\piseq[\theta]$ for $\theta\in\Thetaseq$) cannot necessarily be simulated by even a strong
autoregressive oracle (sampling from $\piautoh$).\loose

\subsection{Improving Computational Efficiency through Multi-Turn
  Exploration}
The guarantee in \cref{prop:ss_autoregressive} depends on the
sequence-level coverage coefficient $\Ccov(\pistarb)$ for
$\pistarb$. While various works have shown
that existing pre-trained models may exhibit favorable coverage for tasks of interest
\citep{brown2024large,snell2024scaling,wu2024empirical}, it is natural
to ask whether we can improve the computational efficiency further,
perhaps by exploiting the autoregressive structure of $\piref$. To this end, 
we will make the autoregressive realizability assumption that
$\pistarb\in\Piauto$ (i.e., there exists
$\thetastar=(\thetastar_1,\ldots,\thetastar_H)$ such that
$\piauto[\thetastar]=\pistarb$). As discussed above, this is not implied by sequence-level realizability in general, but
we will show that when it holds, we can achieve runtime guarantees
that scale with the following \emph{conditional} (or, token-level/action-level)
coverage coefficient:
\begin{align}
\label{eq:coverage_auto}
    \Ccond(\pistarb) \coloneqq \max_{h\in\brk{H}}\sup_{x \in \cX}\sup_{(a_1,\ldots,a_h)\in\cA^{h}}
        \frac{\pistarbh(a_h \mid x,a_{1:h-1})}{\pi_{h,\refe}(a_h \mid x,a_{1:h-1})}.
\end{align}
\mbox{This coefficient can exponentially improve
$\Ccov(\pistarb)$; we can have
$\Ccond(\pistarb)\leq{}2$, yet $\Ccov(\pistarb)\geq{}2^{H}$.}\loose

\paragraph{\texttt{MultiTurnSpannerSampling}}
We introduce a 
\emph{multi-turn} counterpart to \mainalg, \mtalg (\cref{alg:ops-dp}
in \cref{sec:algos_rl}). \mtalg
learns a policy in a multi-turn (dynamic programming) fashion by
fitting softmax policies for each layer $h = H, \dots, 1$, while growing \emph{core-sets} of informative sub-sequences
$(x,a_{1:h})$ for which the algorithm can confidently estimate the
parameter $\thetastar_h$ (generalizing the notion of spanner used in
\mainalg). The use of dynamic programming in the
algorithm is motivated by the fact that whenever $\pistarb$ is
autoregressive (i.e., $\pistarb\in\Piauto$), a certain
\emph{KL-regularized} state-action value function $\Qstarb(x,
a_{1:h})$ is linear up to an action-independent shift. See \cref{sec:prelim_rl,sec:algos_rl} for a detailed overview.

\begin{theorem}[Guarantee for \mtalg; special case of \cref{thm:main}]
  \label{thm:multi}
Suppose \cref{ass:realizable} is satisfied for the class
$\Piauto$. \mtalg, when configured appropriately, returns $\pihat$
such that $J_\beta(\piauto[\thetastar]) - J_\beta(\pihat_{1:H})\leq
\veps$ with probability at least $1-\delta$, and does so with
$\Tdata(\veps,\delta) \leq  \poly(d, H, B, \veps^{-1},\log(\delta^{-1}))$
reward queries and $\Tcompa(\veps,\delta) \leq \poly\left(\Ccond
  (\pistarb),\Tdata(\veps,\delta)\right)$ 
(weak) autoregressive sampling queries.\loose
\end{theorem}
As discussed above, the action-level coverage
coefficient $\Ccond
  (\pistarb)$ in this result can be exponentially smaller than the
  sequence-level coverage coefficient. We view the assumption that
  $\pistarb\in\Piauto$ as a fairly minimal representational assumption
  for working with autoregressive policies (i.e., for \mtalg to learn
  efficiently, all we require is that the base policy $\piref$ and the
  optimal policy $\pistarb$ are autoregressive), analogous to the
  classical notion of $\Qstar$-realizability in linearly-parameterized
  RL \citep{li2021sample,yin2022efficient,
  weisz2022confident}. \arxiv{We remark that the polynomial dependence on other
problem parameters is significantly worse than that of \mainalg; we
view \cref{thm:multi} as a proof-of-concept, and it can likely
be tightened with more effort.}

We remark that while our exposition focuses on the token-level MDP,
the results above also apply to the more realistic setting where each
action $a_h$ represents a sub-sequence of tokens (e.g., a lemma in a
proof) rather than a single token (e.g., \citet{xiong2024building}). Here, the
fact that the runtime and sample complexity for \mtalg are independent
of $\abs{\cA}$ is essential.

\paragraph{Connection to reinforcement learning with linear $\Qstar$}
\mtalg can be applied beyond the
token-level MDP formulation above: Our presentation and
guarantees for the algorithm in \cref{part:multi} of the appendix
apply to \emph{any MDP} for which $\pistarb$ is an
``action-level'' linear softmax policy (a generalization of the
assumption that the optimal KL-regularized value function $\Qstarb$ is linear), provided that resetting to
previously visited states\arxiv{ (\emph{local simulator access})} is
allowed. In this regard, the algorithm can be viewed as a counterpart
to a body of work which shows that MDPs with linear $\Qstar$ and
state-action gap $\Delta$ can be learned under reset access \citep{li2021sample,yin2022efficient,
  weisz2022confident}; the regularization parameter $\beta$
plays a role analogous to the gap $\Delta$ in facilitating favorable
error propagation in our analysis.\loose

\emph{\textbf{See \cref{part:multi} of the appendix for a formal
    presentation of the \mtalg algorithm and guarantees for learning
    autoregressive linear softmax policies in general MDPs
    (generalizing \cref{thm:multi})}}

\arxiv{
\section{Discussion}
\label{sec:discussion}
Our results---via the sampling oracle framework---reveal the
computational, statistical, and representational tradeoffs inherent to
language model exploration, highlighting the fundamental role of the
base model $\piref$ in enabling computational efficiency. We view our
results as an initial step toward a computational foundation for
language model exploration, and more broadly, for efficient decision
making with generative models. To this end, some natural questions are
as follows.

\paragraph{Efficient exploration beyond linear softmax policies}
While our lower bounds are relevant beyond the linear softmax
parameterization, our algorithms are specialized to this
setting. Developing algorithms to support general, nonlinear policy
parameterizations is perhaps the most important question left by our
work. We expect that the basic principle behind our algorithms---expending inference-time computation to identify ``representative''
responses with which to explore---to be useful more broadly, but the
specific notion of spanner used in our results will need to
change.\footnote{We expect it to be fairly straightforward to extend
  our results to accommodate policy classes with bounded eluder
  dimension, but it is less clear how to address realistic classes
  based on, e.g., transformers..}
\paragraph{Better representations for exploration}
Our results in \cref{sec:computational} show that training-time
interventions that produce softmax policies (e.g., modifications to
the \dpo loss) are insufficient for computationally efficient
exploration. This raises the question of whether there exist
training-time interventions that induce different policy
representations (e.g., based on alternative forms of regularization \citep{wang2024beyond,huang2024correcting})
that more readily lend themselves to computationally efficient
exploration. Our results in \cref{sec:algorithms} show that relatively
simple modifications to the linear softmax parameterization (e.g., truncation)
have benefits for exploration, but are there more general principles beyond
the linear setting?\loose

\arxiv{
\subsection*{Acknowledgements}
We thank Qinghua Liu and Tengyang Xie for several helpful discussions
and comments.
}

}

\bibliography{refs}

\begin{thebibliography}{111}
\providecommand{\natexlab}[1]{#1}
\providecommand{\url}[1]{\texttt{#1}}
\expandafter\ifx\csname urlstyle\endcsname\relax
  \providecommand{\doi}[1]{doi: #1}\else
  \providecommand{\doi}{doi: \begingroup \urlstyle{rm}\Url}\fi

\bibitem[Abbasi-Yadkori et~al.(2011)Abbasi-Yadkori, P{\'a}l, and
  Szepesv{\'a}ri]{abbasi2011improved}
Yasin Abbasi-Yadkori, D{\'a}vid P{\'a}l, and Csaba Szepesv{\'a}ri.
\newblock Improved algorithms for linear stochastic bandits.
\newblock In \emph{Advances in Neural Information Processing Systems}, 2011.

\bibitem[Agarwal et~al.(2019)Agarwal, Jiang, Kakade, and
  Sun]{agarwal2019reinforcement}
Alekh Agarwal, Nan Jiang, Sham~M Kakade, and Wen Sun.
\newblock Reinforcement learning: Theory and algorithms.
\newblock \url{https://rltheorybook.github.io/}, 2019.
\newblock Version: January 31, 2022.

\bibitem[Awerbuch and Kleinberg(2008)]{awerbuch2008online}
Baruch Awerbuch and Robert Kleinberg.
\newblock Online linear optimization and adaptive routing.
\newblock \emph{Journal of Computer and System Sciences}, 74\penalty0
  (1):\penalty0 97--114, 2008.

\bibitem[Bagnell et~al.(2003)Bagnell, Kakade, Schneider, and
  Ng]{bagnell2003policy}
James Bagnell, Sham~M Kakade, Jeff Schneider, and Andrew Ng.
\newblock Policy search by dynamic programming.
\newblock \emph{Advances in neural information processing systems}, 16, 2003.

\bibitem[Block and Polyanskiy(2023)]{block2023sample}
Adam Block and Yury Polyanskiy.
\newblock The sample complexity of approximate rejection sampling with
  applications to smoothed online learning.
\newblock In \emph{The Thirty Sixth Annual Conference on Learning Theory},
  pages 228--273. PMLR, 2023.

\bibitem[Bose et~al.(2024)Bose, Xiong, Saha, Du, and Fazel]{bose2024hybrid}
Avinandan Bose, Zhihan Xiong, Aadirupa Saha, Simon~Shaolei Du, and Maryam
  Fazel.
\newblock Hybrid preference optimization for alignment: Provably faster
  convergence rates by combining offline preferences with online exploration.
\newblock \emph{arXiv preprint arXiv:2412.10616}, 2024.

\bibitem[Bradley and Terry(1952)]{bradley1952rank}
Ralph~Allan Bradley and Milton~E Terry.
\newblock Rank analysis of incomplete block designs: I. the method of paired
  comparisons.
\newblock \emph{Biometrika}, 39\penalty0 (3/4):\penalty0 324--345, 1952.

\bibitem[Brown et~al.(2024)Brown, Juravsky, Ehrlich, Clark, Le, R{\'e}, and
  Mirhoseini]{brown2024large}
Bradley Brown, Jordan Juravsky, Ryan Ehrlich, Ronald Clark, Quoc~V Le,
  Christopher R{\'e}, and Azalia Mirhoseini.
\newblock Large language monkeys: Scaling inference compute with repeated
  sampling.
\newblock \emph{arXiv:2407.21787}, 2024.

\bibitem[Brown et~al.(2020)Brown, Mann, Ryder, Subbiah, Kaplan, Dhariwal,
  Neelakantan, Shyam, Sastry, Askell, Agarwal, Herbert-Voss, Krueger, Henighan,
  Child, Ramesh, Ziegler, Wu, Winter, Hesse, Chen, Sigler, Litwin, Gray, Chess,
  Clark, Berner, McCandlish, Radford, Sutskever, and Amodei]{brown2020language}
Tom Brown, Benjamin Mann, Nick Ryder, Melanie Subbiah, Jared~D Kaplan, Prafulla
  Dhariwal, Arvind Neelakantan, Pranav Shyam, Girish Sastry, Amanda Askell,
  Sandhini Agarwal, Ariel Herbert-Voss, Gretchen Krueger, Tom Henighan, Rewon
  Child, Aditya Ramesh, Daniel Ziegler, Jeffrey Wu, Clemens Winter, Chris
  Hesse, Mark Chen, Eric Sigler, Mateusz Litwin, Scott Gray, Benjamin Chess,
  Jack Clark, Christopher Berner, Sam McCandlish, Alec Radford, Ilya Sutskever,
  and Dario Amodei.
\newblock Language models are few-shot learners.
\newblock In \emph{Advances in Neural Information Processing Systems}, 2020.

\bibitem[Bubeck et~al.(2012)Bubeck, Cesa-Bianchi, and
  Kakade]{bubeck2012towards}
S{\'e}bastien Bubeck, Nicolo Cesa-Bianchi, and Sham~M Kakade.
\newblock Towards minimax policies for online linear optimization with bandit
  feedback.
\newblock In \emph{Conference on Learning Theory}, pages 41--1. JMLR Workshop
  and Conference Proceedings, 2012.

\bibitem[Calabro et~al.(2008)Calabro, Impagliazzo, Kabanets, and
  Paturi]{calabro2008complexity}
Chris Calabro, Russell Impagliazzo, Valentine Kabanets, and Ramamohan Paturi.
\newblock The complexity of unique k-sat: An isolation lemma for k-cnfs.
\newblock \emph{Journal of Computer and System Sciences}, 74\penalty0
  (3):\penalty0 386--393, 2008.

\bibitem[Cao and Krishnamurthy(2019)]{cao2019disagreement}
Tongyi Cao and Akshay Krishnamurthy.
\newblock Disagreement-based combinatorial pure exploration: Sample complexity
  bounds and an efficient algorithm.
\newblock In \emph{Conference on Learning Theory}, pages 558--588. PMLR, 2019.

\bibitem[Cen et~al.(2024)Cen, Mei, Goshvadi, Dai, Yang, Yang, Schuurmans, Chi,
  and Dai]{cen2024value}
Shicong Cen, Jincheng Mei, Katayoon Goshvadi, Hanjun Dai, Tong Yang, Sherry
  Yang, Dale Schuurmans, Yuejie Chi, and Bo~Dai.
\newblock Value-incentivized preference optimization: A unified approach to
  online and offline rlhf, 2024.

\bibitem[Chan(2016)]{chan2016approximation}
Siu~On Chan.
\newblock Approximation resistance from pairwise-independent subgroups.
\newblock \emph{Journal of the ACM (JACM)}, 63\penalty0 (3):\penalty0 1--32,
  2016.

\bibitem[Chang et~al.(2024)Chang, Zhan, Oertell, Brantley, Misra, Lee, and
  Sun]{chang2024dataset}
Jonathan~D Chang, Wenhao Zhan, Owen Oertell, Kiant{\'e} Brantley, Dipendra
  Misra, Jason~D Lee, and Wen Sun.
\newblock Dataset reset policy optimization for rlhf.
\newblock \emph{arXiv preprint arXiv:2404.08495}, 2024.

\bibitem[Chen et~al.(2017)Chen, Gupta, Li, Qiao, and Wang]{chen2017nearly}
Lijie Chen, Anupam Gupta, Jian Li, Mingda Qiao, and Ruosong Wang.
\newblock Nearly optimal sampling algorithms for combinatorial pure
  exploration.
\newblock In \emph{Conference on Learning Theory}, pages 482--534. PMLR, 2017.

\bibitem[Chen et~al.(2024)Chen, Zhang, Luo, Chai, and Liu]{chen2024pad}
Ruizhe Chen, Xiaotian Zhang, Meng Luo, Wenhao Chai, and Zuozhu Liu.
\newblock Pad: Personalized alignment at decoding-time.
\newblock \emph{arXiv:2410.04070}, 2024.

\bibitem[Chen et~al.(2022)Chen, Zhong, Yang, Wang, and Wang]{chen2022human}
Xiaoyu Chen, Han Zhong, Zhuoran Yang, Zhaoran Wang, and Liwei Wang.
\newblock Human-in-the-loop: Provably efficient preference-based reinforcement
  learning with general function approximation.
\newblock In \emph{International Conference on Machine Learning}, pages
  3773--3793. PMLR, 2022.

\bibitem[Dani et~al.(2008)Dani, Hayes, and Kakade]{dani2008stochastic}
Varsha Dani, Thomas~P Hayes, and Sham~M Kakade.
\newblock Stochastic linear optimization under bandit feedback.
\newblock In \emph{Conference on Learning Theory (COLT)}, 2008.

\bibitem[Dann et~al.(2018)Dann, Jiang, Krishnamurthy, Agarwal, Langford, and
  Schapire]{dann2018oracle}
Christoph Dann, Nan Jiang, Akshay Krishnamurthy, Alekh Agarwal, John Langford,
  and Robert~E Schapire.
\newblock On oracle-efficient {PAC} {RL} with rich observations.
\newblock In \emph{Advances in neural information processing systems}, pages
  1422--1432, 2018.

\bibitem[Das et~al.(2024)Das, Chakraborty, Pacchiano, and
  Chowdhury]{das2024provably}
Nirjhar Das, Souradip Chakraborty, Aldo Pacchiano, and Sayak~Ray Chowdhury.
\newblock Provably sample efficient rlhf via active preference optimization.
\newblock \emph{arXiv preprint arXiv:2402.10500}, 2024.

\bibitem[Deep{S}eek-{AI}(2025)]{deepseek2025r1}
Deep{S}eek-{AI}.
\newblock Deep{S}eek-{R}1: Incentivizing {R}easoning {C}apability in {LLM}s via
  {R}einforcement {L}earning.
\newblock 2025.
\newblock URL
  \url{https://github.com/deepseek-ai/DeepSeek-R1/blob/main/DeepSeek_R1.pdf}.

\bibitem[Du et~al.(2020)Du, Kakade, Wang, and Yang]{du2019good}
Simon~S Du, Sham~M Kakade, Ruosong Wang, and Lin~F Yang.
\newblock Is a good representation sufficient for sample efficient
  reinforcement learning?
\newblock In \emph{International Conference on Learning Representations}, 2020.

\bibitem[Du et~al.(2024)Du, Winnicki, Dalal, Mannor, and
  Srikant]{du2024exploration}
Yihan Du, Anna Winnicki, Gal Dalal, Shie Mannor, and R~Srikant.
\newblock Exploration-driven policy optimization in {RLHF}: Theoretical
  insights on efficient data utilization.
\newblock \emph{arXiv preprint arXiv:2402.10342}, 2024.

\bibitem[Fisch et~al.(2024)Fisch, Eisenstein, Zayats, Agarwal, Beirami, Nagpal,
  Shaw, and Berant]{fisch2024robust}
Adam Fisch, Jacob Eisenstein, Vicky Zayats, Alekh Agarwal, Ahmad Beirami,
  Chirag Nagpal, Pete Shaw, and Jonathan Berant.
\newblock Robust preference optimization through reward model distillation.
\newblock \emph{arXiv:2405.19316}, 2024.

\bibitem[Foster and Rakhlin(2023)]{foster2023foundations}
Dylan~J Foster and Alexander Rakhlin.
\newblock Foundations of reinforcement learning and interactive decision
  making.
\newblock \emph{arXiv:2312.16730}, 2023.

\bibitem[Foster et~al.(2021)Foster, Kakade, Qian, and
  Rakhlin]{foster2021statistical}
Dylan~J Foster, Sham~M Kakade, Jian Qian, and Alexander Rakhlin.
\newblock The statistical complexity of interactive decision making.
\newblock \emph{arXiv:2112.13487}, 2021.

\bibitem[Gao et~al.(2024{\natexlab{a}})Gao, Chang, Zhan, Oertell, Swamy,
  Brantley, Joachims, Bagnell, Lee, and Sun]{gao2024rebel}
Zhaolin Gao, Jonathan~D Chang, Wenhao Zhan, Owen Oertell, Gokul Swamy,
  Kiant{\'e} Brantley, Thorsten Joachims, J~Andrew Bagnell, Jason~D Lee, and
  Wen Sun.
\newblock {REBEL}: Reinforcement learning via regressing relative rewards.
\newblock \emph{arXiv:2404.16767}, 2024{\natexlab{a}}.

\bibitem[Gao et~al.(2024{\natexlab{b}})Gao, Zhan, Chang, Swamy, Brantley, Lee,
  and Sun]{gao2024regressing}
Zhaolin Gao, Wenhao Zhan, Jonathan~D Chang, Gokul Swamy, Kiant{\'e} Brantley,
  Jason~D Lee, and Wen Sun.
\newblock Regressing the relative future: Efficient policy optimization for
  multi-turn rlhf.
\newblock \emph{arXiv preprint arXiv:2410.04612}, 2024{\natexlab{b}}.

\bibitem[Golowich et~al.(2024)Golowich, Moitra, and
  Rohatgi]{golowich2024exploration}
Noah Golowich, Ankur Moitra, and Dhruv Rohatgi.
\newblock Exploration is harder than prediction: Cryptographically separating
  reinforcement learning from supervised learning.
\newblock \emph{arXiv preprint arXiv:2404.03774}, 2024.

\bibitem[Google(2023)]{anil2023palm}
Google.
\newblock Palm 2 technical report.
\newblock \emph{arXiv:2305.10403}, 2023.

\bibitem[Guo et~al.(2024)Guo, Zhang, Liu, Liu, Khalman, Llinares, Rame,
  Mesnard, Zhao, Piot, Ferret, and Blondel]{guo2024direct}
Shangmin Guo, Biao Zhang, Tianlin Liu, Tianqi Liu, Misha Khalman, Felipe
  Llinares, Alexandre Rame, Thomas Mesnard, Yao Zhao, Bilal Piot, Johan Ferret,
  and Mathieu Blondel.
\newblock Direct language model alignment from online {AI} feedback.
\newblock \emph{arXiv:2402.04792}, 2024.

\bibitem[Hao et~al.(2023)Hao, Gu, Ma, Hong, Wang, Wang, and
  Hu]{hao2023reasoning}
Shibo Hao, Yi~Gu, Haodi Ma, Joshua Hong, Zhen Wang, Daisy Wang, and Zhiting Hu.
\newblock Reasoning with language model is planning with world model.
\newblock In \emph{Proceedings of the 2023 Conference on Empirical Methods in
  Natural Language Processing}, pages 8154--8173, 2023.

\bibitem[Hazan and Karnin(2016)]{hazan2016volumetric}
Elad Hazan and Zohar Karnin.
\newblock Volumetric spanners: An efficient exploration basis for learning.
\newblock \emph{The Journal of Machine Learning Research}, 17\penalty0
  (1):\penalty0 4062--4095, 2016.

\bibitem[Huang et~al.(2024{\natexlab{a}})Huang, Block, Foster, Rohatgi, Zhang,
  Simchowitz, Ash, and Krishnamurthy]{huang2024self}
Audrey Huang, Adam Block, Dylan~J Foster, Dhruv Rohatgi, Cyril Zhang, Max
  Simchowitz, Jordan~T Ash, and Akshay Krishnamurthy.
\newblock Self-improvement in language models: The sharpening mechanism.
\newblock \emph{arXiv preprint arXiv:2412.01951}, 2024{\natexlab{a}}.

\bibitem[Huang et~al.(2024{\natexlab{b}})Huang, Zhan, Xie, Lee, Sun,
  Krishnamurthy, and Foster]{huang2024correcting}
Audrey Huang, Wenhao Zhan, Tengyang Xie, Jason~D Lee, Wen Sun, Akshay
  Krishnamurthy, and Dylan~J Foster.
\newblock Correcting the mythos of {KL}-regularization: Direct alignment
  without overoptimization via {C}hi-squared {P}reference {O}ptimization.
\newblock \emph{arXiv:2407.13399}, 2024{\natexlab{b}}.

\bibitem[Ji et~al.(2024)Ji, Kulkarni, Wang, and Xie]{ji2024selfplay}
Xiang Ji, Sanjeev Kulkarni, Mengdi Wang, and Tengyang Xie.
\newblock Self-play with adversarial critic: Provable and scalable offline
  alignment for language models.
\newblock \emph{arXiv:2406.04274}, 2024.

\bibitem[Jiang et~al.(2017)Jiang, Krishnamurthy, Agarwal, Langford, and
  Schapire]{jiang2017contextual}
Nan Jiang, Akshay Krishnamurthy, Alekh Agarwal, John Langford, and Robert~E
  Schapire.
\newblock Contextual decision processes with low {Bellman} rank are
  {PAC}-learnable.
\newblock In \emph{International Conference on Machine Learning}, 2017.

\bibitem[Jin et~al.(2021)Jin, Liu, and Miryoosefi]{jin2021bellman}
Chi Jin, Qinghua Liu, and Sobhan Miryoosefi.
\newblock Bellman {E}luder dimension: New rich classes of {RL} problems, and
  sample-efficient algorithms.
\newblock \emph{Advances in Neural Information Processing Systems}, 2021.

\bibitem[Jinnai et~al.(2024)Jinnai, Morimura, Ariu, and
  Abe]{jinnai2024regularized}
Yuu Jinnai, Tetsuro Morimura, Kaito Ariu, and Kenshi Abe.
\newblock Regularized best-of-n sampling to mitigate reward hacking for
  language model alignment.
\newblock \emph{arXiv:2404.01054}, 2024.

\bibitem[Kane et~al.(2022)Kane, Liu, Lovett, and
  Mahajan]{kane2022computational}
Daniel Kane, Sihan Liu, Shachar Lovett, and Gaurav Mahajan.
\newblock Computational-statistical gap in reinforcement learning.
\newblock In \emph{Conference on Learning Theory}, pages 1282--1302. PMLR,
  2022.

\bibitem[Katz-Samuels et~al.(2020)Katz-Samuels, Jain, Jamieson,
  et~al.]{katz2020empirical}
Julian Katz-Samuels, Lalit Jain, Kevin~G Jamieson, et~al.
\newblock An empirical process approach to the union bound: Practical
  algorithms for combinatorial and linear bandits.
\newblock \emph{Advances in Neural Information Processing Systems}, 33, 2020.

\bibitem[Kazemnejad et~al.(2024)Kazemnejad, Aghajohari, Portelance, Sordoni,
  Reddy, Courville, and Roux]{kazemnejad2024vineppo}
Amirhossein Kazemnejad, Milad Aghajohari, Eva Portelance, Alessandro Sordoni,
  Siva Reddy, Aaron Courville, and Nicolas~Le Roux.
\newblock Vineppo: Unlocking rl potential for llm reasoning through refined
  credit assignment.
\newblock \emph{arXiv preprint arXiv:2410.01679}, 2024.

\bibitem[Kearns(1998)]{kearns1998efficient}
Michael Kearns.
\newblock Efficient noise-tolerant learning from statistical queries.
\newblock \emph{Journal of the ACM}, 1998.

\bibitem[Khaki et~al.(2024)Khaki, Li, Ma, Yang, and Ramachandra]{khaki2024rs}
Saeed Khaki, JinJin Li, Lan Ma, Liu Yang, and Prathap Ramachandra.
\newblock Rs-dpo: A hybrid rejection sampling and direct preference
  optimization method for alignment of large language models.
\newblock \emph{arXiv preprint arXiv:2402.10038}, 2024.

\bibitem[Khanov et~al.(2024)Khanov, Burapacheep, and Li]{khanov2024args}
Maxim Khanov, Jirayu Burapacheep, and Yixuan Li.
\newblock Args: Alignment as reward-guided search.
\newblock \emph{arXiv:2402.01694}, 2024.

\bibitem[Kumar et~al.(2024)Kumar, Zhuang, Agarwal, Su, Co-Reyes, Singh, Baumli,
  Iqbal, Bishop, Roelofs, et~al.]{kumar2024training}
Aviral Kumar, Vincent Zhuang, Rishabh Agarwal, Yi~Su, John~D Co-Reyes, Avi
  Singh, Kate Baumli, Shariq Iqbal, Colton Bishop, Rebecca Roelofs, et~al.
\newblock Training language models to self-correct via reinforcement learning.
\newblock \emph{arXiv preprint arXiv:2409.12917}, 2024.

\bibitem[Lai and Robbins(1985)]{lai1985asymptotically}
Tze~Leung Lai and Herbert Robbins.
\newblock Asymptotically efficient adaptive allocation rules.
\newblock \emph{Advances in Applied Mathematics}, 6\penalty0 (1):\penalty0
  4--22, 1985.

\bibitem[Lattimore and Szepesv{\'a}ri(2020)]{lattimore2020bandit}
Tor Lattimore and Csaba Szepesv{\'a}ri.
\newblock \emph{Bandit algorithms}.
\newblock Cambridge University Press, 2020.

\bibitem[Lattimore et~al.(2020)Lattimore, Szepesvari, and
  Weisz]{lattimore2020learning}
Tor Lattimore, Csaba Szepesvari, and Gellert Weisz.
\newblock Learning with good feature representations in bandits and in rl with
  a generative model.
\newblock In \emph{International Conference on Machine Learning}, pages
  5662--5670. PMLR, 2020.

\bibitem[Li et~al.(2021)Li, Chen, Chi, Gu, and Wei]{li2021sample}
Gen Li, Yuxin Chen, Yuejie Chi, Yuantao Gu, and Yuting Wei.
\newblock Sample-efficient reinforcement learning is feasible for linearly
  realizable mdps with limited revisiting.
\newblock \emph{Advances in Neural Information Processing Systems}, 2021.

\bibitem[Li et~al.(2023)Li, Yang, and Wang]{li2023reinforcement}
Zihao Li, Zhuoran Yang, and Mengdi Wang.
\newblock Reinforcement learning with human feedback: Learning dynamic choices
  via pessimism.
\newblock \emph{arXiv:2305.18438}, 2023.

\bibitem[Lightman et~al.(2023)Lightman, Kosaraju, Burda, Edwards, Baker, Lee,
  Leike, Schulman, Sutskever, and Cobbe]{lightman2023lets}
Hunter Lightman, Vineet Kosaraju, Yura Burda, Harri Edwards, Bowen Baker, Teddy
  Lee, Jan Leike, John Schulman, Ilya Sutskever, and Karl Cobbe.
\newblock Let's verify step by step.
\newblock \emph{arXiv preprint arXiv:2305.20050}, 2023.

\bibitem[Liu et~al.(2024{\natexlab{a}})Liu, Guo, Bianco, Calandriello, Berthet,
  Llinares, Hoffmann, Dixon, Valko, and Blondel]{liu2024decoding}
Tianlin Liu, Shangmin Guo, Leonardo Bianco, Daniele Calandriello, Quentin
  Berthet, Felipe Llinares, Jessica Hoffmann, Lucas Dixon, Michal Valko, and
  Mathieu Blondel.
\newblock Decoding-time realignment of language models.
\newblock \emph{arXiv:2402.02992}, 2024{\natexlab{a}}.

\bibitem[Liu et~al.(2023)Liu, Zhao, Joshi, Khalman, Saleh, Liu, and
  Liu]{liu2023statistical}
Tianqi Liu, Yao Zhao, Rishabh Joshi, Misha Khalman, Mohammad Saleh, Peter~J
  Liu, and Jialu Liu.
\newblock Statistical rejection sampling improves preference optimization.
\newblock \emph{arXiv preprint arXiv:2309.06657}, 2023.

\bibitem[Liu et~al.(2024{\natexlab{b}})Liu, Lu, Zhang, Liu, Guo, Yang,
  Blanchet, and Wang]{liu2024provably}
Zhihan Liu, Miao Lu, Shenao Zhang, Boyi Liu, Hongyi Guo, Yingxiang Yang, Jose
  Blanchet, and Zhaoran Wang.
\newblock Provably mitigating overoptimization in {RLHF}: Your {SFT} loss is
  implicitly an adversarial regularizer.
\newblock \emph{arXiv:2405.16436}, 2024{\natexlab{b}}.

\bibitem[Lov{\'a}sz and Vempala(2006)]{lovasz2006fast}
L{\'a}szl{\'o} Lov{\'a}sz and Santosh Vempala.
\newblock Fast algorithms for logconcave functions: Sampling, rounding,
  integration and optimization.
\newblock In \emph{Symposium on Foundations of Computer Science}, 2006.

\bibitem[Malladi et~al.(2023)Malladi, Wettig, Yu, Chen, and
  Arora]{malladi2023kernel}
Sadhika Malladi, Alexander Wettig, Dingli Yu, Danqi Chen, and Sanjeev Arora.
\newblock A kernel-based view of language model fine-tuning.
\newblock In \emph{International Conference on Machine Learning}, pages
  23610--23641. PMLR, 2023.

\bibitem[Mhammedi(2024)]{mhammedi2024sampleoracleefficientreinforcement}
Zakaria Mhammedi.
\newblock Sample and oracle efficient reinforcement learning for mdps with
  linearly-realizable value functions, 2024.
\newblock URL \url{https://arxiv.org/abs/2409.04840}.

\bibitem[Mhammedi et~al.(2024)Mhammedi, Foster, and Rakhlin]{mhammedi2024power}
Zakaria Mhammedi, Dylan~J Foster, and Alexander Rakhlin.
\newblock The power of resets in online reinforcement learning.
\newblock \emph{arXiv preprint arXiv:2404.15417}, 2024.

\bibitem[Nemirovski et~al.(1983)Nemirovski, Yudin, and
  Dawson]{nemirovski1983problem}
Arkadii Nemirovski, David~Borisovich Yudin, and Edgar~Ronald Dawson.
\newblock \emph{Problem complexity and method efficiency in optimization}.
\newblock Wiley, 1983.

\bibitem[Novoseller et~al.(2020)Novoseller, Wei, Sui, Yue, and
  Burdick]{novoseller2020dueling}
Ellen Novoseller, Yibing Wei, Yanan Sui, Yisong Yue, and Joel Burdick.
\newblock Dueling posterior sampling for preference-based reinforcement
  learning.
\newblock In \emph{Conference on Uncertainty in Artificial Intelligence}, pages
  1029--1038. PMLR, 2020.

\bibitem[OpenAI(2023)]{achiam2023gpt}
OpenAI.
\newblock {GPT}-4 technical report.
\newblock \emph{arXiv:2303.08774}, 2023.

\bibitem[OpenAI(2024)]{openai2024o1}
OpenAI.
\newblock Introducing openai o1.
\newblock \emph{Blog}, 2024.
\newblock URL \url{https://openai.com/o1/}.

\bibitem[Ouyang et~al.(2022)Ouyang, Wu, Jiang, Almeida, Wainwright, Mishkin,
  Zhang, Agarwal, Slama, Ray, Schulman, Hilton, Kelton, Miller, Simens, Askell,
  Welinder, Christiano, Leike, and Lowe]{ouyang2022training}
Long Ouyang, Jeffrey Wu, Xu~Jiang, Diogo Almeida, Carroll Wainwright, Pamela
  Mishkin, Chong Zhang, Sandhini Agarwal, Katarina Slama, Alex Ray, John
  Schulman, Jacob Hilton, Fraser Kelton, Luke Miller, Maddie Simens, Amanda
  Askell, Peter Welinder, Paul Christiano, Jan Leike, and Ryan Lowe.
\newblock Training language models to follow instructions with human feedback.
\newblock \emph{Advances in Neural Information Processing Systems}, 2022.

\bibitem[Pacchiano et~al.(2021)Pacchiano, Saha, and Lee]{pacchiano2021dueling}
Aldo Pacchiano, Aadirupa Saha, and Jonathan Lee.
\newblock Dueling {RL}: reinforcement learning with trajectory preferences.
\newblock \emph{arXiv preprint arXiv:2111.04850}, 2021.

\bibitem[Polyanskiy and Wu(2025)]{Polyanskiy_Wu_2025}
Yury Polyanskiy and Yihong Wu.
\newblock \emph{Information Theory: From Coding to Learning}.
\newblock Cambridge University Press, 2025.

\bibitem[Qu et~al.(2024)Qu, Zhang, Garg, and Kumar]{qu2024recursive}
Yuxiao Qu, Tianjun Zhang, Naman Garg, and Aviral Kumar.
\newblock Recursive introspection: Teaching language model agents how to
  self-improve.
\newblock \emph{arXiv preprint arXiv:2407.18219}, 2024.

\bibitem[Radford et~al.(2019)Radford, Wu, Child, Luan, Amodei, Sutskever,
  et~al.]{radford2019language}
Alec Radford, Jeffrey Wu, Rewon Child, David Luan, Dario Amodei, Ilya
  Sutskever, et~al.
\newblock Language models are unsupervised multitask learners.
\newblock \emph{OpenAI blog}, 1\penalty0 (8):\penalty0 9, 2019.

\bibitem[Rafailov et~al.(2023)Rafailov, Sharma, Mitchell, Manning, Ermon, and
  Finn]{rafailov2024direct}
Rafael Rafailov, Archit Sharma, Eric Mitchell, Christopher~D Manning, Stefano
  Ermon, and Chelsea Finn.
\newblock Direct preference optimization: Your language model is secretly a
  reward model.
\newblock \emph{Advances in Neural Information Processing Systems}, 2023.

\bibitem[Rafailov et~al.(2024)Rafailov, Hejna, Park, and Finn]{rafailov2024r}
Rafael Rafailov, Joey Hejna, Ryan Park, and Chelsea Finn.
\newblock {From $r$ to $Q^{\star}$}: Your language model is secretly a
  {Q}-function.
\newblock \emph{arXiv:2404.12358}, 2024.

\bibitem[Rohatgi et~al.(2025)Rohatgi, Block, Huang, Krishnamurthy, and
  Foster]{rohatgi2025computational}
Dhruv Rohatgi, Adam Block, Audrey Huang, Akshay Krishnamurthy, and Dylan~J.
  Foster.
\newblock Computational-statistical tradeoffs at the next-token prediction
  barrier: Autoregressive and imitation learning under misspecification.
\newblock \emph{arXiv preprint arXiv:2502.12465}, 2025.

\bibitem[Schulman et~al.(2017)Schulman, Wolski, Dhariwal, Radford, and
  Klimov]{schulman2017proximal}
John Schulman, Filip Wolski, Prafulla Dhariwal, Alec Radford, and Oleg Klimov.
\newblock Proximal policy optimization algorithms.
\newblock \emph{arXiv:1707.06347}, 2017.

\bibitem[Setlur et~al.(2024{\natexlab{a}})Setlur, Garg, Geng, Garg, Smith, and
  Kumar]{setlur2024rl}
Amrith Setlur, Saurabh Garg, Xinyang Geng, Naman Garg, Virginia Smith, and
  Aviral Kumar.
\newblock Rl on incorrect synthetic data scales the efficiency of llm math
  reasoning by eight-fold.
\newblock \emph{arXiv preprint arXiv:2406.14532}, 2024{\natexlab{a}}.

\bibitem[Setlur et~al.(2024{\natexlab{b}})Setlur, Nagpal, Fisch, Geng,
  Eisenstein, Agarwal, Agarwal, Berant, and Kumar]{setlur2024rewarding}
Amrith Setlur, Chirag Nagpal, Adam Fisch, Xinyang Geng, Jacob Eisenstein,
  Rishabh Agarwal, Alekh Agarwal, Jonathan Berant, and Aviral Kumar.
\newblock Rewarding progress: Scaling automated process verifiers for llm
  reasoning.
\newblock \emph{arXiv preprint arXiv:2410.08146}, 2024{\natexlab{b}}.

\bibitem[Shi et~al.(2024{\natexlab{a}})Shi, Chen, Hu, Liu, Smith, Hajishirzi,
  and Du]{shi2024decoding}
Ruizhe Shi, Yifang Chen, Yushi Hu, ALisa Liu, Noah Smith, Hannaneh Hajishirzi,
  and Simon Du.
\newblock Decoding-time language model alignment with multiple objectives.
\newblock \emph{arXiv:2406.18853}, 2024{\natexlab{a}}.

\bibitem[Shi et~al.(2024{\natexlab{b}})Shi, Zhou, and Du]{shi2024crucial}
Ruizhe Shi, Runlong Zhou, and Simon~S Du.
\newblock The crucial role of samplers in online direct preference
  optimization.
\newblock \emph{arXiv preprint arXiv:2409.19605}, 2024{\natexlab{b}}.

\bibitem[Snell et~al.(2024)Snell, Lee, Xu, and Kumar]{snell2024scaling}
Charlie Snell, Jaehoon Lee, Kelvin Xu, and Aviral Kumar.
\newblock Scaling {LLM} test-time compute optimally can be more effective than
  scaling model parameters.
\newblock \emph{arXiv:2408.03314}, 2024.

\bibitem[Tiapkin et~al.(2024)Tiapkin, Belomestny, Calandriello, Moulines,
  Naumov, Perrault, Valko, and Menard]{tiapkin2024regularized}
Daniil Tiapkin, Denis Belomestny, Daniele Calandriello, Eric Moulines, Alexey
  Naumov, Pierre Perrault, Michal Valko, and Pierre Menard.
\newblock Demonstration-regularized rl.
\newblock In \emph{The Twelfth International Conference on Learning
  Representations}, 2024.

\bibitem[Tkocz(2012)]{tkocz2012upper}
Tomasz Tkocz.
\newblock An upper bound for spherical caps.
\newblock \emph{The American Mathematical Monthly}, 119\penalty0 (7):\penalty0
  606--607, 2012.

\bibitem[Touvron et~al.(2023)Touvron, Martin, Stone, Albert, Almahairi, Babaei,
  Bashlykov, Batra, Bhargava, Bhosale, Bikel, Blecher, Ferrer, Chen, Cucurull,
  Esiobu, Fernandes, Fu, Fu, Fuller, Gao, Goswami, Goyal, Hartshorn, Hosseini,
  Hou, Inan, Kardas, Kerkez, Khabsa, Kloumann, Korenev, Koura, Lachaux, Lavril,
  Lee, Liskovich, Lu, Mao, Martinet, Mihaylov, Mishra, Molybog, Nie, Poulton,
  Reizenstein, Rungta, Saladi, Schelten, Silva, Smith, Subramanian, Tan, Tang,
  Taylor, Williams, Kuan, Xu, Yan, Zarov, Zhang, Fan, Kambadur, Narang,
  Rodriguez, Stojnic, Edunov, and Scialom]{touvron2023llama}
Hugo Touvron, Louis Martin, Kevin Stone, Peter Albert, Amjad Almahairi, Yasmine
  Babaei, Nikolay Bashlykov, Soumya Batra, Prajjwal Bhargava, Shruti Bhosale,
  Dan Bikel, Lukas Blecher, Cristian~Canton Ferrer, Moya Chen, Guillem
  Cucurull, David Esiobu, Jude Fernandes, Jeremy Fu, Wenyin Fu, Brian Fuller,
  Cynthia Gao, Vedanuj Goswami, Naman Goyal, Anthony Hartshorn, Saghar
  Hosseini, Rui Hou, Hakan Inan, Marcin Kardas, Viktor Kerkez, Madian Khabsa,
  Isabel Kloumann, Artem Korenev, Punit~Singh Koura, Marie-Anne Lachaux,
  Thibaut Lavril, Jenya Lee, Diana Liskovich, Yinghai Lu, Yuning Mao, Xavier
  Martinet, Todor Mihaylov, Pushkar Mishra, Igor Molybog, Yixin Nie, Andrew
  Poulton, Jeremy Reizenstein, Rashi Rungta, Kalyan Saladi, Alan Schelten, Ruan
  Silva, Eric~Michael Smith, Ranjan Subramanian, Xiaoqing~Ellen Tan, Binh Tang,
  Ross Taylor, Adina Williams, Jian~Xiang Kuan, Puxin Xu, Zheng Yan, Iliyan
  Zarov, Yuchen Zhang, Angela Fan, Melanie Kambadur, Sharan Narang, Aurelien
  Rodriguez, Robert Stojnic, Sergey Edunov, and Thomas Scialom.
\newblock Llama 2: Open foundation and fine-tuned chat models.
\newblock \emph{arXiv:2307.09288}, 2023.

\bibitem[Tran et~al.(2023)Tran, Glaze, and Hancock]{tran2023iterative}
Hoang Tran, Chris Glaze, and Braden Hancock.
\newblock Iterative dpo alignment.
\newblock \emph{Technical report}, 2023.
\newblock URL
  \url{https://huggingface.co/snorkelai/Snorkel-Mistral-PairRM-DPO}.

\bibitem[Wang et~al.(2024{\natexlab{a}})Wang, Jiang, Yang, Liu, and
  Chen]{wang2024beyond}
Chaoqi Wang, Yibo Jiang, Chenghao Yang, Han Liu, and Yuxin Chen.
\newblock Beyond reverse kl: Generalizing direct preference optimization with
  diverse divergence constraints.
\newblock In \emph{The Twelfth International Conference on Learning
  Representations}, 2024{\natexlab{a}}.

\bibitem[Wang et~al.(2024{\natexlab{b}})Wang, Lin, Xiong, Yang, Diao, Qiu,
  Zhao, and Zhang]{wang2024arithmetic}
Haoxiang Wang, Yong Lin, Wei Xiong, Rui Yang, Shizhe Diao, Shuang Qiu, Han
  Zhao, and Tong Zhang.
\newblock Arithmetic control of llms for diverse user preferences: Directional
  preference alignment with multi-objective rewards.
\newblock \emph{arXiv preprint arXiv:2402.18571}, 2024{\natexlab{b}}.

\bibitem[Wang et~al.(2024{\natexlab{c}})Wang, Xiong, Xie, Zhao, and
  Zhang]{wang2024interpretable}
Haoxiang Wang, Wei Xiong, Tengyang Xie, Han Zhao, and Tong Zhang.
\newblock Interpretable preferences via multi-objective reward modeling and
  mixture-of-experts.
\newblock \emph{arXiv preprint arXiv:2406.12845}, 2024{\natexlab{c}}.

\bibitem[Wang et~al.(2021)Wang, Wang, and Kakade]{wang2021exponential}
Yuanhao Wang, Ruosong Wang, and Sham~M Kakade.
\newblock An exponential lower bound for linearly-realizable {MDPs} with
  constant suboptimality gap.
\newblock \emph{Neural Information Processing Systems (NeurIPS)}, 2021.

\bibitem[Wang et~al.(2023{\natexlab{a}})Wang, Liu, and Jin]{wang2023rlhf}
Yuanhao Wang, Qinghua Liu, and Chi Jin.
\newblock Is {RLHF} more difficult than standard {RL}?
\newblock \emph{arXiv preprint arXiv:2306.14111}, 2023{\natexlab{a}}.

\bibitem[Wang et~al.(2023{\natexlab{b}})Wang, Dong, Zeng, Adams, Sreedhar,
  Egert, Delalleau, Scowcroft, Kant, Swope, et~al.]{wang2023helpsteer}
Zhilin Wang, Yi~Dong, Jiaqi Zeng, Virginia Adams, Makesh~Narsimhan Sreedhar,
  Daniel Egert, Olivier Delalleau, Jane~Polak Scowcroft, Neel Kant, Aidan
  Swope, et~al.
\newblock Helpsteer: Multi-attribute helpfulness dataset for steerlm.
\newblock \emph{arXiv preprint arXiv:2311.09528}, 2023{\natexlab{b}}.

\bibitem[Wang et~al.(2024{\natexlab{d}})Wang, Dong, Delalleau, Zeng, Shen,
  Egert, Zhang, Sreedhar, and Kuchaiev]{wang2024helpsteer2}
Zhilin Wang, Yi~Dong, Olivier Delalleau, Jiaqi Zeng, Gerald Shen, Daniel Egert,
  Jimmy~J Zhang, Makesh~Narsimhan Sreedhar, and Oleksii Kuchaiev.
\newblock Helpsteer2: Open-source dataset for training top-performing reward
  models.
\newblock \emph{arXiv preprint arXiv:2406.08673}, 2024{\natexlab{d}}.

\bibitem[Weisz et~al.(2021)Weisz, Amortila, Janzer, Abbasi-Yadkori, Jiang, and
  Szepesv{\'a}ri]{weisz2021query}
Gellert Weisz, Philip Amortila, Barnab{\'a}s Janzer, Yasin Abbasi-Yadkori, Nan
  Jiang, and Csaba Szepesv{\'a}ri.
\newblock On query-efficient planning in mdps under linear realizability of the
  optimal state-value function.
\newblock In \emph{Conference on Learning Theory}, pages 4355--4385. PMLR,
  2021.

\bibitem[Weisz et~al.(2022)Weisz, Gy{\"o}rgy, Kozuno, and
  Szepesv{\'a}ri]{weisz2022confident}
Gell{\'e}rt Weisz, Andr{\'a}s Gy{\"o}rgy, Tadashi Kozuno, and Csaba
  Szepesv{\'a}ri.
\newblock Confident approximate policy iteration for efficient local planning
  in $q^{\pi}$-realizable mdps.
\newblock \emph{Advances in Neural Information Processing Systems},
  35:\penalty0 25547--25559, 2022.

\bibitem[Wu and Sun(2023)]{wu2023making}
Runzhe Wu and Wen Sun.
\newblock Making {RL} with preference-based feedback efficient via
  randomization.
\newblock \emph{arXiv preprint arXiv:2310.14554}, 2023.

\bibitem[Wu et~al.(2024)Wu, Sun, Li, Welleck, and Yang]{wu2024empirical}
Yangzhen Wu, Zhiqing Sun, Shanda Li, Sean Welleck, and Yiming Yang.
\newblock An empirical analysis of compute-optimal inference for
  problem-solving with language models.
\newblock \emph{arXiv:2408.00724}, 2024.

\bibitem[Xie et~al.(2024)Xie, Foster, Krishnamurthy, Rosset, Awadallah, and
  Rakhlin]{xie2024exploratory}
Tengyang Xie, Dylan~J Foster, Akshay Krishnamurthy, Corby Rosset, Ahmed
  Awadallah, and Alexander Rakhlin.
\newblock Exploratory preference optimization: Harnessing implicit
  {Q}*-approximation for sample-efficient {RLHF}.
\newblock \emph{arXiv:2405.21046}, 2024.

\bibitem[Xiong et~al.(2024{\natexlab{a}})Xiong, Dong, Ye, Zhong, Jiang, and
  Zhang]{xiong2024iterative}
Wei Xiong, Hanze Dong, Chenlu Ye, Han Zhong, Nan Jiang, and Tong Zhang.
\newblock Iterative preference learning from human feedback: Bridging theory
  and practice for {RLHF} under {KL}-constraint.
\newblock \emph{International Conference on Machine Learning (ICML)},
  2024{\natexlab{a}}.

\bibitem[Xiong et~al.(2024{\natexlab{b}})Xiong, Shi, Shen, Rosenberg, Qin,
  Calandriello, Khalman, Joshi, Piot, Saleh, et~al.]{xiong2024building}
Wei Xiong, Chengshuai Shi, Jiaming Shen, Aviv Rosenberg, Zhen Qin, Daniele
  Calandriello, Misha Khalman, Rishabh Joshi, Bilal Piot, Mohammad Saleh,
  et~al.
\newblock Building math agents with multi-turn iterative preference learning.
\newblock \emph{arXiv preprint arXiv:2409.02392}, 2024{\natexlab{b}}.

\bibitem[Xu et~al.(2023)Xu, Lee, Sukhbaatar, and Weston]{xu2023some}
Jing Xu, Andrew Lee, Sainbayar Sukhbaatar, and Jason Weston.
\newblock Some things are more cringe than others: Preference optimization with
  the pairwise cringe loss.
\newblock \emph{arXiv preprint arXiv:2312.16682}, 2023.

\bibitem[Xu et~al.(2020)Xu, Wang, Yang, Singh, and Dubrawski]{xu2020preference}
Yichong Xu, Ruosong Wang, Lin Yang, Aarti Singh, and Artur Dubrawski.
\newblock Preference-based reinforcement learning with finite-time guarantees.
\newblock \emph{Advances in Neural Information Processing Systems},
  33:\penalty0 18784--18794, 2020.

\bibitem[Yan et~al.(2024)Yan, Song, Feng, Yang, Zhang, Ammar, and
  Wang]{yan2024efficient}
Xue Yan, Yan Song, Xidong Feng, Mengyue Yang, Haifeng Zhang, Haitham~Bou Ammar,
  and Jun Wang.
\newblock Efficient reinforcement learning with large language model priors.
\newblock \emph{arXiv preprint arXiv:2410.07927}, 2024.

\bibitem[Yang and Barron(1998)]{yang1998asymptotic}
Yuhong Yang and Andrew~R Barron.
\newblock An asymptotic property of model selection criteria.
\newblock \emph{IEEE Transactions on Information Theory}, 44\penalty0
  (1):\penalty0 95--116, 1998.

\bibitem[Yao et~al.(2024)Yao, Yu, Zhao, Shafran, Griffiths, Cao, and
  Narasimhan]{yao2024tree}
Shunyu Yao, Dian Yu, Jeffrey Zhao, Izhak Shafran, Tom Griffiths, Yuan Cao, and
  Karthik Narasimhan.
\newblock Tree of thoughts: Deliberate problem solving with large language
  models.
\newblock \emph{Advances in Neural Information Processing Systems}, 36, 2024.

\bibitem[Ye et~al.(2024)Ye, Xiong, Zhang, Jiang, and Zhang]{ye2024theoretical}
Chenlu Ye, Wei Xiong, Yuheng Zhang, Nan Jiang, and Tong Zhang.
\newblock A theoretical analysis of {Nash} learning from human feedback under
  general {KL}-regularized preference.
\newblock \emph{Neural Information Processing Systems (NeurIPS)}, 2024.

\bibitem[Yin et~al.(2022)Yin, Hao, Abbasi-Yadkori, Lazi{\'c}, and
  Szepesv{\'a}ri]{yin2022efficient}
Dong Yin, Botao Hao, Yasin Abbasi-Yadkori, Nevena Lazi{\'c}, and Csaba
  Szepesv{\'a}ri.
\newblock Efficient local planning with linear function approximation.
\newblock In \emph{International Conference on Algorithmic Learning Theory},
  2022.

\bibitem[Yin et~al.(2023)Yin, Thiagarajan, Lazic, Rajaraman, Hao, and
  Szepesvari]{yin2023sample}
Dong Yin, Sridhar Thiagarajan, Nevena Lazic, Nived Rajaraman, Botao Hao, and
  Csaba Szepesvari.
\newblock Sample efficient deep reinforcement learning via local planning.
\newblock \emph{arXiv preprint arXiv:2301.12579}, 2023.

\bibitem[Zhan et~al.(2023)Zhan, Uehara, Sun, and Lee]{zhan2023query}
Wenhao Zhan, Masatoshi Uehara, Wen Sun, and Jason~D Lee.
\newblock Provable reward-agnostic preference-based reinforcement learning.
\newblock \emph{arXiv preprint arXiv:2305.18505}, 2023.

\bibitem[Zhang et~al.(2024)Zhang, Yu, Sharma, Yang, Wang, Hassan, and
  Wang]{zhang2024self}
Shenao Zhang, Donghan Yu, Hiteshi Sharma, Ziyi Yang, Shuohang Wang, Hany
  Hassan, and Zhaoran Wang.
\newblock Self-exploring language models: Active preference elicitation for
  online alignment, 2024.

\bibitem[Zhao et~al.(2024)Zhao, Ye, Gu, and Zhang]{zhao2024sharp}
Heyang Zhao, Chenlu Ye, Quanquan Gu, and Tong Zhang.
\newblock Sharp analysis for kl-regularized contextual bandits and rlhf.
\newblock \emph{arXiv preprint arXiv:2411.04625}, 2024.

\bibitem[Zhao et~al.(2025)Zhao, Ye, Xiong, Gu, and Zhang]{zhao2025logarithmic}
Heyang Zhao, Chenlu Ye, Wei Xiong, Quanquan Gu, and Tong Zhang.
\newblock Logarithmic regret for online kl-regularized reinforcement learning.
\newblock \emph{arXiv preprint arXiv:2502.07460}, 2025.

\bibitem[Zhou et~al.(2024)Zhou, Zanette, Pan, Levine, and
  Kumar]{zhou2024archer}
Yifei Zhou, Andrea Zanette, Jiayi Pan, Sergey Levine, and Aviral Kumar.
\newblock Archer: Training language model agents via hierarchical multi-turn
  rl.
\newblock In \emph{International Conference on Machine Learning}, pages
  62178--62209. PMLR, 2024.

\bibitem[Zhu et~al.(2023)Zhu, Jordan, and Jiao]{zhu2023principled}
Banghua Zhu, Michael Jordan, and Jiantao Jiao.
\newblock Principled reinforcement learning with human feedback from pairwise
  or k-wise comparisons.
\newblock In \emph{International Conference on Machine Learning}, 2023.

\bibitem[Zhu et~al.(2022)Zhu, Foster, Langford, and Mineiro]{zhu2022contextual}
Yinglun Zhu, Dylan~J Foster, John Langford, and Paul Mineiro.
\newblock Contextual bandits with large action spaces: Made practical.
\newblock In \emph{International Conference on Machine Learning}, pages
  27428--27453. PMLR, 2022.

\end{thebibliography}

\clearpage

\appendix

\renewcommand{\contentsname}{Contents of Appendix}
\addtocontents{toc}{\protect\setcounter{tocdepth}{2}}
{
  \hypersetup{hidelinks}
  \tableofcontents
}

\clearpage

\part{Additional Results and Discussion}

\section{Additional Related Work}
\label{app:related}

In this section we discuss related work not already covered in detail.

\paragraph{Theoretical algorithms for online alignment}
There is a large body of work on theoretical algorithms for
exploration in online alignment (as well as the more abstract problem
of preference-based contextual bandits and RL), but most prior
algorithms are not computationally efficient when the response space
$\cY$ is large
\citep{xu2020preference,novoseller2020dueling,pacchiano2021dueling,wu2023making,zhan2023query,du2024exploration,das2024provably,chen2022human,wang2023rlhf,ye2024theoretical,xiong2024iterative}. As
discussed earlier, the \xpo algorithm of \citet{xie2024exploratory}
(see also \citet{cen2024value,zhang2024self})\footnote{\citet{cen2024value,zhang2024self} concurrently proposed
  similar algorithms to \xpo, but did not provide non-trivial
  theoretical guarantees (e.g., guarantees that indicate benefits over
  purely passive exploration).} is perhaps the
closest to a satisfactory solution from prior
work, as it achieves
optimal data efficiency and only accesses the response space through
sampling from policies $\pitheta$. However, our results in
\cref{sec:computational} show that the \xpo objective cannot be
implemented efficiently in general. More broadly, even if the base model $\piref$ has favorable properties
such as coverage (in the sense of \cref{eq:coverage}), none of the
aforementioned algorithms can take advantage of it for improved computational
efficiency. In this regard, we view them as making somewhat
superficial use of the base policy (i.e., it does not play a role in
algorithm design outside of being used to define the KL-regularized RL objective).

Many works consider the complementary problem of alignment in
\emph{offline} or \emph{hybrid} settings
\citep{zhu2023principled,li2023reinforcement,xiong2024iterative,gao2024rebel,chang2024dataset,
  liu2024provably,cen2024value,fisch2024robust,ji2024selfplay,huang2024correcting,zhao2024sharp}. These
works pay for coverage coefficients similar to $\Ccov(\pistarb)$
\emph{statistically} (i.e., $\Tdatafull=\bigom(\Ccov(\pistarb))$), and
hence are not data-efficient by our definition. One relevant work here
is \citet{bose2024hybrid}, who give a hybrid variant of \xpo which
obtains statistical rates tighter than purely offline or online
methods, but is still computationally inefficient.

Algorithms that use additional inference-time computation for exploration
(e.g., via rejection sampling)
\citep{khanov2024args,chen2024pad,shi2024decoding,liu2024decoding,jinnai2024regularized,shi2024crucial} 
or \emph{multi-turn} techniques that
proceed at the per-step (e.g., token or sub-sequence) level
\citep{lightman2023lets,qu2024recursive,kumar2024training,setlur2024rewarding,setlur2024rl,xiong2024building,kazemnejad2024vineppo,zhou2024archer}
have been explored empirically, but most results we aware of do not
enjoy sample complexity guarantees. \citet{shi2024crucial} explore the role of various sampling schemes on top of \onlinedpo,
but do not give sample complexity guarantees for our
setting. \citet{gao2024regressing} provide a multi-turn algorithm with
sample complexity guarantees, but it engages in passive exploration
and pays for coverage statistically.

\paragraph{Fast rates for regularized regret}
Our algorithm \mainalg achieves a \emph{fast rate} in the sense that
$\Tdatafull\approxleq\frac{1}{\beta\veps}$ when $\veps\leq\beta$,
improving over the $\Tdatafull\approxleq{}\frac{1}{\veps^2}$ rates found in prior
    work \citep{xiong2024iterative,xie2024exploratory,cen2024value} by
    exploiting strong convexity of the KL-regularized regret. Recent work of \citet{zhao2024sharp} achieves a
similar fast rate, but requires access to an offline dataset satisfying a
stringent uniform coverage assumption (and pays for the coverage
coefficient statistically), while concurrent work of
\citet{zhao2025logarithmic} achieves fast rates in the purely online
setting, but is not computationally efficient in our framework.
Also related is the work of
  \citet{tiapkin2024regularized}, which achieves fast rates for
  regularized regret in tabular and linear MDPs, but is not efficient
  when the action space is large.

\paragraph{Algorithms for reinforcement learning with linear-$\Qstar$}
Our multi-turn algorithm, \mtalg, can be viewed as a counterpart
to a body of work which shows that MDPs with linear $\Qstar$ and
state-action gap $\Delta$ can be learned under reset access \citep{li2021sample,yin2022efficient,
  weisz2022confident,mhammedi2024power,mhammedi2024sampleoracleefficientreinforcement}. In
  particular, prior work has shown that RL with linear-$Q^{\star}$ and
  an action gap $\Delta$ is statistically intractable in the episodic
  RL protocol, but is tractable under reset access
  \citep{weisz2021query,li2021sample}. Our results show that the regularization parameter $\beta$
  plays a similar role to the action gap $\Delta$ in enabling
  favorable error propagation, leading to tractability under reset
  access. While \mtalg draws inspiration from the works
  above---particularly
  \citet{mhammedi2024power,mhammedi2024sampleoracleefficientreinforcement}---it
  requires fairly substantial modifications, both in design and
  analysis---to (i) leverage KL regularization, and (ii) achieve
  computational efficiency in the sampling oracle framework.\loose

\subsection{Comparison to Preference-Based Feedback}
\label{app:preference}
Much of prior work on online alignment focuses on
\emph{preference-based feedback}. Here,
the protocol is as follows. At each round $t\in\brk{\Trounds}$, we receive a prompt
$x\ind{t}$ and sample two responses %
$(y_1\ind{t},y_2\ind{t})\sim{}\pi\ind{t}(\cdot\mid{}x\ind{t})$, where $\pi\ind{t}$
denotes the \emph{exploration policy} for round $t$; the exploration
policy may be represented as a language model, or may
correspond to an alternative sampler
\citep{liu2023statistical,khaki2024rs,shi2024crucial}. The responses are then labeled as 
$(\yp\ind{t},\ym\ind{t})$ based on a binary preference
$b\ind{t}\sim{}\bbP(y_1\ind{t}\psdgt{}y_2\ind{t}\mid{}x\ind{t})$, and added to the preference dataset via
$\cD\ind{t+1}\gets\cD\ind{t}\cup\crl{(x\ind{t},\yp\ind{t},\ym\ind{t})}$,
which can then be used to compute an updated policy
$\pi\ind{t+1}$. The preference distribution
$\bbP(y_1\psdgt{}y_2\mid{}x)$ represents the underlying verifier or
oracle of interest; it is typically assumed that preferences follow the Bradley-Terry
model \citep{bradley1952rank}, i.e.
\begin{align}
\label{eq:bt}
\bbP(y_1\psdgt{}y_2\mid{}x) = \frac{\exp\prn*{\rstar(x,y_1)}}{\exp\prn*{\rstar(x,y_1)} + \exp\prn*{\rstar(x,y_2)}},
\end{align}
for an underlying \emph{reward function} 
$\rstar:\cX\times\cY\to\bbR$. As with our setting, the goal is to  use
the collected data $\cD_{\pref}$ to produce a final policy $\pihat$
with high KL-regularized reward $\Jbeta(\pihat)$.

When $N=2$, our absolute reward formulation in \cref{sec:background} is very
closely related to this formulation, and algorithms for one
setting can easily be adapted to the other (typically the only change
is in the objective used to estimate the reward model). We use the absolute reward formulation and general $N$
(as described \cref{sec:background}) because (i) allowing for $N>2$ makes our lower
bounds/impossibility results stronger, even though our algorithms
themselves only use $N=2$; and (ii) the absolute reward
formulation---which has been used in prior work empirically
\citep{wang2023helpsteer,wang2024helpsteer2,wang2024interpretable,xiong2024building}
and in theory
\citep{zhao2024sharp,wang2024arithmetic,xiong2024building}---is a more
realistic model for the motivating problem of learning from a strong
oracle/verifier such as a proof checker.

\paragraph{Adapting preference-based algorithms to reward-based
  feedback}
\onlinedpo
\citep{guo2024direct} proceeds iteratively for $t\in\brk{\Trounds}$ as
follows:
\begin{enumerate}
\item Compute $\pi\ind{t} := \pi_{\theta\ind{t}}$ by solving the \dpo{} objective:
  \begin{small}
  \begin{align}
  \label{eq:dpo}
\theta\ind{t}\gets\argmin_{\theta\in\Theta}
\sum_{(x,\yp,\ym) \in \cD\ind{t}} - \log\left[\sigma\left( \beta\log\frac{\pitheta(\yp\mid{}x)}{\piref(\yp\mid{}x)} - \beta\log\frac{\pitheta(\ym\mid{}x)}{\piref(\ym\mid{}x)} \right) \right],
  \end{align}\end{small}%
  where $\sigma(z)\ldef{}\frac{\exp(z)}{1+\exp(z)}$ is the sigmoid function.
\item Sample
  $y_1\ind{t},y_2\ind{t}\sim{}\pi_{\theta\ind{t}}(\cdot\mid{}x\ind{t})$, then
  label as $(\yp\ind{t},\ym\ind{t})$ and update  $\cD\ind{t+1}\gets\cD\ind{t}\cup\crl{(x\ind{t},\yp\ind{t},\ym\ind{t})}$.
\end{enumerate}
For our reward-based setting, we change \cref{eq:dpo} to
\begin{small}
\begin{align}
\theta\ind{t}\gets\argmin_{\theta\in\Theta}
  \sum_{(x,y_1,y_2,r_1,r_2) \in \cD\ind{t}} \left(
  \beta\log\frac{\pitheta(y_1\mid{}x)}{\piref(y_1\mid{}x)} -
  \beta\log\frac{\pitheta(y_2\mid{}x)}{\piref(y_2\mid{}x)}  - (r_1 - r_2)\right)^2.
\end{align}\end{small}%
It is possible to show that this algorithm obtains
$\Tdatafull=\poly(d,\Rmax,\Ccov(\pistarb),\veps^{-1},\log(\delta^{-1}))$
for our linear softmax setting through standard arguments.

Similarly, at each step $t$, \xpo \citep{xie2024exploratory} minimizes the objective
\begin{small}
  \begin{align*}
    \theta\ind{t}\gets\argmin_{\theta\in\Theta}
    \crl*{\alpha\sum_{i<t}\log\pi_{\theta}(y_2\ind{i}\mid{}x\ind{i})
    +
    \sum_{(x,y_1,y_2,r_1,r_2) \in \cD\ind{t}} - \log\left[\sigma\left(
    \beta\log\frac{\pitheta(\yp\mid{}x)}{\piref(\yp\mid{}x)} -
    \beta\log\frac{\pitheta(\ym\mid{}x)}{\piref(\ym\mid{}x)} \right) \right]
    },
  \end{align*}
\end{small}
for an optimism parameter $\alpha>0$, then samples 
  $y\ind{t}\sim{}\pi_{\theta\ind{t}}(\cdot\mid{}x\ind{t})$ and
  $y_2\ind{t}\sim\piref(\cdot\mid{}x\ind{t})$ and updates
  $\cD\ind{t+1}\gets\cD\ind{t}\cup\crl{(x\ind{t},\yp\ind{t},\ym\ind{t})}$. To
  adapt \xpo to the reward-based setting, we analogously change the
  objective to
  \begin{small}
    \begin{align}
      \label{eq:xpo_square}
      \theta\ind{t}\gets\argmin_{\theta\in\Theta}
      \crl*{\alpha\sum_{i<t}\log\pi_{\theta}(y_2\ind{i}\mid{}x\ind{i})
      +\sum_{(x,y_1,y_2,r_1,r_2) \in \cD\ind{t}} \left(
      \beta\log\frac{\pitheta(y_1\mid{}x)}{\piref(y_1\mid{}x)} -
      \beta\log\frac{\pitheta(y_2\mid{}x)}{\piref(y_2\mid{}x)}  - (r_1 - r_2)\right)^2
      }.
    \end{align}
  \end{small}
  The sample complexity bound %
  \arxiv{\[
    \Tdatafull \approxleq \poly(\Rmax)\cdot\frac{d^{2}\log(\delta^{-1})}{\veps^2}
  \]}
  claimed in \cref{eq:xpo_data} follows by (i) specializing the
  \texttt{SEC}-based bound in \citet{xie2024exploratory} to the linear
  softmax policy class, and (ii) noting that the $\exp(\Rmax)$ factor
  in the sample complexity guarantee in \citet{xie2024exploratory} can be
  removed under reward-based feedback (as it arises due to converting
  between the logistic loss for the Bradley-Terry model and the square
  loss); these calculations can be found in Theorem J.1 and Lemmas J.4 and J.5 of \citet{huang2024self}.

\paragraph{Adapting \spanalg to preference-based feedback}
To adapt \spanalg to preference-based feedback, the only change
required is to switch the reward estimation step in
\cref{eq:spanner_dpo} to the following \dpo-like objective: %
\arxiv{\begin{align}
        \theta\ind{t}\gets\argmin_{\theta\in\Theta}\sum_{(x,y_{+},y_{-})
        \in \cDexp\ind{t}\cup\cDspan} -\log\sigma\prn*{
        \tri*{\theta, \phidel(x,y_+,y_-)}}.
       \end{align}
       }
This leads to identical guarantees, except that the sample complexity
will pay for a $\exp(\Rmax)$ factor due to conversion from logistic
loss to square loss (e.g., Lemma C.8 in \citet{xie2024exploratory}).

  \section{Sampling Oracles: Beyond Linear Policies}
  \label{app:sampling}
We expect that our sampling oracle abstraction for the language model
alignment problem will be of use beyond the linear softmax policy
parameterization we focus on. In this section, we briefly discuss possibilities
for extending \cref{def:oracle} beyond the linear setting, as well as
challenges this entails.

\paragraph{Features versus log-probabilities}

Recall that the sampling oracle in \cref{def:oracle} reveals the features $\phi(x,y)$ for responses $y$ sampled from
      the oracle, but does not reveal 
      the log-probabilities $\log\pitheta(y\mid{}x)$ themselves. As
      highlighted in \cref{sec:coverage}, the observed features are closely related, as
      they can be used to evaluate $\beta\log\frac{\pitheta(y\mid{}x)}{\piref(y\mid{}x)} -
        \beta\log\frac{\pitheta(y'\mid{}x)}{\piref(y'\mid{}x)}=\tri*{\theta,\phi(x,y)-\phi(x,y')}$,
        but they cannot be used to compute $\log\pitheta(y\mid{}x)$ itself
        in general. We adopt this formalism because it simplifies the coverage-based
        lower bounds in \cref{sec:coverage_lower}; our algorithmic results only make use of the
        features $\phi(x,y)$, and hence fall into the framework of \cref{def:oracle}. But to move beyond linear policies, it is more natural to
        directly allow the learner to query the log-probabilities.\loose
        \begin{definition}[Generalized sampling oracle framework]
          \label{def:oracle_generalized}
          In one query, the learner proposes a prompt $x\in\cX$ and
          parameter $\theta\in\Theta$, and receives a conditional sample
          $y\sim\pitheta(\cdot\mid{}x)$, as well as the
          corresponding log-probability $\log\pitheta(y\mid{}x)$ for
          the sampled response (note that $\pi_{\mb{0}}=\piref$).
          
        \end{definition}
There is a technical subtlety here (and in \cref{def:oracle}) as far as the learner's a-priori knowledge. In the linear softmax setting, if $\phi$ is known a-priori, then ruling out algorithms that enumerate over the response space $\cY$ requires additionally assuming that each query $y\ind{t}_i$ made to the reward oracle is a response that has previously been revealed by the sampling oracle. It seems more natural to consider the features (and, in general, the parametrization $\theta\mapsto\pi_\theta$) to be unknown a-priori. %
        
        Proving lower bounds like \cref{thm:coverage} may be more
        challenging in the framework in \cref{def:oracle_generalized}, as the log-probabilities can
        potentially provide more information than the features
        themselves. On the other hand, more assumptions are likely
        required to derive efficient algorithms. For example, it is not clear
        that one can efficiently minimize the \dpo objective under
        \cref{def:oracle_generalized}, and so it might be necessary to
        assume an additional oracle for minimizing the objective. We
        leave a detailed understanding for future work.

\newpage
\part{Proofs from Sections \ref*{sec:coverage} through \ref*{sec:computational}}

\section{Technical Tools}
\label{app:technical}

    For a pair of probability measures $\bbP$ and $\bbQ$, we define
    the total variation distance as
    $\Dtv{\bbP}{\bbQ}=\frac{1}{2}\int\abs{\mathrm{d}\bbP-\mathrm{d}\bbQ}$,
    and define Hellinger distance by
    $\Dhels{\bbP}{\bbQ}=\int(\sqrt{\mathrm{d}\bbQ}-\sqrt{\mathrm{d}\bbQ})^2$. We
    define KL divergence by
    $\Dkl{\bbP}{\bbQ}=\int{}\mathrm{d}\bbP\log\prn[\big]{\frac{\mathrm{d}\bbP}{\mathrm{d}\bbQ}}$ if
    $\bbP\ll\bbQ$ and $\Dkl{\bbP}{\bbQ}=+\infty$ otherwise.

\subsection{Tail Bounds}

\begin{lemma}[Azuma-Hoeffding]
  \label{lem:hoeffding}
    Let $(X_t)_{t\leq{T}}$ be a sequence of real-valued random
    variables adapted to a filtration $\prn{\filt_t}_{t\leq{}T}$. If
    $\abs*{X_t}\leq{}R$ almost surely, then with probability at least
    $1-\delta$, for all $T'\leq{}T$,
    \[
      \abs*{\sum_{t=1}^{T'}X_t - \En_{t-1}\brk*{X_t}} \leq{} R\cdot\sqrt{8T\log(2\delta^{-1})}.
    \]
\end{lemma}

\begin{lemma}[Freedman's inequality]
  \label{lem:freedman}
  Let $(X_t)_{t\leq{T}}$ be a real-valued martingale difference
  sequence adapted to a filtration $\prn{\filt_t}_{t\leq{}T}$. If
  $\abs*{X_t}\leq{}R$ almost surely, then for any $\eta\in(0,1/R)$,
  with probability at least $1-\delta$, for all $T'\leq{}T$,
    \[
      \sum_{t=1}^{T'}X_t \leq{} \eta\sum_{t=1}^{T'}\En_{t-1}\brk*{X_t^{2}} + \frac{\log(\delta^{-1})}{\eta}.
    \]
  \end{lemma}
  The next result is a standard consequence of \cref{lem:freedman}
  (e.g., \citet{foster2021statistical}).
        \begin{lemma}
          \label{lem:mult_freedman}
            Let $(X_t)_{t\leq{T}}$ be a sequence of random
      variables adapted to a filtration $\prn{\filt_{t}}_{t\leq{}T}$. If
  $0\leq{}X_t\leq{}R$ almost surely, then with probability at least
  $1-\delta$, for all $T'\leq{}T$,
  \begin{align}
    &\sum_{t=1}^{T'}X_t \leq{}
                        \frac{3}{2}\sum_{t=1}^{T'}\En_{t-1}\brk*{X_t} +
                        4R\log(2\delta^{-1}),
    \intertext{and}
      &\sum_{t=1}^{T'}\En_{t-1}\brk*{X_t} \leq{} 2\sum_{t=1}^{T'}X_t + 8R\log(2\delta^{-1}).
  \end{align}
\end{lemma}

\subsection{Elliptic Potential}

\begin{lemma}[e.g. Lemma 19.4 in \citet{lattimore2020bandit}]
\label{lem:elliptic_potential}
Let $v_1,\ldots,v_T\in \bbR^d$ satisfy $\norm{v_t}_2\leq 1$
        for all $t\in[T]$. Fix $\lambda>0$, and let $V_t = \lambda{}I
        + \sum_{i<t} v_i v_i^\trn$. Then
	\begin{align}
          \sum_{t=1}^{T}\norm{v_t}^2_{V_{t}^{-1}}\wedge{}1 \leq
          2\sum_{t=1}^{T}\log(1+\norm{v_t}^2_{V_{t}^{-1}}) \leq
          2d \log \prn*{1+\lambda^{-1}T/d}.
	\end{align}
        As a consequence, we have
        \begin{align}
          \sum_{t=1}^{T}\norm{v_t}_{V_{t}^{-1}}\wedge{}1 \leq
          \sqrt{2dT\log \prn*{1+\lambda^{-1}T/d}}.
        \end{align}
\end{lemma}

\subsection{Miscellaneous Lemmas}

\begin{lemma}[Sequential union bound]
  \label{lem:unionbound}
  Let $T,H\in \mathbb{N}$ and $\delta \in (0,1)$ be given.
  Further, let $\cB_1$ be an algorithm that runs in $T\in\mathbb{N}$ iterations. At each iteration, $\cB_1$ makes a sequence of $H$ calls to a subroutine $\cB_2$. Let $\mathfrak{S}$ denote the state space of algorithm $\cB_1$; the space capturing the values of all the internal variables of $\cB_1$. Let $\mathbf{S}\ind{t}_{h,-}\in \mathfrak{S}$ denote the random state of $\cB_1$ immediately before the $h$th call to $\cB_2$ during the $t$th iteration; further, let $\mathbf{S}\ind{t}_{h,+} \in \mathfrak{S}$ denote the random state of $\cB_1$ immediately after this call to $\cB_2$. Suppose that for any $S\ind{t}_{h,-}\in \mathfrak{S}$, there is an event $\cE\ind{t}_{h}(S\ind{t}_{h,-}) \subset \mathfrak{S}$ such that $\P[\mathbf{S}\ind{t}_{h,+}\in \cE\ind{t}_{h}(S\ind{t}_{h,-})]\geq 1-\delta$. Then, with probability at least $1-\delta H T$, for all $t\in[T]$ and $h\in[H]$, we have $\mathbf{S}\ind{t}_{h,+}\in \cE\ind{t}_{h}(\mathbf{S}\ind{t}_{h,-})$.   
  \end{lemma}
  \begin{proof}
    Let $\cE$ be the event defined by 
    \begin{align}
      \cE \coloneqq \left\{\prod_{t=1}^T \prod_{h=1}^H \mathbb{I}\{\mathbf{S}\ind{t}_{h,+} \in  \cE\ind{t}_{h}(\mathbf{S}\ind{t}_{h,-}) \}=1 \right\}.
    \end{align}  
    We need to show that $\P[\cE]\geq 1-\delta H T$. To this end, we note that by the chain rule of probability, we have
    \begin{align}
     \P[\cE] & = \prod_{t=1}^T \prod_{h=1}^H \E\left[ \P[\mathbf{S}\ind{t}_{h,+}\in \cE\ind{t}_{h}(\mathbf{S}\ind{t}_{h,-})\mid \mathbf{S}\ind{t}_{h,-}]\right],\nn\\
     & \geq \prod_{t=1}^T\prod_{h=1}^H \left(1-\delta\right),  \label{eq:penul} \\
     & \geq 1- TH \delta,\nn 
  \end{align}
  where \eqref{eq:penul} follows by the fact that $\P[\mathbf{S}\ind{t}_{h,+}\in \cE\ind{t}_{h}(S\ind{t}_{h,-})]\geq 1-\delta$ for all $S\ind{t}_{h,-}\in \mathfrak{S}$, and the last inequality follows by the fact that for any sequence $x_{1},\dots,x_{T}\in (0,1)$, $\prod_{i\in[T]} (1- x_i)\geq 1- \sum_{i\in[T]}x_i$. 
  
  \end{proof}

\begin{lemma}
  \label{lem:log_bound}
  If $x\geq{}1$ satisfies $x\leq{}a\log(1+bx)$ for $a,b\geq{}3$, then $x\leq{}2a\log(1+ab)$.
\end{lemma}
\begin{proof}[\pfref{lem:log_bound}] First, note that by reparameterizing $x\gets{}bx$ and $c\gets{}ab$, it suffices to show that
$x\leq{}c\log(1+x)$ for $x\geq{}1,c\geq{}3$ implies
$x\leq{}2c\log(1+c)$. Toward proving the latter statement, we first note that if $x\geq{}c$,
then $x\mapsto{}x$ increases faster than $x\mapsto{}c\log(1+x)$, so any point
$x\geq{}c$ for which $x > 2c\log(1+x)$ gives a valid upper bound. Let
us choose $x=2c\log(1+c)$. Then we have
$c\log(1+x)\leq{}
c\log(1+2c\log(1+c))<c\log(1+c^2)\leq{}c\log((1+c)^2)\leq2c\log(1+c) =
x$ as desired, where the strict inequality uses that $2c\log(1+c)\leq c^2$ for $c\geq{}3$.
  
\end{proof}

\newpage
\section{Proofs from \creftitle{sec:coverage}}
\label{app:coverage}

In this section we restate and prove \cref{thm:coverage}.

\coveragelower*

\begin{proof}[\pfref{thm:coverage}]
  If $\beta > 1/2$ then the lower bound on $\Tcompfull$ is vacuously true, so we may assume henceforth that $\beta \leq 1/2$. Similarly, we may assume without loss of generality that $Y \geq 9$. Let $\cS$ be an arbitrary set of size $Y-1$. We take prompt space $\cX = \{\perp\}$ (and henceforth omit all dependences on the prompt $\perp$). We take response space $\cY = \{0\}\cup\cS$. We take parameter space $\Theta = \ball$. 

  For each $\thetastar \in \ball$ and $\ystar \in \cS$, we define an instance $\cI^{\thetastar,\ystar}$ of the online alignment problem (with linear softmax policy class) as follows:
\begin{itemize} 
  \item The reference policy is $\piref^{\ystar} \in \Delta(\cY)$ defined by $\piref^{\ystar}(0) = 1-\epref$ and $\piref^{\ystar}(\ystar) = \epref$ where $\epref := \max\{1/C^\st, e^{-\beta^{-1}/2}\}$.
  \item The feature mapping $\phi^{\thetastar,\ystar}: \cY \to \RR^d$ is defined by
\[\phi^{\thetastar,\ystar}(y) = \begin{cases} \thetastar & \text{ if } y = \ystar \\ 0 & \text{ if } y \neq \ystar \end{cases}.\]
\item The reward function $r^{\thetastar,\ystar}: \cY \to [0,1]$ is defined by $r^{\thetastar,\ystar}(y) = \langle \thetastar, \phi^{\thetastar,\ystar}(y)\rangle = \mathbbm{1}[y=\ystar]$. 
\end{itemize}
Note that $\cX$, $\cY$, and $\Theta$ are fixed, and do not depend on the choice of $(\thetastar,\ystar)$.

We make the following observations about $\cI^{\thetastar,\ystar}$. Since $r^{\thetastar,y^\st}(y) = \langle \thetastar,\phi^{\thetastar,y^\st}(y)\rangle$ and $\thetastar \in \ball = \Theta$, \cref{ass:realizable} is satisfied. In particular, the optimal KL-regularized policy is $\pi_{\thetastar}(y)\propto\piref(y)\exp\prn*{\beta^{-1}\tri*{\thetastar,\phi^{\thetastar,\ystar}(y)}}$. It is straightforward to check that \cref{ass:norm} is satisfied with $\Rmax = B = 1$. From \cref{eq:coverage}, we have 
\[\Ccov(\pi_{\theta^\st}) =\Ccov(\pi_{\thetastar})\leq \max_{y \in \{0, \ystar\}} \frac{1}{\piref(y)} \leq \frac{1}{\epref} \leq C^\st\]
where the second inequality uses the fact that $1-\epref \geq \epref$ (which holds since $C^\st \geq 2$ and $\beta \leq 1/2$). Finally, $|\cY| = Y$ by construction. From the theorem assumptions, we conclude that for all instances $\cI^{\thetastar,\ystar}$, $\Alg$ uses $\Tdatafull$ queries to the reward oracle and $\Tcompfull$ queries to the strong sampling oracle, and with probability at least $1-\delta$ returns a policy $\pihat$ satisfying $J_\beta(\pi_{\thetastar}) - J_\beta(\pihat) \leq \veps$. Assume, for the sake of contradiction, that $\Tdata := \Tdata(1/4,1/4) < Y/8$ and $\Tcomp := \Tcomp(1/4,1/4) < c_0 \cdot\min\{e^{\beta^2 d/2}, e^{\beta^{-1}/2}, C^\st\}$ for a universal constant $c_0>0$ to be determined.

Now, consider the distribution over problem instances induced by sampling $\thetastar$ uniformly from the unit sphere in $\RR^d$, and independently sampling $\ystar \sim \Unif(\cS)$. Then execute $\Alg$ on instance $\cI^{\thetastar,\ystar}$ with error tolerance $\veps = 1/4$ and failure probability $\delta = 1/4$. For the purposes of analysis, for each $q\geq{}0$, let $\Algbar^{[q]}$ denote a modified version of algorithm where the first $q$ oracle queries are answered with $0 \in \cY$ (for the sampling oracle) or $0 \in \RR$ (for the reward oracle), and the algorithm is run unmodified for subsequent steps. Let $\pihat^{[q]}$ denote the output of $\Algbar^{[q]}$. On the one hand, since $\Algbar^{[0]} = \Alg$, we know that
\begin{equation} \Pr[J_\beta(\pi_{\thetastar}) - J_\beta(\pihat^{[0]}) \leq \veps] \geq 1-\delta\label{eq:pihat0}\end{equation}
where the probability is over the random choice of $(\thetastar,\ystar)$ and the randomness of $\Algbar^{[0]}$ (and its oracle calls). On the other hand, since the problem parameters $(\cX,\cY,\Theta)$ are independent of $(\thetastar,\ystar)$ and all queries made by $\Algbar^{[\Tcomp+\Tdata]}$ are answered independently of $(\thetastar,\ystar)$, we have that $\pihat^{[\Tcomp+\Tdata]}$ is independent of $(\thetastar,\ystar)$. We use this to prove the following lower bound on the regret for $\pihat^{[\Tcomp+\Tdata]}$.

\begin{lemma}
  \label{lem:pihatfinal}
  For $\veps=\delta=1/4$, it holds that
  \begin{equation}\Pr[J_\beta(\pi_{\thetastar}) - J_\beta(\pihat^{[\Tcomp+\Tdata]}) \leq \veps] \leq 1/2.\label{eq:pihatfinal}\end{equation}
\end{lemma}

\begin{proof}[\pfref{lem:pihatfinal}]
For any fixed $(\thetastar,\ystar)$, in instance $\cI^{\thetastar,\ystar}$, we have
\begin{align*}
J_\beta(\pi_{\thetastar}) 
&= \pi_{\thetastar}(\ystar) - \beta \Dkl{\pi_{\thetastar}}{\piref} \\
&= \beta \log(1-\epref+\epref e^{\beta^{-1}})  \\ 
&\geq \beta \log(\epref e^{\beta^{-1}}) \\
&= 1 - \beta \log(1/\epref) \\
&\geq 1/2,
\end{align*}
where the first equality is by definition of $J_\beta$ and the reward function $r^{\thetastar,\ystar}$; the second equality is by explicit calculation; and the final inequality is $\epref \geq \frac{1}{2}e^{-\beta^{-1}/2}$. But we also have \[J_\beta(\pihat^{[\Tcomp+\Tdata]}) \leq \pihat^{[\Tcomp+\Tdata]}(\ystar)\]
by definition of $J_\beta$ and non-negativity of KL-divergence. Since $\pihat^{[\Tcomp+\Tdata]}$ is independent of $\ystar$ and $\ystar$ is uniformly distributed in $\cS$, we know that \[\EE[\pihat^{[\Tcomp+\Tdata]}(\ystar)] \leq 1/|\cS|,\] and so by Markov's inequality and the fact that $|\cS| = Y-1 \geq 8$, we have \[\Pr[\pihat^{[\Tcomp+\Tdata]}(\ystar) \geq 1/4] \leq 1/2.\]
Recalling that $\veps = 1/4$, it follows that
\begin{equation*}\Pr[J_\beta(\pi_a) - J_\beta(\pihat^{[\Tcomp+\Tdata]}) \leq \veps] \leq \Pr[J_\beta(\pi_a) - \pihat^{[\Tcomp+\Tdata]}(\ystar) \leq \veps] \leq 1/2.\end{equation*}
\end{proof}
From here, we proceed by relating the regret of $\pihat^{[0]}$ to that of $\pihat^{[\Tcomp+\Tdata]}$. Fix $0 \leq q < \Tcomp+\Tdata$. The probability that $\Algbar^{[q]}$ deviates from $\Algbar^{[q+1]}$ is at most the probability that the response to the $(q+1)$-th oracle query by $\Algbar^{[q]}$ is non-zero. Since all previous oracle queries by $\Algbar^{[q]}$ were answered independently of $(\thetastar,\ystar)$, the $(q+1)$-th \emph{query} (though not its answer) is independent as well. Condition on this query; we distinguish two cases.
\begin{enumerate}
\item If it is a sampling oracle query $\theta \in \Theta$, then the probability that the execution of $\Algbar^{[q+1]}$ deviates from $\Algbar^{[q]}$ (in the optimal coupling of their executions) is precisely the probability that the sampling oracle $y\sim\pi_{\theta}$ yields a non-zero answer $y\neq\ystar$, which is precisely $\pi_\theta(\ystar)$. Moreover, we can bound the expectation (over all randomness) of this probability:
\begin{align*} 
\EE[\pi_\theta(\ystar)] 
&= \EE\left[\frac{\epref \exp(\beta^{-1} \langle \theta, \phi^{\thetastar,\ystar}(\ystar)\rangle)}{(1-\epref) \exp(\beta^{-1}\langle \theta, \phi^{\thetastar,\ystar}(0)\rangle) + \epref \exp(\beta^{-1} \langle \theta, \phi^{\thetastar,\ystar}(\ystar)\rangle)}\right]\\
&= \EE\left[\frac{\epref \exp(\beta^{-1} \langle \theta, \thetastar\rangle)}{1-\epref + \epref \exp(\beta^{-1} \langle \theta, \thetastar\rangle)}\right] \\ 
&\leq \EE\left[\frac{\epref \exp(\beta^{-1} \max(\thetastar_1,0))}{1-\epref + \epref \exp(\beta^{-1} \max(\thetastar_1,0))}\right] 
\leq O(\epref + \exp(-\beta^2 d/2)),
\end{align*}
where the first inequality uses the fact that $\theta$ is independent of $\thetastar$ and hence $\langle \theta,\thetastar\rangle$ is stochastically dominated by $\max(\thetastar_1,0)$, and the final inequality uses \cref{lemma:bounded-cap-bound} (stated and proven in the sequel). %

\loose

\item If it is a reward query $y \in \cY$, then the probability that the execution of $\Algbar^{[q+1]}$ deviates from $\Algbar^{[q]}$ is precisely the probability that the reward oracle yields a non-zero answer, which is $r^{\thetastar,\ystar}(y) = \mathbbm{1}[y=\ystar]$. Since $y$ is independent of $\ystar$, we have $\Pr[r^{\thetastar,\ystar}(y) \neq 0] \leq 1/|\cS|$.
\end{enumerate} 
Therefore by the data processing inequality, %
\begin{align} 
\Pr[\pihat^{[q]} \neq \pihat^{[q+1]}] 
&\leq \Dtv{\Law(\Algbar^{[q]})}{\Law(\Algbar^{[q+1]})} \nonumber\\ 
&\leq O(\epref + \exp(-\beta^2 d/2)) \cdot \Pr[\text{$(q+1)$-th query by $\Alg^{[q]}$ is sampling}] \nonumber\\ 
&\qquad+ \frac{1}{|\cS|} \cdot \Pr[\text{$(q+1)$-th query by $\Alg^{[q]}$ is reward}] \nonumber\\ 
&= O(\epref + \exp(-\beta^2 d/2)) \cdot \Pr[\text{$(q+1)$-th query by $\Alg^{[\Tcomp+\Tdata]}$ is sampling}] \nonumber\\ 
&+ \frac{1}{|\cS|} \cdot \Pr[\text{$(q+1)$-th query by $\Alg^{[\Tcomp+\Tdata]}$ is reward}]\label{eq:pihathybrid}
\end{align}
where the equality uses the fact that the executions of $\Alg^{[q]}$ and $\Alg^{[\Tcomp+\Tdata]}$ are identically distributed up to and including the $(q+1)$-th query. We conclude that
\begin{align*}
\Pr[J_\beta(\pi_{\thetastar}) - J_\beta(\pihat^{[0]}) \leq \epref] 
&\leq \Pr[J_\beta(\pi_{\thetastar}) - J_\beta(\pihat^{[\Tcomp+\Tdata]}) \leq \epref] + \sum_{q=0}^{\Tcomp+\Tdata-1} \Pr[\pihat^{[q]} \neq \pihat^{[q+1]}]  \\ 
&\leq \frac{1}{2} + O(\epref + \exp(-\beta^2 d/2)) \Tcomp + \frac{1}{|\cS|} \Tdata 
< 3/4,
\end{align*}
where the second inequality is by \cref{eq:pihatfinal,eq:pihathybrid}, and the third inequality is by the assumed bounds on $\Tcomp,\Tdata$ and holds so long as $c_0>0$ is a sufficiently small constant. This contradicts \cref{eq:pihat0}, so it must be that either $\Tdata \geq Y/8$ or $\Tcomp \geq c_0 \cdot \min\{e^{\beta^2 d/2},e^{-\beta^{-1}/2},C^\st\}$.
\end{proof}

\begin{lemma}\label{lemma:bounded-cap-bound}
Fix $\epsilon \in (0,1/2)$, $\beta>0$, and $d \in \NN$. Let $X \sim \Unif(\mathbb{S}^{d-1})$ where $\mathbb{S}^{d-1}$ is the unit sphere in $\RR^d$. Then
\[\EE\left[\frac{\epsilon \cdot e^{\beta^{-1} \max(X_1,0)}}{1-\epsilon + \epsilon\cdot e^{\beta^{-1} \max(X_1,0)}}\right] \lesssim \epsilon + e^{-\beta^2 d/2}.\]
\end{lemma}

\begin{proof}[\pfref{lemma:bounded-cap-bound}]
The quantity inside the expectation is always at most $1$. Moreover, if $X_1 \leq \beta$, then $\frac{\epsilon \cdot e^{\beta^{-1} \max(X_1,0)}}{1-\epsilon + \epsilon\cdot e^{\beta^{-1} \max(X_1,0)}} \lesssim \epsilon$. It follows that
\begin{align*} \EE\left[\frac{\epsilon \cdot e^{\beta^{-1} \max(X_1,0)}}{1-\epsilon + \epsilon\cdot e^{\beta^{-1} \max(X_1,0)}}\right] 
\lesssim \epsilon + \Pr[X_1 > \beta] 
\leq \epsilon + e^{-\beta^2 d/2}
\end{align*}
by a standard bound on the volume of a spherical cap \citep{tkocz2012upper}.
\end{proof}

\newpage
\section{\rejection Algorithm and Guarantees}
\label{sec:rejection}

\begin{algorithm}[tp]
    \caption{$\softmaxsample$}
    \label{alg:rejection_density}
    \begin{algorithmic}[1]
      \Statex[0]\multiline{ {\bfseries input:}
        Function $f:\cX\times\cY\to\bbR$, prompt $x$, base policy $\piref:\cX\to\Delta(\cY)$, parameter
        $\beta>0$, rejection threshold $M>0$, failure probability $\delta\in(0,1)$.}
      \State Let $N\ldef{}4M\log(4\delta^{-1})$.
      \Statex[0] \algcommentbig{Estimate normalization constant}
      \State Sample $y_1,\ldots,y_N\sim\piref(\cdot\mid{}x)$ \iid
      \State Set $\Zhat\ldef{}\frac{1}{N}\sum_{i=1}^{n}\exp\prn*{\beta^{-1}f(x,y_i)}$.\label{line:rejection_est_density}
      \Statex[0] \algcommentbig{Rejection sampling}
\For{iteration $i = 1,2,\dotsc,N$}
\label{line:rejection_for_density}
      \State Sample $y\sim \piref(\cdot\mid{}x)$ and $\xi\sim\Ber\prn*{\frac{\exp\prn*{\beta^{-1}f(x,y)}}{\Zhat M}}$.
      \State Set $\rho \gets \frac{\exp\prn{\beta^{-1}f(x,y)}}{\Zhat}$.  \hfill \algcommentlight{$\rho \approx \frac{\exp\prn{\beta^{-1} f(x,y)}}{\E_{y\sim \pi_\refe(\cdot\mid x)}\left[ \exp\prn{\beta^{-1} f(x,y)}\right]}$.}
      \State If $\xi=1$, \textbf{return} $(y, \rho)$.
        \EndFor
        \State Sample $y\sim\pi_\refe(\cdot\mid{}x)$ and set $\rho \gets \frac{\exp\prn{\beta^{-1}f(x,y)}}{\Zhat}$. \hfill %
        \State \textbf{return}
        $(y, \rho)$.\hfill\algcommentlight{Failure event;
          occurs with low probability.}
    \end{algorithmic}
    \end{algorithm}

  In this section, we give self-contained guarantees for the
  \rejection algorithm (\cref{alg:rejection}) used within \spanalg, as
  well as a slightly more general version of the algorithm,
  \softmaxsample{} (\cref{alg:rejection_density}), which is used
  within \mtalg (\cref{sec:algos_rl}). Both algorithms take as input a
  function $f(x,y)$ and use rejection
sampling to generate samples from the softmax policy
\begin{align}
  \pif(y\mid{}x)\propto\piref(y\mid{}x)\exp\prn*{\beta^{-1}f(x,y)} \label{eq:loads}
\end{align}
given sample access to $\piref$. \softmaxsample{} only differs from
\rejection in that, in addition to using rejection sampling
to generate samples from \eqref{eq:loads}, it also returns
an estimate for density ratio $\frac{\pif(y\mid{}x)}{\piref(y\mid{}x)}$
for the sampled response; since the \rejection algorithm already estimates the normalization constant for the target policy, which is the only non-trivial part of the density ratio to compute, this requires no change outside of explicitly
  returning the density ratio estimate.

\paragraph{Algorithm overview}
Let us briefly describe the
algorithm. \cref{line:rejection_for_density} of \rejection and \softmaxsample{} applies vanilla rejection
sampling to generate samples from $\pi_f$, sampling multiple responses
from $\piref$ and using the density ratio to decide whether to accept
each response. The only subtlety is that
the density ratio
\[\frac{\pi_f(y\mid{}x)}{\piref(y\mid{}x)}=\frac{\exp\prn*{\beta^{-1}f(x,y)}}{\En_{y'\sim\piref}\brk*{\exp(\beta^{-1}f(x,y'))}}
  ,
  \]
  depends on the normalization constant
  $Z(x)\ldef\En_{y'\sim\piref}\brk*{\exp(\beta^{-1}f(x,y'))}$, which is
  unknown. To address this, \cref{line:rejection_est_density} estimates the
  normalization constant via sampling from $\piref$ and computing the
  empirical mean. The estimated normalization constant is then used to
  set the rejection threshold.

  The main guarantee for \softmaxsample{} is as follows.

\begin{theorem}[Guarantee for \softmaxsample]
  \label{thm:rejection_density}
  Let $f:\cX\times\cY\to\bbR$, $x\in\cX$, and $\beta>0$ be given, and define
  \begin{align}
    \label{eq:cinf0}
  \pif(\cdot \mid{}x)\propto\piref(\cdot\mid{}x)\exp\prn*{\beta^{-1}f(x,\cdot)},\mathand    \Cinf\ldef{} 1 \vee \nrm*{\frac{\pif(\cdot\mid{}x)}{\piref(\cdot\mid{}x)}}_{\infty}.
  \end{align}
  Fix $\delta\in(0,1)$, and suppose that
  $M\geq{}4\Cinf$. There is an event $\cEaccept$ with
  $\bbP(\cEaccept)\geq{}1-\delta$ under which the output
  $(y,\rho)$ of $\softmaxsample_{\beta, M, \delta}(f;x,\piref)$ satisfies
\begin{gather}
\bbP(y=\cdot\mid{}\cEaccept) = \pif(\cdot\mid{}x) 
\shortintertext{and} 
\mathbb{I}\{M \geq 4 \Cinf^2\} \cdot \left|\log \rho - \log \frac{\pi_f(y \mid x)}{\piref(y\mid x)}\right| \leq C_\infty \cdot \sqrt{\frac{2}{M}}. \label{eq:whole}
\end{gather}
Furthermore, if $|f(\cdot,\cdot)|\leq \Rmax$, then $\rho \in [e^{-2\Rmax/\beta}, e^{2\Rmax/\beta}]$ with probability 1.
The total number of sampling queries $y\sim\piref(\cdot\mid{}x)$ used
by the algorithm is at most $\Tsample=8 M \log(4\delta^{-1})+1$.
\end{theorem}
  
  We now state the guarantee for \rejection{}, which follows immediately from \cref{thm:rejection_density}.
\begin{theorem}[Guarantee for \rejection]
  \label{thm:rejection}
  Let $f:\cX\times\cY\to\bbR$, $x\in\cX$, and $\beta>0$ be given, and define
  \begin{align}
    \label{eq:cinf}
  \pif(y\mid{}x)\propto\piref(y\mid{}x)\exp\prn*{\beta^{-1}f(x,y)},\mathand    \Cinf\ldef{}\nrm*{\frac{\pif(\cdot\mid{}x)}{\piref(\cdot\mid{}x)}}_{\infty}.
  \end{align}
  Fix $\delta\in(0,1)$, and suppose that
  $M\geq{}4\Cinf$. There is an event $\cEaccept$ with
  $\bbP(\cEaccept)\geq{}1-\delta$ such that the response
  $y\sim{}\rejectionargs$ satisfies
  \begin{align}
    \bbP(y=\cdot\mid{}\cEaccept) = \pif(\cdot\mid{}x).
  \end{align}
The total number of sampling queries $y\sim\piref(\cdot\mid{}x)$ used
by the algorithm is at most $\Tsample = 8M\log(4\delta^{-1})+1$.
\end{theorem}

\paragraph{Further guarantees}
We now state some additional results, both of which are fairly
straightforward consequences of \cref{thm:rejection_density,thm:rejection}.

\begin{lemma}
  \label{lem:rejection_kl}
  Let $\pihatf(\cdot\mid{}x)$ denote the distribution over
  $y\sim\rejectionargs$. Suppose that $\abs*{f(x,y)}\leq\Rmax$. Then
  under the conditions of \cref{thm:rejection}, it holds that
  \[
    \Dtv{\pihatf(x)}{\pif(x)}\leq\delta,\quad
    \Dhels{\pihatf(x)}{\pif(x)}\leq2\delta,\mathand
    \Dkl{\pihatf(x)}{\pif(x)}
    \leq{} 4\prn*{\frac{\Rmax}{\beta} +\log{}N}\delta.
  \]
In addition, 
  \[
\Dkl{\pihatf(x)}{\piref(x)}
-     \Dkl{\pif(x)}{\piref(x)}
\leq{} 6\prn*{\frac{\Rmax}{\beta} +\log{}N}\delta
\]
and \[
  \frac{\pihatf(y\mid{}x)}{\pif(y\mid{}x)}
  \leq{} \exp(\Rmax/\beta)\cdot{}N.
  \]
\end{lemma}

\begin{lemma}
    \label{cor:rejection_density}
    Let $f:\cX\times\cY\to\bbR$, $x\in\cX$, and $\beta>0$ be given, and define
    \begin{align}
      \label{eq:cinf02}
    \pif(\cdot \mid{}x)\propto\piref(\cdot\mid{}x)\exp\prn*{\beta^{-1}f(x,\cdot)},\mathand    \Cinf\ldef{}\nrm*{\frac{\pif(\cdot\mid{}x)}{\piref(\cdot\mid{}x)}}_{\infty}.
\end{align}
    Fix $\delta\in(0,1)$, and suppose that
    $M\geq{}4\Cinf^2$. Consider a call to \arxiv{\[\softmaxsample_{\beta, M, \delta}(f;x,\piref).\]} Let $(y,\rho)$ denote its random output and let $\pihat_f(\cdot \mid x)$ denote the probability distribution of $y$. Then, we have 
\begin{align}
 \left|\E\left[\log \rho\right] - \E_{y'\sim \pihat_f(\cdot\mid x)}\left[\log \frac{\pihat_f(y' \mid x)}{\piref(y'\mid x)}\right] \right| \leq C_\infty \cdot \sqrt{\frac{2}{M}} + 4\prn*{\frac{2B}{\beta} +\log (4 M \log (4 \delta^{-1}))}\delta. \label{eq:whole2}
\end{align}
\end{lemma}

  Finally, we have the following change-of-measure guarantee.
  \begin{lemma}
    \label{lem:rejection_tv_average_new}
    For any function $g(x,y)\in\brk{0,1}$, let
    $\pihat(x)\ldef{}\softmaxsample_{\beta, M, \delta}(f;x,\piref)$
    denote the distribution over responses induced by
    \cref{alg:rejection_density}. Then for any $\rho \in \Delta(\cX)$,
    \begin{align*}
      \abs*{\En_{x\sim\rho,y\sim\pi_f(x)}\brk*{g(x,y)}-
      \En_{x\sim\rho,y\sim\pihat(x)}\brk*{g(x,y)}}%
      &\leq{} \delta +  \bbP_{x\sim\rho}\brk*{M < 4 \Ccond(\pi_f\mid x)},
    \end{align*}
    where $\Ccond(\pif\mid{}x)\ldef{}\sup_{y\in\cY}\frac{\pif(y\mid{}x)}{\piref(y\mid{}x)}$.
  \end{lemma}

\label{sec:rejection_density}

\subsection{Proofs}

\begin{proof}[\pfref{thm:rejection_density}]
    Our starting point is the
    following standard guarantee for rejection sampling.\loose
    \begin{lemma}[Rejection sampling \citep{block2023sample}]
      \label{lem:rejection}
      Let $\mu\in\Delta(\cY)$ be a proposal distribution, and let $\nu$
      denote a target distribution that we wish to sample from. Suppose that
      $\nrm*{\frac{\mathrm{d}\nu}{\mathrm{d\mu}}}_{\infty}\leq{}M$. Consider the
    algorithm which, for $i=1,2,\ldots,N$, samples $X_i\sim\mu$, then samples
    $\xi_i\in\crl{0,1}$ such that
    $\bbP(\xi_i=1\mid{}X_i)=\frac{1}{M}\frac{\mathrm{d}\nu}{\mathrm{d\mu}}$,
    and returns $X_i$ if $\xi_i=1$; we return $\perp$ if $\xi_i=0$ for all
    $i$. If $N\geq{}M\log(\delta^{-1}))$, then
    $\xi_i=1$ for some $i$ with probability at least $1-\delta$, and we
    have
    \[
    \bbP(X_i\in{}A\mid{}\xi_i=1)=\nu(A).
    \]
    \end{lemma}
In what follows, we omit dependence on $x$. Let
$Z\ldef{}\En_{y\sim\piref}\brk*{\exp(\beta^{-1}f(y))}$ denote the
``true'' normalization constant for $\pif$, and observe that we have
\begin{align}
  \label{eq:cinf_ratio0}
  \Cinf = \frac{\max_{y\in\cY}\exp(\beta^{-1}f(y))}{Z}.
\end{align}
We begin by giving a guarantee for the estimated normalization constant
$\Zhat$. We observe that by \cref{lem:mult_freedman}, there is an event $\cE$ of probability
at least $1-\delta/2$, under which
\begin{align*}
\Zhat
\leq{}
\frac{3}{2}Z +
      \frac{4 \max_{y\in\cY}\exp(\beta^{-1}f(y))\log(4\delta^{-1})}{N},
\shortintertext{and}
      Z\leq{} 2\Zhat + \frac{8 \max_{y\in\cY}\exp(\beta^{-1}f(y))\log(4\delta^{-1})}{N}.
\end{align*}
  Using \cref{eq:cinf_ratio0}, we can equivalently write this as
    \begin{align*}
    \Zhat\leq{}
                        \frac{3}{2}Z +
      \frac{4 \Cinf\log(4\delta^{-1})}{N}\cdot{}Z,
\mathand
      Z\leq{} 2\Zhat + \frac{8 \Cinf\log(4\delta^{-1})}{N}\cdot{}Z.
    \end{align*}
    It follows that as long as $N\geq{}16\Cinf\log(4\delta^{-1})$ (or
    $N\geq{}4M\log(4\delta^{-1})$ if $M\geq{}4\Cinf$), we
    have that
    \begin{align}
      \label{eq:zhat0}
      \frac{1}{4}Z \leq \Zhat \leq 2Z.
    \end{align}
    Let us condition on $\cE$ until otherwise stated. To proceed, we
    observe that the for loop in \cref{line:rejection_for_density} can be
    interpreted as applying the rejection sampling algorithm in
    \cref{lem:rejection} with $\mu=\piref(\cdot)$, $\nu=\pif(\cdot)$,
    and threshold
    \begin{align}
      M' \ldef{} M\cdot\frac{\Zhat}{Z}.
    \end{align}  
Hence, as long as $M'\geq{}
\nrm*{\frac{\pif(\cdot)}{\piref(\cdot)}}_{\infty}=\Cinf$ and
$N\geq{}M'\log(2\delta^{-1})$, with probability at least $1-\delta/2$,
there will be some $i$ such that $\xi_i=1$ and
$\bbP(y=\cdot\mid{}\xi_i=1)=\pif(\cdot)$. Note that under \cref{eq:zhat0},
we have
\[
M' = M\cdot\frac{\Zhat}{Z} \in\brk*{\frac{M}{4}, 2M}.
\]
so setting $M\geq{}4\Cinf$ and $N\geq{}2M\log(2\delta^{-1})$ suffices to prove that 
\begin{align}
\P(y = \cdot \mid \cE_\texttt{accept}) = \pi_f(\cdot \mid x),
\end{align}
under $\cE$. We now prove the second claim on the approximation of the density ratio.

\paragraphi{Estimating the density ratio} 
We no longer condition on $\cE$. First, note that 
\begin{align}
\frac{\pi_f(\cdot)}{\piref(\cdot)} = \frac{e^{\beta^{-1}f(\cdot)}}{Z}.
\end{align}
On the other hand, the output $\rho$ of \softmaxsample{} satisfies $\rho =\frac{e^{\beta^{-1}f(y)}}{\Zhat}$. Thus, to get our desired bound on $ 
\left|\log \rho - \log \frac{\pi_f(y)}{\piref(y)} \right|
$, it suffices to show that 
\begin{align}
   \left| \log \frac{\Zhat}{Z}\right| \leq 4 C_\infty \cdot \sqrt{\frac{2\log (4\delta^{-1})}{N}}.
\end{align}
 By \cref{lem:hoeffding} (Hoeffding inequality), there is an event $\cE'$ of probability
at least $1-\delta/2$, under which
\begin{align}
\left| \Zhat-  Z\right|
&\leq{} \max_{y\in\cY}\exp(\beta^{-1}f(y))\cdot\sqrt{\frac{8\log(4\delta^{-1})}{N}},\nn \\
& =  C_\infty Z\cdot\sqrt{\frac{8\log(4\delta^{-1})}{N}}. \label{eq:toarrange}
\end{align}
We now condition on $\cE'$.
Rearraning \eqref{eq:toarrange} and dividing by $Z$, we get that
\begin{align}
  \frac{\Zhat}{Z}  \leq 1 + C_\infty\cdot \sqrt{\frac{8 \log (4 \delta^{-1})}{N}},
  \mathand
  \frac{\Zhat}{Z}  \geq 1 - C_\infty\cdot  \sqrt{\frac{8 \log (4 \delta^{-1})}{N}}.
\end{align}
Therefore, using that $N \geq 16 C_\infty^2 \log (4 \delta^{-1})$ together with the facts that $\log (1+x)\leq x$ and $\log(1-x)\geq 1 - 2x$, for $x\in [0,1/2]$, we get 
\begin{align}
 - 2C_\infty \cdot \sqrt{\frac{8 \log (4 \delta^{-1})}{N}} \leq  \log \frac{\Zhat}{Z} \leq C_\infty \cdot \sqrt{\frac{8 \log (4 \delta^{-1})}{N}}.
\end{align}
This shows the desired bound on the log ratio $\log \frac{\Zhat}{Z}$ after plugging-in the choice $N = 4 M \log (4 \delta^{-1})$. Now, by the union bound, the probability of the event $\cE\cap \cE'$ is at least $1-\delta$. Thus, the event $\cEaccept= \cE \cap \cE'$ satisfies the desired properties.

Now, we no longer condition on $\cE'$. Note that $\rho$ is of the form
\begin{align}
  \frac{e^{\beta^{-1} f(y)}}{\frac{1}{N} \sum_{i=1}^N e^{\beta^{-1} f(y_i)}}.
\end{align}
Thus, when $|f(\cdot,\cdot)|\leq \Rmax$, we immediately have that
$\rho \in [e^{-2\Rmax/\beta},e^{2\Rmax/\beta}]$ as desired.
\end{proof}

\begin{proof}[\pfref{lem:rejection_kl}]
It is an immediate consequence of \cref{thm:rejection} that
$\Dtv{\pihatf(x)}{\pif(x)}\leq\delta$ and $\Dhels{\pihatf(x)}{\pif(x)}\leq2\delta$.
  We begin by writing
  \begin{align*}
    &\Dkl{\pihatf(x)}{\piref(x)}
    -     \Dkl{\pif(x)}{\piref(x)}\\
    &=     \En_{y\sim\pihatf(x)}\brk*{\log(\pif(y\mid{}x)/\piref(y\mid{}x))}
    - \En_{y\sim\pif(x)}\brk*{\log(\pif(y\mid{}x)/\piref(y\mid{}x))}\\
&~~~~    +
      \En_{y\sim\pihatf(x)}\brk*{\log(\pihatf(y\mid{}x)/\pif(y\mid{}x))}\\
    &\leq \frac{2\Rmax\delta}{\beta}
    + \En_{y\sim\pihatf(x)}\brk*{\log(\pihatf(y\mid{}x)/\pif(y\mid{}x))},
  \end{align*}
  where the inequality uses that
  $\abs*{\log(\pif(y\mid{}x)/\piref(y\mid{}x))}\leq\Rmax/\beta$. To
  proceed, we bound
  \begin{align*}
    \Dkl{\pihatf(x)}{\pif(x)} &=
                                \En_{y\sim\pihatf(x)}\brk*{\log(\pihatf(y\mid{}x)/\pif(y\mid{}x))}
  \end{align*}
  We first note that %
  \begin{align*}
    \frac{\pihatf(y\mid{}x)}{\pif(y\mid{}x)}
    \leq\exp(\Rmax/\beta)
    \cdot\frac{\pihatf(y\mid{}x)}{\piref(y\mid{}x)}
    \leq{} \exp(\Rmax/\beta)\cdot{}N.
  \end{align*}
  The latter inequality is a standard property of rejection sampling:
  If we let $\cY_N$ denote the set of responses the algorithm
  considers accepting, then we have $\pihatf(y\mid{}x)
  = \En_{\cY_N}\brk*{\pihatf(y\mid{}x,\cY_N)}
  \leq{} \En_{\cY_N}\brk*{\indic\crl*{y\in\cY_N}}\leq{}
  N\cdot\piref(y\mid{}x)$. From here, it follows from Lemma A.10 of
  \citet{foster2021statistical} that
  $    \Dkl{\pihatf(x)}{\pif(x)}\leq{} 2(\frac{\Rmax}{\beta} +\log{}N)
  \Dhels{\pihatf(x)}{\pif(x)}
  \leq{} 4(\frac{\Rmax}{\beta} +\log{}N)\delta$.

\end{proof}

\begin{proof}[\pfref{cor:rejection_density}]
Let $(y,\rho)$ be the random output of \softmaxsample.
    First, by Jensen's inequality, we have that
    \begin{align}
\left|\E\left[\log \rho\right] - \E\left[\log \frac{\pi_f(y \mid x)}{\piref(y\mid x)}\right] \right| 
    & \leq \E\left[\left|\log \rho-\log \frac{\pi_f(y\mid x)}{\piref(y\mid x)}\right|\right], \nn \\
\intertext{and so, by letting $\cEaccept$ be the event $\cEaccept$ in \cref{thm:rejection_density}, we have}
& = \E\left[ \mathbb{I}\{\cEaccept\} \cdot \left|\log \rho-\log \frac{\pi_f(y\mid x)}{\piref(y\mid x)}\right|\right]\nn \\
& \quad + \E\left[(1-\mathbb{I}\{\cEaccept\})\cdot \left|\log \rho-\log \frac{\pi_f(y\mid x)}{\piref(y\mid x)}\right|\right], \nn \\
& \leq C_\infty \cdot \sqrt{\frac{2}{M}} + \frac{4 B \delta}{\beta}, \label{eq:needle}
\end{align}
where the last inequality follows from \cref{thm:rejection_density} and that $\rho$ and $\frac{\pi_f(y\mid x)}{\piref(y\mid x)}$ are in $[e^{2 B/\beta}, e^{-2 B/\beta}]$. 

On the other hand, we have that 
\begin{align}
    \E\left[\log \frac{\pi_f(y \mid x)}{\piref(y\mid x)}\right] & = \E_{y \sim \pihat_f(\cdot\mid x)}\left[\log \frac{\pi_f(y \mid x)}{\piref(y\mid x)}\right]\\
  &= \E_{y \sim \pihat_f(\cdot\mid x)}\left[\log \frac{\pihat_f(y \mid x)}{\piref(y\mid x)}\right] - \kl{\pihat_f(\cdot \mid x)}{\pi_f(\cdot \mid x)}.
\end{align}
Now, by \cref{lem:rejection_kl}, we have that $\kl{\pihat_f(\cdot \mid x)}{\pi_f(\cdot \mid x)} \leq 4\prn*{\frac{\Rmax}{\beta} +\log{}N}\delta$. Combining this with \eqref{eq:needle} and the triangle inequality, we get the desired result.

\end{proof}

    \begin{proof}[\pfref{lem:rejection_tv_average_new}]
      By Jensen's inequality and the triangle inequality, we have that
      \begin{align}
     &   \abs*{\En_{x\sim\rho,y\sim\pi_f}\brk*{g(x,y)}-
        \En_{x\sim\rho,y\sim\pihat}\brk*{g(x,y)}} \nn \\
       & \leq  \En_{x\sim\rho}\brk*{\abs*{\En_{y\sim\pi_f(\cdot\mid{}x)}\brk*{g(x,y)}-
        \En_{y\sim\pihat(\cdot\mid{}x)}\brk*{g(x,y)}}\cdot \indic\crl*{M \geq  4 \Ccond(\pi_f\mid x)}} \nn \\
        & \quad + \P_{x\sim \rho}\left[M < 4 \Ccond(\pi_f\mid x)\right],\nn \\
        & \leq \delta + \P_{x\sim \rho}\left[M < 4 \Ccond(\pi_f\mid x)\right],
      \end{align}
      where the last inequality follows from the fact that 
  \begin{align}
    \Dtv{\pihat(\cdot\mid{}x)}{\pi_f(\cdot\mid{}x)}\leq\delta,\nn 
  \end{align}
  for all $x$ such that $M \geq 4 \Ccond(\pi_f\mid x)$,
  thanks to \cref{thm:rejection}. This completes the proof.
    \end{proof}

\newpage
\arxiv{\section{Proofs from \creftitle{sec:algorithms}}}
\label{app:algorithms}

\arxiv{This section is dedicated to proving the main guarantee for
  \spanalg, \cref{thm:spanner}. \cref{sec:spanner_lemmas} presents standard
  technical lemmas, and \cref{sec:spanner_regret} presents our central
  regret decomposition for truncated softmax policies. Finally, in
  \cref{sec:spanner_proof} we combine these results to prove \cref{thm:spanner}.
  }

  \subsection{Technical Lemmas}
  \label{sec:spanner_lemmas}

\subsubsection{Basic Results}
  
\begin{lemma}[Differences in rewards are linear]
  \label{lem:reward_difference}
  If \cref{ass:realizable} holds, then for all $x\in\cX$ and $y,y'\in\cY$, 
\begin{align}
  \rstar(x,y) - \rstar(x,y') = \tri*{\thetastar,\phi(x,y) - \phi(x,y')}.
\end{align}
\end{lemma}
\begin{proof}[\pfref{lem:reward_difference}]
  If \cref{ass:realizable} holds, then for all $x\in\cX$,
  $\pistarb(y\mid{}x)=\pist(y\mid{}x)$, where
  $\pistarb(y\mid{}x)\propto\piref(y\mid{}x)\exp\prn*{\beta^{-1}\rstar(x,y)}$
  is the optimal KL-regularized policy. Taking logarithms, this
  implies that for all $x\in\cX$, $y\in{}\cY$, 
  \[
    \beta\log\frac{\pistarb(y\mid{}x)}{\piref(y\mid{}x)}
    = \rstar(x,y) - \log Z_{\rstar}(x)
    =     \beta\log\frac{\pist(y\mid{}x)}{\piref(y\mid{}x)}
    = \tri*{\thetastar,\phi(x,y)} - \log Z_{\thetastar}(x),
  \]
  where $Z_{\rstar}(x)\ldef
  \En_{y\sim\piref(x)}\brk*{\exp\prn*{\beta^{-1}\rstar(x,y)}}$ and
  $Z_{\thetastar}(x)\ldef
  \En_{y\sim\piref(x)}\brk*{\exp\prn*{\beta^{-1}\tri*{\thetastar,\phi(x,y)}}}$. Picking
  any $y,y'\in\cY$ and take the difference then implies that
  \[
  \rstar(x,y) - \rstar(x,y') = \tri*{\thetastar,\phi(x,y) - \phi(x,y')}.
  \]
as claimed.
\end{proof}

\begin{lemma}[Density ratio bound for softmax policies]
  \label{lem:softmax_rmax}
  For a function $f:\cX\times\cY\to\bbR$, let
\begin{align}
\pi_f(y\mid{}x)\propto\piref(y\mid{}x)\exp(\beta^{-1}f(x,y)),
\end{align}
Then for all $x\in\cX$, it holds that
\begin{align}
  \nrm*{\frac{\pif(\cdot\mid{}x)}{\piref(\cdot\mid{}x)}}_{\infty}
  \leq{} \exp\prn*{\beta^{-1}\prn*{\max_{y\in\cY}f(x,y)
  - \En_{y\sim\piref(x)}\brk*{f(x,y)}}
  }.
\end{align}
  
\end{lemma}
\begin{proof}[\pfref{lem:softmax_rmax}]
  For any $y\in\cY$, we can use Jensen's inequality to bound
  \begin{align*}
    \frac{\pif(y\mid{}x)}{\piref(y\mid{}x)}
    &=
      \frac{\exp\prn*{\beta^{-1}f(x,y)}}{\En_{y'\sim\piref(x)}\brk*{\exp\prn*{\beta^{-1}f(x,y')}}}\\
    &=
      \frac{\exp\prn*{\beta^{-1}(f(x,y)-\En_{y''\sim\piref(x)}\brk*{f(x,y'')}}}{\En_{y'\sim\piref(x)}\brk*{\exp\prn*{\beta^{-1}(f(x,y')-\En_{y''\sim\piref(x)}\brk*{f(x,y')}}}}\\
        &\leq{}
          \frac{\exp\prn*{\beta^{-1}(f(x,y)-\En_{y''\sim\piref(x)}\brk*{f(x,y'')}}}{\exp\prn*{\beta^{-1}(\En_{y'\sim\piref(x)}\brk*{
          f(x,y')}-\En_{y''\sim\piref(x)}\brk*{f(x,y')}}}\\
            &=\exp\prn*{\beta^{-1}(f(x,y)-\En_{y''\sim\piref(x)}\brk*{f(x,y'')}}.
  \end{align*}
as claimed.
\end{proof}

\subsubsection{Guarantees for Least Squares}
The following result presents a standard guarantee for least
squares with dependent data.
\begin{lemma}
  \label{lem:least_squares}
  Consider a sequentially generated dataset $\crl*{(x\ind{t},
    y_1\ind{t}, y_2\ind{t}, r_1\ind{t}, r_2\ind{t})}_{t\in\brk{T}}$ in
  which for all $t$,
  \[
    \En\brk*{r_1\ind{t}- r_2\ind{t}\mid{} x\ind{t}, y_1\ind{t},
      y_2\ind{t}, \filt\ind{t-1}}
    = \tri*{\thetastar,\phi(x\ind{t}, y_1\ind{t})-\phi(x\ind{t}, y_2\ind{t})},
  \]
  where $\filt\ind{t-1}\ldef{}\sigma\prn*{\crl*{(x\ind{i},
    y_1\ind{i}, y_2\ind{i}, r_1\ind{i}, r_2\ind{i})}_{i<t}}$.
  Define the least-squares estimator
  \[
      \theta\ind{t}=\argmin_{\theta\in\Theta}\sum_{i<t} \prn*{
        \tri*{\phi(x\ind{i}, y_1\ind{i})-\phi(x\ind{i}, y_2\ind{i}),\theta}
        -(r\ind{i}_1-r\ind{i}_2)}^2,
    \]
    and let
    $\Sigma\ind{t}\ldef{}\sum_{i<t}(\phi(x\ind{i},y\ind{i}_1)-\phi(x\ind{i},
    y\ind{i}_2)) (\phi(x\ind{i},y\ind{i}_1)-\phi(x\ind{i},
    y\ind{i}_2))^{\trn}
    $.
Assume that $r_1\ind{t}, r_2\ind{t}\in\brk{0, \Rmax}$ almost surely,
that $\thetastar\in\Theta$, and that \cref{ass:norm} holds with
parameter $B$. Define $\lambda=\frac{\Rmax^2}{B^2}$. Then with probability at least $1-\delta$, for all
$t\in\brk{T}$,
\begin{align*}
  \nrm*{\theta\ind{t}-\thetastar}_{\Sigma\ind{t}+\lambda{}\Id}^2
  \leq{} \bigoh(d\Rmax^2\log(B\Rmax^{-1}\delta^{-1}T)).
\end{align*}
\end{lemma}
\begin{proof}[\pfref{lem:least_squares}]
  By a standard concentration result for well-specified regression (e.g., Lemma 39 in
  \citet{jin2021bellman}), we have that with probability at least
  $1-\delta$, for all $t\in\brk{T}$,
  \begin{align*}
\nrm*{\theta\ind{t}-\thetastar}_{\Sigma\ind{t}}^2 =
    \sum_{i<t}\tri*{\theta\ind{t}-\thetastar, \phi(x\ind{i},y_1\ind{i})-\phi(x\ind{i},y_2\ind{i})}^2
    \approxleq\bigoh(d\Rmax^2\log(B\Rmax^{-1}\delta^{-1}T)).
  \end{align*}
  By \cref{ass:norm} and choice of $\lambda$, we have
  \begin{align*}
    \nrm*{\theta\ind{t}-\thetastar}_{\Sigma\ind{t}+\lambda\Id}^2
    = \nrm*{\theta\ind{t}-\thetastar}_{\Sigma\ind{t}}^2
    + \lambda\nrm*{\theta\ind{t}-\thetastar}^2
    \leq \nrm*{\theta\ind{t}-\thetastar}_{\Sigma\ind{t}}^2
    + 4\lambda{}B^2
    \leq \nrm*{\theta\ind{t}-\thetastar}_{\Sigma\ind{t}}^2
    + 4\Rmax^2,
  \end{align*} 
  and combining with the preceding bound completes the proof.
\end{proof}

\subsubsection{Elementary Properties of KL-Regularized Regret}
We now state some generic properties of the
KL-regularized regret. Suppose the true reward is $\fstar(x,y)$ for an arbitrary
function $\fstar$, and let
\[
  \Jbeta(\pi\midsem{}\fstar, x) = \En_{y\sim{}\pi(x)}\brk*{\fstar(x,y)}
  -\beta\Dkl{\pi(x)}{\piref(x)}.
\]
For a function $f$, let
\begin{align*}
\pi_f(y\mid{}x)\propto\piref(y\mid{}x)\exp(\beta^{-1}f(x,y)),
\end{align*}
and let
$Z_f(x)\ldef{}\En_{y\sim\piref(\cdot\mid{}x)}\brk*{\exp\prn*{\beta^{-1}f(x,y)}}$
denote the normalization constant. The following result follows from elementary manipulations.
\begin{lemma}
  \label{lem:kl_regret}
  For all $f:\cX\times\cY\to\bbR$ and $x\in\cX$, it holds that
  \begin{align*}
\notag    &J_\beta(\pi_{\fstar}\midsem{}\fstar, x) -
            J_\beta(\pi_f\midsem{}\fstar, x) \\&=
    \beta\cdot\Dkl{\pi_f(x)}{\pi_{\fstar}(x)}
    \\
    &= \beta\log(Z_{\fstar}(x)/Z_{f}(x)) +
      \En_{y\sim\pif(x)}\brk*{f(x,y)-\fstar(x,y)}.\\
                                                             &=
                                                               \beta\log\prn*{\En_{y\sim\pif(x)}\exp\prn*{\beta^{-1}(\fstar(x,y)-f(x,y))}}
                                                               +
                                                               \En_{y\sim\pif(x)}\brk*{f(x,y)-\fstar(x,y)}.
\end{align*}
  
\end{lemma}
\begin{proof}[\pfref{lem:kl_regret}]
    We can directly calculate that
  \begin{align*}
    \label{eq:first_step}
    \beta\Dkl{\pif}{\pifstar}
    = \beta\En\brk*{\log(Z_{\fstar}(x)/Z_{f}(x))} + \En_{y\sim\pif(x)}\brk*{f(x,y)-\fstar(x,y)},
  \end{align*}
  where
  $Z_f(x)\ldef{}\En_{y\sim\piref(x)}\brk*{\exp(\beta^{-1}f(x,y))}$.
  The first identity now follows by noting that for any $f$,
  \begin{align*}
    J_\beta(\pi_{f}\midsem{}\fstar, x)
    = \En_{y\sim\pif(x)}\brk*{\fstar(x,y)}
    - \En_{y\sim\pif(x)}\brk*{f(x,y)}
     + \beta\log(Z_f(x)).
  \end{align*}
  We
  finally observe that 
  \begin{align*}
    \log(Z_{\fstar}(x)/Z_{f}(x))
    = \log\prn*{\En_{y\sim\pif(x)}\brk*{\exp(\beta^{-1}(\fstar(x,y)-f(x,y)))}}
  \end{align*}
which completes the proof.
\end{proof}

\subsection{KL-Regularized Regret Decomposition for Truncated
  Softmax Policies}
\label{sec:spanner_regret}

This section gives tight bounds on the KL-regularized
regret for truncated softmax policies of the type used in \spanalg and
formally defined below. The main 
results, \cref{lem:truncated_density,lem:truncated_regret_simple},
allow for fast $1/\veps$-type rates by exploiting regularization, as
well as efficient rejection sampling.

\paragraph{Truncated softmax policies} Let $\Sigma\psdgt{}0$ be a given matrix and
$\nu>0$ be a parameter. Recalling that $\phidel(x,y,y') := \phi(x,y)-\phi(x,y')$, define a \emph{truncated} feature map by
\[
  \phidb(x,y,y')=\phidel(x,y,y')\indic\crl*{\nrm*{\phidb(x,y,y')}_{\Sigma^{-1}}\leq\nu}.
\]
For a parameter $\theta\in\bbR^{d}$ we will define a \emph{truncated
  softmax policy}
$\pibtheta(y,y'\mid{}x)$ as follows. Fixing $x\in\cX$, first define
$\pibtheta(y'\mid{}x)=\piref(y'\mid{}x)$. Next, define
\[
\pibtheta(y\mid{}x,y') \propto \piref(y\mid{}x)\exp\prn*{\beta^{-1}\tri*{\theta,\phidb(x,y,y')}}.
\]
We use $\pibtheta(y\mid{}x)=\En_{y'\sim\piref(\cdot\mid{}x)}\brk*{\pibtheta(y\mid{}x,y')}$
to denote the marginal over $y$ given $x$. We will overload notation
slightly and use $\Jbeta(\pibtheta)$ to denote the KL-regularized
regret of the marginalized policy $\pibtheta(y\mid{}x)$. We also define
\[
   \Jbbeta(\pibtheta)
   = \En_{x\sim\rho,y'\sim\piref(\cdot\mid{}x), y\sim\pibtheta(\cdot\mid{}x,y')}\brk*{
     \rstar(x,y)
     -\beta\Dkl{\pibtheta(\cdot\mid{}x,y')}{\piref(\cdot\mid{}x)}
     }
   \]
   and, for any $x\in\cX$,
   \[
   \Jbbeta(\pibtheta; x)
   = \En_{y'\sim\piref(\cdot\mid{}x), y\sim\pibtheta(\cdot\mid{}x,y')}\brk*{
     \rstar(x,y)
     -\beta\Dkl{\pibtheta(\cdot\mid{}x,y')}{\piref(\cdot\mid{}x)}
     }.
   \]

We first give a bound on the regret of the truncated softmax policies
that scales with (i) the squared estimation error (allowing for fast
$1/\veps$-type rates), and (ii) truncation probability under responses
drawn from $\piref$ and $\pist$.\loose
\begin{lemma}[Basic regret decomposition for truncated softmax policies]
  \label{lem:truncated_regret}
  Fix $x\in\cX$ and define $  \vepsstat^2 \ldef \nrm*{\theta-\thetastar}_{\Sigma}^2$.
Then under \cref{ass:realizable,ass:norm}, if $\nu\leq\beta/\vepsstat$, we have
\begin{align*}
  \Jbeta(\pist; x) - \Jbeta(\pibtheta; x)
  &\leq\Jbeta(\pist; x) - \Jbbeta(\pibtheta; x)\\
  &\leq\beta^{-1}\En_{(y,y')\sim\pibtheta(x)}\brk*{\tri*{\thetastar-\theta,\phidb(x,y,y')}^2}\\
  &~~~~+ \Rmax\Big(\bbP_{y\sim\pist(x),y'\sim\piref(x)}\brk*{\nrm*{\phidel(x,y,y')}_{\Sigma^{-1}}>\nu}
    \\
    &\qquad\qquad+ \bbP_{(y,y')\sim\pibtheta(x)}\brk*{\nrm*{\phidel(x,y,y')}_{\Sigma^{-1}}>\nu}\Big).
\end{align*}
\end{lemma}
Next, we show that the density ratio between $\pibar_{\theta}$ and
$\piref$ can be bounded by the optimal density ratio
$\Ccov(\pist)$ for $\pist$.
\begin{lemma}[Density ratio bound for truncated softmax policies]
  \label{lem:truncated_density}
  Fix $x\in\cX$ and $y'\in\cY$. Define
\[
  \vepsspan(x,y')\ldef{} \bbP_{y\sim\pist(\cdot\mid{}x)}\brk*{\nrm*{\phidel(x,y,y')}_{\Sigma^{-1}}>\nu},\mathand
  \vepsstat^2 \ldef \nrm*{\theta-\thetastar}_{\Sigma}^2.
\]
Suppose
  $\nu\leq{}\beta/\vepsstat$ and
  $\vepsspan(x,y')\leq{}1/2$. Then for all $y\in\cY$,
  \[
    \pibar_{\theta}(y\mid{}x,y')
    \leq{}2e^2\Ccov(\pist)\cdot\piref(y\mid{}x).
  \]
\end{lemma}
Finally \cref{lem:truncated_regret} and \cref{lem:truncated_density}
gives our main regret decomposition for truncated softmax policies.
\begin{lemma}[Main regret decomposition for truncated softmax policies]
  \label{lem:truncated_regret_simple}
  Define $  \vepsstat^2 \ldef \nrm*{\theta-\thetastar}_{\Sigma}^2$, and for any $\veps>0$, define
  \[
  \cXspan(\veps) \ldef
\crl*{x\in\cX\mid{}\bbP_{(y,y')\sim\piref(x)}\brk*{\nrm*{\phidel(x,y,y')}_{\Sigma^{-1}}>\nu} \leq\veps}.
  \]
Then under \cref{ass:realizable,ass:norm}, if $\nu\leq\beta/\vepsstat$, \arxiv{we have}\loose
\begin{align*}
  &\Jbeta(\pist) - \Jbeta(\pibtheta)
  \leq\Jbeta(\pist) - \Jbbeta(\pibtheta)\\
  &\leq
    \beta^{-1}\En_{(y,y')\sim\pibtheta(x)}\brk*{\tri*{\thetastar-\theta,\phidb(x,y,y')}^2}
        + 18\Rmax\Ccov(\pist)\cdot\veps
        + 2\Rmax\bbP\brk*{x\notin\cXspan(\veps)}.
\end{align*}
  
\end{lemma}

\subsubsection{Proof of Lemmas \ref*{lem:truncated_regret} through \ref*{lem:truncated_regret_simple}}

\begin{proof}[\pfref{lem:truncated_regret}]
To keep notation compact, in this proof we omit all dependence on the fixed
$x\in\cX$ (so that below, $J_\beta(\pist)$ refers to $J_\beta(\pist;x)$, $\pist$ refers to $\pist(x)$, etc.). We have
  \begin{align*}
    &\Jbeta(\pist) - \Jbeta(\pibtheta)\\
      &= \En_{y\sim\pist}\brk*{\rstar(y)}
    - \En_{y\sim\pibtheta}\brk*{\rstar(y)}
        - \beta\Dkl{\pist}{\piref} + \beta\Dkl{\pibtheta}{\piref}\\
    &\leq{} \En_{y'\sim\piref}\brk*{\En_{y\sim\pist}\brk*{\rstar(y)}
    - \En_{y\sim\pibtheta(\cdot\mid{}y')}\brk*{\rstar(y)}
        - \beta\Dkl{\pist}{\piref} +
      \beta\Dkl{\pibtheta(\cdot\mid{}y')}{\piref}},\\
    &=\Jbeta(\pist) - \Jbbeta(\pibtheta)
  \end{align*}
since, by convexity of KL-divergence,
\[\Dkl{\pibtheta}{\piref}=\Dkl{\En_{y'\sim\piref}\brk*{\pibtheta(\cdot\mid{}y')}}{\piref}
\leq{}
\En_{y'\sim\piref}\brk*{\Dkl{\pibtheta(\cdot\mid{}y')}{\piref}}.\] Under
\cref{ass:realizable}, we
can further write the quantity above as
\begin{align}
&  \En_{y'\sim\piref}\brk*{\En_{y\sim\pist}\brk*{\rstar(y)}
    - \En_{y\sim\pibtheta(\cdot\mid{}y')}\brk*{\rstar(y)}
        - \beta\Dkl{\pist}{\piref} +
                 \beta\Dkl{\pibtheta(\cdot\mid{}y')}{\piref}}\notag\\
  &= \En_{y'\sim\piref}\brk*{\En_{y\sim\pist}\brk*{\rstar(y)-\rstar(y')}
    - \En_{y\sim\pibtheta(\cdot\mid{}y')}\brk*{\rstar(y)-\rstar(y')}
        - \beta\Dkl{\pist}{\piref} +
    \beta\Dkl{\pibtheta(\cdot\mid{}y')}{\piref}}\notag\\
    &= \En_{y'\sim\piref}\brk*{\En_{y\sim\pist}\brk*{\tri*{\thetast,\phidel(y,y')}}
    - \En_{y\sim\pibtheta(\cdot\mid{}y')}\brk*{\tri*{\thetast,\phidel(y,y')}}
        - \beta\Dkl{\pist}{\piref} +
      \beta\Dkl{\pibtheta(\cdot\mid{}y')}{\piref}}\notag\\
      &\leq \En_{y'\sim\piref}\brk*{\En_{y\sim\pist}\brk*{\tri*{\thetast,\phidb(y,y')}}
    - \En_{y\sim\pibtheta(\cdot\mid{}y')}\brk*{\tri*{\thetast,\phidb(y,y')}}
        - \beta\Dkl{\pist}{\piref} +
        \beta\Dkl{\pibtheta(\cdot\mid{}y')}{\piref}}\notag\\
  &~~~~+\Rmax\prn*{\bbP_{y\sim\pist,y'\sim\piref}\brk*{\nrm*{\phidel(y,y')}_{\Sigma^{-1}}>\nu}
    + \bbP_{(y,y')\sim\pibtheta}\brk*{\nrm*{\phidel(y,y')}_{\Sigma^{-1}}>\nu}} \notag\\ 
    &\leq \En_{y'\sim\piref}\big[\En_{y\sim\pibst(\cdot\mid{}y')}\brk*{\tri*{\thetast,\phidb(y,y')}}
    - \En_{y\sim\pibtheta(\cdot\mid{}y')}\brk*{\tri*{\thetast,\phidb(y,y')}}
        \notag\\
        &\qquad\qquad- \beta\Dkl{\pibst(\cdot\mid{}y')}{\piref} +
        \beta\Dkl{\pibtheta(\cdot\mid{}y')}{\piref}\big]\notag\\
  &~~~~+\Rmax\prn*{\bbP_{y\sim\pist,y'\sim\piref}\brk*{\nrm*{\phidel(y,y')}_{\Sigma^{-1}}>\nu}
    + \bbP_{(y,y')\sim\pibtheta}\brk*{\nrm*{\phidel(y,y')}_{\Sigma^{-1}}>\nu}}, \label{eq:trunc_step0}
\end{align}
where the first inequality is by definition of $\phidb$, and the second is because for any fixed $y'$, we have
\begin{align*}
&  \En_{y\sim\pist}\brk*{\tri*{\thetast,\phidb(y,y')}}
                 - \beta\Dkl{\pist}{\piref}\\
  &\leq{} \max_{\pi:\cX\to\Delta(\cY)}\crl*{\En_{y\sim\pi}\brk*{\tri*{\thetast,\phidb(y,y')}}
    - \beta\Dkl{\pi}{\piref}}\\
  &=  \En_{y\sim\pibst(\cdot\mid{}y')}\brk*{\tri*{\thetast,\phidb(y,y')}}
                 - \beta\Dkl{\pibst(\cdot\mid{}y')}{\piref}.
\end{align*}
With this upper bound, for any fixed $y'$, we can interpret the
quantity 
\begin{align*}
  \En_{y\sim\pibst(\cdot\mid{}y')}\brk*{\tri*{\thetast,\phidb(y,y')}}
    - \En_{y\sim\pibtheta(\cdot\mid{}y')}\brk*{\tri*{\thetast,\phidb(y,y')}}
        - \beta\Dkl{\pibst(\cdot\mid{}y')}{\piref} +
        \beta\Dkl{\pibtheta(\cdot\mid{}y')}{\piref}
\end{align*}
in \cref{eq:trunc_step0} as the KL-regularized regret of $\pibtheta(\cdot\mid{}y')$ to
$\pibst(\cdot\mid{}y')$ under the reward $\tri*{\thetastar,\phidel(y,
  y')}$. Consequently, \cref{lem:kl_regret} allows us to bound this regret by
\begin{align}
  \beta\log\prn*{\En_{y\sim\pibtheta(\cdot\mid{}y')}\exp\prn*{\beta^{-1}\tri*{\thetastar-\theta,\phidb(y,y')}}}
                                                               +
                                                               \En_{y\sim\pibtheta(\cdot\mid{}y')}\brk*{\tri*{\theta-\thetastar,\phidb(y,y')}}.
\label{eq:beta-exp-lin}
\end{align}
  Note
  that for all $y,y'\in\cY$, under the condition on $\nu$ in the
  lemma statement,
  \begin{align}
    \label{eq:dot_bound}
    \abs*{\tri*{\theta-\thetastar,\phidb(y,y')}}
    \leq{} \nrm*{\phidb(y,y')}_{\Sigma^{-1}}\nrm*{\theta-\thetastar}_{\Sigma}\leq\nu\vepsstat\leq\beta,
  \end{align}
  since $\phidb(y,y')=\mb{0}$ if
  $\nrm*{\phidel(y,y')}_{\Sigma^{-1}}\leq\nu$ does not hold. Hence, using that $e^{z}\leq{}1+z+z^2$ for all $z\leq{}1$, we have
  \begin{align*}
&\beta\log\prn*{\En_{y\sim\pibtheta(\cdot\mid{}y')}\brk*{\exp\prn*{\beta^{-1}\tri*{\thetastar-\theta,\phidb(y,y')}}}}
    \\
    &\leq{}\beta\log\prn*{1 +
      \beta^{-1}\En_{y\sim\pibtheta(\cdot\mid{}y')}\brk*{\tri*{\thetastar-\theta,\phidb(y,y')}}+\beta^{-2}\En_{y\sim\pibtheta(\cdot\mid{}y')}\brk*{\tri*{\thetastar-\theta,\phidb(y,y')}^2}}\\
&\leq{}\En_{y\sim\pibtheta(\cdot\mid{}y')}\brk*{\tri*{\thetastar-\theta,\phidb(y,y')}}+\beta^{-1}\En_{y\sim\pibtheta(\cdot\mid{}y')}\brk*{\tri*{\thetastar-\theta,\phidb(y,y')}^2},
  \end{align*}
  which we can substitute into the bound from \cref{eq:beta-exp-lin} (cancelling out the linear term) to get the bound
  \begin{align*}
      &\En_{y\sim\pibst(\cdot\mid{}y')}\brk*{\tri*{\thetast,\phidb(y,y')}}
    - \En_{y\sim\pibtheta(\cdot\mid{}y')}\brk*{\tri*{\thetast,\phidb(y,y')}}
        - \beta\Dkl{\pibst(\cdot\mid{}y')}{\piref} +
        \beta\Dkl{\pibtheta(\cdot\mid{}y')}{\piref}\\
      &\leq \beta^{-1}\En_{y\sim\pibtheta(\cdot\mid{}y')}\brk*{\tri*{\thetastar-\theta,\phidb(y,y')}^2}.
  \end{align*}
  Since this holds uniformly for all $y'\in\cY$, returning to \cref{eq:trunc_step0}, we conclude that
  \begin{align*}
    \Jbeta(\pist) - \Jbeta(\pibtheta)
    &\leq{}
    \beta^{-1}\En_{(y,y')\sim\pibtheta}\brk*{\tri*{\thetastar-\theta,\phidb(y,y')}^2}\\
    &~~~~+ \Rmax\prn*{\bbP_{y\sim\pist,y'\sim\piref}\brk*{\nrm*{\phidel(y,y')}_{\Sigma^{-1}}>\nu}
    + \bbP_{(y,y')\sim\pibtheta}\brk*{\nrm*{\phidel(y,y')}_{\Sigma^{-1}}>\nu}}
  \end{align*}
  as claimed.
\end{proof}

\begin{proof}[\pfref{lem:truncated_density}]
  To keep notation compact, we again omit all dependence on $x$. Fix
  $y'\in\cY$. Then we have
  \begin{align*}
    \frac{\pibar_{\theta}(y\mid{}y')}{\piref(y)}
    = \frac{\exp\prn*{\beta^{-1}\tri*{\theta,\phidb(y,y')}}}{Z_{\pibtheta}(y')},
  \end{align*}
  where
  $Z_{\pibtheta}(y')\ldef{}\En_{y\sim\piref}\brk*{\exp\prn*{\beta^{-1}\tri*{\theta,\phidb(y,y')}}}$. Note
  that for all $y,y'\in\cY$, under the conditions in the lemma
  statement, we have
  \begin{align}
    \label{eq:dot_bound-2}
    \abs*{\tri*{\theta-\thetastar,\phidb(y,y')}}
    \leq{} \nrm*{\phidb(y,y')}_{\Sigma^{-1}}\nrm*{\theta-\thetastar}_{\Sigma}\leq\nu\vepsstat\leq\beta.
  \end{align}
We begin by giving a lower bound on the normalization constant
$Z_{\pibtheta}(y')$. Observe that
  \begin{align*}
    Z_{\pibtheta}(y')&\ldef{}\En_{y\sim\piref}\brk*{\exp\prn*{\beta^{-1}\tri*{\theta,\phidb(y,y')}}}\\
                 &\geq{}\En_{y\sim\piref}\brk*{\exp\prn*{\beta^{-1}\tri*{\theta,\phidb(y,y')}}\indic\crl*{\nrm*{\phidel(y,y')}_{\Sigma^{-1}}\leq\nu}}\\
                 &\geq{}e^{-1}\En_{y\sim\piref}\brk*{\exp\prn*{\beta^{-1}\tri*{\thetastar,\phidb(y,y')}}\indic\crl*{\nrm*{\phidel(y,y')}_{\Sigma^{-1}}\leq\nu}}\\
                     &=e^{-1}\En_{y\sim\piref}\brk*{\exp\prn*{\beta^{-1}\tri*{\thetastar,\phidel(y,y')}}\indic\crl*{\nrm*{\phidel(y,y')}_{\Sigma^{-1}}\leq\nu}},\\
&=e^{-1}\En_{y\sim\piref}\brk*{\exp\prn*{\beta^{-1}\tri*{\thetastar,\phi(y)}}\indic\crl*{\nrm*{\phidel(y,y')}_{\Sigma^{-1}}\leq\nu}}
                   \cdot\exp\prn*{-\beta^{-1}\tri*{\thetastar,\phi(y')}},
  \end{align*}
  where the second inequality uses \cref{eq:dot_bound-2} and the second-to-last
  equality uses the definition of the indicator. Now, define
  \[
    Z_{\pi_{\thetastar}}=\En_{y\sim\piref}\brk*{\exp\prn*{\beta^{-1}\tri*{\thetastar,\phi(y)}}}
  \]
  as the normalization constant for $\pi_{\thetastar}$. We can write
  \begin{align*}
    &\En_{y\sim\piref}\brk*{\exp\prn*{\beta^{-1}\tri*{\thetastar,\phi(y)}}\indic\crl*{\nrm*{\phidel(y,y')}_{\Sigma^{-1}}\leq\nu}}\\
    &=    Z_{\pi_{\thetastar}} - \En_{y\sim\piref}\brk*{\exp\prn*{\beta^{-1}\tri*{\thetastar,\phi(y)}}\indic\crl*{\nrm*{\phidel(y,y')}_{\Sigma^{-1}}>\nu}}.
  \end{align*}
  We can further bound
  \begin{align*}
    &\En_{y\sim\piref}\brk*{\exp\prn*{\beta^{-1}\tri*{\thetastar,\phi(y)}}\indic\crl*{\nrm*{\phidel(y,y')}_{\Sigma^{-1}}>\nu}}\\
    &=\En_{y\sim\piref}\brk*{\frac{\exp\prn*{\beta^{-1}\tri*{\thetastar,\phi(y)}}}{
      Z_{\pi_{\thetastar}}}\indic\crl*{\nrm*{\phidel(y,y')}_{\Sigma^{-1}}>\nu}}\cdot{}
      Z_{\pi_{\thetastar}}\\
    &=\En_{y\sim\pi_{\thetastar}}\brk*{\indic\crl*{\nrm*{\phidel(y,y')}_{\Sigma^{-1}}>\nu}}\cdot{}    Z_{\pi_{\thetastar}}.
  \end{align*}
  It follows that as long as
  $\vepsspan(x,y')\ldef{}\En_{y\sim\pi_{\thetastar}}\brk*{\indic\crl*{\nrm*{\phidel(y,y')}_{\Sigma^{-1}}>\nu}}\leq1/2$,
  we have
  \[
\En_{y\sim\piref}\brk*{\exp\prn*{\beta^{-1}\tri*{\thetastar\phi(y)}}\indic\crl*{\nrm*{\phidel(y,y')}_{\Sigma^{-1}}\leq\nu}}\geq{}\frac{1}{2}Z_{\pi_{\thetastar}}.
\]
Combining this with the preceding steps gives
\begin{align*}
      \frac{\pibar_{\theta}(y\mid{}y')}{\piref(y)}
  &\leq{}
  2e\cdot\frac{\exp\prn*{\beta^{-1}(\tri*{\theta,\phidb(y,y')}+\tri*{\thetastar,\phi(y')})}}{Z_{\pist}}\\
  &\leq{} 2e^2\cdot\frac{\exp\prn*{\beta^{-1}(\tri*{\thetastar,\phidb(y,y')}+\tri*{\thetastar,\phi(y')})}}{Z_{\pist}}
\end{align*}
where the second inequality is by \cref{eq:dot_bound-2}. To proceed, we consider two cases. First, if
$\tri*{\thetastar,\phidel(y,y')}\geq{}0$, then
\[
  \tri*{\thetastar,\phidb(y,y')}+\tri*{\thetastar,\phi(y')}
  \leq{} \tri*{\thetastar,\phidel(y,y')}+\tri*{\thetastar,\phi(y')}
  = \tri*{\thetastar,\phi(y)}.
\]
Otherwise,
\[
  \tri*{\thetastar,\phidb(y,y')}+\tri*{\thetastar,\phi(y')}
  \leq{} \tri*{\thetastar,\phi(y')}.
\]
Combining these cases gives
\begin{align*}
  \frac{\exp\prn*{\beta^{-1}(\tri*{\thetastar,\phidb(y,y')}+\tri*{\thetastar,\phi(y')})}}{Z_{\pist}}
  &\leq{}
    \max\crl*{\frac{\exp\prn*{\beta^{-1}\tri*{\thetastar,\phi(y)}}}{Z_{\pist}},
\frac{\exp\prn*{\beta^{-1}\tri*{\thetastar,\phi(y')}}}{Z_{\pist}}
    }\\
    &=
    \max\crl*{\frac{\pist(y)}{\piref(y)}, \frac{\pist(y')}{\piref(y')}
    }\\
&\leq{} \Ccov(\pist)
\end{align*}
which completes the proof.
\end{proof}

  \begin{proof}[\pfref{lem:truncated_regret_simple}]
  By \cref{lem:truncated_regret}, taking expectation over $x\sim\rho$, we have
  \begin{align*}
    \Jbeta(\pist) - \Jbeta(\pibtheta)
    &\leq\Jbeta(\pist) - \Jbbeta(\pibtheta)\\
  &\leq
    \beta^{-1}\En_{x\sim\rho, (y,y')\sim\pibtheta(x)}\brk*{\tri*{\thetastar-\theta,\phidb(x,y,y')}^2}\\
    &~~~~+
      \Rmax\bbP_{x\sim\rho,y\sim\pist(x),y'\sim\piref(x)}\brk*{\nrm*{\phidel(x,y,y')}_{\Sigma^{-1}}>\nu} \\
    &~~~~+ \Rmax\bbP_{x\sim\rho, (y,y')\sim\pibtheta(x)}\brk*{\nrm*{\phidel(x,y,y')}_{\Sigma^{-1}}>\nu}.
  \end{align*}
  We need to bound the second and third terms. Let $\veps>0$ be
  fixed, and let us abbreviate
  \[
\cXspan \equiv \cXspan(\veps) = 
\crl*{x\in\cX\mid{}\bbP_{(y,y')\sim\piref(x)}\brk*{\nrm*{\phidel(x,y,y')}_{\Sigma^{-1}}>\nu} \leq\veps}.
  \]
  For the second term, it is immediate that
  \begin{align*}
    &\bbP_{x\sim\rho,y\sim\pist(x),y'\sim\piref(x)}\brk*{\nrm*{\phidel(x,y,y')}_{\Sigma^{-1}}>\nu}\\
    &\leq{}
\En_{x\sim\rho}\brk*{\bbP_{y\sim\pist(x),y'\sim\piref(x)}\brk*{\nrm*{\phidel(x,y,y')}_{\Sigma^{-1}}>\nu}\indic\crl{x\in\cXspan}}
      + \bbP_{x\sim\rho}\brk*{x\notin\cXspan}\\
        &\leq{}
\Ccov(\pist)\cdot\En_{x\sim\rho}\brk*{\bbP_{(y,y')\sim\piref(x)}\brk*{\nrm*{\phidel(x,y,y')}_{\Sigma^{-1}}>\nu}\indic\crl{x\in\cXspan}}
      + \bbP_{x\sim\rho}\brk*{x\notin\cXspan}\\
    &\leq{} \Ccov(\pist)\cdot\veps + \bbP_{x\sim\rho}\brk*{x\notin\cXspan}.
  \end{align*}
  
  To handle the third term, define
  \[
    \Zgood\ldef{}\crl*{(x,y')\in\cX\times\cY\mid{}
      \bbP_{y\sim\pist(\cdot\mid{}x)}\brk*{\nrm*{\phidel(x,y,y')}_{\Sigma^{-1}}>\nu}\leq\frac{1}{2}}.
  \]
  We can bound
  \begin{align*}
    &\bbP_{x\sim\rho,
    (y,y')\sim\pibtheta(x)}\brk*{\nrm*{\phidel(x,y,y')}_{\Sigma^{-1}}>\nu}\\
    &= \bbP_{x\sim\rho,
      y'\sim\piref(\cdot\mid{}x),y'\sim\pibtheta(\cdot\mid{}x,y')}\brk*{\nrm*{\phidel(x,y,y')}_{\Sigma^{-1}}>\nu}\\
    &\leq \En_{x\sim\rho,
      y'\sim\piref(\cdot\mid{}x)}\brk*{\bbP_{y'\sim\pibtheta(\cdot\mid{}x,y')}\brk*{\nrm*{\phidel(x,y,y')}_{\Sigma^{-1}}>\nu}\indic\crl{(x,y')\in\Zgood,x\in\cXspan}} \\
&~~~~   + \En_{x\sim\rho,
     y'\sim\piref(\cdot\mid{}x)}\brk*{\indic\crl{(x,y')\notin\Zgood,
                                                                                                                                                                           x\in\cXspan}}
                                                                                                                                                                           + \bbP_{x\sim\rho}\brk*{x\notin\cXspan}.
  \end{align*}
  For the first term, \cref{lem:truncated_density} implies that when
  $(x,y')\in\Zgood$, $\pibtheta(y\mid{}x,y')\leq
  2e^2\Ccov(\pist)\cdot\piref(y\mid{}x)$ for all $y\in\cY$, so we can bound
  \begin{align*}
&    \En_{x\sim\rho,
    y'\sim\piref(\cdot\mid{}x)}\brk*{\bbP_{y'\sim\pibtheta(\cdot\mid{}x,y')}\brk*{\nrm*{\phidel(x,y,y')}_{\Sigma^{-1}}>\nu}\indic\crl{(x,y')\in\Zgood,
                   x\in\cXspan}}\\
    &\leq{}
      2e^2\Ccov(\pist)\cdot\En_{x\sim\rho}\brk*{\bbP_{(y,y')\sim\piref(x)}\brk*{\nrm*{\phidel(x,y,y')}_{\Sigma^{-1}}>\nu}\indic\crl*{x\in\cXspan}}\\
    &\leq{} 2e^2\Ccov(\pist)\cdot\veps.
  \end{align*}
  For the second term, we can use Markov's inequality to bound
  \begin{align*}
    \En_{x\sim\rho,
    y'\sim\piref(\cdot\mid{}x)}\brk*{\indic\crl{(x,y')\notin\Zgood,x\in\cXspan}}
    &\leq{} 2 \En_{x\sim\rho}\brk*{\bbP_{y\sim\pist(\cdot\mid{}x),y'\sim\piref(\cdot\mid{}x)}\brk*{\nrm*{\phidel(x,y,y')}_{\Sigma^{-1}}>\nu}\indic\crl*{x\in\cXspan}}\\
    &\leq{}
      2\Ccov(\pist)\cdot\En_{x\sim\rho}\brk*{\bbP_{(y,y')\sim\piref(x)}\brk*{\nrm*{\phidel(x,y,y')}_{\Sigma^{-1}}>\nu}\indic\crl*{x\in\cXspan}}\\
    &\leq{} 2\Ccov(\pist)\cdot\veps.
  \end{align*}
Combining the preceding bounds completes the proof.
\end{proof}

\subsection{Proof of \creftitle{thm:spanner} (Guarantee for \spanalg)}
\label{sec:spanner_proof}

In this section we prove \cref{thm:spanner}, restated below.

\spannermain*

Note that the high-probability bound is over the randomness of the policies $\pihat\ind{1},\dots,\pihat^{\Texp}$, but $\pihat$ is chosen uniformly from these; a true high-probability bound on $J_\beta(\pistarb)-J_\beta(\pihat)$ could be obtained by estimating each $J_\beta(\pihat\ind{t})$ and choosing $\pihat$ as the minimizer over $t\in[\Texp]$ (as we do in $\mtalg$), but we omit this extra complication here. We begin by proving a number of intermediate results. We then use these
results to prove \cref{thm:spanner} in \sssref{sec:fast_rate}.\loose

\subsubsection{Intermediate Guarantee for Spanner Construction}
In this section we give self-contained guarantees for
\cref{line:spanner_outer} of \cref{alg:spanner}, which aims to construct a \emph{spanner}: a
collection $\Psispan$ of tuples $(x,y,y')$ for which
$\sum_{(x,y,y')\in\Psispan}\phidel(x,y,y')\phidel(x,y,y')$ covers the
feature space as least as well as $(y,y')\sim\piref(\cdot\mid{}x)$.

Concretely, let $\Psispan$ denote the collection of all tuples $(x\ind{t},
y_1\ind{t,i}, y_1\ind{t,i})$ for which the if statement in
\cref{line:spanner_if} is triggered, so that
$\Sigmas=\lambda\Id + \sum_{(x,y,y')\in\Psispan}\phidel(x,y,y')\phidel(x,y,y')$
when the outer for loop completes. Our first lemma gives a bound on the size
of $\Psispan$.

\begin{lemma}
  \label{lem:spanner_size}
  Suppose that $\nu,\lambda \leq 1$. With probability $1$, we have
  \[
    |\Psispan| \lesssim 
\frac{d\log(1+\nu^{-1}\lambda^{-1})}{\nu^2}.
    \]
\end{lemma}
\begin{proof}[\pfref{lem:spanner_size}] Order $\Psispan$ as  $\Psispan = \{(x\ind{1},
  y_1\ind{1},y_2\ind{1})\dots,(x\ind{k},y_1\ind{k},y_2\ind{k})\}$ and
  let \[\Gamma_j = \lambda \Id + \sum_{i=1}^j \phidel(x\ind{i},
  y_1\ind{i},y_2\ind{i})\phidel(x\ind{i},
  y_1\ind{i},y_2\ind{i})^\trn.\] We will bound $k\in\bbN$. From the standard elliptic potential lemma argument, we have
  \[
    \log\det \Gamma_k - \log\det \Gamma_0 \geq \sum_{j=1}^k \log\left(1 + \phidel(x\ind{j},
  y_1\ind{j},y_2\ind{j})^\trn (\Gamma_{j-1})^{-1} \phidel(x\ind{j},
  y_1\ind{j},y_2\ind{j})\right) \geq k\log(1 + \nu^2) \geq
\frac{k\nu^2}{2}.
\]
Moreover, $\nrm*{\phidel}\leq{}2$, $\log\det \Gamma_k \leq
d\log(\lambda+4k/d)$ (e.g., Lemma 10 of \citet{abbasi2011improved}), whereas $\log\det\Gamma_0 = d\log\lambda$. Hence, we have
\[k \leq \frac{2d\log(1+4k/(d\lambda))}{\nu^2},\]
and \cref{lem:log_bound} further implies that $k\approxleq{}
\frac{d\log(1+\nu^{-1}\lambda^{-1})}{\nu^2}$ as claimed.
  
\end{proof}

Our second lemma gives a guarantee on the quality of the spanner.

\begin{lemma}
  \label{lem:spanner_quality}
  Let $\delta\in(0,1)$ be fixed, and define
  \[
\vepsspan\ldef{}\frac{8\log(4\Tprompt\delta^{-1})}{\Tspan}.
  \]
  With probability at least $1-\delta$, \cref{alg:spanner} satisfies
  \begin{align*}
        \bbP_{x\sim\rho}\brk*{   \bbP_{
    (y,y')\sim\piref(\cdot\mid{}x)}\brk*{\nrm*{\phidel(x,y,y')}_{\Sigmas^{-1}}>\nu}
    >\vepsspan}
    \approxleq{}
\frac{d\log(1+\nu^{-1}\lambda^{-1})}{\Tprompt\nu^2} + \frac{\log(\delta^{-1})}{\Tprompt}.
  \end{align*}
  
\end{lemma}
\begin{proof}[\pfref{lem:spanner_quality}]
    Let $\Sigmas\ind{t}$ and $\Psispan\ind{t}$ denote the value of
  $\Sigmas$ and $\Psispan$ at the beginning
  of the iteration $t$ of the for loop in
  \cref{line:spanner_outer}. For each $t\in\brk{\Tprompt}$, let
  $i_t$ denote the first index $i$ such that the if statement in
  \cref{line:spanner_if} is triggered, and let $i_t=\Tspan$
  otherwise. Using \cref{lem:mult_freedman} and a union bound, we have that with
  probability at least $1-\delta/2$, for all $t\in\brk{\Tprompt}$,
  \begin{align*}
    \bbP_{
    (y,y')\sim\piref(\cdot\mid{}x\ind{t})}\brk*{\nrm*{\phidel(x,y,y')}_{(\Sigmas\ind{t})^{-1}}>\nu}
    &\leq{}
\frac{2}{i_t}\sum_{i=1}^{i_t}\indic\crl*{\nrm*{\phidel(x\ind{t},y\ind{t,i},y\ind{t,i})}_{(\Sigmas\ind{t})^{-1}}>\nu}
      + \frac{8\log(4\Tprompt\delta^{-1})}{i_t}\\
        &\leq{}
2\indic\crl[\Big]{\exists i : \nrm*{\phidel(x\ind{t},y\ind{t,i},y\ind{t,i})}_{(\Sigmas\ind{t})^{-1}}>\nu}
          + \frac{8\log(4\Tprompt\delta^{-1})}{i_t}.
  \end{align*}
If
$\nrm*{\phidel(x\ind{t},y\ind{t,i},y\ind{t,i})}_{(\Sigmas\ind{t})^{-1}}\leq\nu$
for all $i$, then $i_t=\Tspan$, and consequently the right-hand side above is
bounded by
$\vepsspan\ldef{}\frac{8\log(4\Tprompt\delta^{-1})}{\Tspan}$. It
follows that under the concentration event above, we have that for all $t\in\brk{\Tprompt}$,
  \begin{align*}
\indic\crl*{    \bbP_{
    (y,y')\sim\piref(\cdot\mid{}x\ind{t})}\brk*{\nrm*{\phidel(x,y,y')}_{(\Sigmas\ind{t})^{-1}}>\nu}
    >\vepsspan}
    \leq{} \indic\crl[\Big]{\exists i : \nrm*{\phidel(x\ind{t},y\ind{t,i},y\ind{t,i})}_{(\Sigmas\ind{t})^{-1}}>\nu}.
  \end{align*}
  Now, define
  \[
    p\ind{t} = \bbP_{x\sim\rho}\brk*{   \bbP_{
    (y,y')\sim\piref(\cdot\mid{}x)}\brk*{\nrm*{\phidel(x,y,y')}_{(\Sigmas\ind{t})^{-1}}>\nu}
    >\vepsspan}.
\]
  Then $p\ind{\Tprompt+1}\leq{}p\ind{\Tprompt}\leq{}\ldots{}p\ind{1}$, so
  \[
    \bbP_{x\sim\rho}\brk*{   \bbP_{
    (y,y')\sim\piref(\cdot\mid{}x)}\brk*{\nrm*{\phidel(x,y,y')}_{\Sigmas^{-1}}>\nu}
    >\vepsspan}
= p\ind{\Tprompt+1}
    \leq\frac{1}{\Tprompt}\sum_{t=1}^{\Tprompt}p\ind{t}.
  \]
  Since $\Sigmaspan\ind{t}$ does not depend on $x\ind{t}$, \cref{lem:mult_freedman} implies that with probability at least
  $1-\delta/2$,
  \begin{align*}
    \sum_{t=1}^{\Tprompt}p\ind{t}
    &\leq{} 2\sum_{t=1}^{\Tprompt}\indic\crl*{    \bbP_{
    (y,y')\sim\piref(\cdot\mid{}x\ind{t})}\brk*{\nrm*{\phidel(x,y,y')}_{(\Sigmas\ind{t})^{-1}}>\nu}
    >\vepsspan}
      + 8\log(4\delta^{-1})\\
        &\leq{} 2\sum_{t=1}^{\Tprompt}\indic\crl[\Big]{\exists i : \nrm*{\phidel(x\ind{t},y\ind{t,i},y\ind{t,i})}_{(\Sigmas\ind{t})^{-1}}>\nu}
          + 8\log(4\delta^{-1})\\
        &\leq{} 2\abs*{\Psispan\ind{\Tprompt+1}}
    + 8\log(4\delta^{-1}).
  \end{align*}
    From here, the result follows from \cref{lem:spanner_size}.

\end{proof}

\subsubsection{Proof of \creftitle{thm:spanner}}
\label{sec:fast_rate}

\begin{proof}[Proof of \cref{thm:spanner}]
  Recall that we define $\vepsstat\ldef
  c\cdot\sqrt{d\Rmax^2\log(B\Rmax^{-1}\delta^{-1}\nexp)}$ for a
  sufficiently large absolute constant $c>0$, and use the parameter settings
  $\lambda\gets{}(\nicefrac{\Rmax}{B})^2$, 
  $\nu\ldef{}\beta/\vepsstat$, $\Mrej\ldef{}8e^2\cdot\Ccov(\pist)$, and
  $\deltarej\ldef{}\nexp^{-1}$. We will show that under these
  settings, for any choice
  of $\Tprompt$, $\nspan$, and $\Texp$, we have
    that with probability at least $1-\delta$, 
    \[      
    \En_{t\sim\unif\prn*{\brk{\nexp}}}\brk*{\Jbeta(\pist) - \Jbeta(\pihat\ind{t})}
    \approxleq{}
    \frac{\Rmax^2}{\beta}\cdot\bigoht\prn*{\frac{d^2 \log^2
    \prn*{\delta^{-1}}}{\nexp} + 
  \frac{
    d^2\Rmax^2\log^2(\delta^{-1})}{\beta^2\Tprompt}
  + \frac{\Ccov(\pist)\cdot\log(\delta^{-1})}{\Tspan}}.
  \]
  We will use this to give bounds on $\Tdatafull$ and $\Tsamplefull$
  at the end of the proof.

  \paragraph{Preliminaries: Least squares}
  We begin with some preliminary
  observations. First, for each $t\in\brk{\nexp}$, define
  \[
    \Sigmafull\ind{t}=\lambda\Id + \sum_{(x,y_1,y_2)\in\Psispan}\phidel(x\ind{t},y_1,y_2) \phidel(x\ind{t},y_1,y_2)^{\trn}
  + \sum_{i<t}\phidel(x\ind{t},y_1\ind{t},y_2\ind{t})
  \phidel(x\ind{t},y_1\ind{t},y_2\ind{t})^{\trn}
\]
and $\Sigmaexp\ind{t}=\lambda\Id + \sum_{i<t}\phidel(x\ind{t},y_1\ind{t},y_2\ind{t})   \phidel(x\ind{t},y_1\ind{t},y_2\ind{t})^{\trn}$.

  We invoke \cref{lem:least_squares}, which implies that for
  the choice of $\lambda$ in \cref{line:spanner_params}, we are guaranteed
  that with probability at least $1-\delta/3$, for all $t\in\brk{\nexp}$,
  \begin{align}
    \nrm*{\theta\ind{t}-\thetastar}_{\Sigmafull\ind{t}}^2
  \leq{} \underbrace{c\cdot{}d\Rmax^2\log(B\Rmax^{-1}\delta^{-1}\nexp)}_{\rdef\vepsstat^2}.
  \end{align}
  for an absolute constant $c>0$. We denote this event by
  $\cEconc$ and condition on it going forward. In particular, under
  this event, we have
  \begin{align}
    \label{eq:spanner_ls}
    \nrm*{\theta\ind{t}-\thetastar}_{\Sigmas}^2 \leq \vepsstat^2,\mathand
        \nrm*{\theta\ind{t}-\thetastar}_{\Sigmaexp\ind{t}}^2 \leq \vepsstat^2.
  \end{align}

  \paragraph{Preliminaries: Truncated policies}
  Next, recall that we define
  \[
        r\ind{t}(x,y,y')\ldef{}
        \tri*{\theta\ind{t},\phidel(x,y,y')}\indic\crl[\big]{\nrm*{\phidel(x,y,y')}_{\Sigmas^{-1}}\leq\nu}
        =\tri*{\theta\ind{t},\phidb(x,y,y')}
  \]
  in \cref{line:spanner_reward}, where
  $\phidb(x,y,y')\ldef\phidel(x,y,y')\indic\crl[\big]{\nrm*{\phidel(x,y,y')}_{\Sigmas^{-1}}\leq\nu}$. It
    will be helpful to define some
    intermediate policies. 
    First, define 
  \[
    \pi\ind{t}(y\mid{}x,
    y')\propto\piref(y\mid{}x)\exp\prn*{\beta^{-1}r\ind{t}(x,y, y')}
    = \piref(y\mid{}x)\exp\prn*{\beta^{-1}\tri*{\theta\ind{t},\phidb(x,y, y')}}
  \]
  be the softmax policy induced by
  $r\ind{t}(\cdot,\cdot,y')$. Clearly, we have
  \[
    \pi\ind{t}(y\mid{}x,y') = \pibar_{\theta\ind{t}}(y\mid{}x,y'),
  \]
  where $\pibar_{\theta}(y\mid{}x,y')$ is the truncated softmax
  policy defined in \cref{sec:spanner_regret} for parameters $\Sigmaspan$
  and $\nu$. We further define
  \begin{align*}
    \pi\ind{t}(y, y'\mid{}x)\ldef \pibar_{\theta\ind{t}}(y,y'\mid{}x)
    =
    \pibar_{\theta\ind{t}}(y\mid{}x, y')\cdot \piref(y'\mid{}x)
  \end{align*}
  as the joint distribution over $(y,y')$ induced by sampling
  $y'\sim\piref(\cdot\mid{}x)$ and
  $y\sim\pi\ind{t}(\cdot\mid{}x,y')$, and define
  $\pi\ind{t}(y\mid{}x)=\pibar_{\theta\ind{t}}(y\mid{}x)\ldef\En_{y'\sim\piref(\cdot\mid{}x)}\brk{\pi\ind{t}(y\mid{}x,y')}$
  as the induced ``marginal'' policy over $y$.

  Note that by definition of $\nu$, whenever $\cEconc$ holds, we have
  \begin{align}
    \abs*{r\ind{t}(x,y,y')}
    &\leq{}
    \nrm*{\theta\ind{t}-\thetastar}_{\Sigmaspan}\nrm*{\phidb(x,y,y')}_{\Sigmaspan^{-1}}
    + \abs*{\tri*{\thetastar,\phidel(x,y,y')}}\notag\\
    &\leq{} \nu\vepsstat + \Rmax
    \leq{} \beta + \Rmax \leq 2\Rmax.     \label{eq:r_bound}
  \end{align}

  \paragraph{Preliminaries: Spanner construction}
  Define
  \[
\vepsspan\ldef{}\frac{8\log(12\Tprompt\delta^{-1})}{\Tspan}.
  \]
  \cref{lem:spanner_quality} implies that with probability at least
  $1-\delta/3$,
  \begin{align}
    \label{eq:spanner_bound}
       \bbP_{x\sim\rho}\brk*{   \bbP_{
    (y,y')\sim\piref(\cdot\mid{}x)}\brk*{\nrm*{\phidel(x,y,y')}_{\Sigmas^{-1}}>\nu}
    >\vepsspan} \leq \vepsprompt
  \end{align}
  for
  \[
\vepsprompt
    \approxleq{}
\frac{d\log(1+\nu^{-1}\lambda^{-1})}{\Tprompt\nu^2} + \frac{\log(\delta^{-1})}{\Tprompt}.
  \]
We denote this event by $\cEspan$ and condition on it going
forward. It will be convenient to define
\[
  \cXspan \ldef{} \crl*{x\in\cX \mid \bbP_{
    (y,y')\sim\piref(\cdot\mid{}x)}\brk*{\nrm*{\phidel(x,y,y')}_{\Sigmas^{-1}}>\nu}
  \leq\vepsspan}
\]
so that \cref{eq:spanner_bound} can be equivalently written as
$\bbP_{x\sim\rho}\brk*{x\notin\cXspan}\leq\vepsprompt$ under this event.

  \paragraph{Preliminaries: Rejection sampling}
We define
  \[
    \pihat\ind{t}(\cdot\mid{}x, y')
    \ldef{}
    \rejection_{\beta,\Mrej,\deltarej}(r\ind{t}(\cdot,\cdot, y')\midsem x, \piref)
  \]
  denote the distribution over $y_1\ind{t}$ in
  \cref{line:spanner_sampling} (given $x\ind{t}=x$ and $y_2\ind{t}=y'$), which aims to
  approximate $\pibar_{\theta\ind{t}}(\cdot\mid{}x,y')$, and define
  \[
\pihat\ind{t}(y,y'\mid{}x)
\]
as the law of  $y'\sim\piref(\cdot\mid{}x)$ and
$y\sim\rejection_{\beta,\Mrej,\deltarej}(\rbar\ind{t}(\cdot,\cdot,
y')\midsem x, \piref)$, using $\pihat\ind{t}(y\mid{}x)=\En_{y'\sim\piref(\cdot\mid{}x)}\brk*{\pihat\ind{t}(y\mid{}x,y')}$ to denote the
marginal. 

Define 
  \[
    \cZgood\ldef{}\crl*{(x,y')\in\cX\times\cY\mid{}
      \bbP_{y\sim\pist(\cdot\mid{}x)}\brk*{\nrm*{\phidel(x,y,y')}_{\Sigmaspan^{-1}}>\nu}\leq\frac{1}{2}}.
  \]
By \cref{lem:truncated_density}, we
  have that under $\cEconc$,
  \begin{align}
    \Ccov(\pibar_{\theta\ind{t}}(\cdot\mid{}x,y'))\leq{}2e^2\Ccov(\pist)
    \quad\text{for all $(x,y')\in\cZgood$,}\label{eq:zgood}
  \end{align}
  so \cref{thm:rejection}
  implies that for the choice for $\Mrej$ in
  \cref{line:rejection_params_spanner}, we have
  \begin{align}
    \label{eq:rejection_tv}
    \Dtv{\pihat\ind{t}(\cdot\mid{}x,y')}{\pi\ind{t}(\cdot\mid{}x,y')}\leq\deltarej
  \end{align}
  for all $(x,y')\in\cZgood$. We can further derive the following
  consequence.
  \begin{lemma}
    \label{lem:rejection_tv_average}
    Under the event $\cEconc$, for any function $f(x,y,y')\in\brk{0,1}$,
    \begin{align*}
      \abs*{\En_{x\sim\rho,(y,y')\sim\pi\ind{t}}\brk*{f(x,y,y')}-
      \En_{x\sim\rho,(y,y')\sim\pihat\ind{t}}\brk*{f(x,y,y')}}%
            &\leq{} \deltarej + 2\Ccov(\pist)\cdot\vepsspan + \vepsprompt.
    \end{align*}
    
  \end{lemma}
  \begin{proof}[\pfref{lem:rejection_tv_average}]
By \cref{eq:rejection_tv}, we can bound
    \begin{align*}
      &\abs*{\En_{x\sim\rho,(y,y')\sim\pi\ind{t}}\brk*{f(x,y,y')}-
      \En_{x\sim\rho,(y,y')\sim\pihat\ind{t}}\brk*{f(x,y,y')}}\\
      &\leq{}
        \En_{x\sim\rho,y'\sim\piref(\cdot\mid{}x)}\brk*{\abs*{\En_{y\sim\pi\ind{t}(\cdot\mid{}x,y')}\brk*{f(x,y,y')}-
        \En_{y\sim\pihat\ind{t}(\cdot\mid{}x,y')}\brk*{f(x,y,y')}}\indic\crl*{(x,y')\in\cZgood,x\in\cXspan}}\\
      &~~~~+
        \bbP_{x\sim\rho,y'\sim\piref(\cdot\mid{}x)}\brk*{\indic\crl*{(x,y')\notin\cZgood}}
      + \bbP_{x\sim\rho}\brk*{x\notin\cXspan}\\
      &\leq{}\deltarej +
        \En_{x\sim\rho,y'\sim\piref(\cdot\mid{}x)}\brk*{\indic\crl*{(x,y')\notin\cZgood,
        x\in\cXspan}} + \vepsprompt.
    \end{align*}
    Using Markov's inequality, we can further bound
    \begin{align*}
      \En_{x\sim\rho,y'\sim\piref(\cdot\mid{}x)}\brk*{\indic\crl*{(x,y')\notin\cZgood,
      x\in\cXspan}} 
      &\leq{} 2
        \En_{x\sim\rho}\brk*{\bbP_{y'\sim\piref(\cdot\mid{}x),y\sim\pist(x)}\brk*{\nrm*{\phidel(x,y,y')}_{\Sigmaspan^{-1}}>\nu}\indic\crl*{x\in\cXspan}}\\
            &\leq{} 2\Ccov(\pist)
\En_{x\sim\rho}\brk*{\bbP_{(y,y')\sim\piref(\cdot\mid{}x)}\brk*{\nrm*{\phidel(x,y,y')}_{\Sigmaspan^{-1}}>\nu}\indic\crl*{x\in\cXspan}}\\
      &\leq{} 2\Ccov(\pist) \cdot\vepsspan,
    \end{align*}
    where the final inequality follows from the definition of $\cXspan$.
  \end{proof}

  \paragraph{Moving to idealized softmax policies}
  Our aim is to bound the regret
  \begin{align}
    \En_{t\sim\unif\prn*{\brk{\nexp}}}\brk*{\Jbeta(\pist) -
    \Jbeta(\pihat\ind{t})} =     \frac{1}{\nexp}
        \sum_{t=1}^{\nexp}\Jbeta(\pist) - \Jbeta(\pihat\ind{t}).\notag
  \end{align}
 Define
 \[
   \Jbbeta(\pibtheta)
   = \En_{x\sim\rho,y'\sim\piref(\cdot\mid{}x), y\sim\pibtheta(\cdot\mid{}x,y')}\brk*{
     \rstar(x,y)
     -\beta\Dkl{\pibtheta(\cdot\mid{}x,y')}{\piref(\cdot\mid{}x)}
     }.
 \]
  We invoke \cref{lem:spanner_rejection} below (proven in the sequel)
  to bound
  \begin{align}
\frac{1}{\nexp}
\sum_{t=1}^{\nexp}\Jbeta(\pist) - \Jbeta(\pihat\ind{t})
\approxleq{} \frac{1}{\nexp}
\sum_{t=1}^{\nexp}\Jbeta(\pist) - \Jbbeta(\pibt)
+ \Rmax \log\log(\nexp)\cdot\prn*{\deltarej +
    \Ccov(\pist)\cdot\vepsspan+\vepsprompt}. \label{eq:idealized_softmax}
  \end{align}
  \begin{lemma}
    \label{lem:spanner_rejection}
    Under the event $\cEconc$, for any $\deltarej\in(0,1)$, we have
    \begin{align*}
      \Jbbeta(\pibt)
      - \Jbeta(\pihat\ind{t})
      &\approxleq
        \bigoh(\Rmax
        \log\log(\deltarej^{-1})\cdot\prn*{\deltarej + \Ccov(\pist)\cdot\vepsspan + \vepsprompt}.
      \end{align*}
  \end{lemma}

  \paragraph{Regret bound for truncated softmax policy}
  For the next step, we note that for the choice of
  $\nu=\beta/\vepsstat$, under $\cEconc$,
  our central regret decomposition for truncated softmax policies (\cref{lem:truncated_regret_simple}) implies that \loose
  \begin{align*}
&\frac{1}{\nexp}
\sum_{t=1}^{\nexp}\Jbeta(\pist) - \Jbbeta(\pibt)\\
&\leq{}
    \frac{1}{\beta\nexp}
\sum_{t=1}^{\nexp}\En_{x\sim\rho,(y,y')\sim\pibt(x)}\brk*{\tri*{\theta\ind{t}-\thetastar,\phidb(x,y,y')}^2}
                                                     + \bigoh(\Rmax)\cdot \prn*{\Ccov(\pist)\cdot\vepsspan+\vepsprompt}.
  \end{align*}
  Using \cref{eq:spanner_ls} and \cref{eq:r_bound}, we can bound
  \begin{align*}
    \sum_{t=1}^{\nexp}\En_{x\sim\rho, (y,y')\sim\pibt(x)}\brk*{\tri*{\theta\ind{t}-\thetastar,\phidb(x,y,y')}^2}%
    &\leq{}4\sum_{t=1}^{\nexp}\En_{x\sim\rho, (y,y')\sim\pibt(x)}\brk*{\tri*{\theta\ind{t}-\thetastar,\phidb(x,y,y')}^2\wedge\Rmax^2}\\
    &\leq{}4\sum_{t=1}^{\nexp}\En_{x\sim\rho, (y,y')\sim\pibt(x)}\brk*{\vepsstat^2\nrm*{\phidel(x,y,y')}_{(\Sigmaexp\ind{t})^{-1}}^2\wedge\Rmax^2}.
  \end{align*}
  We can
  further use \cref{lem:rejection_tv_average} to bound
  \begin{align*}
    &\sum_{t=1}^{\nexp}\En_{x\sim\rho,
    (y,y')\sim\pibt(x)}\brk*{\vepsstat^2\nrm*{\phidel(x,y,y')}_{(\Sigmaexp\ind{t})^{-1}}^2\wedge{}\Rmax^2}\\
    &\leq{} \sum_{t=1}^{\nexp}\En_{x\sim\rho,
    (y,y')\sim\pihat\ind{t}(x)}\brk*{\vepsstat^2\nrm*{\phidel(x,y,y')}_{(\Sigmaexp\ind{t})^{-1}}^2\wedge{}\Rmax^2}
    + \bigoh(\Rmax^2\nexp(\deltarej + \Ccov(\pist)\cdot\vepsspan+\vepsprompt))\\
    &\leq{} \vepsstat^2\sum_{t=1}^{\nexp}\En_{x\sim\rho,
    (y,y')\sim\pihat\ind{t}(x)}\brk*{\nrm*{\phidel(x,y,y')}_{(\Sigmaexp\ind{t})^{-1}}^2\wedge{}1}
    + \bigoh(\Rmax^2\nexp(\deltarej + \Ccov(\pist)\cdot\vepsspan+\vepsprompt)),
  \end{align*}
  where the last inequality uses that $\vepsstat\geq{}\Rmax$.
  Now, by \cref{lem:mult_freedman}, we are
guaranteed that with probability at least $1-\delta/3$,
\begin{align*}
  \sum_{t=1}^{\nexp}\En_{x\sim\rho,(y,y')\sim\pihat\ind{t}(x)}\brk*{\min\crl*{\nrm*{\phidel(x,y,y')}^2_{(\Sigmaexp\ind{t})^{-1}},
  1}}
  \leq{}
  \frac{3}{2}\sum_{t=1}^{\nexp}\min\crl*{\nrm*{\phidel(x\ind{t},y_1\ind{t},y_2\ind{t})}^2_{(\Sigmaexp\ind{t})^{-1}},
  1}
  + 4\log(6\delta^{-1}).
\end{align*}
Finally, since $\Sigmaexp\ind{t}=\lambda\Id +
\sum_{i<t}\phidel(x\ind{i},y_1\ind{i},y_2\ind{i})
\phidel(x\ind{i},y_1\ind{i},y_2\ind{i})^{\trn}$,
\cref{lem:elliptic_potential} implies that
\[
\sum_{t=1}^{\nexp}\min\crl*{\nrm*{\phidel(x\ind{t},y_1\ind{t},y_2\ind{t})}^2_{(\Sigma\ind{t})^{-1}},
  1}
\leq 2d \log \prn*{1+\lambda^{-1}\nexp/d}.
\]

\paragraph{Putting everything together: Final bounds on $\Tdata$ and
  $\Tsample$}
Combining all of the preceding inequalities and simplifying (using
that $\vepsstat\geq\Rmax\geq\beta$ and $\deltarej=\nexp^{-1}$), we conclude that with
probability at least $1-\delta$,
\begin{align*}
&
\En_{t\sim\unif\prn*{\brk{\nexp}}}\brk*{\Jbeta(\pist) -
    \Jbeta(\pihat\ind{t})}
                 \\
  &\approxleq \frac{\vepsstat^2\cdot{}d \log \prn*{\lambda^{-1}\nexp\delta^{-1}}}{\beta\nexp}
  + 
  \frac{\Rmax^2
    \log\log(\nexp)}{\beta}\cdot\prn*{\frac{1}{\Texp} +
    \Ccov(\pist)\cdot\vepsspan+\vepsprompt}\\
    &\approxleq \frac{\vepsstat^2\cdot{}d \log \prn*{B\Rmax^{-1}\nexp\delta^{-1}}}{\beta\nexp}
  +
      \frac{d\Rmax^2\log(B\Rmax^{-1}\nu^{-1}\delta^{-1})\log\log(\Texp)}{\beta\nu^2\Tprompt}
      +
      \frac{\Rmax^2\log(\delta^{-1})\log\log(\Texp)\cdot\Ccov(\pist)}{\beta\Tspan}\\
      &\approxleq \frac{\vepsstat^2\cdot{}d \log \prn*{B\Rmax^{-1}\nexp\delta^{-1}}}{\beta\nexp}
  +
      \frac{\vepsstat^2\cdot{}d\Rmax^2\log(B\Rmax^{-1}\nu^{-1}\delta^{-1})\log\log(\Texp)}{\beta^3\Tprompt}
      + \frac{\Rmax^2\log(\delta^{-1})\log\log(\Texp)\cdot\Ccov(\pist)}{\beta\Tspan},
\end{align*}
where the second inequality uses \cref{eq:spanner_bound} and the third
inequality uses that $\nu\ldef{}\beta/\vepsstat$. Choosing
\[
  \Tprompt = \wt{\Theta}\prn*{\frac{\Rmax^2}{\beta^2}\cdot\Texp},\mathand
\Tspan = \wt{\Theta}\prn*{\Ccov(\pist)\cdot\Texp}
\]
suffices to give
\begin{align*}
  \En_{t\sim\unif\prn*{\brk{\nexp}}}\brk*{\Jbeta(\pist) -
    \Jbeta(\pihat\ind{t})}
  \leq \bigoht\prn*{\frac{\vepsstat^2\cdot{}d \log
    \prn*{\delta^{-1}}}{\beta\nexp}
  } = \bigoht\prn*{\frac{d^2 \Rmax^2\log^2
  \prn*{\delta^{-1}}}{\beta\nexp}
  }.
\end{align*}
so that setting
\arxiv{\begin{align*}
  \Texp = \wt{\Theta}\prn*{\frac{d^2 \Rmax^2\log^2
  \prn*{\delta^{-1}}}{\beta\veps}
  }
\end{align*}}
suffices to achieve $\En_{t\sim\unif\prn*{\brk{\nexp}}}\brk*{\Jbeta(\pist) -
    \Jbeta(\pihat\ind{t})}\leq\veps$. We now bound the number of reward/prompt
queries and sampling oracle queries. First, note that during the
spanner construction phase, the algorithm queries the reward oracle
twice whenever it expands $\Psispan$, and does not query it
otherwise. Meanwhile, it queries the reward oracle twice at each round
of the exploration phase. Consequently, by \cref{lem:spanner_size}, we have
\begin{align*}
  \Tdatafull
  \leq 2\prn*{\abs*{\Psispan} + \Texp}
  \leq\bigoht\prn*{
  \frac{d}{\nu^2}
  + \Texp
    }
  \leq\bigoht\prn*{
  \frac{d^2\Rmax^2\log(\delta^{-1})}{\beta^2}
  + \frac{d^2 \Rmax^2\log^2
  \prn*{\delta^{-1}}}{\beta\veps}
  }.
\end{align*}
where we have used that $\nu=\beta/\vepsstat$. We also observe that
the number of prompts used by the algorithm is
\begin{align*}
  \Tprompt + \Texp = \bigoht\prn*{\frac{\Rmax^2}{\beta^2}\cdot\Texp + \Texp}
=
  \bigoht\prn*{
  \frac{d^2 \Rmax^4\log^2
  \prn*{\delta^{-1}}}{\beta^3\veps}
  }.
\end{align*}
To bound the number of
sampling oracle queries, we note that the algorithm queries the sampling
oracle twice during each inner loop iteration of the spanner construction phase, and
calls it $\bigoh(\Mrej\log(\deltarej^{-1}))=\bigoht(\Ccov(\pist))$ times during each round
of the exploration phase (through the invocation of \rejection
(\cref{alg:rejection})). We can thus bound
\begin{align*}
  \Tsamplefull
  \leq \bigoht\prn*{\Tprompt\Tspan
  + \Ccov(\pist)\cdot{}\Texp
  }
  \leq \bigoht\prn*{\Ccov(\pist)\cdot\frac{\Rmax^2}{\beta^2}\cdot\Texp^2}.
\end{align*}

\end{proof}

\subsubsection{Proofs for Supporting Lemmas}

  \begin{proof}[\pfref{lem:spanner_rejection}]
    We begin by writing
        \begin{align*}
          \Jbbeta(\pibt)
      - \Jbeta(\pihat\ind{t})
      &= \En_{x\sim\rho,y'\sim\piref(\cdot\mid{}x),y\sim\pibt(\cdot\mid{}x,y')}\brk*{
     \rstar(x,y)
        -\beta\Dkl{\pibt(\cdot\mid{}x,y')}{\piref(\cdot\mid{}x)}}\\
          &~~~~- \En_{x\sim\rho,y'\sim\piref(\cdot\mid{}x),y\sim\pihat\ind{t}(\cdot\mid{}x,y')}\brk*{
            \rstar(x,y)}
            + \beta\En_{x\sim\rho}\brk*{
            \Dkl{\pihat\ind{t}(\cdot\mid{}x)}{\piref(\cdot\mid{}x)}
            }\\
          &\leq{} \En_{x\sim\rho,y'\sim\piref(\cdot\mid{}x),y\sim\pibt(\cdot\mid{}x,y')}\brk*{
     \rstar(x,y)
        -\beta\Dkl{\pibt(\cdot\mid{}x,y')}{\piref(\cdot\mid{}x)}}\\
          &~~~~- \En_{x\sim\rho,y'\sim\piref(\cdot\mid{}x),y\sim\pihat\ind{t}(\cdot\mid{}x,y')}\brk*{
            \rstar(x,y)- \beta
            \Dkl{\pihat\ind{t}(\cdot\mid{}x,y')}{\piref(\cdot\mid{}x)}
            },
        \end{align*}
        where the inequality uses convexity of KL-divergence. We can
        further bound this by
        \begin{align*}
          &\underbrace{\En_{x\sim\rho,y'\sim\piref(\cdot\mid{}x)}
          \brk*{(\En_{y\sim\pibt(\cdot\mid{}x,y')}\brk*{\rstar(x,y)}
          -
          \En_{y\sim\pihat\ind{t}(\cdot\mid{}x,y')}\brk*{\rstar(x,y)})\indic\crl*{(x,y')\in\cZgood}}}_{\termi}\\
          &+          \underbrace{\beta\En_{x\sim\rho,y'\sim\piref(\cdot\mid{}x)}
            \brk*{(\Dkl{\pihat\ind{t}(\cdot\mid{}x,y')}{\piref(\cdot\mid{}x)}
            - \Dkl{\pibt(\cdot\mid{}x,y')}{\piref(\cdot\mid{}x)})
            \indic\crl*{(x,y')\in\cZgood}}}_{\termii}\\
          &+\underbrace{\En_{x\sim\rho,y'\sim\piref(\cdot\mid{}x)}
          \brk*{(\Rmax + \beta \Dkl{\pihat\ind{t}(\cdot\mid{}x,y')}{\piref(\cdot\mid{}x)})\indic\crl*{(x,y')\notin\cZgood}}}_{\termiii}.
        \end{align*}
For the first two terms above, our choice for $\Mrej$ and
\cref{eq:zgood} imply that
whenever $(x,y')\in\cZgood$, the conditions of \cref{lem:rejection_kl}
apply, so we have
\begin{align*}
  \termi \leq{} 2\Rmax\deltarej,\mathand
  \termii \leq{} \beta\cdot\bigoh\prn*{\frac{\Rmax}{\beta} +
  \log(\Ccov(\pist)\log(\deltarej^{-1}))}\cdot\deltarej
  \leq\bigoh(\Rmax\deltarej\log\log(\deltarej^{-1}))
\end{align*}
as long as $\beta\leq\Rmax$. Meanwhile, \cref{lem:rejection_kl} also implies that for all $(x,y')$,
\[
\beta\Dkl{\pihat\ind{t}(\cdot\mid{}x,y')}{\piref(\cdot\mid{}x)}\leq\bigoh(\Rmax+\beta\log(\Ccov(\pist)\log(\deltarej^{-1}))\leq\bigoh(\Rmax\log\log(\deltarej^{-1})),
\]
and hence
\begin{align*}
  \termiii
  &\leq \bigoh(\Rmax\log\log(\deltarej^{-1}))\cdot
    \En_{x\sim\rho,y'\sim\piref(\cdot\mid{}x)}\brk*{\indic\crl*{(x,y')\notin\Zgood}}.
\end{align*}
To conclude, we use the definition of $\cXspan$ to bound
\begin{align*}    
\En_{x\sim\rho,y'\sim\piref(\cdot\mid{}x)}\brk*{\indic\crl*{(x,y')\notin\Zgood}}    &\leq   \En_{x\sim\rho,y'\sim\piref(\cdot\mid{}x)}\brk*{\indic\crl*{(x,y')\notin\Zgood,x\in\cXspan}}
      + \vepsprompt\\
  &\leq 
2\En_{x\sim\rho}\brk*{\bbP_{y'\sim\piref(\cdot\mid{}x),y\sim\pist(\cdot\mid{}x)}\brk*{\nrm*{\phidel(x,y,y')}_{\Sigmaspan^{-1}}>\nu}\indic\crl*{x\in\cXspan}}
    + \vepsprompt\\
    &\leq 
2\Ccov(\pist)\En_{x\sim\rho}\brk*{\bbP_{(y,y')\sim\piref(\cdot\mid{}x)}\brk*{\nrm*{\phidel(x,y,y')}_{\Sigmaspan^{-1}}>\nu}\indic\crl*{x\in\cXspan}}
  + \vepsprompt\\
    &\leq 2\Ccov(\pist)\cdot\vepsspan + \vepsprompt.
\end{align*}
where the second inequality above is Markov's inequality.
  \end{proof}

\newpage
\section{Proofs from \creftitle{sec:computational}}
\label{app:computational}

In this section we prove \cref{thm:training}; we formally state the Randomized Exponential Time Hypothesis and restate the theorem below.
\begin{conjecture}[Randomized Exponential Time Hypothesis \citep{calabro2008complexity}]\label{conj:randeth}
There is no randomized algorithm with time complexity $2^{o(n)}$ that, given a \textsc{3-SAT} formula $\varphi$ with $n$ clauses, has the following guarantee:
\begin{itemize}
\item If $\varphi$ is satisfiable, then the output is \textsc{Yes} with probability at least $1/2$.
\item If $\varphi$ is unsatisfiable, then the output is \textsc{No}.
\end{itemize}
\end{conjecture}

\complower*

Recall the definition of a proper alignment algorithm from \cref{def:proper-exploration}. We note in passing that our proof shows that \cref{thm:training} holds in a stronger sampling oracle model where the algorithm directly observes the log-probability $\log\pitheta(y\mid{}x)$ for each response sampled from the oracle.\loose

\paragraph{Preliminaries} We say that $\Alg$ is an online alignment algorithm for linear softmax policies with parameter set $\Theta$ if, for any given $d \in \NN$, $\beta > 0$, and $\veps,\delta>0$, $\Alg$ solves any $d$-dimensional instance with regularization parameter $\beta$, regret $\veps$, and failure probability $\delta$. In order to be explicit, we write $\Tdata(d,\beta,\veps,\delta)$ to denote the number of reward oracle queries used by $\Alg$, and $\Tcomp(d,\beta,\veps,\delta)$ to denote the number of strong sampling oracle queries.\loose

\subsection{Overview of Proof}

Note that \cref{thm:training} does not require the \emph{output} of the alignment algorithm to itself be proper, i.e. lie in the policy class $\Pi$; per the definition of a proper alignment algorithm (\cref{def:proper-exploration}), it only requires the exploratory policies to be proper. To prove \cref{thm:training}, the primary building block is the following weaker result, \cref{thm:training-proper}, which gives hardness under the additional assumption that the output policy is required to lie in $\Pi$. We deduce \cref{thm:training} from this result by showing that one can use imitation learning to efficiently convert any improper output policy into a proper one (\cref{lemma:bc-distill}); this leverages the fact that behavior cloning with the log-loss is computationally efficient for linearly parametrized softmax policies \citep{rohatgi2025computational}.

\begin{theorem}
  \label{thm:training-proper}
Under the Randomized Exponential Time Hypothesis
(\cref{conj:randeth}), there is no proper alignment algorithm, even
with a \emph{strong oracle}\arxiv{ (\cref{def:oracle})} and a Euclidean projection oracle for $\Theta$, that (i)
has $\Tdatafull \leq
\poly\prn*{d,\beta^{-1},\veps^{-1},\delta^{-1}, \log \frac{1}{\min_{x,y} \piref(y\mid{}x)}}$ and \[\Tsamplefull \leq \poly\prn*{d, \exp\prn{\beta^{-1}},
  \veps^{-1},\delta^{-1}, |\cY|, \log \frac{1}{\min_{x,y} \piref(y\mid{}x)}}\] under \cref{ass:norm} (with $\Rmax=1$, $B=\sqrt{d}$), (ii) has runtime
$\poly\prn[\Big]{d,\exp\prn{\beta^{-1}},\veps^{-1},\delta^{-1},|\cY|,\log \frac{1}{\min_{x,y} \piref(y\mid{}x)}}$, and (iii) has output $\pihat \in \Pi$.\loose
\end{theorem}

We prove this hardness
result in the simpler fixed-prompt setting (i.e. $\cX =
\{\perp\}$). For notational convenience, we henceforth omit all
dependencies on $\perp$, i.e. we write $\pi(y) := \pi(y\mid{}\perp)$
and $\rstar(y)\ldef\rstar(\perp,y) $ for any response $y \in
\cY$. We prove the result for parameter set $\Theta :=
\{\theta \in \RR^d: \norm{\theta}_\infty \leq 1\}$, which is indeed contained in the Euclidean ball of radius $B=\sqrt{d}$, and admits an efficient Euclidean projection oracle. The proof is based on a reduction from the NP-hard \emph{Max-$k$-DNF}
problem.%

\begin{definition}[Max-$k$-DNF formula]\label{def:dnf}
Fix $n,m,k \in \NN$. A \emph{Max-$k$-DNF formula with $n$ variables and $m$ clauses} is a tuple $\varphi = (\MC_1,\dots,\MC_m)$, where each \emph{clause} $\MC_i$ consists of a subset $S_i \subseteq [n]$ of size $|S_i| \leq k$, and a partial assignment $f_i: S_i \to \{-1,1\}$. The \emph{value} of $\varphi$ is 
\[\valDNF(\varphi) := \max_{x \in \{-1,1\}^n} \valDNF(\varphi;x),\] 
where 
\[\valDNF(\varphi;x) := \sum_{i=1}^m \mathbbm{1}[\forall j \in S_i: x_j = f_i(j)].\]
\end{definition}

The Max-$k$-DNF problem is to compute $\valDNF(\varphi)$ for a given
formula $\varphi$. Under the randomized Exponential Time Hypothesis
(\cref{conj:randeth}), even \emph{approximating} this value is
computationally hard \--- see \cref{thm:boosted-dnf-hardness} in
\cref{sec:maxkdnf} for the precise statement that we will need. This motivates the following reduction, which shows that any proper online alignment algorithm for the linear softmax policy class gives an approximation algorithm for Max-$k$-DNF.

\begin{lemma}\label{thm:rlhf-to-dnf}
Let $\Alg$ be a proper (\cref{def:proper-exploration}) online alignment algorithm for linear softmax policies, in the strong oracle setting, with parameter set $\Theta$, which uses $\Tdata(\cdot)$ reward oracle queries and has time complexity bounded by $\Tcomp(\cdot)$. Suppose also that the output of $\Alg$ lies in $\Pi$. %
Define
\[
\beta(k,\delta) := \frac{1}{k^2 \log(16/\delta)}, \mathand \veps(k,\delta) := \frac{\delta^2}{16k^2 \log(16/\delta)}.
\]
Then there is an algorithm $\Alg'$ for Max-$k$-DNF with the following guarantee: given any parameter $\delta>0$ and Max-$k$-DNF formula $\varphi$ with $d$ variables and $m$ clauses,
\begin{itemize}[leftmargin=*]
    \item If $\valDNF(\varphi) \geq \delta m$ and $\Tdata(d,\beta(k,\delta),\veps(k,\delta),1/4) \leq 2^k$, then $\Alg'$ outputs \textsc{Yes} with probability at least $1/4$.\loose
    \item If $\valDNF(\varphi) \leq \frac{\delta m}{16 \cdot \Tdata(d, \beta(k,\delta), \veps(k,\delta),1/4)}$, then $\Alg'$ outputs \textsc{No}.
\end{itemize}
Moreover, the time complexity of $\Alg'$ is $\poly(d,m) \cdot \Tcomp(d,\beta(k,\delta),\veps(k,\delta),1/4)$.
\end{lemma}

\paragraph{Organization of appendix} In \cref{sec:rlhf-to-dnf_proof}, we prove \cref{thm:rlhf-to-dnf}. In \cref{sec:training_proper_proof}, we use this result to prove \cref{thm:training-proper}: in particular, if the proper alignment algorithm hypothesized in \cref{thm:training-proper} exists, then \cref{thm:rlhf-to-dnf} gives an algorithm for approximating Max-$k$-DNF that, by \cref{thm:boosted-dnf-hardness}, violates the randomized Exponential Time Hypothesis (\cref{conj:randeth}). Note that in \cref{thm:rlhf-to-dnf}, the approximation factor for $\Alg'$ depends on the algorithm's sample complexity $\Tdata(d, \beta(k,\delta), \veps(k,\delta),1/4)$, and consequently the assumption that the algorithm is data-efficient (i.e., $\Tdata$ does not scale with $\exp(\beta^{-1})$ or $\Ccov(\pistarb)$) is essential for the argument to hold. Finally, in \cref{sec:training_proof} we complete the proof of \cref{thm:training} by showing that properness of the output policy is essentially without loss of generality from a computational perspective. The hardness of approximation result for Max-$k$-DNF is deferred to \cref{sec:maxkdnf}.

\subsection{Proof of \creftitle{thm:rlhf-to-dnf}}
\label{sec:rlhf-to-dnf_proof}

Before proceeding to the proof of \cref{thm:training}, we prove \cref{thm:rlhf-to-dnf} by introducing a method for
embedding a Max-$k$-DNF formula into an instance of the online
alignment problem satisfying\arxiv{ \cref{eq:linear_reward}}.

\paragraph{Embedding a DNF formula} Given a Max-$k$-DNF formula $\varphi = (\MC_1,\dots,\MC_m)$ with $d$ variables and $m$ clauses, and some $\beta>0$, we define an instance $\cI(\varphi)$ of the online alignment problem (\cref{sec:background}) as follows. As discussed above, we set prompt space $\cX := \{\perp\}$ and omit dependences on $\perp$ henceforth. We set response space $\cY := \{0\} \cup [m]$. We set the regularization parameter to be $\beta$. The reference policy $\piref \in \Delta(\cY)$ is defined by $\piref(0) := 1-\epref$ and $\piref(i) := \epref/m$ for all $i \in [m]$, where we write $\epref := e^{-1/\beta}$. We consider the linear softmax policy class $\Pi=\{\pi_\theta:\theta\in\Theta\}$ with $\Theta := \{\theta\in\RR^d: \norm{\theta}_\infty \leq 1\}$, and with feature mapping $\phi: \cY \to \RR^d$ defined by $\phi(0) := 0$ and 
\[\phi(i)_j := \begin{cases} \frac{f_i(j)}{k} & \text{ if } j \in S_i \\ 0 & \text{ otherwise } \end{cases}\]
for each $i \in [m]$ and $j \in [d]$ (recall from \cref{def:dnf} that $S_i$ and $f_i$ are the variable set and partial assignment, respectively, corresponding to clause $\MC_i$). Finally, the reward function $\rstar: \cY \to [-1,1]$ is defined by $\rstar(y) = \langle \phi(y), \thetastar \rangle$ where $\thetastar \in \{-1,1\}^d \subset \Theta$ is any vector satisfying $\valDNF(\varphi;\thetastar) = \valDNF(\varphi)$. Since the reward is linear in $\thetastar$, \cref{ass:realizable} is satisfied, and since $\norm{\phi(y)}_2 \leq \norm{\phi(y)}_1 \leq 1$ for all $y\in\cY$ and $\Theta$ is contained in the Euclidean ball of radius $\sqrt{d}$, we see that \cref{ass:norm} is satisfied with $\Rmax := 1$ and $B := \sqrt{d}$.\footnote{Technically, \cref{ass:norm} requires the rewards to lie in $[0,1]$. This is straightforward to fix by adding a constant feature $\phi(i)_{d+1} := 1/2$, scaling all other features by a factor of $1/2$, and setting $\thetastar_{d+1} = 1$. With this modification, \cref{lemma:vis-val-bounds} still holds except the additive term $e^{-1/(\beta k)}$ becomes $e^{-1/(2\beta k)}$, and the proof of \cref{thm:rlhf-to-dnf} goes through unchanged so long as $k \geq 2$.}\loose

The following lemma relates the value of $\varphi$ to the maximum likelihood (over all policies in the policy class $\Pi$) of observing some non-zero response. More precisely, it shows that if $\valDNF(\varphi)$ is large, then $\pi_{\thetastar}$ (the optimal KL-regularized policy) puts non-trivial mass on non-zero responses; conversely, if $\valDNF(\varphi)$ is small and $\beta$ is sufficiently small, then \emph{no} policy puts non-trivial mass on non-zero responses.

\begin{lemma}\label{lemma:vis-val-bounds}
Suppose that $\beta \leq 1/\log(2)$. It holds that $\sum_{i=1}^m \pi_{\thetastar}(i) \geq \frac{\valDNF(\varphi)}{2m}$ and, for any $\theta \in \Theta$,
\[\sum_{i=1}^m \pi_\theta(i) \leq 2\left(\frac{\valDNF(\varphi;\sgn(\theta))}{m} + e^{-1/(\beta k)}\right).\]
\end{lemma}

\begin{proof}[\pfref{lemma:vis-val-bounds}]
For each $i \in [m]$ such that the assignment $\thetastar$ satisfies clause $\MC_i$, we have by definition that $\thetastar_j = f_i(j)$ for all $j \in S_i$; hence, by definition of $\phi(i)$, \[e^{\frac{1}{\beta}\langle \thetastar, \phi(i)\rangle} = e^{\frac{1}{k\beta} \sum_{j \in S_i} f_i(j) \thetastar_j} = e^{1/\beta}.\]
Since $\thetastar$ satisfies $\valDNF(\varphi;\thetastar)$ clauses, and $\piref(i) = e^{-1/\beta}/m$ for all $i \in [m]$, we get
\[\sum_{i=1}^m \pi_{\thetastar}(i) = \frac{\sum_{i=1}^m \piref(i) e^{\frac{1}{\beta} \langle \thetastar, \phi(i)\rangle}}{\piref(0) + \sum_{i=1}^m \piref(i) e^{\frac{1}{\beta} \langle \thetastar, \phi(i)\rangle}} \geq \frac{\valDNF(\varphi;\thetastar)/m}{\piref(0) + \valDNF(\varphi;\thetastar)/m} \geq \frac{\valDNF(\varphi;\thetastar)}{2m}\]
where the first inequality uses monotonicity of $z \mapsto \frac{z}{\piref(0)+z}$, and the final inequality uses that $\piref(0) \leq 1$ and $\valDNF(\varphi;\thetastar) \leq m$.

Next, for any $\theta \in \Theta$, set $x := \sgn(\theta)$. For each $i \in [m]$, we have the bound $e^{\frac{1}{\beta}\langle \theta, \phi(i)\rangle} \leq e^{\frac{1}{\beta} \norm{\theta}_\infty \norm{\phi(i)}_1} \leq e^{1/\beta}$. Additionally, if the assignment $x$ does not satisfy clause $\MC_i$, then there is some $j^\star \in S_i$ such that $\phi(i)_{j^\star} \theta_{j^\star} \leq 0$, and thus 
\[e^{\frac{1}{\beta}\langle \theta,\phi(i)\rangle} \leq e^{\frac{1}{\beta k}\sum_{j \in S_i\setminus\{j^\star\}} f_i(j) \theta_j} \leq e^{\frac{k-1}{k\beta}},\]
where the final inequality uses that $|f_i(j)| \leq 1$ and $|\theta_j| \leq 1$ for all $j \in S_i \setminus \{j^\star\}$. Recalling that $\piref(i)=\frac{1}{e^{1/\beta}m}$ for $i\in\brk{m}$, it follows that
\[\sum_{i=1}^m \piref(i) e^{\frac{1}{\beta}\langle \theta,\phi(i)\rangle} \leq \frac{1}{e^{1/\beta}m}\left(me^{\frac{k-1}{k\beta}} + \valDNF(\varphi;x) e^{1/\beta}\right) = \frac{\valDNF(\varphi;x)}{m} + e^{-1/(\beta k)}.\]
Thus,
\[\sum_{i=1}^m \pi_\theta(i) = \frac{\sum_{i=1}^m \piref(i) e^{\frac{1}{\beta}\langle \theta,\phi(i)\rangle}}{\piref(0) + \sum_{i=1}^m \piref(i) e^{\frac{1}{\beta}\langle \theta,\phi(i)\rangle}} \leq 2\sum_{i=1}^m \piref(i) e^{\frac{1}{\beta}\langle \theta,\phi(i)\rangle} \leq 2 \left(\frac{\valDNF(\varphi;x)}{m} + e^{-1/(\beta k)}\right)\]
where the first inequality uses the bound $\piref(0) = 1-e^{-1/\beta} \geq 1/2$.
\end{proof}

The following lemma implies that approximately maximizing $\sum_{i=1}^m \pi_\theta(i)$ is necessary in order to obtain low (KL-regularized) regret.

\begin{lemma}\label{lemma:regret-vis-lb}
For any $\pi \in \Delta(\cY)$,
\[J_\beta(\pi_{\thetastar}) - J_\beta(\pi) = \beta \Dkl{\pi}{\pi_{\thetastar}} \geq \beta \left( \sum_{i=1}^m \pi(i) - \sum_{i=1}^m \pi_{\thetastar}(i)\right)^2.\]
\end{lemma}

\begin{proof}[\pfref{lemma:regret-vis-lb}]
The inequality is a consequence of Pinsker's inequality; the equality is via the following standard manipulation. Define $Z_{\thetastar} := \sum_{y \in \cY} \piref(y) e^{\frac{1}{\beta} \langle \thetastar,\phi(y)\rangle}$. For any $y \in \cY$, we have by definition of $\pi_\thetastar(y)$ that
\[\langle \thetastar,\phi(y)\rangle - \beta \log \frac{\pi_{\thetastar}(y)}{\piref(y)} = \beta \log Z_{\thetastar}\]
which is notably independent of $y$. Thus,
\begin{align*}
J_\beta(\pi_{\thetastar}) - J_\beta(\pi)
&= \EE_{y \sim \pi_{\thetastar}}\left[\langle {\thetastar},\phi(y)\rangle - \beta \log \frac{\pi_{\thetastar}(y)}{\piref(y)}\right] - \EE_{y \sim \pi}\left[\langle \thetastar,\phi(y)\rangle - \beta \log \frac{\pi(y)}{\piref(y)}\right] \\ 
&= \EE_{y \sim \pi_{\thetastar}}\left[\langle \thetastar,\phi(y)\rangle - \beta \log \frac{\pi_{\thetastar}(y)}{\piref(y)}\right] - \EE_{y \sim \pi}\left[\langle \thetastar,\phi(y)\rangle - \beta \log \frac{\pi_{\thetastar}(y)}{\piref(y)}\right] \\ 
&\qquad + \EE_{y \sim \pi}\left[\beta \log \frac{\pi(y)}{\pi_{\thetastar}(y)}\right] \\ 
&= \beta\Dkl{\pi}{\pi_{\thetastar}}.
\end{align*}
This completes the proof.
\end{proof}

Together, \cref{lemma:vis-val-bounds,lemma:regret-vis-lb} suggest that
for appropriate parameter choices, solving the constructed online
alignment problem necessitates solving the original Max-$k$-DNF
problem. The remaining (important!) subtlety is that it is impossible to efficiently simulate the data collection when only given $\varphi$, since the reward function $\rstar$ depends on the maximally satisfying assignment $\thetastar$. Instead, we \emph{approximately} simulate the data collection by always producing reward $0$. This is only incorrect on non-zero responses, so in round $t$, it is unlikely to be incorrect unless $\sum_{i=1}^m \pi^t(i)$ is already non-negligible. In this event, the simulation may fail, but since $\Alg$ was assumed to be a proper-exploration algorithm, we can apply \cref{lemma:vis-val-bounds} to $\pi^t = \pi_{\theta_t}$ to show that $\valDNF(\varphi;\sgn(\theta_t))$ is a decent approximation for $\valDNF(\varphi)$. Thus, we get a win-win reduction, where the approximation factor scales with the number of data collection rounds. We make this argument formal below, proving \cref{thm:rlhf-to-dnf}.\loose
\vspace{1em}
\begin{proof}[Proof of \cref{thm:rlhf-to-dnf}]
  Given a proper online alignment algorithm $\Alg$ and a
  Max-$k$-DNF formula $\varphi$ with $d$ variables and $m$ clauses, and a parameter $\delta \in (0,1)$, we define $\Alg'$ to have the below behavior. Throughout the remainder of the proof, we abbreviate $\beta \ldef \beta(k,\delta)$ and $\veps \ldef \veps(k,\delta)$ for simplicity:
\begin{enumerate} 
\item Simulate $\Alg$ on the online alignment instance $\cI(\varphi)$ defined above (with regularization parameter $\beta$, error tolerance $\veps$, and failure probability $1/4$), but use reward function $\wh r(y) := 0$ instead of $\rstar$. In particular:\loose
\begin{itemize}
\item When $\Alg$ queries the sampling oracle with $\theta \in \Theta$, $\Alg'$ computes $\pi_\theta \in \Delta(\cY)$, samples $y \sim \pi_\theta$, and passes response $y$ to $\Alg$.
\item When $\Alg$ initiates data collection round $t$ with exploration policy $\pi^t = \pi_{\theta_t}$, $\Alg'$ computes $\pi_{\theta_t}$, samples $y \sim \pi_{\theta_t}$, and passes response $y$ and reward $0$ to $\Alg$. 
\end{itemize}
Let $q$ denote the number of data collection rounds. Let $\wh\pi = \pi_{\thetafinal}$ denote the final policy output by $\Alg$. %
\item Compute
\[\thetahat := \argmax_{\theta \in\{\theta_1,\dots,\theta_q,\thetafinal\}} \sum_{i=1}^m \pi_\theta(i).\]
\item Set $\wh x := \sgn(\thetahat) \in \{-1,1\}^d$. $\Alg'$ outputs \textsc{Yes} if $\valDNF(\varphi;\wh x) > \frac{\delta m}{16 \cdot S(d,\beta(k,\delta),\veps(k,\delta))}$ and \textsc{No} otherwise.
\end{enumerate}

We now analyze $\Alg'$. Suppose that $\val(\varphi) \geq \delta m$ and $\Tdata(d,\beta,\veps,1/4) \leq 2^k$. Let $\Algtil'$ denote the idealized modification of $\Alg'$ that simulates $\Alg$ with the true reward function $\rstar$ (which is computationally hard to implement). Further, let $\Algbar'$ denote the idealized modification of $\Alg'$ that simulates $\Alg$ with the true reward function only on queries $\theta_t$ with $\sum_{i=1}^m \pi_{\theta_t}(i) > \frac{1}{4 \Tdata(d,\beta,\veps,1/4)}$ (and with reward $0$ otherwise). We can couple the executions of $\Algtil'(\varphi,\delta)$ and $\Algbar'(\varphi,\delta)$ with the execution of $\Alg'(\varphi,\delta)$. Observe that the execution of $\Algtil'$ only deviates from the execution of $\Algbar'$ if there is some round $t$ where $\sum_{i=1}^m \pi_{\theta_t}(i) \leq \frac{1}{4 \Tdata(d,\beta,\veps,1/4)}$ and yet the sample $y \sim \pi_{\theta_t}$ is non-zero. For any fixed $t$, this occurs with probability at most $\frac{1}{4 \Tdata(d,\beta,\veps,1/4)}$, so by a union bound and the assumption that $\Alg$ uses at most $\Tdata(d,\beta,\veps,1/4)$ rounds, the two algorithms deviate with probability at most $1/4$. Next, the execution of $\Algbar'$ deviates from the execution of $\Alg'$ only if there is some round $t$ such that $\sum_{i=1}^m \pi_{\theta_t}(i) > \frac{1}{4\Tdata(d,\beta,\veps,1/4)}$. Thus,  %
\begin{align*} 
\Pr[\Algtil'(\varphi,\delta) \neq \Alg'(\varphi,\delta)]
&\leq \Pr[\Algtil'(\varphi,\delta) \neq \Algbar'(\varphi,\delta)] + \Pr[\Algbar'(\varphi,\delta) \neq \Alg'(\varphi,\delta)] \\ 
&\leq \frac{1}{4} + \underbrace{\Pr^{\Alg'}\left[\max_{1 \leq j \leq q} \sum_{i=1}^m \pi_{\theta_j}(i) > \frac{1}{4 \Tdata(d,\beta,\veps,1/4)}\right]}_{\text{\textdagger}}.
\end{align*}

We distinguish two cases based on the value of \textdagger. %
\begin{itemize}
    \item In the first case, if \textdagger{} is at most $1/4$, then 
    \[\Pr[\Algtil'(\varphi,\delta) \neq \Alg'(\varphi,\delta)] \leq 1/2.\]
    By the guarantee of $\Alg$, it holds with probability at least $3/4$ over the execution of $\Algtil'$ that the output $\thetafinal$ of the simulated $\Alg$ satisfies $J_\beta(\pi_\thetastar) - J_\beta(\pi_{\thetafinal}) \leq \veps$. Hence, the same bound holds with probability at least $1/4$ over the execution of $\Alg'$. Condition on this event. Then
    \[\sum_{i=1}^m \pi_{\thetahat}(i) \geq \sum_{i=1}^m \pi_{\thetafinal}(i) \geq \sum_{i=1}^m \pi_\thetastar(i) - \sqrt{\frac{\veps}{\beta}} \geq \frac{\valDNF(\varphi)}{2m} - \sqrt{\frac{\veps}{\beta}} \geq \frac{\delta}{4}\]
    where the first inequality is by definition of $\wh\theta$, the second inequality is by \cref{lemma:regret-vis-lb}, the third inequality is by \cref{lemma:vis-val-bounds}, and the fourth inequality is by assumption that $\valDNF(\varphi) \geq \delta m$ and choice of $\veps = \veps(k,\delta)$. It follows that
    \[\frac{\valDNF(\varphi;\sgn(\thetahat))}{m} \geq \frac{1}{2}\sum_{i=1}^m \pi_{\thetahat}(i) - e^{-1/(\beta k)} 
      \geq\frac{1}{2}\prn*{\frac{\delta}{4}-\prn*{\frac{\delta}{16}}^{k}} \geq \frac{\delta}{16},\]
    where the first inequality is by \cref{lemma:vis-val-bounds}, and the second inequality is by choice of $\beta = \beta(k,\delta)$.\loose
    \item In the second case, \[\text{\textdagger} = \Pr^{\Alg'}\left[\max_{1 \leq j \leq q} \sum_{i=1}^m \pi_{\theta_j}(i) > \frac{1}{4 \Tdata(d,\beta,\veps,1/4)}\right] > \frac{1}{4}.\] Condition on the event within the probability occurring. In this event, which occurs with probability at least $1/4$, we have $\sum_{i=1}^m \pi_{\thetahat}(i) \geq \frac{1}{4 \cdot \Tdata(d,\beta,\veps,1/4)}$ by definition of $\wh\theta$, and hence 
    \begin{align*} \frac{\valDNF(\varphi;\sgn(\thetahat))}{m} 
    &\geq \frac{1}{2}\sum_{i=1}^m \pi_{\thetahat}(i) - e^{-1/(\beta k)} \\
    &\geq \frac{1}{8 \cdot \Tdata(d,\beta,\veps,1/4)} - \left(\frac{\delta}{16}\right)^k \\
    &> \frac{\delta}{16 \cdot \Tdata(d,\beta,\veps,1/4)}
    \end{align*}
    where the first inequality is by \cref{lemma:vis-val-bounds}, the second is by the conditioning and the choice of $\beta$, and the third inequality is by the theorem assumption that $\Tdata(d,\beta,\veps,1/4) \leq 2^k$ and $\delta\in(0,1)$.
\end{itemize}
In both cases, the output of $\Alg'$ is therefore \textsc{Yes} with probability at least $1/4$. On the other hand, if $\valDNF(\varphi) \leq \frac{\delta m}{16 \cdot \Tdata(d, \beta(k,\delta), \veps(k,\delta),1/4)}$, then it is immediate that $\Alg'$ outputs \textsc{No}. Finally, the time complexity of $\Alg'$ is
\[\poly(d,m) \cdot \Tcomp\left(d, \frac{1}{k^2 \log(16/\delta)}, \frac{\delta^2}{16k^2 \log(16/\delta)}, 1/4\right)\]
by the assumed time complexity bound for $\Alg$ and the fact that for any given $\theta \in \Theta$, the distribution $\pi_\theta$ can be explicitly computed from $\varphi$ in time $\poly(d,m)$.
\end{proof}

\subsection{Proof of \creftitle{thm:training-proper}}
\label{sec:training_proper_proof}

We now prove \cref{thm:training-proper} by combining \cref{thm:rlhf-to-dnf} with a hardness of approximation result for Max-$k$-DNF (\cref{thm:boosted-dnf-hardness}, stated and proven in \cref{sec:maxkdnf}).
\vspace{1em}

\begin{proof}[Proof of \cref{thm:training-proper}]
  Suppose that there is a proper alignment algorithm $\Alg$ with proper output, using a strong sampling oracle and a Euclidean projection oracle for $\Theta$, that has $\Tdata(d,\beta,\veps,\delta) \leq \poly(d,\beta^{-1},\veps^{-1},\delta^{-1},\log \frac{1}{\min_{x,y} \piref(y\mid{}x)})$ and has runtime bounded by \[\Tcomp(d,\beta,\veps,\delta) =
   \ poly\prn*{d,\exp\prn{\beta^{-1}},\veps^{-1},\delta^{-1},|\cY|,\log \frac{1}{\min_{x,y} \piref(y\mid{}x)}}.\] 
The construction in \cref{thm:rlhf-to-dnf} uses $\min_{x,y}\piref(y\mid{}x) = 1/(e^{1/\beta}m)$ and, without loss of generality, $m \leq 2^d$ (in the regime $m > 2^d$, the Max-$k$-DNF problem can be solved in $\poly(m)$ time unconditionally). Thus $\log(1/\min_{x,y}\piref(y\mid{}x)) \leq \poly(\beta^{-1}, d)$. It follows from the assumed bound on $\Tdata$ that for the problem instances constructed in the proof of \cref{thm:rlhf-to-dnf}, there is a constant $c_1>0$ such that $\Tdata(d,\beta,\veps,1/4) \leq (d/(\beta \veps))^{c_1}$. Set $\delta = \delta(d) := 1/d$ and $k = k(d) := C_{\ref{thm:boosted-dnf-hardness}} \cdot (4c_1)^2 \log(d)$. Then, recalling the definitions of $\beta(k,\delta)$ and $\veps(k,\delta)$ from \cref{thm:rlhf-to-dnf},
\[\Tdata(d,\beta(k,\delta),\veps(k,\delta),1/4) \leq \left(\frac{16d k^4 \log^2(16/\delta)}{\delta^2}\right)^{c_1} \leq O(c_1^{c_1}) \cdot (d^3 \log^6(16d))^{c_1} \leq \frac{d^{4c_1}}{16} \leq 2^k\]
where the third inequality holds for all sufficiently large $d$. By \cref{thm:rlhf-to-dnf}, there is an algorithm $\Alg'$ with the following guarantees on an input $k(d)$-DNF formula $\varphi$ with $d$ variables and $m$ clauses:
\begin{itemize}
\item If $\valDNF(\varphi) \geq \delta m = m/d$, then $\Alg'$ outputs \textsc{Yes} with probability at least $1/4$.
\item If $\valDNF(\varphi) \leq m/d^{4c_1}$, then since $16 \cdot \Tdata(d,\beta(k,\delta),\veps(k,\delta),1/4) \leq d^{4c_1}$ for sufficiently large $d$, we have $\valDNF(\varphi) \leq m/(16 \cdot \Tdata(d,\beta(k,\delta),\veps(k,\delta)),1/4)$, so $\Alg'$ outputs \textsc{No}.
\end{itemize}
Next we analyze the time complexity of $\Alg'$. The construction in \cref{thm:rlhf-to-dnf} uses $|\cY| = m$ and $\min_{x,y}\piref(y\mid{}x) = 1/(e^{1/\beta}m)$. Thus, $\Tcomp(d,\beta(k,\delta),\veps(k,\delta),1/4) \leq \poly(d,m,\exp\prn*{1/\beta(k,\delta)},1/\veps(k,\delta)) \leq \poly(2^{\bigoh(\log^3(d))},m)$. So from \cref{thm:rlhf-to-dnf}, we get that the time complexity of $\Alg'$ is $\poly(2^{\bigoh(\log^3(d))},m)$.
Hence, applying \cref{thm:boosted-dnf-hardness} with parameter $c := 4c_1$ and using that $k(d) \geq C_{\ref{thm:boosted-dnf-hardness}} (4c_1)^2 \log(d)$, we get that the randomized Exponential Time Hypothesis is false (concretely, \cref{thm:boosted-dnf-hardness} rules out the possibility of time complexity $\poly(2^{\bigoh(d^{\eta})},m)$ for a constant $\eta=\eta(4c_1)$).
\end{proof}

\subsection{Proof of \creftitle{thm:training}}\label{sec:training_proof}

We now prove \cref{thm:training} by combining \cref{thm:training-proper} with the following result, which shows that any proper alignment algorithm with improper output can be efficiently bootstrapped to a proper alignment algorithm with proper output. The proof of \cref{lemma:bc-distill} follows from an analysis of log-loss behavior cloning due to \cite{rohatgi2025computational}.\loose

\begin{lemma}\label{lemma:bc-distill}
  Let $\Alg$ be a proper alignment algorithm with a strong sampling oracle and a Euclidean projection oracle for $\Theta$, that uses $\Tdata(\cdot)$ reward oracle queries and has time complexity $\Tcomp(\cdot)$, under \cref{ass:norm} with $\Rmax := 1$ and $B := \sqrt{d}$. Suppose that the output of $\Alg$ is sampleable in time $T$ per prompt. Then, in the same setting, there is a proper alignment algorithm $\Alg'$ \emph{with proper output $\pihat \in \Pi$}, that uses $\Tdata'(d,\beta, \epsilon,\delta)=\Tdata(d,\beta, \epsilon_0,\delta/2)$ reward oracle queries and has sampling oracle complexity \[
    \Tsample'(d,\beta, \epsilon,\delta) = \Tsample(d,\beta, \epsilon_0,\delta/2) + \poly\prn*{d,|\cY|, B, T,\epsilon^{-1}, \log(\delta^{-1}), \log \frac{1}{\min_{x,y} \piref(y\mid{}x)}},\]
where $\epsilon_0 := \frac{c\epsilon}{B\log(1/\min_{x,y}\piref(y\mid{}x))}$ for a universal constant $c>0$, and we assume that $\beta \in (0,1)$. Furthermore, the time complexity of $\Alg'$ is polynomial in the time complexity of $\Alg$ with parameters $(d,\beta, \epsilon_0,\delta/2)$ and $\poly\prn*{d,|\cY|, B, T,\epsilon^{-1}, \log(\delta^{-1}), \log \frac{1}{\min_{x,y} \piref(y\mid{}x)}}$.
\end{lemma}

\begin{proof}[Proof of \cref{lemma:bc-distill}]
For any $\theta \in \Theta$, $x\in\cX$, and $y \in \cY$, observe that
$\pi_\theta(y\mid{}x) \geq e^{-2B} \piref(y\mid{}x)$. It follows that for any $\pi:\cX\to\Delta(\cY)$,
\[\max_{x,y} \max_{\theta\in\Theta} \frac{\pi(y\mid{}x)}{\pi_\theta(y\mid{}x)} \leq W := \frac{e^{2B}}{\min_{x,y} \piref(y\mid{}x)}.\]
We will use a slight generalization of Proposition~D.2 in \cite{rohatgi2025computational}, which shows that there is an algorithm that takes an integer $n$, the norm bound $B$ from \cref{ass:norm}, some $\delta \in (0,1)$, and $n$ i.i.d. samples $(x\ind{i},y\ind{i})_{i=1}^n$ from $\pi$, and produces $\thetahat \in\Theta$ satisfying, with probability at least $1-\delta$,
\begin{align}
  \label{eq:bc}\Dhels{\pi_{\thetahat}}{\pi} \lesssim \frac{(d\log(Bn)+\log(W))\log(1/\delta)}{n} + \log(W) \cdot \min_{\theta\in\Theta} \Dhels{\pi_\theta}{\pi}.\end{align}
Additionally, the time complexity is $\poly(n,d,|\cY|, B, \log(1/\delta))$ (to be precise, Proposition~D.2 is specialized to the case where $\piref$ is uniform on $\cY$, so that $W = e^{2B}/|\cY|$, but this generalization is immediate from inspecting the proof; also, in our setting the horizon parameter $H$ from the proposition is equal to one). The algorithm is essentially gradient ascent on the log-likelihood for a set of \iid samples from $\pi$, which is concave in $\theta$ \citep{rohatgi2025computational}.

This motivates defining $\Alg'$ as follows. Execute $\Alg$ with parameters $\epsilon_0$ and $\delta/2$, and let $\pi$ be the output. Then, for a parameter $n$ to be determined, draw $n$ samples $(x\ind{i})_{i=1}^n$ from the prompt distribution $\rho$, and for each draw $y\ind{i} \sim \pi(\cdot\mid{}x\ind{i})$. Execute the algorithm described above with failure probability $\delta/2$, to compute $\thetahat \in \Theta$, and output $\pihat := \pi_{\thetahat}$.

We now analyze $\Alg'$. With probability at least $1-\delta/2$, it holds that $\beta \Dkl{\pi}{\pi_{\thetastar}} = J_\beta(\pi_{\thetastar}) - J_\beta(\pi) \leq \epsilon_0$, where the equality is by \cref{lem:kl_regret}. Also, by \cref{eq:bc}, with probability at least $1-\delta/2$ it holds that 
\begin{equation} 
\Dhels{\pihat}{\pi} \lesssim \epsilon_0 + \log(W) \cdot \Dhels{\pi_{\thetastar}}{\pi},
\label{eq:distill-hels}
\end{equation}
so long as $n \geq \poly(d, \epsilon_0^{-1}, \log(BW/\delta))$. Condition on the event that both of these bounds hold. Then
\begin{align*}
J_\beta(\pi_{\thetastar}) - J_\beta(\pihat) 
&= \beta \Dkl{\pihat}{\pi_{\thetastar}} \\ 
&\lesssim \beta \log(W) \cdot \Dhels{\pihat}{\pi_{\thetastar}} \\ 
&\lesssim \beta \log(W) \cdot (\Dhels{\pihat}{\pi} + \Dhels{\pi}{\pi_{\thetastar}}) \\ 
&\lesssim \beta\log(W) \cdot \epsilon_0 + \beta \log^2(W) \Dhels{\pi}{\pi_{\thetastar}} \\ 
&\leq \log(W) \cdot \epsilon_0 + \beta \log^2(W) \Dkl{\pi}{\pi_{\thetastar}} \\ 
&\lesssim \log^2(W) \cdot \epsilon_0
\end{align*}
where the equality is by \cref{lem:kl_regret}, the first inequality is by Lemma~4 of \cite{yang1998asymptotic}, the second inequality uses the fact that $\Dhel{\cdot}{\cdot}$ is a metric, the third inequality is by \cref{eq:distill-hels}, and the fourth inequality uses e.g. (7.33) in \cite{Polyanskiy_Wu_2025} as well as the bound $\beta \leq 1$. By definition of $W$ and $\epsilon_0$ (so long as the constant $c>0$ is sufficiently small), we can bound the above by $\epsilon$, so $\Alg'$ satisfies the correctness desideratum of a proper alignment algorithm. 

Since the second phase of $\Alg'$ uses no reward oracle queries, the claimed bound on total reward oracle queries is immediate from the parameter choices in the call to $\Alg$. Additionally, the claimed time complexity bound follows from the guarantee from \cite{rohatgi2025computational}, and the assumption that the output of $\Alg$ is sampleable in time $T$, so long as we choose $n = \poly(d, \epsilon_0^{-1}, \log(BW/\delta))$.
\end{proof}

\begin{proof}[Proof of \cref{thm:training}]
Suppose that there is a proper alignment algorithm with access to a strong sampling oracle and a Euclidean projection oracle for $\Theta$, that uses at most $\poly(d,\beta^{-1},\veps^{-1},\delta^{-1})$ reward oracle queries and has time complexity at most $\poly(d,\exp(\beta^{-1}),\veps^{-1},\delta^{-1})$, under \cref{ass:norm} with $\Rmax:= 1$ and $B := \sqrt{d}$. Without loss of generality (assuming that the output policy is represented by a circuit), the output policy is sampleable in time $T := \poly(d,\exp(\beta^{-1}),\veps^{-1},\delta^{-1})$. Thus, by \cref{lemma:bc-distill} and choice of $B$, in the same setting, there is a proper alignment algorithm with proper output that uses at most $\poly(d,\beta^{-1},\veps^{-1},\delta^{-1},\log(1/\min_{x,y}\piref(y\mid{}x)))$ reward oracle queries and has time complexity at most $\poly(d,\exp(\beta^{-1}),\veps^{-1},\delta^{-1},|\cY|, \log(1/\min_{x,y}\piref(y\mid{}x)))$. Note that the time complexity also bounds the number of sampling oracle queries. It follows from \cref{thm:training-proper} that the Randomized Exponential Time Hypothesis is false.
\end{proof}

\subsection{Hardness of Approximation for Max-$k$-DNF}
\label{sec:maxkdnf}

In this section we prove the following hardness of approximation result\arxiv{, which is needed in the proof of} \cref{thm:training}.\loose

\begin{theorem}\label{thm:boosted-dnf-hardness}
Fix any $c > 1$ and function $k: \NN \to \NN$. Suppose that $k(n) \geq C_{\ref{thm:boosted-dnf-hardness}}c^2\log(n)$ for a sufficiently large universal constant $C_{\ref{thm:boosted-dnf-hardness}}$. There is $\eta = \eta(c) > 0$ with the following property. Suppose that there is a time-$\poly(2^{n^\eta}, m)$ algorithm $\Alg$ that, given a $k(n)$-DNF formula $\varphi$ with $n$ variables and $m$ clauses, has the following behavior:
\begin{enumerate}
\item If $\valDNF(\varphi) \geq m/n$, then $\Alg$ outputs \textsc{Yes} with probability at least $1/4$.
\item If $\valDNF(\varphi) \leq m/n^c$, then $\Alg$ outputs \textsc{No}.
\end{enumerate}
Then the randomized Exponential Time Hypothesis is false.
\end{theorem}

Essentially, the above result states that for a Max-$k(n)$-DNF formula with $n$ variables, for any constant $c>1$ and so long as $k(n)$ is sufficiently large, it is computationally hard to distinguish between the case that $a := 1/n$ fraction of clauses are satisfiable versus $b := 1/n^c$ fraction of clauses are satisfiable. For the application to \cref{thm:training}, it is critical that the approximation gap $a/b$ is larger than $1/a$, since the gap generated by the reduction \cref{thm:rlhf-to-dnf} scales with $\Tdata(n,\beta(k,\delta),\veps(k,\delta),1/4)$, which in turn scales polynomially with $1/\delta = 1/a$. Additionally, we remark that it is important for $k$ to grow with $n$, since sampling a random assignment gives an efficient $2^k$-approximation algorithm.

To prove \cref{thm:boosted-dnf-hardness}, we use a result by \cite{chan2016approximation}, which states that for any constant $k$, there is a sparse $k$-predicate $P$ so that Max-$P$ (i.e. Max-$k$-CSP with predicate $P$) is hard to approximate to any factor better than $2^k/(2k)$. We then reduce Max-$P$ to Max-$k$-DNF (using sparsity of $P$ to control the decay in satisfiability thresholds) and then apply serial repetition to boost the gap. To make this approach formal, we start with the following definition.

\begin{definition}[Max-$P$ formula]\label{def:max-p}
Fix $n,m,k \in \NN$ and any $P: \{-1,1\}^k \to \{0,1\}$. A \emph{Max-$P$ formula with $n$ variables and $m$ clauses} is a tuple $\varphi = (\MC_1,\dots,\MC_m)$, where each \emph{clause} $\MC_i$ is a tuple $(v_i,b_i)$ where $v_i \in [n]^k$ and $b_i \in \{-1,1\}^k$. The \emph{value} of $\varphi$ is 
\[\val_P(\varphi) := \max_{x \in \{-1,1\}^n} \val_P(\varphi;x)\] 
where 
\[\val_P(\varphi;x) := \sum_{i=1}^m \mathbbm{1}[P(b_{i1}x_{v_{i1}},\dots,b_{ik}x_{v_{ik}}) = 1].\]
We say that the set of \emph{accepting assignments} of $P$ is $P^{-1}(1)$.
\end{definition}

 Note, for example, that \cref{def:dnf} corresponds to the predicate $P(x_1,\dots,x_k) = \mathbbm{1}[x_1=\dots=x_k=1]$. The following hardness result is essentially due to \cite{chan2016approximation}.

\begin{theorem}\label{thm:max-csp-hardness}
Let $k \in \NN$ and $\veps>0$. There is a predicate $P: \{-1,1\}^k \to \{0,1\}$ with at most $2k$ accepting assignments, and a real number $\gamma = \gamma(k,\veps) \in (0,1)$ such that the following property holds: suppose that there is a time-$O(2^{n^\gamma})$ algorithm that, given a Max-$P$ formula $\varphi$ with at most $n$ variables and $n$ clauses, has the following behavior:
\begin{enumerate}
\item If $\varphi$ is at least $(1-\veps)$-satisfiable, then it outputs \textsc{Yes} with probability at least $1/2$.
\item If $\varphi$ is at most $(2k/2^k + \veps)$-satisfiable, then it outputs \textsc{No}.
\end{enumerate}
Then the Randomized Exponential Time Hypothesis (\cref{conj:randeth}) is false.
\end{theorem}

\begin{proof}[\pfref{thm:max-csp-hardness}]Corollary~1.2 of \cite{chan2016approximation} states that for any positive integer $k$ and any $\veps>0$, for the Hadamard predicate $P: \{-1,1\}^k \to \{0,1\}$, which has at most $2k$ accepting assignments, distinguishing between $(1-\veps)$-satisfiability and $(2k/2^k+\veps)$-satisfiability of a Max-$P$ formula is NP-hard.\footnote{In fact, \cite{chan2016approximation} shows that one can even take $\veps = o(1)$, but for our purposes a constant suffices.} Thus, there is a polynomial-time reduction $\cR$ that takes as input any 3SAT formula $\tau$ with $n$ variables and $n$ clauses, and produces a Max-$P$ formula $\varphi$ with $m \leq \poly(n)$ clauses such that $\val_P(\varphi) \geq (1-\veps)m$ if $\tau$ is satisfiable, and $\val_P(\varphi) \leq (2k/2^k+\veps)m$ otherwise. There is a constant $C = C(k,\veps)>0$ such that $\cR$ has time complexity $O(n^{C(k,\veps)})$. Now let $\gamma := 1/(2 C(k,\veps))$ and suppose that there is indeed a time-$O(2^{n^\gamma})$ randomized algorithm $\Alg$ with the behavior specified in the theorem statement. Then on input $\tau$ with $n$ variables and $n$ clauses, it follows from the time complexity bound on $\cR$ that the output $\varphi$ of $\cR$ has at most $O(n^{C(k,\veps)})$ variables and clauses, so $\Alg(\varphi)$ runs in time $2^{O(n^{\gamma \cdot C(k,\veps)})} = 2^{o(n)}$ by choice of $\gamma$. If $\tau$ is satisfiable, then it outputs \textsc{Yes} with probability at least $1/2$, and otherwise it outputs \textsc{No}. This contradicts \cref{conj:randeth}.
\end{proof}

\begin{lemma}[DNF embedding]\label{lemma:dnf-embedding}
Fix $k \in \NN$ and let $P: \{-1,1\}^k \to \{0,1\}$ be a predicate with $\ell$ accepting assignments. Then there is a $\poly(n,m)$-time algorithm that takes as input a $P$-formula $\varphi$ with $n$ variables and $m$ clauses, and outputs a Max-$k$-DNF formula $\varphi'$ with $n$ variables and $\ell m$ clauses, satisfying $\val_P(\varphi) = \valDNF(\varphi')$.\loose
\end{lemma}

\begin{proof}[\pfref{lemma:dnf-embedding}]
Given $\varphi = (\MC_1,\dots,\MC_m)$, we define $\varphi'$ on the same variable set as follows. For each clause $\MC_i = (v_i, b_i)$, for each $y \in P^{-1}(1)$, we add to $\varphi'$ the clause $(S_{i,y},f_{i,y})$ defined by $S_{i,y} = \{v_{i1},\dots,v_{ik}\}$ and $f_{i,y}(v_{ij}) = b_{ij}y_{ij}$. 

Since $P$ has $\ell$ accepting assignments, $\varphi'$ has $\ell m$ clauses. Moreover, fix any assignment $x \in \{-1,1\}^n$ and clause $\MC_i$. If $x$ satisfies $\MC_i$ in $\varphi$, then there is exactly one $y \in P^{-1}(1)$ such that $(b_{i1}x_{v_{i1}},\dots,b_{ik}x_{v_{ik}}) = y$. Equivalently, there is exactly one $y \in P^{-1}(1)$ such that $x_a = f_{i,y}(a)$ for all $a \in S_{i,y}$. On the other hand, if $x$ does not satisfy $\MC_i$, then there is no such $y$. Thus, $\val_P(\varphi;x) = \valDNF(\varphi';x)$. Since this holds for all $x$, we get $\val_P(\varphi) = \valDNF(\varphi')$. Finally, it's clear that the reduction is polynomial-time in the input size.
\end{proof}

\begin{lemma}[Serial repetition]\label{lemma:serial-repetition}
There is an algorithm that takes as input a Max-$k$-DNF formula $\varphi$ with $n$ variables and $m$ clauses, and a parameter $t \in \NN$, and outputs a $kt$-DNF formula $\varphi'$ with $n$ variables and $m^t$ clauses, that has value $\valDNF(\varphi') = (\valDNF(\varphi))^t$. Moreover, the time complexity of the algorithm is $\poly(n, k, m^t)$.
\end{lemma}

\begin{proof}[\pfref{lemma:serial-repetition}]
Given $\varphi = (\MC_1,\dots,\MC_m)$, we define $\varphi'$ on the same variable set as $\varphi$, with clauses defined as follows. For each ordered tuple $(i_1,\dots,i_t) \in [m]^t$, we introduce the clause $\MC_{i_1,\dots,i_t} := \MC_{i_1} \land \dots \land \MC_{i_t}$ to $\varphi'$. Then $\varphi'$ has exactly $m^t$ clauses. Moreover, for any assignment $x \in \{-1,1\}^n$ which satisfies some subset $\{\MC_i:i \in S\}$ of the original clauses, we have that $x$ satisfies $\MC_{i_1,\dots,i_t}$ if and only if $(i_1,\dots,i_t) \in S^t$. Thus $\valDNF(\varphi';x) = \valDNF(\varphi;x)^t$ and so, since $x$ was arbitrary, $\valDNF(\varphi') = \valDNF(\varphi)^t$. Finally, it's clear that the reduction is polynomial-time in the output size.
\end{proof}

\begin{proof}[Proof of \cref{thm:boosted-dnf-hardness}]
Fix $k_0 = k_0(c) \in \NN$ sufficiently large and $\veps_0 = \veps_0(c) \in (0,1)$ sufficiently small that the following inequality holds: %
\begin{equation} 
2^{-k_0} + \veps_0 \leq \left(\frac{1-\veps_0}{2k_0}\right)^{2c}.\label{eq:k0-choice}
\end{equation}
In particular, there is a universal constant $C_{\ref{thm:boosted-dnf-hardness}}$ so that we can always take $k_0 \leq C_{\ref{thm:boosted-dnf-hardness}}c^2$ and $\veps_0 = 2^{-C_{\ref{thm:boosted-dnf-hardness}}c^2}$. Set $\eta := \gamma(k_0,\veps_0)/2$, where $\gamma(k_0,\veps_0) \in (0,1)$ is the parameter guaranteed by \cref{thm:max-csp-hardness}. Next, let $P: \{0,1\}^{k_0} \to \{0,1\}$ be the predicate guaranteed by \cref{thm:max-csp-hardness}. We define an algorithm $\Alg'$ for Max-$P$ that does the following, given a $P$-formula with $n$ variables and $m$ clauses:
\begin{enumerate}
\item Using \cref{lemma:dnf-embedding} and the fact that $P$ has only $2k_0$ accepting assignments, construct a $k_0$-DNF formula $\varphi'$ with $n$ variables, $2k_0m$ clauses, and with $\valDNF(\varphi') = \val_P(\varphi)$.
\item Set $t := \left\lfloor \frac{\log n}{\log \frac{2k_0}{1-\veps_0}} \right\rfloor$. Using \cref{lemma:serial-repetition}, construct a $k_0t$-DNF formula $\varphi''$ with $n$ variables, $(2k_0 m)^t$ clauses, and with $\valDNF(\varphi'') = \valDNF(\varphi')^t = \val_P(\varphi)^t$.
\item Apply $\Alg$ to $\varphi''$ and output whatever $\Alg$ outputs.
\end{enumerate}
Notice that $k_0 t \leq C_{\ref{thm:boosted-dnf-hardness}}c^2 \log(n) \leq k(n)$, so $\varphi''$ is a Max-$k$-DNF formula and hence the application of $\Alg$ to $\varphi''$ is well-defined. If $\val_P(\varphi) \geq (1-\veps_0)m$, then 
\[\valDNF(\varphi'') \geq (1-\veps_0)^t m^t = (2k_0m)^t \left(\frac{1-\veps_0}{2k_0}\right)^t \geq \frac{(2k_0 m)^t}{n}\]
since $t \leq \frac{\log n}{\log \frac{2k_0}{1-\veps_0}}$. Thus, $\Alg'$ outputs \textsc{Yes} in this case, with probability at least $1/4$; we can boost this probability to $1/2$ by running the algorithm independently $O(1)$ times. On the other hand, if $\val_P(\varphi) \leq (2k_0/2^{k_0}+\veps_0)m$, then 
\begin{align*} 
\valDNF(\varphi'') 
&\leq (2k_0/2^{k_0} + \veps_0)^t m^t \\
&\leq (2^{-k_0} + \veps_0)^t (2k_0 m)^t \\
&\leq \left(\frac{1-\veps_0}{2k_0}\right)^{2ct} (2k_0 m)^t \\
&\leq \left(\frac{1-\veps_0}{2k_0}\right)^{2c\left(\frac{\log n}{\log \frac{2k_0}{1-\veps_0}} - 1\right)} (2k_0 m)^t \arxiv{\\ }
\arxiv{&}= \left(\frac{2k_0}{1-\veps_0}\right)^{2c} \frac{(2k_0 m)^t}{n^{2c}} \leq \frac{(2k_0 m)^t}{n^c}
\end{align*}
where the third inequality is by \cref{eq:k0-choice}, the fourth inequality is by definition of $t$, and the final inequality holds so long as $n \geq (2k_0/(1-\veps_0))^2$. Thus, $\Alg'$ outputs \textsc{No} with probability $1$ in this case.

The time complexity of $\Alg'$, dominated by the invocation of $\Alg$ on $\varphi''$, is $\poly(2^{n^\eta}, (2k_0 m)^t)$. For $m \leq n$, this is dominated by $2^{O(n^\eta)}$ for sufficiently large $n$, since $t \leq \log(n)$. Since $\eta = \gamma(k_0,\veps_0)/2$, we conclude that $\Alg'$ has time complexity at most $O(2^{n^{\gamma(k_0)}})$, so by \cref{thm:max-csp-hardness}, the Randomized Exponential Time Hypothesis is false. 
\end{proof}

\newpage
\part{Multi-Turn Exploration: Learning Autoregressive Softmax Policies}
\label{part:multi}

This section of the appendix is dedicated to presenting and analyzing the \mtalg algorithm (\cref{alg:ops-dp}). \mtalg learns a near-optimal policy for any Markov Decision Process in which the optimal KL-regularized policy has autoregressive linear softmax structure (a generalization of linear-$\Qstarb$ realizability) under reset access. This formulation subsumes the setting in \cref{sec:multi}, encompassing general Markov Decision Processes (MDPs) that extend well beyond the token-level MDP; we prove \cref{thm:multi}, our main result for the multi-turn exploration setting in \cref{sec:multi} as a special case. We adopt this formulation because (i) it makes the essential ingredients in our algorithm and analysis as clear as possible; and (ii) we believe the results are likely to be of interest more broadly, beyond language modeling. This section is organized as follows:
\begin{itemize}
\item \cref{sec:prelim_rl}: We introduce the general reinforcement learning setting and statistical assumptions.\loose
    \item \cref{sec:algos_rl}: We present and describe the multi-turn algorithm, \texttt{MultiTurnSpannerSampling} (\mtalg; \cref{alg:ops-dp}), and state its main guarantee, \cref{thm:main} (generalizing \cref{thm:multi}).\loose
    \item \cref{sec:rejection_density,sec:proof_uncertain,sec:proof_fitvalue}: We provide the main guarantees for the subroutines used within \mtalg.
    \item \cref{sec:proof_main}: We combine the preceding results to prove the main guarantee for \mtalg (\cref{thm:main}).
    \item \cref{sec:technical_rl}: Supporting technical lemmas for the proofs of the results above.
    \end{itemize}

 \section{Preliminaries for Multi-Turn Exploration}
\label{sec:prelim_rl}

In this section, we introduce formally introduce the setting and assumptions for our general multi-turn exploration results.
Recall that the language model alignment problem in \cref{sec:multi} is a special case of episodic reinforcement learning in a specific (``token-level'') Markov Decision Process, where \emph{actions} are tokens or sub-sequences of tokens, and the \emph{state} consists of the prompt and all of the tokens generated so far, with the \emph{reward}  determined by the alignment objective. This is a rather simple MDP, as the transition dynamics are deterministic, and are known a-priori.

Our results in this section encompass a more general setting where the MDP transitions are unknown and stochastic, but where the agent has the ability to reset to previously observed states during the learning process (in addition to standard episodic access); this setting is also known as reinforcement learning with \emph{local simulator access} \citep{li2021sample, yin2022efficient, weisz2022confident,mhammedi2024power}. In the context of language model alignment, the assumption of a local simulator is without loss of generality because the MDP dynamics are known; resetting the state simply involves feeding the prompt and partial prefix of tokens back into the policy \citep{chang2024dataset}. We present our results for \emph{any stochastic MDP}, provided that local simulator access is available. To achieve statistical and computational efficiency, we make statistical and computational assumptions that generalize \cref{sec:multi}, focusing on KL-regularized regret, and assuming that the optimal regularized policy has linear softmax structure that generalizes \cref{eq:auto}.

\paragraph{A remark on notation} Throughout \cref{part:multi} of the appendix, we use boldface notation (e.g., $\x_h$, $\by_h$, and $\a_h$) to denote realizations of random variables. This allows certain arguments that require conditioning to be presented in the clearest way possible.

\subsection{MDP Setting and Multi-Turn Reinforcement Learning Framework}

 A Markov Decision Process (MDP) is a tuple $\cM^\star = ( \cX, \cA, H, P, \rstar,)$, where
$\cX$ is a (large or potentially infinite) state space, $\cA$ is the
action space (we abbreviate $A=\abs{\cA}$), $H \in \bbN$ is the horizon, $\rstar = \{\rstar_h\}_{h=1}^H$ is
the reward function (where $\rstar_h : \cX \times \cA \rightarrow [0,1]$)
and $P = \{P_h\}_{h=0}^{H}$ is the transition distribution (where $P_h
: \cX \times \cA \rightarrow \Delta(\cX)$), with the convention that
$P_0(\cdot\mid\emptyset)$ is the initial state distribution. A
policy is a sequence of functions $\pi = \{\pi_h: \cX
\rightarrow \Delta(\cA)\}_{h=1}^H$. When a policy $\pi$ is executed, it
generates a trajectory $(\bx_1,\ba_1,\br_1), \dots, (\bx_H, \ba_H,
\br_h)$ via the process $\ba_h \sim \pi_h(\bx_h), \br_h \sim
\rstar_h(\bx_h,\ba_h), \bx_{h+1} \sim P_h(\cdot\mid\bx_h,\ba_h)$,
initialized from $\bx_1 \sim P_0(\cdot\mid\emptyset)$ (we use
$\bx_{H+1}$ to denote a terminal state with zero
reward). We write $\bbP_{\pi}\brk*{\cdot}$ and $\En_{\pi}\brk*{\cdot}$
to denote the law and expectation under this process. We assume that $\br_h\in\brk*{0,\Rmax}$ for all $h$.

As discussed above, the \emph{token-level MDP} formulation of language model reinforcement learning (e.g., \citet{rafailov2024r}) corresponds to the case where $\bx_1$ is the prompt, $\ba_h$ is the next token to predict, and $\bx_h=(\bx_1, \ba_1,\ldots,\ba_{h-1})$ is the concatenation of the prompt with all of the tokens so far; the final response is $\by=(\ba_1,\ldots,\ba_H)$. For the formulation in \cref{sec:multi}, we have $\rstar_h=0$ for all $h<H$, and $\rstar_H$ represents the reward for the complete response. A slightly more general formulation (e.g., \citet{xiong2024building}) which our setup also encompasses, is where each $\ba_h$ represents a sub-sequence of tokens rather than a single token (e.g., corresponding to a portion of a proof).

\paragraph{Online reinforcement learning framework}
In the standard online reinforcement learning framework, the learner
repeatedly interacts with an unknown stochastic MDP (where the transition distribution is not known) by executing a policy and
observing the resulting trajectory, with the goal of maximizing the
total reward. We begin with a base policy $\piref=\crl*{\pirefh}_{h=1}^{H}$. Formally, for each episode $t\in\brk{\Trounds}$, the learner
selects a policy $\pi\ind{t} = \{\pi_h\ind{t}\}_{h=1}^H$, executes it
in the underlying MDP $\cM^\star$ and observes the trajectory
$\{(\bx_h\ind{t},\ba_h\ind{t},\br_h\ind{t})\}_{h=1}^H$. After all $\Trounds$
episodes conclude, the goal of the learner is to produce a policy $\pihat$ such that 
\begin{align}
  \sup_{\pi} J_\beta(\pi) - J_\beta(\pihat) \leq \veps, \label{eq:goal}
\end{align}
for some $\veps,\beta>0$, where $J_\beta(\pi) \coloneqq \bbE_\pi\brk[\big]{\sum_{h=1}^H \br_h} - \beta \kl{\pi}{\piref}$ denotes the \emph{regularized} cumulative reward, with
  \begin{align}
    \label{eq:kl}
\kl{\pi}{\piref}\ldef\En_{\pi}\brk*{\sum_{h=1}^{H}\kl{\pi_h(\bx_h)}{\pirefh(\bx_h)}}.
  \end{align}
We define $\pistarb=\crl[\big]{\pistarbh}_{h=1}^{H}$ as the optimal KL-regularized policy: $\pistarb=\argmax_{\pi}\Jbeta(\pi)$.

\paragraph{Online RL with resets (local simulator access)} We focus on online RL with state resetting (local simulator access) \citep{weisz2021query,li2021sample,yin2022efficient,weisz2022confident,yin2023sample,mhammedi2024power}, which augments the online RL protocol as follows: At each episode $t \in \brk{\Trounds}$, instead of starting from a random initial state $\bx_1 \sim P_0(\cdot \mid \emptyset)$, the agent can \emph{reset} the MDP to any layer $h \in \brk{H}$ and any state $\bx_h$ previously encountered, and proceed with a new episode starting from this point. As in the online RL protocol, the goal is to produce a policy $\pihat \in \Pi$ that satisfies \eqref{eq:goal} with as few episodes of interaction as possible. Note that when $\cM^{\star}$ is the token-level MDP, this formulation precisely corresponds to the setting in \cref{sec:multi}.\footnote{In particular, the KL-divergence in \cref{eq:kl} coincides with the sequence-level KL used in \cref{sec:multi} via the chain rule.}

\paragraph{Linear softmax policies} We consider the class of linear softmax policies given by
\[
  \Pi = \crl*{\pi_{\theta}=\crl*{\pi_{h,\theta}}_{h=1}^{H}\mid{} \theta_1,\ldots,\theta_H\in\mathbb{B}_d(B)}
\]
for a parameter $B\geq\Rmax$, where
\[
  \pi_{h,\theta}(a\mid{}x)\propto\pirefh(a\mid{}x)
  \cdot\exp\prn{\beta^{-1}\inner{\theta_{h}}{\phi_h(x,a)}}.
\]
We assume that the feature map $\phi$ satisfies $\sup_{h,\in[H],(x,a)\in \cX \times \cA}\|\phi_h(x,a)\|\leq 1$. Our main assumption is that the optimal regularized policy is itself softmax-linear.
\begin{assumption}[Linear $\pistar_{\beta}$] 
	\label{assum:linear}
        We assume that for all $h\in[H]$, there exists $\theta_{h,\beta}^\star\in \bbB_d(B)$ such that
        \begin{align}\forall (x,a)\in \cX\times \cA,\quad  \pistar_{h,\beta}(a\mid x) \propto \pi_{h,\refe}(a \mid x) \cdot \exp\prn{\beta^{-1}\inner{\thetastar_{h,\beta}}{\phi_h(x,a)}}.\end{align}
      \end{assumption}

      \paragraph{KL-regularized dynamic programming}
      The KL-regularized RL formulation admits regularized counterparts to the standard $Q$- and $V$-value functions, defined as follows (e.g., \citet{rafailov2024r,xie2024exploratory}).\loose
\begin{definition}[State-action value function]
	\label{def:qfunc}
	For any $\pi_{1:H}\colon \cX \rightarrow \Delta(\cA)$ and $(x,a)\in \cX\times \cA$, define	
	\arxiv{
\begin{align}
	Q_{h,\beta}^{\pi}(x,a) &\coloneqq \rstar_h(x,a)+ \E_{\pi}\left[\sum_{\ell=h+1}^H \rstar_\ell(\x_\ell,\a_\ell)-\beta \sum_{\ell=h+1}^H \log \frac{\pi_\ell(\a_\ell \mid \x_\ell)}{\pi_{\ell,\refe}(\a_\ell \mid \x_\ell)} \mid \x_h = x,\a_h = a\right]. \label{eq:qpi}
\end{align}
	}
\end{definition}

\begin{definition}[KL-regularized value function ($Q^\star_{h,\beta}$)]
	\label{def:inition}
We define the optimal regularized state-action value functions $(Q^\star_{h,\beta})_{h\in[H]}$ via backward induction over $h\in[H]$ as follows: for all $(x,a)\in \cX \times \cA$, $Q_{H+1,\beta}^\star(x,a)=0$ and for $h=H,\dots,1$:
\arxiv{
\begin{gather}
Q_{h,\beta}^{\star}(x,a) = \rstar_h(x,a) + \cT_{h,\beta}[Q_{h+1,\beta}^\star](x,a),
\shortintertext{where}
\cT_{h,\beta}[f](x,a)  \coloneqq  \E_{\pi_\refe}\left[ V_{f}(\x_{h+1}) \mid \x_h = x,\a_h = a\right] \quad \text{and}\quad V_f(x) \coloneqq \beta \log \sum_{a\in \cA} \pi_\refe(a\mid x) \cdot e^{f(x,a)/\beta}.
\end{gather}
}
\end{definition}
We show in \cref{lem:wonky} that the optimal KL-regularized policy $\pistarb$ satisfies
\begin{align}
\label{eq:softmax_qstar}\pistar_{h,\beta}(\cdot \mid x)\propto \pi_{h,\refe}(\cdot\mid x) \cdot \exp\prn*{\beta^{-1}Q^{\star}_{h,\beta}(x,\cdot)}\end{align}
and $\pistar_{h+1:H,\beta} \in \argmax_{\pi_{h+1:H}\in \Pi} Q^{\pi}_{h,\beta}(x,a)$, for all $(x,a)\in \cX\times \cA$ and $h\in[H]$. Consequently, \cref{assum:linear} is equivalent to asserting that for all $h$,
\begin{align}
  \label{eq:value_difference}
  Q_{h,\beta}^{\star}(x,a)
  -Q_{h,\beta}^{\star}(x,a')
  =\tri*{\theta^{\star}_{h,\beta},\phi_h(x,a)-\phi_h(x,a')}.
\end{align}
Thus, \cref{assum:linear} may be viewed as a KL-regularized analogue of the so-called \emph{linear-$Q^{\star}$} assumption explored throughout prior work \citep{du2019good,wang2021exponential,weisz2021query,li2021sample,yin2022efficient,weisz2022confident,yin2023sample,mhammedi2024power}. In particular, prior work has shown that RL with linear-$Q^{\star}$ and an action gap $\Delta$ is statistically intractable in the episodic RL protocol, but is tractable under reset access \citep{weisz2021query,li2021sample}, motivating this access model. Our results show that the regularization parameter $\beta$ plays a similar role to the action gap $\Delta$ in enabling tractability. 

\paragraph{Non-triviality of autoregressive realizability}
The following result, as mentioned in \cref{sec:multi}, shows that \cref{assum:linear} is a non-trivial assumption, in the sense that it may not be satisfied even if the rewards themselves are linear.
\begin{proposition}
  \label{prop:autoregressive}
  Consider the token-level MDP. Let $\MX = \{\perp\}$, $\cA = [2]$, $H = d = 2$, and $\beta=1$. For any
  $\delta\in(0,1/2)$, there exist feature maps $\phi_h\in\bbR^{d}$
  with $\nrm*{\phi_h(x,y_{1:h})}\leq{}1$ and parameters
  $\thetastar_h\in\bbR^{d}$ with
  $\nrm*{\thetastar_h}\leq{}B\ldef{}\log(3/\delta)$, so that (i) the
  optimal KL-regularized policy $\pistarb$ for rewards $\rstar(x,y) =
  \sum_{h=1}^{2}\tri*{\thetastar_h,\phi_h(x,a_{1:h})}$ satisfies
  \[
    \pistarb(x, y_{1:2})=\piseq[\thetastar](x,y_{1:2}) \geq 1-\delta > \frac{1}{2}
  \]
  for some $y_{1:2}\in\cY$, yet (ii) for all $\theta_h\in\bbR^{d}$, $\piauto(x, y_{1:2}) \leq \frac{1}{2}$.
  In particular, this means there are no $\theta_h\in\bbR^{d}$ such that $\piauto=\piseq[\thetastar]$.
\end{proposition}
\begin{proof}[\pfref{prop:autoregressive}]
We adapt the proof of Proposition E.2 in
\citet{huang2024self}. Throughout, we omit the dependence on the
prompt $\perp$. Let $\pirefh\ldef\unif(\cA)$. %
Define $\phi_1$ by
$\phi_1(i) = e_1$,
and define $\phi_2$ by:
\[\phi_2(i,j) = \begin{cases} e_1 & \text { if } i = 2, j = 1 \\ e_2 &
    \text { o.w. } \end{cases}.
\]
For $h\in\crl{1,2}$, define $\rstar_h(y_{1:h}) =
\tri*{\thetastar_h,\phi_h(y_{1:h})}$, where 
 $\thetastar_1=\thetastar_2=B\cdot{}e_i$ for $B\geq{}\log{}(3/\delta)$. Then
  we have
  \[
    \pistarb(y_1=2, y_2=1) = \frac{e^{2B}}{e^{2B} + 3e^{B}}
    = \frac{1}{1 + 3e^{-B}} = \frac{1}{1+\delta} \geq 1-\delta.
  \]
On the other hand, since $\phi_1(1)=\phi_1(2)=B\cdot{}e_1$, all $\theta\in\bbR^d$ have %
\arxiv{\[
\piauto(y_1=2, y_2=1)\leq{}\piauto(y_1=2) = \frac{1}{2}.
\]}
\end{proof}

\subsection{Sample Complexity, Computational Oracles, and Coverage}
\label{sec:mtss_oracle}
Recall that our goal is to design an algorithm that achieves the objective in \eqref{eq:goal} with as few episodes of interaction with the environment as possible. We measure the sample efficiency of an algorithm in terms of the total number $T_\texttt{data}$ of reward, transition, and local simulator (reset) queries required to achieve \eqref{eq:goal} for $\veps>0$. To allow for computationally efficient algorithms, we assume access to the following sampling oracle for $\piref$, generalizing \cref{def:oracle_autoregressive}.
  \begin{definition}[Policy sampling oracle (weak version)]
    \label{def:oracle_policy}
    In one query, the learner can
  propose a state $x\in\cX$ and layer $h\in\brk{H}$, and receive a conditional sample
  $\a_h\sim\pirefh(\cdot\mid{}x)$ as well as the corresponding feature
  $\phi_h(x,\a_h)$. Additionally, in one query the learner can propose a state $x\in\cX$, action $a\in\cA$, and layer $h\in\brk{H}$, and receive $\phi_h(x,a)$. We let $\Tcomp$ denote the total number of policy sampling
  queries used by the algorithm.
\end{definition}

For technical reasons, we require explicit query access to $\phi_h(x,a)$ in this framework, whereas in our original framework (\cref{def:oracle}) we only required observing $\phi(x,y)$ for sampled $(x,y)$ pairs. This is because our multi-turn algorithm $\mtalg$ uses an ``anchor action'' to normalize features across all states; it may be possible to this avoid with a method similar to what is used by $\spanalg$---see \cref{remark:mtss-query-access}.\loose

This technical point aside, note that \cref{def:oracle_policy} is a generalization of the \emph{weak} autoregressive sampling oracle in \cref{def:oracle_autoregressive}, which instantiates \cref{def:oracle_policy} in the token-level MDP. An analogue of the \emph{strong} oracle in \cref{def:oracle_autoregressive} would be to allow sampling $\a_h\sim\pi_{h,\theta}(\cdot\mid{}x)$ in unit time for any $\theta$, but the weak oracle is all that is required by our main algorithm, \mtalg.\loose

Finally, our results depend on the following coverage coefficient, generalizing \cref{eq:coverage_auto}.
\begin{definition}[Conditional Coverage]
	\label{def:condcoverability}
	For any policy $\pi =\{\pi_h\}_{h=1}^H$ and state $x \in \cX$, the conditional coverage of $\pi_\refe =\{\pi_{h,\refe}\}_{h=1}^H$ relative to a reference policy $\pi_\refe = $ at $x$ is defined as:
	\begin{align}
		\Ccond(\pi \mid x) \coloneqq \sup_{a \in \cA} \max_{h\in [H]}\frac{\pi_{h}(a \mid x)}{\pi_{h,\refe}(a \mid x)}. \label{eq:condx}
	\end{align}
	Similarly, the conditional coverage of $\pi$ relative to $\pi_\refe$ is defined as:
	\begin{align}
		\Ccond(\pi) \coloneqq \sup_{(x, a) \in \cX \times \cA} \max_{h\in[H]}\frac{\pi_h(a \mid x)}{\pi_{h,\refe}(a \mid x)}. \label{eq:condcond}
  \end{align}
  
For $h\in[H]$, we occasionally overload notation and write $\Ccond(\pi_h)$ and $\Ccond(\pi_h\mid x)$ to indicate the quantities $\sup_{a \in \cA} \frac{\pi_{h}(a \mid x)}{\pi_{h,\refe}(a \mid x)} $ and $\sup_{(x, a) \in \cX \times \cA}\frac{\pi_h(a \mid x)}{\pi_{h,\refe}(a \mid x)}$, respectively; that is, the quantities in \eqref{eq:condx} and \eqref{eq:condcond} without the max over $h$.
  
\end{definition}
  
Our goal is to ensure $T_\texttt{data} \leq \poly(d, B, H, \beta^{-1}, \veps^{-1},\log(\delta^{-1}))$ and $T_\texttt{comp} \leq \poly(\Ccond(\pistar_\beta), T_\texttt{data})$, where $d$ is the dimension of the feature map $\phi$ in \cref{assum:linear}, $B$ is the bound on the parameter norm, and $\Ccond(\pistar_\beta)$ is the coverage number of the optimal regularized policy.\loose

  \paragraph{Additional notation}
    For any $m,n \in\mathbb{N}$, we denote by $[m\ldotst{}n]$
 the integer interval $\{m,\dots, n\}$. We also let
 $[n]\coloneqq [1\ldotst{}n]$. %
For any $(x,a)\in \cX \times \cA$, we use the convention that
\begin{align}\E[\cdot \mid \x_0=x, \a_0=a] \equiv \E[\cdot] \quad \text{and} \quad  \P[\cdot \mid \x_0=x, \a_0=a] \equiv \P[\cdot].\end{align}

\newpage
\section{\mtalg Algorithm and Guarantees}
\label{sec:algos_rl}
In this section, we formally introduce our algorithm, \mtalg, present some intuition behind its design, and state its guarantee (\cref{thm:main}).

\begin{algorithm}[ht]
    \caption{\mtalg: Multi-Turn Spanner Sampling.
	}
	\label{alg:ops-dp}
	\begin{algorithmic}[1]
        \setstretch{1.2}
		\Statex[0] 	{\bfseries input:} Base policy $\pi_\refe$. Parameters $\beta, \delta,\veps\in(0,1)$ and $B>0$.
        \Statex[0] {\bfseries initialize:}
        Set $\tauindl \gets 1$, and $\Sigma_{h}\ind{1} \gets \lambda I$, $\forall h\in[H]$.
        \State Set $\Trounds$, $\Nreg$ $\Nspann$, $\Nspanb$, $\Mrej$, $\deltarej$, $\lambda, \nu>0$ as in \sssref{app:params}. 
        \State Set $\cC\ind{1}_{h}\gets \emptyset$, for all $h\in[0 \ldotst H]$, and set $\cC_0\ind{1} \gets \cC_0\ind{1}\cup \{(x_0,a_0)\}$ for arbitrary $(x_0,a_0)\in \cX\times \cA$.
        \For{$t= 1, \dots, T$}
        \Statex[1] \algcommentbiglight{Fit state-action value function in a dynamic programming fashion.}
        \For{$h=H, \dots, 1$} \label{line:bigbang}
        \State Update $\theta_h\ind{t} \gets \fitval_{h}(\cC\ind{t}_h, \theta\ind{t}_{h+1:H}, \Sigma\ind{t}_{h+1:H};B,\fraka,\Nreg, \Mrej,\deltarej, \piref)$. %
        \State \multiline{ For all $x\in \cX$, set
        \vspace{-.2cm}
        \begin{align}
        \pihat\ind{t}_h(\cdot \mid x) \propto   \rejection_{\beta,\Mrej,\deltarej}(\inner{\varphibar_h(x,\cdot)}{\theta\ind{t}_h} \midsem x, \piref),\nn 
        \end{align}
        where $\varphibar\ind{t}_h(\cdot,\cdot)= \varphi_h(\cdot,\cdot) \cdot \mathbb{I}\left\{\|\varphi_h(\cdot,\cdot)\|^2_{(\Sigma_h\ind{t})^{-1}}\leq \nu^2\right\}$ and $\varphi_h(\cdot, \cdot)\coloneqq \phi_h(\cdot,\cdot)- \phi_h(\cdot, \fraka)$.} \label{line:truncated}
        \EndFor
        \vspace{-.1cm}
        \Statex[1] \algcommentbiglight{Add uncertain state action pairs to the core set.}
        \State Set $\cC_{1:H}\ind{t+1}\gets \cC\ind{t}_{1:H}$ and $\Sigma\ind{t+1}_{1:H}\gets \Sigma\ind{t}_{1:H}$.
        \For{$h=1,\dots, H$} \label{line:setinit}
        \State $(x\ind{t}_h,a_{h}\ind{t})\gets \uncertsa_h(\cC\ind{t}_{0:h-1}, \pihat\ind{t}_{1:h}, \Sigma_h\ind{t}; \Nspann, \Nspanb)$. \hfill \algcommentlight{\cref{alg:fitbonus_p}}  \label{line:designdir}         \State  Update $\cC_{h}\ind{t+1} \gets \cC_h\ind{t+1} \cup \{(x_h\ind{t},a_{h}\ind{t}),(x_h\ind{t}, \fraka)\}$.\hfill  \algcommentlight{$\cC\ind{t+1}_h$ is a multiset.}\label{line:updatePsi}
        \State Update $\Sigma_h\ind{t+1} \gets \Sigma_h\ind{t+1} +  \varphi_h(x_h\ind{t},a_h\ind{t})\varphi_h(x_h\ind{t},a_h\ind{t})^\top$.
        \EndFor
       \Statex[1] \algcommentbiglight{If $(x\ind{t}_h,a\ind{t}_h)$ is not too uncertain, $\pihat\ind{t}$ is a good candidate policy to return.}
       \If{$\max_{h\in[H]}\|\varphi_h(x_h\ind{t},a\ind{t}_h)\|^2_{(\Sigma\ind{t}_h)^{-1}}\leq\nu^2/4$}
       \State $\tauindl \gets t$.
       \EndIf
       \EndFor
		\State \textbf{return} $\pihat_{1:H} =\pihat_{1:H}\ind{\tauindl}$.
\end{algorithmic}
\end{algorithm}

\subsection{\mtalg Pseudocode and Overview}
Our main algorithm, \mtalg (\cref{alg:ops-dp}),  learns a policy in a dynamic programming fashion by fitting value functions for each layer $h = H, \dots, 1$, while maintaining a growing \emph{core-set} of informative state-action pairs. This core-set guides exploration, and is closely related to recent works on linearly parameterized RL with local simulators \citep{li2021sample, yin2022efficient, weisz2022confident,mhammedi2024power,mhammedi2024sampleoracleefficientreinforcement}; of the works, the structure of \mtalg{} is most closely aligned with the \texttt{Optimistic-PSDP} algorithm introduced by \cite{mhammedi2024sampleoracleefficientreinforcement} for the linearly-$Q^\pi$ realizable RL setting, which itself builds on ideas from the classical Policy Search by Dynamic Programming (\texttt{PSDP}) algorithm (see, e.g., \cite{bagnell2003policy}) and the Recursive Value Function Search (\texttt{RVFS}) algorithm of \citet{mhammedi2024power}. We combine this with the spanner technique and truncated softmax policies from \mainalg, which are critical to achieve computational efficiency under the sampling oracle model in \cref{def:oracle_policy}. At various points in the section, we will comment on key similarities and differences between these approaches.

\mtalg{} comprises three main subroutines: \texttt{FitValue} (\cref{alg:fitval}) and \uncertsa (\cref{alg:fitbonus_p}). Before getting into the details of \mtalg, we first provide an overview of the key variables used in \cref{alg:ops-dp}.

\subsubsection{Key Variables in \mtalg (\cref{alg:ops-dp})} \mtalg takes a fixed reference policy $\piref$ as input (e.g., a pre-trained language model) and runs for $\Trounds = \wtilde{O}(d)$ iterations. At each iteration $t \in [\Trounds]$, the algorithm maintains the following key variables:
\begin{itemize}[leftmargin=*]
    \item $\cC\ind{t}_{h}$: A core-set of $t$ state-action pairs at layer $h$. The subroutine \texttt{FitValue}$_h$ uses these state-action pairs as starting points to generate rollout trajectories, which are then used to fit the parameters of the optimal policy via regression. The core-set $\cC\ind{t}_h$---generalizing the notion of spanner in \mainalg---aims at provide sufficient coverage of the state-action space at layer $h$\arxiv{, as we explain in the sequel.} Note that $\cC_h\ind{t}$ is a \emph{multiset}, allowing for multiple instances of the same state-action pair.
    \item $\fraka$: A fixed, arbitrarily chosen ``anchor'' action used in the regression problem solved by \texttt{FitValue}.\loose
    \item $\Sigma\ind{t}_h$: A ``design matrix'' for layer $h$, defined as the sum of outer products of feature differences $\varphi_h(x_h, a_h) \coloneqq \phi_h(x_h, a_h) - \phi_h(x_h, \fraka)$ for the core-set states-action pairs $(x_h, a_h) \in \cC\ind{t}_h$:
    \begin{align}
        \Sigma\ind{t}_h = \lambda I + \sum_{(x_h, a_h) \in \cC\ind{t}_h} \varphi_h(x_h, a_h) \varphi_h(x_h, a_h)^\top, \label{eq:SigmaMat}
    \end{align}
    where  $\lambda > 0$ is a regularization parameter defined in \cref{alg:ops-dp}.
    
    \item $\theta\ind{t}_h$: An estimate of the parameter vector $\thetastar_{h, \beta}$ associated with\arxiv{ the optimal policy} $\pistar_{h, \beta}$ (\cref{assum:linear}).
    
    \item $\pihat\ind{t}_h$: The policy used to generate actions at layer $h$ during iteration $t$. This policy is computed using the \softmaxsample{} subroutine (\cref{alg:rejection_density} in \cref{sec:rejection}) and approximates the ``truncated'' policy:
    \begin{align}
        \pibar\ind{t}_h(\cdot \mid x) \propto \piref(\cdot \mid x) \cdot e^{\beta^{-1} \cdot \inner{\varphibar\ind{t}_h(x, \cdot)}{\theta\ind{t}_h}}, \quad\text{where}\quad
        \varphibar\ind{t}_h(\cdot, \cdot) = \varphi_h(\cdot, \cdot) \cdot \mathbb{I}\left\{\|\varphi_h(\cdot, \cdot)\|^2_{(\Sigma_h\ind{t})^{-1}} \leq \nu^2\right\};
    \end{align}
    this notion generalizes the truncated softmax policies used in \mainalg.
     As $t$ increases, $\theta\ind{t}_h$ converges to $\thetastar_{h, \beta}$, and $\|\varphi_h(x, a)\|^2_{(\Sigma\ind{t}_h)^{-1}}$ decreases for all $(x, a)$ pairs. Eventually, $\|\varphi_h(x, a)\|^2_{(\Sigma\ind{t}_h)^{-1}}$ becomes smaller than $\nu^2$ for ``most'' state-action pairs $(x,a)$, ensuring that $\pibar\ind{t}_h$ (and thus $\pihat\ind{t}_h$) approximates the optimal policy $\pistar_{h, \beta}$.
\end{itemize}

In each iteration $t \in [\Trounds]$, \mtalg computes the estimates $\theta\ind{t}_{1:H}$ in a dynamic programming fashion by fitting the difference of value functions $Q^{\pihat\ind{t}}_{h, \beta}(\cdot, \cdot) - Q^{\pihat\ind{t}}_{h, \beta}(\cdot, \fraka)$ at each layer $h$ (motivated by \cref{eq:value_difference}). In what follows, we describe the \texttt{FitValue} subroutine responsible for this step.

\subsubsection{\texttt{FitValue} (\cref{alg:fitval})}
\label{sec:fitvalue}
\arxiv{\begin{algorithm}[ht]}
\caption{$\texttt{FitValue}_{h}$: Estimate KL-regularized value function using rollouts.}
\label{alg:fitval}
\begin{algorithmic}[1]
\setstretch{1.2}
\Statex[0]{\bfseries input:} Layer $h$, core set $\cC_h$, $\theta_{h+1:H}$, $\Sigma_{h+1:H}$, $B>0$, fixed action $\fraka$, and $N$, $M$, $\tilde\delta$, $\piref$.
\State \multiline{For all $\ell\in[h+1 \ldotst H]$ and $x\in\cX$, define
\[\varphibar_\ell(\cdot,\cdot)= \varphi_\ell(\cdot,\cdot) \cdot \mathbb{I}\left\{\|\varphi_\ell(\cdot,\cdot)\|^2_{\Sigma_\ell^{-1}}\leq \nu^2\right\},
\]  
where $\varphi_\ell(\cdot, \cdot)\coloneqq \phi_\ell(\cdot,\cdot)- \phi_\ell(\cdot, \fraka)$.}
\Statex[0] \algcommentbiglight{Gather trajectory data.}
\For{$(x_h,a_h)\in \cC_h$}
\State Set $\cD(x_h,a_h) \gets \emptyset$.
\State Set $\x'_h = \x''_h= x_h$, $\a'_h = a_h$, and $\a''_h = \afrak$.
\State Set $\br'_h\sim \rstar_h(\x'_h,\a'_h)$ and $\br_h''\sim \rstar_h(\x''_h,\a''_h)$.
\For{$N$ times} 
\For{$\ell=h+1, \dots, H$}
\State Sample $\x'_{\ell}\sim \P[\cdot \mid \x_{\ell-1}= \x'_{\ell-1}, \a_{\ell-1}= \a'_{\ell-1}]$.
\State Sample $\x_\ell''\sim \P[\cdot \mid \x_{\ell-1}= \x''_{\ell-1}, \a_{\ell-1}= \a''_{\ell-1}]$.
\State Set $(\a'_\ell,\brho'_\ell)\gets
\softmaxsample_{\beta,M,\tilde\delta}(\inner{\varphibar_\ell(\x'_\ell,\cdot)}{\theta_\ell}
\midsem \x'_\ell,\piref)$. \hfill\algcommentlight{\cref{alg:rejection_density}.}
\State  Set $(\a_\ell'',\brho''_\ell)\gets \softmaxsample_{\beta,M,\tilde\delta}(\inner{\varphibar_\ell(\x_\ell'',\cdot)}{\theta_\ell}\midsem \x''_\ell,\piref)$.
\State Set $\br'_\ell\sim \rstar_\ell(\x'_\ell,\a'_\ell)$ and $\br_\ell''\sim \rstar_\ell(\x''_\ell,\a''_\ell)$.
\EndFor
 \State Compute $\by'_h \gets \br'_h  + \sum_{\ell=h+1}^H \left(\br'_\ell - \beta \log \brho'_\ell \right)$.\hfill \algcommentlight{$\brho'_\ell \approx \frac{\pibar_{\ell,\theta}(\by'_\ell \mid \x'_\ell)}{\pi_{\ell,\refe}(\by'_\ell\mid \x'_\ell)}$.}
 \State  Compute $\by''_h \gets \br_h'' +\sum_{\ell=h+1}^H \left(\br''_\ell - \beta \log \brho''_\ell\right)$. \hfill \algcommentlight{$\brho''_\ell \approx \frac{\pibar_{\ell,\theta}(\by_\ell'' \mid \x''_\ell)}{\pi_{\ell,\refe}(\by_\ell''\mid \x''_\ell)}$.}
 \State Set $\z_h \gets \by'_h - \by''_h$.
\State Update $\cD(x_h,a_h) \gets \cD(x_h, a_h) \cup \{\z_{h}\}$. \hfill \algcommentlight{$\cD$ is a multiset.}
\EndFor
\EndFor
\Statex[0] \algcommentbiglight{Fit value function.}
\State \textbf{if} $\cD\not= \emptyset$ \textbf{then} $\thetahat_{h} \gets \argmin_{\tilde\theta\in \bbB(B)}\sum_{(x_h,a_h)\in \cC_h}\sum_{z_h\in \cD(x_{h},a_{h})}\left( \inner{\varphi_h(x_{h}, a_h)}{\tilde\theta}- z_h\right)^2$, \textbf{else} $\thetahat_{h} \gets 0$.
\State \textbf{return} $\thetahat_h$.
\end{algorithmic}
\end{algorithm}
 \label{subsec:fitoptv} In each iteration $t \in [\Trounds]$, starting from $h = H$ and progressing down to $h = 1$, \mtalg invokes $\texttt{FitValue}_h$ with the input $(\cC\ind{t}_{h}, \theta\ind{t}_{h+1:H})$. This subroutine returns the vector $\theta\ind{t}_h$, an estimate of the parameter vector $\thetastar_{h,\beta}$ for the optimal policy $\pistar_{h,\beta}$ (see \cref{assum:linear}).  To compute the estimate $\theta\ind{t}_h$, for each state-action pair $(x_h, a_h) \in \cC\ind{t}_h$, \texttt{FitValue} generates multiple i.i.d.~regression targets $\z_h$ by sampling two trajectories $(\bx'_h,\a'_h,\br'_h,\brho'_h,\dots, \bx'_{H},\a'_{H},\br'_H, \brho'_H)$ and $(\bx''_h,\a''_h,\br''_h,\dots, \bx''_{H},\a''_{H},\br''_H, \brho''_H)$ via the following process. Initialize $\x'_h = \x''_h= x_h$, $\a'_h = a_h$, and $\a''_h = \afrak$ (recall that $\fraka$ is a fixed, arbitrary action defined in \cref{alg:ops-dp}), and sample $\br'_h\sim \rstar_h(\x'_h,\a'_h)$ and $\br_h''\sim \rstar_h(\x''_h,\a''_h)$. Then, for $\ell=h+1,\dots, H$, use \softmaxsample{} (\cref{alg:rejection_density} in \cref{sec:rejection}) tao approximately sample from the policy $\pibar\ind{t}$ as follows:
\begin{itemize} 
\item Sample $\x'_{\ell}\sim \P[\cdot \mid \x_{\ell-1}= \x'_{\ell-1}, \a_{\ell-1}= \a'_{\ell-1}]$ and $\x_\ell''\sim \P[\cdot \mid \x_{\ell-1}= \x''_{\ell-1}, \a_{\ell-1}= \a''_{\ell-1}]$;
\item Set $(\a'_\ell,\brho'_\ell)\gets \softmaxsample_{\beta,\Mrej,\deltarej}(\inner{\varphibar\ind{t}_\ell(\x'_\ell,\cdot)}{\theta\ind{t}_\ell} \midsem \x'_\ell,\piref)$;
\item Set $(\a_\ell'',\brho''_\ell)\gets \softmaxsample_{\beta,\Mrej,\deltarej}(\inner{\varphibar\ind{t}_\ell(\x_\ell'',\cdot)}{\theta\ind{t}_\ell}\midsem \x_\ell'',\piref)$;
\item Sample rewards $\br'_\ell\sim \rstar_\ell(\x'_\ell,\a'_\ell)$ and $\br_\ell''\sim \rstar_\ell(\x''_\ell,\a''_\ell)$;
    \end{itemize}
	then, finally, set
	\begin{align}
		\z_h = \br'_h  + \sum_{\ell=h+1}^H \left(\br'_\ell - \beta \log \brho'_\ell\right) -\br''_h -\sum_{\ell=h+1}^H \left(\br''_\ell - \beta \log \brho_\ell'' \right).\nn %
	\end{align}
        Here, $\brho'_\ell$ and $\brho''_\ell$ approximate $\frac{\pibar\ind{t}_{\ell}(\by'_\ell \mid \x'_\ell)}{\pi_{\ell,\refe}(\by'_\ell\mid \x'_\ell)}$ and $\frac{\pibar\ind{t}_{\ell}(\by''_\ell \mid \x''_\ell)}{\pi_{\ell,\refe}(\by''_\ell\mid \x''_\ell)}$, respectively, and the expected value of $\z_h$ (up to small approximation error) corresponds to the difference
        \[
          Q^{\pihat\ind{t}}_{h,\beta}(x_h, a_h) - Q^{\pihat\ind{t}}_{h,\beta}(x_h, \fraka).
        \]
        We regress onto $\z_h$ with least squares, setting the vector $\theta\ind{t}_h$ as the minimizer of the sum of squared errors across all $(x_h, a_h) \in \cC\ind{t}_h$. 

        Note that if we could somehow ensure $\pihat\ind{t}=\pistar$, the difference $Q^{\pihat\ind{t}}_{h,\beta}(x_h, a_h) - Q^{\pihat\ind{t}}_{h,\beta}(x_h, \fraka)$ would be linear in $\varphi_h(x_h, a_h)$ by \cref{eq:value_difference}, but it is not guaranteed to be linear in general; 
        this poses challenges for deriving a regression guarantee, as the regression problem may not realizable or even approximately realizable. Fortunately, the components of \mtalg{} (in particular \uncertsa) ensure that for sufficiently large $\Trounds$, there exists an iteration $t \in [\Trounds]$ such that for all $h \in [H]$ and state-action pairs $(x_h,a_h)$ satisfying $\|\varphi_h(x_h,a_h)\|^2_{(\Sigma_h\ind{t})^{-1}}\leq O(1)$, the following quantity is small:
\begin{align}
  \label{eq:q_misspecification}
    \left|\Qstar_{h,\beta}(x_h, a_h) - \Qstar_{h,\beta}(x_h, \fraka) 
    - Q^{\pihat\ind{t}}_h(x_h, a_h) + Q^{\pihat\ind{t}}_h(x_h, \fraka)\right|.
\end{align}
Consequently, for such an iteration $t$, the regression problem becomes approximately linearly realizable. This allows us to establish that for sufficiently large $\Trounds$, there exists an iteration $t \in [\Trounds]$ such that the subroutine $\texttt{FitValue}_h$ returns $\theta\ind{t}_h$ satisfying:
\begin{align}
    \|\theta\ind{t}_h - \thetastar_{h,\beta}\|^2_{\Sigma_h\ind{t}} = \lambda I + \sum_{(x_h, a_h) \in \cC\ind{t}_h} \inner{\varphi_h(x_h, a_h)}{\theta\ind{t}_h - \thetastar_{h,\beta}}^2 \leq \veps_\reg^2, \label{eq:linreg}
\end{align}
with high probability, for some small $\veps_\reg > 0$.

\begin{remark}[Fitting the difference]
    The reason \texttt{FitValue} targets the difference $\Qstar_{h,\beta}(\cdot, \cdot) - \Qstar_{h,\beta}(\cdot, \fraka)$ rather than $\Qstar_{h,\beta}(\cdot, \cdot)$ directly is that the former (but not the latter) is linear in $\phi_h(\cdot, \cdot) - \phi_h(\cdot, \fraka)$, as guaranteed by the softmax-linear assumption in \cref{assum:linear} and \cref{lem:reward_difference}. Without additional assumptions, the same would not hold for the un-centered value $\Qstar_{h,\beta}(\cdot, \cdot)$. 
\end{remark}
The \texttt{Optimistic}-\texttt{PSDP} algorithm of \citet{mhammedi2024sampleoracleefficientreinforcement} uses a subroutine similar to \texttt{FitValue}. The difference is that \texttt{Optimistic}-\texttt{PSDP} fits \emph{optimistic} value functions, whereas \texttt{FitValue} does not. Optimism is not needed in our setting because we can effectively drive exploration using the core-sets under reset/local simulator access.

\subsubsection{\uncertsa (\cref{alg:fitbonus_p})}
\label{sec:uncertain}
\begin{algorithm}[ht]
    \caption{
	$\uncertsa_{h}$: Identify uncertain state-action pair.
	}
	\label{alg:fitbonus_p}
	\begin{algorithmic}[1]
		\setstretch{1.2}
\Statex[0]{\bfseries input:}  \multiline{$h$, $\cC_{0:h-1}$, $\pihat_{1:h}$, $\Sigma_h$, $\fraka$, $N$, $\Nbar$.}
\Statex[0] \algcommentbiglight{Gathering candidate state action pairs.}
\State Set $\kappa \gets 0$ and define $\varphi_h(\cdot,\cdot)= \phi_h(\cdot, \cdot)- \phi_h(\cdot, \fraka)$.
    \For{$\ell=0,\dots, h-1$}
    \For{$(x_\ell,a_\ell)\in \cC_{\ell}$} 
	\State Initialize $\cD_{\ell}(x_\ell,a_\ell) \gets \emptyset$.  \hfill \algcommentlight{$\cD_\ell$ is a multiset (only used in the analysis).}
	\For{$N$ times}
	\State Sample $(\x_{\ell+1},\a_{\ell+1}, \dots, \x_h, \a_h) \sim \P_{\pihat_{\ell+1:h}}[\cdot \mid \x_\ell=x_\ell, \a_{\ell}=a_{\ell}]$.
	\Statex[3] \algcommentlight{Above, we use the convention that $\P_{\pihat_{\ell+1:h}}[\cdot \mid \x_0=x_0, \a_{0}=a_0]\equiv \P_{\pihat_{\ell+1:h}}[\cdot]$.}
	\State Update $\cD_\ell(x_\ell,a_\ell) \gets \cD_\ell(x_\ell,a_\ell) \cup \{(\x_h,\a_h)\}$.
	\If{$\|\varphi_h(\x_h,\a_h)\|_{\Sigma_h^{-1}}>\kappa$}
	\State Set $\kappa \gets \|\varphi_h(\x_h,\a_h)\|_{\Sigma_h^{-1}}$. \arxiv{\hfill\algcommentlight{$\kappa$ captures the maximum increase in the elliptical objective.}}
	\State Set $(\hat x_h,\hat a_{h}) \gets (\x_h,\a_{h})$. 
	\EndIf
	\Statex[3] \algcommentbiglight{Testing $\Nbar$ samples from $\pi_\refe$ given $\x_h$}
	\State Initialize $\widebar\cD_\ell(\x_{h})$. \hfill \algcommentlight{$\widebar\cD_\ell$ is a multiset (only used in the analysis).}
	\For{$\Nbar$ times}
	\State Sample $\bar\a_h \sim \pi_{h,\refe}(\cdot \mid \x_h)$.
	\State Update $\widebar\cD_\ell(\x_h) \gets \widebar\cD_\ell(\x_h) \cup \{\bar\a_h\}$.
	\If{$\|\varphi_h(\x_h,\bar\a_h)\|_{\Sigma_h^{-1}}>\kappa$}
	\State Set $\kappa \gets \|\varphi_h(\x_h,\bar\a_h)\|_{\Sigma_h^{-1}}$. 
	\State Set $(\hat x_h,\hat a_{h}) \gets (\x_h,\bar\a_{h})$.
	\EndIf
	\EndFor
	\EndFor
\EndFor
\EndFor
\State \textbf{return} $(\hat x_h,\hat a_{h})$.
\end{algorithmic}
\end{algorithm}
 
 \mtalg uses the subroutine \uncertsa to update the core-sets $(\cC\ind{t}_h)$. When \mtalg invokes it with the input $(\cC\ind{t}_{0:h-1}, \pihat\ind{t}_{1:h}, \Sigma\ind{t}_{1:h})$, $\uncertsa_h$ uses $\pihat\ind{t}_{1:h}$ to generate multiple partial trajectories starting from the state-action pairs $(x_\ell, a_\ell) \in \cC\ind{t}_\ell$ at all layers $\ell \in [0 \ldotst h-1]$ and terminating at layer $h$. Among all state-action pairs reached at layer $h$ through this process, \uncertsa selects the pair $(x\ind{t}_h, a\ind{t}_h)$ that maximizes the elliptical objective $\|\varphi_h(\cdot, \cdot)\|_{(\Sigma\ind{t}_h)^{-1}}$. As a result, the output $(x_h\ind{t}, a_h\ind{t})$ of \uncertsa satisfies the following property: with high probability, for all $\ell \in [0 \ldotst h-1]$ and $(x_\ell, a_\ell) \in \cC\ind{t}_\ell$,
\begin{align}
    \P_{\pihat\ind{t}_{\ell+1:h}}\left[ \|\varphi_h(\x_h, \a_h)\|^2_{(\Sigma\ind{t}_h)^{-1}} > \nu^2 \vee \left(2\|\varphi_h(x\ind{t}_h, a\ind{t}_h)\|^2_{(\Sigma_h\ind{t})^{-1}}\right) \mid \x_\ell = x_\ell, \a_\ell = a_\ell\right] \leq \vepslip, \label{eq:uncertsapre}
\end{align}
for some small $\vepslip > 0$; this generalizes the spanner property used in \mainalg.

 Using the definition of $\Sigma_h\ind{t}$ in \eqref{eq:SigmaMat}, a standard elliptical potential argument (see the proof of \cref{lem:test}) implies that for sufficiently large $t = \Omega(d)$ in \mtalg, the tuple $(x_h\ind{t},a_h\ind{t})$ returned by \uncertsa{} must satisfy $\|\varphi_h(x\ind{t}_h, a\ind{t}_h)\|_{(\Sigma\ind{t}_h)^{-1}} \leq \nu^2$. Substituting this into \eqref{eq:uncertsapre} ensures that for such $t$, the following holds for all $\ell \in [0 \ldotst h-1]$ and $(x_\ell, a_\ell) \in \cC\ind{t}_\ell$:
\begin{align}
    \P_{\pihat\ind{t}_{\ell+1:h}}\left[ \|\varphi_h(\x_h, \a_h)\|^2_{(\Sigma\ind{t}_h)^{-1}} > \nu^2 \mid \x_\ell = x_\ell, \a_\ell = a_\ell\right] \leq \vepslip. \label{eq:uncertsapre_2}
\end{align}
This property is crucial in the analysis of $\texttt{FitValue}_\ell$, as it allows us to bound the misspecification error for the regression problems solved in \texttt{FitValue} (cf. \cref{eq:q_misspecification}) at future layers $h \in [\ell+1 \ldotst H]$. Concretely, using H\"older's inequality, we have for all $(x_\ell, a_\ell) \in \cC\ind{t}_\ell$,
\begin{align}
  & \E_{\pihat\ind{t}_{\ell+1:h}}\left[ \inner{\varphi_h(\x_h, \a_h)}{\theta\ind{t}_h - \thetastar_{h,\beta}}^2 \mid \x_\ell = x_\ell, \a_\ell = a_\ell\right] \nn \\
  & \leq  \E_{\pihat\ind{t}_{\ell+1:h}}\left[ \|\varphi_h(\x_h, \a_h)\|^2_{(\Sigma_h\ind{t})^{-1}} \mid \x_\ell = x_\ell, \a_\ell = a_\ell\right] \cdot \|\theta\ind{t}_h - \thetastar_{h,\beta}\|^2_{\Sigma\ind{t}_h}, \nn \\
  & \leq \left( \nu^2 + \lambda^{-1} \P_{\pihat\ind{t}_{\ell+1:h}}\left[ \|\varphi_h(\x_h, \a_h)\|^2_{(\Sigma_h\ind{t})^{-1}} \mid \x_\ell = x_\ell, \a_\ell = a_\ell\right]\right) \cdot \|\theta\ind{t}_h - \thetastar_{h,\beta}\|^2_{\Sigma\ind{t}_h}, \nn \\
  & \leq \left( \nu^2 + \lambda^{-1} \vepslip\right) \cdot \|\theta\ind{t}_h - \thetastar_{h,\beta}\|^2_{\Sigma\ind{t}_h}.
  \label{eq:estimation_error_bound}
\end{align}
where the last inequality follows from $\sigma_{\min}(\Sigma\ind{t}_h)\geq \lambda$ and \cref{assum:linear}. This bound is particularly useful as it allows us to control the ``on-policy'' error:
\[
\E_{\pihat\ind{t}_{\ell+1:h}}\left[\inner{\varphi_h(\x_h, \a_h)}{\theta\ind{t}_h - \thetastar_{h,\beta}}^2 \mid \x_\ell = x_\ell, \a_\ell = a_\ell\right],
\]
for all state-action pairs $(x_\ell, a_\ell) \in \cC_\ell\ind{t}$, in terms of the regression error at layer $h$. 
\begin{remark}
\opsdp{} \citep{mhammedi2024sampleoracleefficientreinforcement} features a subroutine similar to \uncertsa. The key difference is that \opsdp{} uses a core-set of policies rather than state-action pairs. Consequently, its corresponding subroutine greedily selects the policy that maximizes the elliptical objective \(\|\phi_h(\x_h, \a_h)\|_{(\Sigma\ind{t}_h)^{-1}}\), evaluated in expectation over the core-set of policies and the algorithm's current policy \(\pihat\ind{t}_h\).\loose
\end{remark}

\subsubsection{Parameter Choices for \mtalg}
\label{app:params}
\newcommand{\frakc}{\mathfrak{c}}
For $\frakc = \polylog(d,\Ccond(\pistar_\beta),1/\delta, 1/\veps, H,
B)$ sufficiently large, we set the parameters in \mtalg as:
\begin{gather}
   \veps_\reg^2 = \veps, \ \ \nu = 1/2,\ \ 
    \Trounds = d H^2 \frakc,\ \
    \lambda  = \frac{\veps_\reg^2}{\frakc B^2},\ \
    \Mrej  = \frac{\Ccond(\pistar_\beta)^2 \Trounds^2 H^2 B^4 \frakc}{\veps_\reg^2},\nn \\
    \deltarej  = \frac{\veps^2_\reg}{H^2 B^3 \Trounds \frakc},  \ \
    \Nbar_\spann  = \frac{\Trounds H^4 B^4 \Ccond(\pistar_\beta)\frakc}{\veps_\reg^2}, \ \  N_\spann  = \frac{\Trounds H^2 B^4\frakc}{\veps_\reg^2},\nn \\ 
    N_\reg = \frac{H^2 B^4 d \Trounds \frakc}{\veps_\reg^2}.\nn
\end{gather}

\subsection{Main Guarantee for \mtalg{} (Generalization of \creftitle{thm:multi})}

Building on the intuition in the prequel, the main guarantee for \mtalg is as follows.
  
\begin{theorem}[Main guarantee for \mtalg]
\label{thm:main}
Let $\beta, \veps,\delta\in(0,1)$, $B>0$, and $\piref$ be such that $\veps\leq \beta^2/4$ and suppose \cref{assum:linear} holds with $B>0$. Then, the policy $\pihat$ returned by $\mtalg(\beta, \delta, \veps, B,\piref)$ (\cref{alg:ops-dp}) satisfies 
\begin{align}
J_\beta(\pistar_\beta) - J_\beta(\pihat)\leq \veps.\nn 
\end{align}
Furthermore, the algorithm requires $\Tdata(\veps,\delta) \leq
\poly(d, B, H,\beta^{-1},\veps^{-1} \log(1/\delta))$ reward, transition, and local simulator queries and $\Tcomp(\veps,\delta) \leq \poly\left(\Ccond (\pistarb),\Tdata(\veps,\delta)\right)$ runtime and sampling oracle queries.
\end{theorem}
Critically, we see that the sample complexity
$\Tcomp(\veps,\delta)$ is polynomial in all of the relevant problem
parameters, and the runtime and oracle complexity
$\Tcomp(\veps,\delta)$ scales with the action-level coverage
coefficient $\max_{h\in[H]}\Ccond
(\pistar_{h,\beta})$. \cref{thm:multi} is an immediate corollary.
Overall, the polynomial dependence on other
problem parameters is significantly worse than that of \mainalg; this can likely
be tightened with more effort, but we leave this for future
work.\loose

  The remainder of \cref{part:multi} is dedicated to proving \cref{thm:main}. The main idea behind the proof is to show
that after a sufficient number of iterations, the linear regression
problems solved by \texttt{FitValue} become approximately
realizabile, in the sense that the error
\[
    \left|\Qstar_{h,\beta}(x_h, a_h) - \Qstar_{h,\beta}(x_h, \fraka) 
    - Q^{\pihat\ind{t}}_h(x_h, a_h) + Q^{\pihat\ind{t}}_h(x_h, \fraka)\right|
\]
is small on average. For this, the key challenge is to show that
misspecification errors propagate favorably across the layers $h\in\brk{H}$, avoiding
the dreaded \emph{error amplification} phenomenon
\citep{wang2021exponential} in which misspecification errors compound
exponentially. For this, our main insight is that the regularization
parameter $\beta$ allows for benign error propagation, with errors at
layer $h+1,\ldots,H$ only having a higher-order impact on the
misspecification at layer $h$. This shows that regularization enables
statistical tractability in a similar way to the assumption of an
action gap $\Delta$ found in prior work on linearly-realizable
$Q^{\star}$ \citep{weisz2021query,li2021sample,mhammedi2024power}, an
observation which we expect to find broader use.

We emphasize that while \mtalg draws inspiration from prior work on
the linear-$\Qstar$ problem and relatives---particularly
  \citet{mhammedi2024power,mhammedi2024sampleoracleefficientreinforcement}---it
  requires fairly substantial modifications, both in design and
  analysis---to (i) leverage KL regularization, and (ii) achieve
  computational efficiency in the sampling oracle framework.

\begin{remark}[Action-level coverage]
  One can always choose $\pirefh(\cdot\mid{}x)=\unif(\cA)$,
  so $\max_{h\in[H]}\Ccond (\pistar_{h,\beta})\leq\abs{\cA}$. For
  the token-level MDP formulation described in
  \cref{sec:prelim_rl}, where actions correspond to tokens, this bound may be reasonable (though likely pessimistic). However, for the multi-turn language modeling
  formulation (\citet{xiong2024building}) where each action
  $\ba_h$ represents a sub-sequence of tokens rather than a single
  token (e.g., corresponding to a portion of a proof), paying for
  $\abs{\cA}$ is unacceptable, and access to a base policy with good
  coverage is crucial.\loose
\end{remark}

\begin{remark}[On the anchor action]\label{remark:mtss-query-access}
As discussed in \cref{sec:mtss_oracle}, the sampling oracle used in \mtalg (\cref{def:oracle_policy}) goes slightly
outside of the sampling oracle definition in \cref{def:oracle_autoregressive}
by assuming that we can query the features $\phi_h(\x_h,\mathfrak{a})$
at an arbitrary fixed \emph{anchor action} $\mathfrak{a}$ for all of the states
$\x_h$ encountered by the algorithm. We use the anchor action
$\mathfrak{a}$ to regress onto differences in regularized rewards,
which is motivated by the fact that the difference $\Qstar_{h,\beta}(x,a)-\Qstar_{h,\beta}(x,\mathfrak{a})$ is
linear (per \cref{eq:value_difference}), while $\Qstar_{h,\beta}(x,a)$
itself may not be. This assumption of access to $\phi_h(\x_h,\mathfrak{a})$ can be
avoided by incorporating ``pairwise'' truncated policies
$\pibar_{\theta}$ of the type used in \spanalg (see
\cref{eq:truncated_softmax_body}), which can be thought of as sampling
a fresh anchor action $\a_h\sim\pirefh(\cdot\mid\x_h)$ for each state
the algorithm encounters. We opt to use the
fixed anchor approach---in spite of requiring a slightly stronger oracle---to keep presentation as simple as
possible, as the analysis of \mtalg is already quite involved. We mention in passing that the anchor action assumption can also be
removed if we directly assume that $\Qstar_{h,\beta}(x,a)$ is linear,
by regressing onto absolute rewards.
\end{remark}

\newpage
\section{Guarantee for \uncertsa}

\label{sec:guarantee_uncertain}
In this section, we present the main guarantee of \texttt{UncertainStateAction} (\cref{alg:fitbonus_p}) as a standalone algorithm; see \cref{lem:unctraj}. Then, in \cref{lem:uncertaintraj}, we provide its guarantee when used as a subroutine within \mtalg (\cref{alg:ops-dp}). For a discussion of the motivation for these results, we refer back to \sssref{sec:uncertain}.
\begin{lemma}
	\label{lem:unctraj}
	Consider a call to $\uncertsa_h(\cC_{0:h-1},\pihat_{1:h},\Sigma_h;\fraka, N,\Nbar)$ (\cref{alg:fitbonus_p}) for some given $h, \cC_{0:h-1},\pihat_{1:h},\Sigma_h$, $\fraka\in \cA$, $N$, and $\Nbar$ such that $\sigma_{\min}(\Sigma_h)\geq \lambda$, for some $\lambda \in(0,1)$. Then, for any $\delta' \in (0,1)$ and $\zeta \in (0,1/2)$, with probability at least $1-\delta'$, the output $(\hat{x}_h,\hat{a}_{h})$ of $\uncertsa$ satisfies:
    \begin{itemize}[leftmargin=*]
        \item For all $\ell\in[0 \ldotst h-1]$ and $(x_\ell,a_\ell)\in \cC_{\ell}$,
    \arxiv{
	\begin{align}
	\P_{\pihat_{\ell+1:h}}\left[ \|\varphi_h(\x_h,\a_h)\|^2_{\Sigma_h^{-1}}> 2 \left( \zeta \vee \|\varphi_h(\hat x_h,\hat a_h)\|^2_{\Sigma_h^{-1}}\right) \mid \x_\ell=x_\ell, \a_\ell=a_\ell\right]& \leq \max_{h\in[H]}\frac{4\log \left(\frac{16H |\cC_h|}{\lambda \delta'\zeta}\right)}{N},\label{eq:firstone}
	\end{align}
    }
	where $\varphi_h(\cdot, \cdot)\coloneqq \phi_h(\cdot, \cdot) - \phi_h(\cdot,\fraka)$. 
    \item Furthermore, there exists $\cX_{h,\spann} \subseteq \cX$ such that for all $\ell\in[0\ldotst h-1]$ and $(x_\ell,a_\ell) \in \cC_{\ell}$, $\P_{\pihat_{\ell+1:h-1}}[\x_h\in \cX_{h,\spann} \mid \x_\ell=x_\ell, \a_\ell=a_\ell]\geq 1 - \max_{h\in[H]}\frac{4}{N}\log \frac{32 H N |\cC_h|}{\lambda \delta'\zeta}$ and %
	\begin{align}
	\arxiv{\forall x_h \in \cX_{h,\spann}, \quad} \P_{\a \sim \pi_{h,\refe}(\cdot\mid  x_h)}\left[\|\varphi_h(x_h,\a)\|^2_{\Sigma_h^{-1}}> 2 \left( \zeta \vee \|\varphi_h(\hat x_h,\hat a_h)\|^2_{\Sigma_h^{-1}}\right)\right] &\leq \max_{h\in[H]}\frac{4\log \left(\frac{16H |\cC_h|}{\lambda \delta'\zeta}\right)}{\Nbar}. \label{eq:secondone}
	\end{align}
\end{itemize}
\end{lemma}
\begin{proof}[\pfref{lem:unctraj}]
	Fix $\delta' \in (0,1)$ and $\zeta \in(0,1/2)$, and let $\Gamma  \coloneqq \{\zeta, 2 \zeta , \dots,   \ceil{\frac{4}{\zeta \lambda}}\zeta\}$. Further, for $\ell \in[0\ldotst h-1]$ and $(x_\ell,a_\ell)\in \cC_\ell$, let $\cD_\ell(x_\ell,a_\ell)$ be the dataset in \cref{alg:fitbonus_p} when the algorithm returns. Note that $\cD_\ell(x_\ell,a_\ell)$ consists of $N$ i.i.d.~pairs sampled from $\P_{\pihat_{\ell+1:h}}[(\x_h,\a_h)=\cdot\mid \x_\ell=x_\ell, \a_{\ell}=a_{\ell}]$. Thus, by Freedman's inequality (\cref{lem:freedman}) and the union bound over $\ell\in[0\ldotst h-1]$, $(x_\ell,a_\ell)\in \cC_\ell$, and $\gamma \in \Gamma$, there is an event $\cE$ of probability at least $1-\delta'/2$ under which, %
	\begin{align}
	\arxiv{\forall \ell \in [0\ldotst h-1], \forall (x_\ell,a_\ell)\in \cC_\ell, \forall \gamma\in \Gamma, \quad} &\P_{\pihat_{\ell+1:h}}\left[\|\varphi_h(\x_h,\a_h)\|^2_{\Sigma^{-1}_h} \geq \gamma \mid \x_\ell=x_\ell, \a_{\ell}=a_{\ell}\right] \nn \\
	& \leq \frac{2}{N}\sum_{(x,a)\in \cD_\ell(x_\ell,a_\ell)} \mathbb{I}\left\{\|\varphi_h(x,a)\|^2_{\Sigma^{-1}_h} > \gamma \right\} + \frac{4 \log (2H |\cC_\ell| |\Gamma|/\delta')}{N}, \nn \\
	& \leq \frac{2}{N}\sum_{(x,a)\in \cD_\ell(x_\ell,a_\ell)} \mathbb{I}\left\{\|\varphi_h(x,a)\|^2_{\Sigma^{-1}_h} > \gamma \right\} + \frac{4\log \left(\frac{16H |\cC_\ell|}{\lambda \delta'\zeta }\right)}{N}, \label{eq:probbound}
	\end{align}
	where the last step follows by the facts that $|\Gamma|\leq \ceil{\frac{4}{\zeta \lambda}}$, $\lambda\in(0,1)$, and $\zeta \in (0,1/2)$. Now, since $\sigma_{\min}(\Sigma_h)\geq \lambda$ and $\sup_{(x,a)\in \cX\times \cA}\|\phi_h(\cdot, \cdot)\| \leq 1$ (\cref{assum:linear}), we have that \arxiv{$\sup_{(x,a)\in \cX\times \cA}\|\varphi_h(x,a)\|^2_{\Sigma_h^{-1}}\leq \frac{4}{\lambda}$.} Therefore, by the definition of $\Gamma$, we have that for all $\ell \in [0\ldotst h-1]$ and $(x_\ell,a_\ell)\in \cC_\ell$, there exists $\gamma_\ell(x_\ell,a_\ell)\in \Gamma$ such that 
	\begin{align}
		\max_{(x,a)\in \cD_\ell(x_\ell,a_\ell)}\|\varphi_h(x,a)\|^{2}_{\Sigma_h^{-1}}  & \leq \gamma_\ell(x_{\ell},a_{\ell}), \nn\\  &  \leq   \max_{(x,a)\in \cD_\ell(x_\ell,a_\ell)}\|\varphi_h(x,a)\|^{2}_{\Sigma_h^{-1}} + \zeta,\nn \\& \leq 2 \left(\zeta \vee \max_{(x,a)\in \cD_\ell(x_\ell,a_\ell)}\|\varphi_h(x,a)\|^{2}_{\Sigma_h^{-1}}\right),\nn \\& \leq 2 \left(\zeta \vee \|\varphi_h(\hat{x}_h,\hat{a}_h)\|^{2}_{\Sigma_h^{-1}}\right),\label{eq:sandwitch2}
	\end{align}
	where the last inequality follows by the fact that 
\[ \max_{\ell\in[0\ldotst h-1]}\max_{(x,a)\in \cD_\ell(x_\ell,a_\ell)}\|\varphi_h(x,a)\|^{2}_{\Sigma_h^{-1}}\leq  \|\varphi_h(\hat{x}_h,\hat{a}_h)\|^{2}_{\Sigma_h^{-1}}\]
by definition of $(\hat{x}_h,\hat{a}_h)$ (see \cref{alg:fitbonus_p}).
Substituting $\gamma_\ell(x_\ell,a_\ell)$ for $\gamma$ in \eqref{eq:probbound} and using \eqref{eq:sandwitch2}, we get that under $\cE$, for all $\ell\in[0\ldotst h-1]$ and $(x_\ell,a_\ell)\in \cC_\ell$:
\begin{align}
  \P_{\pihat_{\ell+1:h}}\left[\|\varphi_h(\x_h,\a_h)\|^2_{\Sigma^{-1}_h} > 2 \left(\zeta \vee \|\varphi_h(\hat{x}_h,\hat{a}_h)\|^{2}_{\Sigma_h^{-1}}\right) \mid \x_\ell=x_\ell, \a_{\ell}=a_{\ell}\right] 
 \leq \frac{4\log \left(\frac{16H |\cC_\ell|}{\lambda \delta'\zeta }\right)}{N}.\nn 
\end{align}
This shows that there is an event of probability at least $1-\delta'/2$ under which \eqref{eq:firstone} holds. 

\paragraphi{Second claim} We now prove the second claim of the lemma. Let $(\cD_\ell')$ be the datasets in \cref{alg:fitbonus_p} when the algorithm returns. Note that for $\ell \in[0\ldotst h-1]$, $(x_\ell,a_\ell)\in \cC_\ell$, and $(x_h,a_h)\in \cD_\ell(x_\ell,a_\ell)$, the dataset $\widebar{\cD}_\ell(x_h)$ consists of $\Nbar$ i.i.d.~actions sampled from $\pi_{h,\refe}(\cdot \mid x_h)$. Thus, by Freedman's inequality (\cref{lem:freedman}) and a union bound over $\ell\in[h- 1]$, $(x_\ell,a_\ell)\in \cC_\ell$, $(x_h,a_h)\in \cD_\ell(x_\ell,a_\ell)$, and $\gamma\in \Gamma$, there is an event $\cE'$ of probability at least $1-\delta'/4$ under which for all $\ell \in [0\ldotst h-1]$, $(x_\ell,a_\ell)\in \cC_\ell$, $(x,a)\in \cD(x_\ell,a_\ell)$, and  $\gamma\in \Gamma=\{\zeta, 2 \zeta , \dots,   \ceil{\frac{4}{\zeta \lambda}}\zeta\}$:
\arxiv{
\begin{align}
 \P_{\pi_{\refe}}\left[\|\varphi_h(\x_h,\a_h)\|^2_{\Sigma^{-1}_h} > \gamma \mid \x_h=x\right]
	& \leq \frac{2}{\Nbar}\sum_{a'\in \widebar{\cD}_\ell(x)} \mathbb{I}\left\{\|\varphi_h(x,a')\|^2_{\Sigma^{-1}_h} > \gamma \right\} + \frac{4\log \left(\frac{32 H N |\cC_\ell|}{\lambda \delta'\zeta }\right)}{\Nbar}.\label{eq:probbound2}
	\end{align} 
}
Substituting $\gamma_\ell(x_\ell,a_\ell)$ for $\gamma$ in \eqref{eq:probbound2} and using \eqref{eq:sandwitch2}, we get that under $\cE'$, for all $\ell\in[0\ldotst h-1]$, $(x_\ell,a_\ell)\in \cC_\ell$, and $(x,a)\in \cD_\ell(x_\ell,a_\ell)$:
\begin{align}
    \P_{\pi_{\refe}}\left[\|\varphi_h(\x_h,\a_h)\|^2_{\Sigma^{-1}_h} > 2 \left(\zeta \vee \|\varphi_h(\hat{x}_h,\hat{a}_h)\|^{2}_{\Sigma_h^{-1}}\right) \mid \x_h=x\right]
       & \leq \frac{4\log \left(\frac{32H N |\cC_\ell|}{\lambda \delta'\zeta }\right)}{\Nbar}.\label{eq:neat}
       \end{align}  
	On the other hand, since for all $(x_\ell,a_\ell)\in \cC_\ell$, $\cD_\ell(x_\ell,a_\ell)$ consists of $N$ i.i.d.~pairs sampled from $\P_{\pihat_{\ell+1:h}}[(\x_h,\a_h)=\cdot\mid \x_\ell=x_\ell, \a_{\ell}=a_{\ell}]$, Freedman's inequality (\cref{lem:freedman}) and a union bound over $\ell \in[0\ldotst h-1]$, $(x_\ell,a_\ell)\in \cC_\ell$, and $\gamma \in \Gamma$ implies that there is an event $\cE''$ of probability at least $1-\delta'/4$  under which for all $\ell\in[0\ldotst h-1]$, $(x_\ell,a_\ell)\in \cC_\ell$, and $\gamma\in \Gamma$ we have that
    \arxiv{
	\begin{align}
	& \P_{\pihat_{\ell+1:h-1}} \left[  \P_{\pi_{\refe}}\left[\|\varphi_h(\x_h,\a_h)\|^2_{\Sigma^{-1}_h} > \gamma \mid \x_h\right]
	 >  \frac{4\log \left(\frac{32H N |\cC_\ell|}{\lambda \delta'\zeta }\right)}{\Nbar} \mid \x_\ell=x_\ell, \a_{\ell}=a_{\ell} \right]\nn \\
	& \leq \frac{2}{N} \sum_{(x,a)\in \cD_\ell(x_\ell,a_\ell)}  \mathbb{I}\left\{\P_{\pi_{\refe}}\left[\|\varphi_h(\x_h,\a_h)\|^2_{\Sigma^{-1}_h} > \gamma \mid \x_h=x\right]
	>   \frac{4\log \left(\frac{32H N |\cC_\ell|}{\lambda \delta'\zeta }\right)}{\Nbar}\right\} \nn \\
	& \qquad + \frac{4 \log \left(\frac{16 H  |\cC_\ell|}{\lambda \delta' \zeta} \right)}{N}. \label{eq:tosub}
	\end{align}
    }
	Substituting $\gamma_\ell(x_\ell,a_\ell)$ for $\gamma$ in \eqref{eq:tosub} and using \eqref{eq:neat}, we get that under $\cE' \cap\cE''$, for all $\ell\in[0\ldotst h-1]$ and $(x_\ell,a_\ell)\in \cC_\ell$: 
    \arxiv{
	\begin{align}
		& \P_{\pihat_{\ell+1:h-1}} \left[\P_{\pi_{\refe}}\left[\|\varphi_h(\x_h,\a_h)\|^2_{\Sigma^{-1}_h} > 2 \left(\zeta \vee \|\varphi_h(\hat{x}_h,\hat{a}_h)\|^{2}_{\Sigma_h^{-1}}\right) \mid \x_h\right]
		> \frac{4\log \left(\frac{32H N |\cC_\ell|}{\lambda \delta'\zeta }\right)}{\Nbar} \mid \x_\ell=x_\ell, \a_{\ell}=a_{\ell} \right]\nn \\
		& \leq \frac{4 \log \left(\frac{16 H  |\cC_\ell|}{\lambda \delta' \zeta} \right)}{N}. \label{eq:postunion}
	\end{align}
    }
	We define 
	\begin{align}
		\cX_{h,\spann} \coloneqq \left\{x \in \cX : \begin{array}{l} \P_{\pi_{\refe}}\left[ \|\varphi_h(\x_h,\a_h)\|^2_{\Sigma^{-1}_h} > 2 \left(\zeta \vee \|\varphi_h(\hat{x}_h,\hat{a}_h)\|^{2}_{\Sigma_h^{-1}}\right) \mid \x_h=x\right]\\ 
		\quad > \frac{4\log \left(\frac{32H N |\cC_\ell|}{\lambda \delta' \zeta} \right)}{\Nbar} \end{array} \right\}.\nn 
	\end{align}
By the union bound, $\P[\cE \cap \cE' \cap \cE'']\geq 1 - \delta'$, which combined with \eqref{eq:postunion} completes the proof.\loose
\end{proof} 
\begin{lemma}[Guarantee of \uncertsa for \mtalg] 
\label{lem:uncertaintraj}
Let $\beta, \delta, \veps \in (0,1)$, $B>0$, and $\piref$ be given and consider a call to $\mtalg(\beta, \delta, \veps,B, \piref)$ (\cref{alg:ops-dp}). Let $\lambda, \nu\in(0,1)$ and $\Trounds$ be as in \cref{alg:ops-dp}. Then, there is an event $\cE^\spann$ of probability at least $1-\delta/2$ under which for all $t\in [\Trounds]$ and $h\in[H]$, the output $(x\ind{t}_h,a_{h}\ind{t})$ of $\uncertsa_h$ in \cref{line:designdir} satisfies
\begin{itemize}[leftmargin=*]
    \item  For all $\ell\in[0\ldotst h-1]$ and $(x_\ell,a_\ell)\in \cC\ind{t}_\ell$,\arxiv{
	\begin{align}
		\P_{\pihat\ind{t}_{\ell+1:h}}\left[ \|\varphi_h(\x_h,\a_h)\|^2_{(\Sigma\ind{t}_h)^{-1}}>  \nu^2 \vee \left(2\|\varphi_h(x\ind{t}_h,a\ind{t}_h)\|^2_{(\Sigma_h\ind{t})^{-1}}\right) \mid \x_\ell=x_\ell, \a_\ell=a_\ell\right] \leq \vepslip \coloneqq  \frac{8\log \left(\frac{32 H \Trounds }{\lambda \delta \nu^2}\right)}{\Nspann},
	\end{align}
}
	where $\varphi_h(\cdot, \cdot)\coloneqq \phi_h(\cdot, \cdot)- \phi_h(\cdot, \fraka)$ ($\fraka$ as in \cref{alg:ops-dp}) and $\pihat\ind{t}$ is as in \cref{alg:ops-dp}.
    \item  
	Furthermore, there exists $\cX\ind{t}_{h,\spann}\subseteq \cX$ such that for all $\ell\in[0\ldotst h-1]$ and $(x_\ell,a_\ell)\in \cC\ind{t}_\ell$, $\P_{\pihat\ind{t}}[\x_h\in \cX\ind{t}_{h,\spann} \mid \x_\ell=x_\ell, \a_\ell=a_\ell]\geq 1 -\veps_\spann$ and 
	\begin{align}
	 \forall x_h\in \cX\ind{t}_{h,\spann}, \ \  \P_{\a\sim \pi_{h,\refe}(\cdot \mid x_h)}\left[ \|\varphi_h(x_h,\a)\|^2_{(\Sigma\ind{t}_h)^{-1}}>  \nu^2 \vee \left(2\|\varphi_h(x\ind{t}_h,a\ind{t}_h)\|^2_{(\Sigma_h\ind{t})^{-1}}\right)\right] \leq \uveps,\nn
	\end{align}
where $\uveps\coloneqq \frac{8 }{\Nspanb}\log \frac{16 H \Trounds}{\lambda \delta\nu^2}$.
\end{itemize}
\end{lemma}
\begin{proof}[\pfref{lem:uncertaintraj}]
	The result follows from \cref{lem:unctraj} with parameters \[(\Sigma_h, \delta',\zeta, N, \Nbar)= (\Sigma\ind{t}_h,\delta/(2H\Trounds), \nu^2/2, \Nspann,\Nspanb),\] and \cref{lem:unionbound} (essentially the union bound over $t\in[\Trounds]$ and $h\in [H]$), and the fact that $|\cC\ind{t}_h|\leq 2 \Trounds$, for all $t\in[\Trounds]$ and $h\in[H]$. 
\end{proof}

 \label{sec:proof_uncertain}
\newpage
\section{Guarantee for \texttt{FitValue}}
\label{sec:proof_fitvalue}

\label{proof:guarantee_fitvalue}
In this section, we present the main guarantee of \texttt{FitValue} (\cref{alg:fitval}) as a standalone algorithm, as shown in \cref{lem:fitval_pre}. In \cref{sec:helpfit}, we state and prove supporting lemmas for \cref{lem:fitval_pre}. Then, in \cref{sec:guaranteefit}, we describe the guarantee of \texttt{FitValue} when used as a subroutine within \mtalg (\cref{alg:ops-dp}). For a discussion of the significance of these results and their implications, refer to \sssref{sec:fitvalue}. 

To state the result, we recall that
  \begin{align}
    \pibar_{h,\theta}(\cdot\mid x) \propto \pi_{h,\refe}(\cdot\mid x) \cdot e^{\varphibar_h(x,\cdot)^\top \theta/\beta};\quad  \text{and} \quad \pibar^\star_{h,\beta}(\cdot\mid x) \propto \pi_{h,\refe}(\cdot\mid x) \cdot e^{\varphibar_h(x,\cdot)^\top \theta_{h,\beta}^\star/\beta}, \label{eq:truncatedpolicies}
\end{align}
with $\varphibar_{h}(x,\cdot) \coloneqq \varphi_h(x,\cdot) \cdot \mathbb{I}\left\{ \|\varphi_h(x,\cdot)\|^2_{\Sigma^{-1}_h}  \leq \nu^2\right\}$ (with $\nu$ as in \cref{alg:ops-dp}). Further, we recall that $\theta^\star_{h,\beta}$ and $\pi^\star_{h,\beta}$ denote the optimal KL-regularized policy and corresponding parameter,  as in \cref{assum:linear} and \cref{def:inition}.\loose

\begin{lemma}
	\label{lem:fitval_pre}
Let $h\in[H]$, $\cC_h$, $\theta_{h+1:H}\subset \bbB(B)$, $\Sigma_{h+1:H}$, $\lambda$, $\fraka$, $N$, $M,\tilde\delta$, and $\piref$ be given and suppose \cref{assum:linear} holds with $B>0$. Further, suppose that $\cC_h$ is a multiset of the form $\cC_h = \bigcup_{i\in{n}} \{(x_i,a_i),(x_i,\afrak)\}$ for some sequence $(x_i,a_i)_i\subset\cX \times \cA$ and $n\geq 1$.
Consider a call to $\texttt{FitValue}_h(\cC_h, \theta_{h+1:H}, \Sigma_{h+1:H}; \fraka, N, M, \tilde\delta, \piref)$ (\cref{alg:fitval}) and let $\cD$ be the dataset in the algorithm when it returns. Further, define $\varphi_h(\cdot,\cdot)\coloneqq \phi_h(\cdot,\cdot)- \phi_h(\cdot,\fraka)$ and $\Sigma_h \coloneqq \lambda I + \frac{1}{N}\sum_{(x_{h:H},a_{h:H},r_{h:H})\in \cD} \varphi_h(x_h,a_h)\varphi_h(x_h,a_h)^\top$. Then, for any $\delta'\in (0,1)$, with probability at least $1 - \delta'$, the output $\thetahat_h$ of \texttt{FitValue} satisfies: 
\begin{align}
& \|\thetahat_h -\theta^\star_{h,\beta}\|^2_{\Sigma_h}\nn \\
& \leq 4 \lambda B^2 + \frac{C_1}{N}  +  C_2 \tilde \delta  + C_3\sum_{(x_h,a_h)\in \cC_h} \sum_{\ell=h+1}^H \P_{\pihat_\theta}\left[ M < 16\Ccond(\pibar_{\ell,\theta}\mid \x_\ell)^2 \mid \x_h =x_h, \a_h =a_h\right]\nn \\
& \quad + 2304  H B \beta \sum_{(x_h,a_h)\in \cC_h}\sum_{\ell=h+1}^H \E_{\pihat_{\theta}}\left[\min \left(1,\Ccond(\pibar_{\ell,\theta}\mid \x_\ell) \cdot  \sqrt{\frac{2}{M}}\right)\mid  \x_h =x_h, \a_h=a_h \right] \nn \\
& \quad + 19200  H\beta^2\sum_{(x_h,a_h)\in \cC_h}\sum_{\ell=h+1}^H \E_{\pihat_{\theta}}\left[\kl{\pibar_{\ell,\theta}(\cdot \mid \x_\ell)}{\pibar^\star_{\ell,\beta}(\cdot \mid \x_\ell)}^2 \mid \x_h=x_h,\a_h = a_h \right] \nn \\
& \quad + 7680  H \Rmax^2 \sum_{(x_h,a_h)\in \cC_h}\sum_{\ell=h+1}^H \E_{\pihat_{\theta}}\left[ B_{\ell,\theta}(\x_\ell)^2\mid \x_h=x_h,\a_h=a_h \right], \label{eq:fitvalue_bound}
\end{align}
where 
\begin{gather}
\arxiv{ 
    C_1\coloneqq 6400 B^2   H^2 d\log(3 N/\delta')+  3840  H^2 B^2 |\cC_h|, \nn \\
C_2 \coloneqq  19200 |\cC_h| \cdot \big(12\Rmax^2+ 48 \beta^2 \log(4 M \log(4 \tilde\delta^{-1}))^2 + 85 H^2 B^2\big) + 3072 |\cC_h|   H^2 B\beta\log (4 M \log (4 \tilde\delta^{-1})),\nn \\
	C_3 \coloneqq 16 H \cdot \big(8 \Rmax^2 \log(4 M \log(4 \tilde\delta^{-1}))^2 + 200 H^2B^2 \big)  + 3072   H B^2, \nn\\ 
    \shortintertext{and for $x\in \cX$:}
    B_{\ell,\theta}(x) \coloneqq \min \left( 1, \max_{\pi \in \{\pibar_{\ell,\theta},\pibar^\star_{\ell,\beta}, \pi^\star_{\ell,\beta}\}} \Ccond(\pi\mid x) \cdot \P_{\a \sim \pi_{\ell,\refe}(\cdot \mid x)}\left[\|\varphi_\ell(x,\a)\|^2_{\Sigma_\ell^{-1}}>\nu^2\right] \right).\nn
}
  \end{gather}
\end{lemma}
We remark that for our application of this result within \mtalg, the fact that the right-hand side of \cref{eq:fitvalue_bound} scales with the squared KL divergence $\kl{\pibar_{\ell,\theta}(\cdot \mid \x_\ell)}{\pibar^\star_{\ell,\beta}(\cdot \mid \x_\ell)}^2$ is crucial in enabling favorable error propagation across layers $h$.

\begin{proof}[\pfref{lem:fitval_pre}]
Fix $\delta' \in (0,1)$. Throughout this proof, for any $h\in[H]$ and $x\in \cX$, we let $\pihat_{h,\theta}(\cdot \mid x)$ denote the distribution of $\a$, where $(\a,\brho) = \softmaxsample_{\beta,M,\tilde\delta}(\inner{\varphibar_h(x,\cdot)}{\theta_h} \midsem x, \piref)$. For $(x_h,a_h)\in \cC_h$, let $\cD(x_h,a_h)$ be as in \cref{alg:fitval} when the algorithm returns. Note that $\cD(x_h,a_h)$ consists of $N$ i.i.d.~points $\z_h$ which are obtained by first sampling two trajectories $(\bx'_h,\a'_h,\br'_h,\brho'_h,\dots, \bx'_{H},\a'_{H},\br'_H, \brho'_H)$ and $(\bx''_h,\a''_h,\br''_h,\dots, \bx''_{H},\a''_{H},\br''_H, \brho''_H)$ via the following process. Initialize $\x'_h = \x''_h= x_h$, $\a'_h = a_h$, and $\a''_h = \afrak$, and sample $\br'_h\sim \rstar_h(\x'_h,\a'_h)$ and $\br_h''\sim \rstar_h(\x''_h,\a''_h)$. Then, for $\ell=h+1,\dots, H$,
\begin{itemize} 
\item Sample $\x'_{\ell}\sim \P[\cdot \mid \x_{\ell-1}= \x'_{\ell-1}, \a_{\ell-1}= \a'_{\ell-1}]$ and $\x_\ell''\sim \P[\cdot \mid \x_{\ell-1}= \x''_{\ell-1}, \a_{\ell-1}= \a''_{\ell-1}]$;
\item Set $(\a'_\ell,\brho'_\ell)\gets \softmaxsample_{\beta,M,\tilde\delta}(\inner{\varphibar_\ell(\x'_\ell,\cdot)}{\theta_\ell} \midsem \x'_\ell,\piref)$;
\item Set $(\a_\ell'',\brho''_\ell)\gets \softmaxsample_{\beta,M,\tilde\delta}(\inner{\varphibar_\ell(\x_\ell'',\cdot)}{\theta_\ell}\midsem \x_\ell'',\piref)$;
\item Sample rewards $\br'_\ell\sim \rstar_\ell(\x'_\ell,\a'_\ell)$ and $\br_\ell''\sim \rstar_\ell(\x''_\ell,\a''_\ell)$;
    \end{itemize}
	then, finally, set
	\begin{align}
		\z_h(x_h,a_h) = \br'_h  + \sum_{\ell=h+1}^H \left(\br'_\ell - \beta \log \brho'_\ell\right) -\br''_h -\sum_{\ell=h+1}^H \left(\br''_\ell - \beta \log \brho_\ell'' \right).\label{eq:zh} %
	\end{align}
	For the rest of this proof, for any $(x_{h},a_{h},z_{h})\in \cX \times \cA \times \reals$, we define
\begin{align}
\fhat(x_h,a_h) & \coloneqq \varphi_h(x_h,a_h)^\top \thetahat_h, \nn \\	
f_\star(x_h,a_h) & \coloneqq  Q^\star_{h,\beta}(x_h,a_h)-Q^\star_{h,\beta}(x_h,\fraka), \nn \\
b(x_h,a_h) &\coloneqq   Q^{\pihat_\theta}_{h,\beta}(x_h,a_h) - Q^{\pihat_\theta}_{h,\beta}(x_h,\fraka) - Q^\star_{h,\beta}(x_h,a_h)-Q^\star_{h,\beta}(x_h,\fraka), \nn \\
  \xi(x_h,a_h,z_h)  & \coloneqq   z_h - Q^{\pihat_{\theta}}_{h,\beta}(x_h, a_h) + Q^{\pihat_{\theta}}_{h,\beta}(x_h, \fraka). \label{eq:xi}
\end{align}
Further, for all $f,g: \cX \times \cA \rightarrow \reals$, define
\begin{align}
  \Lhat(f) & = \sum_{(x_h,a_h)\in \cC_h}\sum_{z_h\in \cD(x_h,a_h)  } (f(x_h,a_h)- z_h)^2, \nn \\
\tnorm{f - g}^2 & \coloneqq   N\sum_{(x_h,a_h)\in \cC_h} (f(x_h,a_h)- g(x_h,a_h))^2.\nn 
\end{align} 

\paragraphi{Basic least squares analysis}
We begin with a standard analysis of least squares. If $\theta_\ell=\thetastar_\ell$ for all $\ell>h$, then the conditional mean of $\z_h$ is equal to $\fstar(x_h,a_h)=Q^\star_{h,\beta}(x_h,a_h)-Q^\star_{h,\beta}(x_h,\fraka)$ up to negligible error caused by approximate sampling via \softmaxsample{}, and hence $\theta_h$ is solving (nearly) well-specified linear regression. The crux of the proof that follows will be to bound the misspecification corresponding to the term $b(x_h,a_h)$, which reflects the fact that $Q^{\pihat_{\theta}}_{h,\beta}(x_h,a_h)-Q^{\pihat_{\theta}}_{h,\beta}(x_h,\fraka)$ may not be linear in general.

To begin, note that with the notation introduced so far, $\thetahat_h$ in \cref{alg:fitval} satisfies
\begin{align}
	\thetahat_h \in \argmin_{\tilde\theta \in \bbB(B)} \Lhat(\inner{\varphi_h(\cdot,\cdot)}{\tilde\theta}). \label{eq:ERM}
\end{align}
This, together with the facts that $f_\star(x,a)=\varphi(x,a)^\top \theta^\star_{h,\beta}$ (by \cref{assum:linear} and \cref{lem:reward_difference}) and $\theta^\star_{h,\beta}\in \bbB(B)$ implies that
\begin{align}
0 &\geq \Lhat_n(\fhat)- \Lhat_n(f_\star),\nn \\ &= 2  \sum_{(x_h,a_h)\in \cC_h}\sum_{z_h\in \cD(x_h,a_h)  } (f_\star(x_h,a_h)- z_h) \cdot (\fhat(x_h,a_h)- f_\star(x_h,a_h)) + \tnorm{\fhat - f_\star}^2.\nn 
\end{align}
Rearranging, we get that 
\begin{align}
& \tnorm{\fhat - f_\star}^2 \nn \\
& \leq  4 \sum_{(x_h,a_h)\in \cC_h}\sum_{z_h\in \cD(x_h,a_h)  } (z_h - f_\star(x_h,a_h))\cdot  (\fhat(x_h,a_h)- f_\star(x_h,a_h))- \tnorm{\fhat - f_\star}^2, \nn \\
& \leq 4 \sum_{(x_h,a_h)\in \cC_h}\sum_{z_h\in \cD(x_h,a_h)  } (\xi(x_h,a_h,z_h) + b(x_h,a_h))\cdot  (\fhat(x_h,a_h)- f_\star(x_h,a_h))\nn \\
& \quad - \tnorm{\fhat - f_\star}^2, \nn \\
& \leq 4\sum_{(x_h,a_h)\in \cC_h}\sum_{z_h\in \cD(x_h,a_h)  } \xi(x_h,a_h,z_h) \cdot (\fhat(x_h,a_h)- f_\star(x_h,a_h))- \tnorm{\fhat - f_\star}^2 \nn \\
& \quad  +4 \sum_{(x_h,a_h)\in \cC_h}\sum_{z_h\in \cD(x_h,a_h)  }  b(x_h,a_h)\cdot  (\fhat(x_h,a_h)- f_\star(x_h,a_h)),\nn \\
& \leq 4\sum_{(x_h,a_h)\in \cC_h}\sum_{z_h\in \cD(x_h,a_h)  } \xi(x_h,a_h,z_h) \cdot (\fhat(x_h,a_h)- f_\star(x_h,a_h))- \tnorm{\fhat - f_\star}^2 \nn \\
& \quad +8 \sum_{(x_h,a_h)\in \cC_h}\sum_{z_h\in \cD(x_h,a_h) }  b(x_h,a_h)^2 + \frac{1}{2}\tnorm{\fhat- f_\star}^2,
\label{eq:twoterms_pre}
\end{align}
where the last inequality follows by AM-GM. Rearranging \eqref{eq:twoterms_pre}, we get that
\begin{align}
	\tnorm{\fhat - f_\star}^2 
	& \leq  \underbrace{8\sum_{(x_h,a_h)\in \cC_h}\sum_{z_h\in \cD(x_h,a_h)  } \xi(x_h,a_h,z_h) \cdot (\fhat(x_h,a_h)- f_\star(x_h,a_h))- 2\tnorm{\fhat - f_\star}^2}_{\text{I}} \nn \\
	& \quad +\underbrace{16 \sum_{(x_h,a_h)\in \cC_h}\sum_{z_h\in \cD(x_h,a_h) }  b(x_h,a_h)^2}_{\text{II}}.\label{eq:twoterms}
\end{align}
\paragraphi{Bounding Term I} We first bound Term I, which reflects the (nearly) mean-zero noise in the regression targets. Concretely, by \cref{lem:xibound} (stated and proven in the sequel), we have for all $(x_h,a_h)\in \cC_h$:  \begin{align}
	& \left|\E[ \xi(x_{h},a_{h},\z_h(x_h,a_h))]\right|\nn \\
	& \leq   \beta\sum_{\ell=h+1}^H\sum_{a\in \{a_h,\afrak\}} \E_{\pihat_{\theta}}\left[\mathbb{I}\{M \geq  4 \Ccond(\pibar_{\ell,\theta} \mid \x_\ell)^2 \} \cdot\Ccond(\pibar_{\ell,\theta}\mid \x_\ell) \cdot  \sqrt{\frac{2}{M}}\mid  \x_h =x_h, \a_h=a \right] \nn \\
	 & \quad + 16 H B \tilde\delta +8 H\beta\log (4 M \log (4 \tilde\delta^{-1})) \tilde\delta \nn \\
                                                                                                                                                                                                                                         & \quad + 8  B \sum_{\ell=h+1}^H\sum_{a\in \{a_h,\afrak\}}\P_{\pihat_\theta}\left[M < 4\Ccond(\pibar_{\ell,\theta}\mid \x_\ell)^2\mid\x_h=x_h,\a_h=a\right],\label{eq:concrete} \end{align}
where $\En\brk*{\cdot}$ denotes the expectation over $\z_h(x_h,a_h)$ under the process in \cref{eq:twoterms_pre} (beginning from $(x_h,a_h)$). Given this, we apply Freedman's inequality (\cref{lem:freedman}), using
\begin{itemize}[leftmargin=*]
\item The union bound over an $\frac{1}{N}$-net of $\bbB(B)$ in $\|\cdot\|$-distance; 
\item The fact that \begin{align}& \max(\|\fhat\|_{\infty},\|f^\star\|_{\infty}, \|\xi\|_{\infty})\nn \\&\leq 2 BH + 2\beta H\max_{\ell\in[h+1\ldotst H]} \sup_{(x,a)\in \cX \times \cA}\max \left( \left|\log \frac{\pibar_{\ell,\theta}(a \mid x)}{\piref(a\mid x)}\right|,  
	\left|\log \frac{\pistar_{\ell,\beta}(a \mid x)}{\piref(a\mid x)}\right|, 
	\left|\log \brho_\ell'\right|,  
	\left|\log \brho_\ell''\right|
 \right) ,\nn \\
	&  \leq 6 HB,\label{eq:inftybound}\end{align}
 since $e^{-2B/\beta}\leq \frac{\pihat_{\ell,\theta} (\cdot\mid \cdot)}{\piref(\cdot \mid \cdot)}\wedge  \frac{\pistar_{\ell,\beta} (\cdot\mid \cdot)}{\piref(\cdot \mid \cdot)}\leq\frac{\pihat_{\ell,\theta} (\cdot\mid \cdot)}{\piref(\cdot \mid \cdot)}\vee  \frac{\pistar_{\ell,\beta} (\cdot\mid \cdot)}{\piref(\cdot \mid \cdot)}\leq e^{2B/\beta}$ for all $\theta_\ell\in \bbB(B)$, and for all $\ell\in[H]$, $\br'_\ell,\br_\ell'' \in[0,B]$ and $\brho_\ell', \brho_\ell''\in[e^{-2B/\beta},e^{2 B/\beta}]$ (by \cref{thm:rejection_density}); 
\end{itemize}
to conclude that with probability at least $1-\delta'$, 
\begin{align}
 &\sum_{(x_h,a_h)\in \cC_h}\sum_{z_h\in \cD(x_h,a_h)} \xi(x_h,a_h,z_h) \cdot (\fhat(x_h,a_h)- f_\star(x_h,a_h))\nn \\
 & \leq N\sum_{(x_h,a_h)\in \cC_h} \E[\xi(x_h,a_h,\z_h(x_h,a_h))] \cdot (\fhat(x_h,a_h)- f_\star(x_h,a_h)) \nn \\ & \quad + \frac{N}{144 B^2 H^2}\sum_{(x_h,a_h)\in \cC_h}\E\left[\left(\xi(x_h,a_h,\z_h(x_h,a_h))^2\cdot (\fhat(x_h,a_h)- f_\star(x_h,a_h))\right)^2\right] \nn \\
 & \quad  + 100 B^2   H^2 d\log(3 N/\delta')+  60  H^2 B^2 |\cC_h|,\nn \\
 & \leq 6 H B N \sum_{(x_h,a_h)\in \cC_h} |\E[\xi(x_h,a_h,\z_h(x_h,a_h))]|  + \frac{1}{4}\tnorm{\fhat -f_\star}^2  + 100 B^2   H^2 d\log(3 N/\delta')+  60  H^2 B^2 |\cC_h|,\nn \\
 \intertext{and so by \eqref{eq:concrete} and \cref{lem:multiset} (and that $\cC_h$ is a multiset satisfying $\cC_h = \bigcup_{i\in{n}} \{(x_i,a_i),(x_i,\afrak)\}$)}
 &  \leq \frac{1}{4}\tnorm{\fhat -f_\star}^2 + 100 B^2   H^2 d\log(3 N/\delta')+  60  H^2 B^2 |\cC_h| \nn \\
 & \quad + 288 N H B  \beta \sum_{(x_h,a_h)\in \cC_h}\sum_{\ell=h+1}^H\E_{\pihat_{\theta}}\left[\mathbb{I}\{M \geq  4 \Ccond(\pibar_{\ell,\theta} \mid \x_\ell)^2 \} \cdot\Ccond(\pibar_{\ell,\theta}\mid \x_\ell) \cdot  \sqrt{\frac{2}{M}}\mid  \x_h =x_h, \a_h=a_h \right] \nn \\ 
  & \quad +    384|\cC_h| N  H^2 B^2 \tilde\delta + 48 |\cC_h| N  H^2 B\beta\log (4 M \log (4 \tilde\delta^{-1})) \tilde\delta \nn  \\
 & \quad + 288  N H B^2 \sum_{(x_h,a_h)\in \cC_h}\sum_{\ell=h+1}^H\P_{\pihat_\theta}\left[M < 4\Ccond(\pibar_{\ell,\theta}\mid \x_\ell)^2\mid\x_h=x_h,\a_h=a_h\right]. \label{eq:later}
\end{align}
Using this together with the expression of Term I in \eqref{eq:twoterms}, we have that with probability at least $1-\delta'$,  
\begin{align}
	&\text{Term I}\nn \\ & \leq  800 B^2   H^2 d\log(3 N/\delta')+  480  H^2 B^2 |\cC_h| \nn \\
	& \quad + 2304 N H B  \beta \sum_{(x_h,a_h)\in \cC_h}\sum_{\ell=h+1}^H\E_{\pihat_{\theta}}\left[\mathbb{I}\{M \geq  4 \Ccond(\pibar_{\ell,\theta} \mid \x_\ell)^2 \} \cdot\Ccond(\pibar_{\ell,\theta}\mid \x_\ell) \cdot  \sqrt{\frac{2}{M}}\mid  \x_h =x_h, \a_h=a_h \right] \nn \\ 
	 & \quad +    3072|\cC_h| N  H^2 B^2 \tilde\delta + 2304 |\cC_h| N  H^2 B\beta\log (4 M \log (4 \tilde\delta^{-1})) \tilde\delta \nn  \\
	& \quad + 2304  N H B^2 \sum_{(x_h,a_h)\in \cC_h}\sum_{\ell=h+1}^H\P_{\pihat_\theta}\left[M < 4\Ccond(\pibar_{\ell,\theta}\mid \x_\ell)^2\mid\x_h=x_h,\a_h=a_h\right].  \label{eq:term1}
	\end{align}
	\paragraphi{Bounding Term II} To bound the second term in \eqref{eq:twoterms}, which reflects the misspecification level in the regression problem, we need to bound $b(x_h,a_h) = Q^{\pihat_\theta}_{h,\beta}(x_h,a_h)-Q^{\pihat_\theta}_{h,\beta}(x_h,\fraka) - Q^\star_{h,\beta}(x_h,a_h)+Q^\star_{h,\beta}(x_h,\fraka)$ for $(x_h,a_h)\in \cC_h$. 
	By the performance difference lemma (\cref{rem:general}) and \cref{lem:wonky}, we have that for any $(x_h,a_h)\in \cX \times \cA$:
\begin{align}
&\left|	Q^\star_{h,\beta}(x_h,a_h)- Q^{\pihat_{\theta}}_{h,\beta}(x_h,a_h) \right| \nn \\
& \leq \left| \sum_{\ell=h+1}^H \E_{\pihat_{\theta}}\left[\sum_{a\in \cA} \pistar_{\ell,\beta}(a\mid \x_\ell)\cdot \left(Q^{\star}_{\ell,\beta}(\x_\ell,a)- \beta \cdot \log \frac{\pistar_{\ell,\beta}(a\mid \x_\ell)}{\pi_{\ell,\refe}(a\mid \x_\ell)} \right) \mid \x_h = x_h ,\a_h = a_h\right]\right.\nn \\
& \quad   \left. - \sum_{\ell=h+1}^H \E_{\pihat_{\theta}}\left[\sum_{a\in \cA} \pihat_{\ell,\theta}(a\mid \x_\ell)\cdot \left( Q_{\ell,\beta}^{\star}(\x_\ell,a) - \beta \cdot \log \frac{\pihat_{\ell,\theta}(a\mid \x_\ell)}{\pi_{\ell,\refe}(a\mid \x_\ell)} \right) \mid \x_h = x_h, \a_h =a_h\right]\right|,\nn \\
& \leq \left| \sum_{\ell=h+1}^H \E_{\pihat_{\theta}}\left[\sum_{a\in \cA} \pistar_{\ell,\beta}(a\mid \x_\ell)\cdot \left(Q^{\star}_{\ell,\beta}(\x_\ell,a)- \beta \cdot \log \frac{\pistar_{\ell,\beta}(a\mid \x_\ell)}{\pi_{\ell,\refe}(a\mid \x_\ell)} \right) \mid \x_h = x_h ,\a_h = a_h\right]\right.\nn \\
& \quad   \left. - \sum_{\ell=h+1}^H \E_{\pihat_{\theta}}\left[\sum_{a\in \cA} \pihat_{\ell,\theta}(a\mid \x_\ell)\cdot \left( Q_{\ell,\beta}^{\star}(\x_\ell,a) - \beta \cdot \log \frac{\pibar_{\ell,\theta}(a\mid \x_\ell)}{\pi_{\ell,\refe}(a\mid \x_\ell)} \right) \mid \x_h = x_h, \a_h =a_h\right]\right|\nn \\
& \quad +\sum_{\ell=h+1}^H  \beta\cdot \E_{\pihat_\theta}\left[\kl{\pihat_{\ell,\theta}(\cdot \mid \x_\ell)}{\pibar_{\ell,\theta}(\cdot \mid \x_\ell)} \mid \x_h =x_h, \a_h =a_h\right]. \label{eq:goback0}
\end{align}
Now, by \cref{lem:innerKL_first} (stated and proven in the sequel), we can bound the KL term in \eqref{eq:goback0} as follows: for all $\ell\in [H]$ and $(x_h,a_h)\in \cX \times \cA$:
\begin{align}
	&\E_{\pihat_{\theta}}\left[\kl{\pihat_{\ell,\theta}(\cdot \mid \x_\ell)}{\pibar_{\ell,\theta}(\cdot \mid \x_\ell)}\mid \x_h = x_h, \a_h =a_h\right]\nn \\ &\leq 4\prn*{\frac{\Rmax}{\beta} +\log(4 M \log(4 \tilde\delta^{-1}))} \tilde\delta  \nn \\
	& \quad + \E_{\pihat_\theta}\left[\mathbb{I}\{M <\Ccond(\pibar_{\ell,\theta}\mid \x_\ell)\} \cdot \kl{\pihat_{\ell,\theta}(\cdot \mid \x_\ell)}{\pibar_{\ell,\theta}(\cdot \mid \x_\ell)} \mid \x_h =x_h, \a_h =a_h\right],\nn \\
	& \leq 4\prn*{\frac{\Rmax}{\beta} +\log(4 M \log(4 \tilde\delta^{-1}))} \tilde\delta \nn \\
	& \quad  + \frac{\Rmax}{\beta} \log(4 M \log(4 \tilde\delta^{-1})) \cdot \P_{\pihat_\theta}\left[ M < 4\Ccond(\pibar_{\ell,\theta}\mid \x_\ell) \mid \x_h =x_h, \a_h =a_h\right].\nn
\end{align}
Now, since $M\geq 1$, we have that $M< 4 \Ccond(\pibar_{\ell,\theta}\mid \x_\ell)$ only if $M < 16 \Ccond(\pibar_{\ell,\theta}\mid \x_\ell)^2$, and so for all $\ell\in[h+1\ldots H]$ and $(x_h,a_h)\in \cX\times \cA$:
\begin{align}
	\P_{\pihat_\theta}\left[M < 4\Ccond(\pibar_{\ell,\theta}\mid \x_\ell)\mid\x_h=x_h,\a_h=a_h\right] \leq \P_{\pihat_\theta}\left[M < 16\Ccond(\pibar_{\ell,\theta}\mid \x_\ell)^2\mid\x_h=x_h,\a_h=a_h\right]. \label{eq:met} 
\end{align}
Therefore, we have 
\begin{align}
	&\E_{\pihat_{\theta}}\left[\kl{\pihat_{\ell,\theta}(\cdot \mid \x_\ell)}{\pibar_{\ell,\theta}(\cdot \mid \x_\ell)}\mid \x_h = x_h, \a_h =a_h\right]\nn \\ 
	& \leq 4\prn*{\frac{\Rmax}{\beta} +\log(4 M \log(4 \tilde\delta^{-1}))} \tilde\delta \nn \\
	& \quad  + \frac{\Rmax}{\beta} \log(4 M \log(4 \tilde\delta^{-1})) \cdot \P_{\pihat_\theta}\left[ M < 16\Ccond(\pibar_{\ell,\theta}\mid \x_\ell)^2 \mid \x_h =x_h, \a_h =a_h\right].\label{eq:klterm0} 
\end{align}  
It remains to bound the absolute value term on the right-hand side of \eqref{eq:goback0}. As a starting point, note that for all $\ell\in[h+1 \ldotst H]$, \[\left|Q_{\ell,\beta}^{\star}(\cdot,\cdot) - \beta \cdot \log \frac{\pibar_{\ell,\theta}(\cdot\mid \cdot)}{\pi_{\ell,\refe}(\cdot \mid \cdot)}\right| \leq  H B +2H\beta \sup_{(x,a)\in \cX \times \cA}\left|\log \frac{\pibar_{\ell,\theta}(a \mid x)}{\pi_{\ell,\refe}(a\mid x)}\right| \leq 5 HB,\] since $e^{-2B/\beta}\leq \frac{\pibar_{\ell,\theta'}(\cdot \mid \cdot)}{\pi_{\ell,\refe}(\cdot\mid \cdot)}\leq e^{2B/\beta}$ for all $\theta'_\ell\in \bbB(B)$. 
Thus, by \cref{lem:rejection_tv_average_new}, we have that for all $\ell \in [h+1\ldotst H]$ and $(x_h,a_h)\in \cX \times \cA$:
\begin{align}
&\left|\E_{\pihat_{\theta}}\left[\sum_{a\in \cA} \pihat_{\ell,\theta}(a\mid \x_\ell)\cdot \left( Q_{\ell,\beta}^{\star}(\x_\ell,a) - \beta \cdot \log \frac{\pibar_{\ell,\theta}(a\mid \x_\ell)}{\pi_{\ell,\refe}(a\mid \x_\ell)} \right) \mid \x_h = x_h, \a_h =a_h\right]\right.\nn \\
& \quad  \left. -  \E_{\pihat_{\theta}}\left[\sum_{a\in \cA} \pibar_{\ell,\theta}(a\mid \x_\ell)\cdot \left(Q_{\ell,\beta}^{\star}(\x_\ell,a) - \beta \cdot \log \frac{\pibar_{\ell,\theta}(a\mid \x_\ell)}{\pi_{\ell,\refe}(a\mid \x_\ell)} \right) \mid \x_h = x_h, \a_h =a_h\right] \right|\nn \\
& \leq 5 H B\tilde\delta + 5 HB\cdot  \P_{\pihat_\theta}\left[ M < 4\Ccond(\pibar_{\ell,\theta}\mid \x_\ell) \mid \x_h =x_h, \a_h =a_h\right], \nn \\
& \leq5 H B\tilde\delta + 5 HB\cdot  \P_{\pihat_\theta}\left[ M < 16\Ccond(\pibar_{\ell,\theta}\mid \x_\ell)^2 \mid \x_h =x_h, \a_h =a_h\right], \label{eq:offday0}
\end{align}
where the last inequality follows by \eqref{eq:met}.
On the other hand, by Jensen's inequality and the triangle inequality, we have for all $(x_h,a_h)\in \cX \times \cA$:
\begin{align}
& \left| \sum_{\ell=h+1}^H \E_{\pihat_{\theta}}\left[\sum_{a\in \cA} \pistar_{\ell,\beta}(a\mid \x_\ell)\cdot \left(Q^{\star}_{\ell,\beta}(\x_\ell,a)- \beta \cdot \log \frac{\pistar_{\ell,\beta}(a\mid \x_\ell)}{\pi_{\ell,\refe}(a\mid \x_\ell)} \right) \mid \x_h = x_h ,\a_h = a_h\right]\right.\nn \\
& \quad   \left. - \sum_{\ell=h+1}^H \E_{\pihat_{\theta}}\left[\sum_{a\in \cA} \pibar_{\ell,\theta}(a\mid \x_\ell)\cdot \left( Q_{\ell,\beta}^{\star}(\x_\ell,a) - \beta \cdot \log \frac{\pibar_{\ell,\theta}(a\mid \x_\ell)}{\pi_{\ell,\refe}(a\mid \x_\ell)} \right) \mid \x_h = x_h, \a_h =a_h\right]\right|\nn \\
& \leq \sum_{\ell=h+1}^H \E_{\pihat_{\theta}}\left[ \left| \sum_{a\in \cA} \pistar_{\ell,\beta}(a\mid \x_\ell)\cdot \left(Q^{\star}_{\ell,\beta}(\x_\ell,a)- \beta \cdot \log \frac{\pistar_{\ell,\beta}(a\mid \x_\ell)}{\pi_{\ell,\refe}(a\mid \x_\ell)} \right) \right. \right.\nn \\
& \quad  \qquad \left. \left. - \sum_{a\in \cA} \pibar_{\ell,\theta}(a\mid \x_\ell)\cdot \left( Q_{\ell,\beta}^{\star}(\x_\ell,a) - \beta \cdot \log \frac{\pibar_{\ell,\theta}(a\mid \x_\ell)}{\pi_{\ell,\refe}(a\mid \x_\ell)} \right) \right|\   \mid \x_h = x_h, \a_h =a_h\right], \nn \\
\shortintertext{and so by \cref{lem:truncation}, we have for $B_{\ell,\theta}$ as in the lemma statement:}
& \leq  \beta\sum_{\ell=h+1}^H \E_{\pihat_{\theta}}\left[\kl{\pibar_{\ell,\theta}(\cdot \mid \x_\ell)}{\pibar^\star_{\ell,\beta}(\cdot \mid \x_\ell)} \mid \x_h=x_h,\a_h = a_h \right]\nn \\
& \quad + 2 \Rmax \sum_{\ell=h+1}^H \E_{\pihat_{\theta}}\left[ B_{\ell,\theta}(\x_\ell)\mid \x_h=x_h,\a_h=a_h \right]. \nn 
\end{align}
Combining this with \eqref{eq:goback0}, \eqref{eq:klterm0}, and \eqref{eq:offday0}, we get that for all $(x_h,a_h)\in \cX \times \cA$: %
\begin{align}
& |b(x_h,a_h)|\nn \\
& \leq 2 H\big(4 \Rmax+ 4\beta\log(4 M \log(4 \tilde\delta^{-1})) + 5 H B\big) \cdot \tilde\delta   \nn \\
& \quad  + \left(2 \Rmax \log(4 M \log(4 \tilde\delta^{-1})) + 10 HB \right) \sum_{a\in \{a_h,\afrak\}}\sum_{\ell=h+1}^H \P_{\pihat_\theta}\left[ M < 16\Ccond(\pibar_{\ell,\theta}\mid \x_\ell)^2 \mid \x_h =x_h, \a_h =a\right] \nn \\
& \quad + 2\beta \sum_{a\in \{a_h,\afrak\}}\sum_{\ell=h+1}^H \E_{\pihat_{\theta}}\left[\kl{\pibar_{\ell,\theta}(\cdot \mid \x_\ell)}{\pibar^\star_{\ell,\beta}(\cdot \mid \x_\ell)} \mid \x_h=x_h,\a_h = a \right] \nn \\
& \quad + 4 \Rmax \sum_{a\in \{a_h,\afrak\}} \sum_{\ell=h+1}^H \E_{\pihat_{\theta}}\left[ B_{\ell,\theta}(\x_\ell)\mid \x_h=x_h,\a_h=a \right].
\end{align}
Thus, using Jensen's inequality and \cref{lem:multiset} (together with the fact that $\cC_h$ is multiset satisfying $\cC_h = \bigcup_{i\in[n]} \{(x_i,a_i),(x_i,\fraka)\}$), we have 
\begin{align}
	& \sum_{(x_h,a_h)\in \cC_h}\sum_{z_h\in \cD(x_h,a_h) }  b(x_h,a_h)^2  \nn \\
& = 20 N |\cC_h| \cdot \big(12\Rmax^2+ 48 \beta^2 \log(4 M \log(4 \tilde\delta^{-1}))^2 + 75 H^2 B^2\big) \cdot \tilde\delta^2  \nn \\
& \quad  + 6 H N \cdot \big(8 \Rmax^2 \log(4 M \log(4 \tilde\delta^{-1}))^2 + 200 H^2B^2 \big) \sum_{(x_h,a_h)\in \cC_h} \sum_{\ell=h+1}^H \P_{\pihat_\theta}\left[ M < 16\Ccond(\pibar_{\ell,\theta}\mid \x_\ell)^2 \mid \x_h =x_h, \a_h =a_h\right]^2 \nn \\
& \quad + 120 N H\beta^2\sum_{(x_h,a_h)\in \cC_h}\sum_{\ell=h+1}^H \E_{\pihat_{\theta}}\left[\kl{\pibar_{\ell,\theta}(\cdot \mid \x_\ell)}{\pibar^\star_{\ell,\beta}(\cdot \mid \x_\ell)}^2 \mid \x_h=x_h,\a_h = a_h \right] \nn \\
& \quad + 240 N H  \Rmax^2 \sum_{(x_h,a_h)\in \cC_h}\sum_{\ell=h+1}^H \E_{\pihat_{\theta}}\left[ B_{\ell,\theta}(\x_\ell)^2\mid \x_h=x_h,\a_h=a_h \right]. \label{eq:all}
\end{align}

\paragraph{Putting everything together}
	Combining \eqref{eq:all} with \eqref{eq:term1} and \eqref{eq:twoterms}, we get that with probability at least $1-\delta'$, 
\begin{align}
	&\tnorm{\fhat - f_\star}^2 \nn \\
	& \leq  C_1  + N C_2 \tilde \delta \nn \\
	& \quad  +N C_3 \sum_{(x_h,a_h)\in \cC_h}\sum_{\ell=h+1}^H \P_{\pihat_\theta}\left[ M < 16\Ccond(\pibar_{\ell,\theta}\mid \x_\ell)^2 \mid \x_h =x_h, \a_h =a_h\right] \nn \\
	& \quad + 2304 N H B \beta \sum_{(x_h,a_h)\in \cC_h}\sum_{\ell=h+1}^H \E_{\pihat_{\theta}}\left[\mathbb{I}\{M \geq  4 \Ccond(\pibar_{\ell,\theta} \mid \x_\ell)^2 \} \cdot\Ccond(\pibar_{\ell,\theta}\mid \x_\ell) \cdot  \sqrt{\frac{2}{M}}\mid  \x_h =x_h, \a_h=a_h \right] \nn \\
& \quad + 19200 N H\beta^2\sum_{(x_h,a_h)\in \cC_h}\sum_{\ell=h+1}^H \E_{\pihat_{\theta}}\left[\kl{\pibar_{\ell,\theta}(\cdot \mid \x_\ell)}{\pibar^\star_{\ell,\beta}(\cdot \mid \x_\ell)}^2 \mid \x_h=x_h,\a_h = a_h \right] \nn \\
& \quad + 7680 N H \Rmax^2  \sum_{(x_h,a_h)\in \cC_h}\sum_{\ell=h+1}^H \E_{\pihat_{\theta}}\left[ B_{\ell,\theta}(\x_\ell)^2\mid \x_h=x_h,\a_h=a_h \right],
\end{align}
where \begin{gather}
	C_1\coloneqq 6400 B^2   H^2 d\log(3 N/\delta')+  3840  H^2 B^2 |\cC_h|, \nn \\
C_2 \coloneqq  19200 |\cC_h| \cdot \big(12\Rmax^2+ 48 \beta^2 \log(4 M \log(4 \tilde\delta^{-1}))^2 + 85 H^2 B^2\big) + 3072 |\cC_h|   H^2 B\beta\log (4 M \log (4 \tilde\delta^{-1})),\nn \\
	C_3 \coloneqq 16 H \cdot \big(8 \Rmax^2 \log(4 M \log(4 \tilde\delta^{-1}))^2 + 200 H^2B^2 \big)  + 3072   H B^2.\nn 
\end{gather} 
 Combining this with the fact that 
\begin{align}
	\|\thetahat_h -\theta^\star_{h,\beta}\|^2_{\Sigma_h} & = \lambda \|\thetahat_h -\theta^\star_{h,\beta}\|^2 + \frac{1}{N} \tnorm{\fhat - f_\star}^2, \quad (\text{by definition of $\Sigma_h$})	\nn \\
	& \leq 4\lambda B^2 + \frac{1}{N} \tnorm{\fhat - f_\star}^2, \quad \text{(by \cref{assum:linear} and $\thetahat_h\in \bbB(B)$)},\nn 
\end{align}
we obtain the desired result. It remains to prove \cref{lem:xibound} and \cref{lem:innerKL_first}.

\end{proof}

\subsection{Helper Lemmas for \texttt{FitValue} Guarantee}
\label{sec:helpfit}
\begin{lemma}
	\label{lem:xibound}
	 Consider the setting of \cref{lem:fitval_pre} and the notation in its proof. Let $\theta_{h+1:H}\in \reals^{d\cdot (H-h)}$ be as \cref{lem:fitval_pre}. Fix $(x_h,a_h)\in \cC_h$, and let $\z_h(x_h,a_h)$ be the random variable in \eqref{eq:zh} in the proof of \cref{lem:fitval_pre}. Then, the function $\xi$ in \eqref{eq:xi} satisfies
	  \begin{align}
		& \left|\E[\xi(x_{h},a_{h},\z_h(x_h,a_h))]\right|\nn \\
		& \leq   \beta\sum_{\ell=h+1}^H \sum_{a\in \{a_h,\afrak\}} \E_{\pihat_{\theta}}\left[\mathbb{I}\{M \geq  4 \Ccond(\pibar_{\ell,\theta} \mid \x_\ell)^2 \} \cdot\Ccond(\pibar_{\ell,\theta}\mid \x_\ell) \cdot  \sqrt{\frac{2}{M}}\mid  \x_h=x_h, \a_h=a \right] \nn \\
		 & \quad + 16 H B \tilde\delta +8 H\beta\log (4 M \log (4 \tilde\delta^{-1})) \tilde\delta \\
		& \quad + 4  B \sum_{\ell=h+1}^H \sum_{a\in \{a_h,\afrak\}}\P_{\pihat_\theta}\left[M < 4\Ccond(\pibar_{\ell,\theta}\mid \x_\ell)^2\mid\x_h=x,\a_h=a\right]. \nn \end{align}
	\end{lemma}
\begin{proof}[Proof of \cref{lem:xibound}]
	Let $(x_h,a_h)$ be fixed, as in the lemma statement. In addition to $\theta_{h+1:H}$ as in the lemma statement, fix $\theta_{1:h}\in \reals^{d h}$. Let $(\bx_1,\ba_1,\brho_1,\br_1), \dots, (\bx_H, \ba_H,\brho_H,
	\br_h)$ be the sequence of random variables generated via the process $(\ba_h, \brho_h) = \softmaxsample_{\beta,M,\tilde\delta}(\inner{\varphibar_h(\x_h,\cdot)}{\theta_h} \midsem \x_h,\piref), \br_h \sim
	\rstar_h(\bx_h,\ba_h), \bx_{h+1} \sim P_h(\cdot\mid\bx_h,\ba_h)$,
	initialized from $\bx_1 \sim P_0(\cdot\mid\emptyset)$ (we use
	$\bx_{H+1}$ to denote a terminal state with zero
	reward). We write $\bbP_{\pihat_\theta}\brk*{\cdot}$ and $\En_{\pihat_\theta}\brk*{\cdot}$
	to denote the law and expectation under this process.
	
	With this observe that $\xi(x_h,a_h,\z_h(x_h,a_h))$ satisfies
	\begin{align}
		\E[\xi(x_h,a_h,\z_h(x_h,a_h))] &= \E_{\pihat_\theta}\left[\br_h  + \sum_{\ell=h+1}^H \left(\br_\ell - \beta \log \brho_\ell\right) \mid \x_h=x_h, \a_h=a_h\right] - Q_{h,\beta}^{\pihat_\theta}(x_h,a_h),\nn\\ 
& \quad - \E_{\pihat_\theta}\left[\br_h  + \sum_{\ell=h+1}^H \left(\br_\ell - \beta \log \brho_\ell\right) \mid \x_h=x_h, \a_h=\afrak \right] + Q^{\pihat_\theta}_{h,\beta}(x_h,\afrak).\label{eq:realization}
	\end{align}
Thus, to prove the claim, we will bound the absolute differences 
\begin{align}
	\left|\E_{\pihat_\theta}\left[\br_h  + \sum_{\ell=h+1}^H \left(\br_\ell - \beta \log \brho_\ell\right) \mid \x_h=x_h, \a_h=a\right] - Q_{h,\beta}^{\pihat_\theta}(x_h,a)\right|, \nn
\end{align}
for $a\in \{a_h,\afrak\}$, and then apply the triangle inequality.

	For all $\ell\in[h+1 \ldotst H]$ and $a\in \cA$, we can write:
	\begin{align}
	 \E_{\pihat_\theta}\left[ \log \brho_\ell    \mid \x_h=x_h, \a_h =a\right] 
	&=\E_{\pihat_\theta}\left[ \mathbb{I}\{M \geq 4 \Ccond(\pibar_{\ell,\theta} \mid \x_\ell)^2 \} \cdot  \log \brho_\ell    \mid \x_h=x_h, \a_h =a\right] \nn \\
	& \quad + \E_{\pihat_\theta}\left[ \mathbb{I}\{M < 4 \Ccond(\pibar_{\ell,\theta} \mid \x_\ell)^2 \} \cdot\log \brho_\ell   \mid \x_h=x_h, \a_h =a\right]. \label{eq:tts}
	\end{align}
	Now, by \cref{cor:rejection_density} (guarantee of \softmaxsample), for all $\ell\in [h+1\ldotst H]$ there exists $\zeta_\ell:   \cA \rightarrow \reals$ such that for all $a\in \cA$ \begin{align}
 |\zeta_\ell(a)| &\leq  \E_{\pihat_{\theta}}\left[\mathbb{I}\{M \geq  4 \Ccond(\pibar_{\ell,\theta} \mid \x_\ell)^2 \} \cdot \Ccond(\pibar_{\ell,\theta}\mid \x_\ell) \cdot   \sqrt{\frac{2}{M}}\mid  \x_h=x_h, \a_h=a \right]  \nn \\
	& \quad + \left(\frac{8 B}{\beta} + 4 \log(4 M \log(4 \tilde\delta^{-1}))\right) \cdot \tilde \delta,\label{eq:zeta} \end{align}
	 and 
	 \begin{align}
	&  \E_{\pihat_\theta}\left[ \mathbb{I}\{M \geq 4 \Ccond(\pibar_{\ell,\theta} \mid \x_\ell)^2 \} \cdot \log \brho_\ell  \mid \x_h=x_h, \a_h =a\right] \nn \\
	& = \E_{\pihat_\theta}\left[ \mathbb{I}\{M \geq 4 \Ccond(\pibar_{\ell,\theta} \mid \x_\ell)^2 \} \cdot \log \frac{\pihat_{\ell,\theta}(\a_\ell \mid \x_\ell)}{\pi_{\ell, \refe}(\a_\ell\mid \x_\ell)}    \mid \x_h=x_h, \a_h =a\right] + \zeta_\ell(a).\nn 
	 \end{align}
	 Plugging this into \eqref{eq:tts}, we get that for all $a\in  \cA$:
	\begin{align}
	& \E_{\pihat_\theta}\left[ \log \brho_\ell   \mid \x_h=x_h, \a_h =a\right]\nn \\
	& = \E_{\pihat_\theta}\left[ \mathbb{I}\{M \geq 4 \Ccond(\pibar_{\ell,\theta} \mid \x_\ell)^2 \} \cdot \log \frac{\pihat_{\ell,\theta}(\a_\ell \mid \x_\ell)}{\pi_{\ell, \refe}(\a_\ell\mid \x_\ell)}    \mid \x_h=x_h, \a_h =a\right]  +  \zeta_\ell(a)\nn \\
	& \quad + \E_{\pihat_\theta}\left[ \mathbb{I}\{M < 4 \Ccond(\pibar_{\ell,\theta} \mid \x_\ell)^2 \} \cdot \log \brho_\ell  \mid \x_h=x_h, \a_h =a\right] ,\nn \\
	& = \E_{\pihat_\theta}\left[\log \frac{\pihat_{\ell,\theta}(\a_\ell \mid \x_\ell)}{\pi_{\ell, \refe}(\a_\ell\mid \x_\ell)}     \mid \x_h=x_h, \a_h =a\right]\nn \\
	& \quad - \E_{\pihat_\theta}\left[ \mathbb{I}\{M < 4 \Ccond(\pibar_{\ell,\theta} \mid \x_\ell)^2 \} \cdot \log \frac{\pihat_{\ell,\theta}(\a_\ell \mid \x_\ell)}{\pi_{\ell, \refe}(\a_\ell\mid \x_\ell)}     \mid \x_h=x_h, \a_h =a\right]  +  \zeta_\ell(a)\nn \\
	& \quad + \E_{\pihat_\theta}\left[ \mathbb{I}\{M < 4 \Ccond(\pibar_{\ell,\theta} \mid \x_\ell)^2 \} \cdot  \log \brho_\ell   \mid \x_h=x_h, \a_h =a\right] .\nn
	\end{align}	
	Thus, rearranging and using that $\brho_\ell, \frac{\pihat_{\ell,\theta}(a \mid x)}{\pi_{\ell, \refe}(a\mid x)} \in [e^{-2B/\beta},e^{2B/\beta}]$, for all $\ell\in [h+1\ldotst H]$ and $(x,a)\in \cX\times \cA$, we get that for all $\ell\in[h+1\ldotst H]$ and $a\in  \cA$:
	\begin{align}
		 \left| \E_{\pihat_\theta}\left[ \log \brho_\ell  - \log \frac{\pihat_{\ell,\theta}(\a_\ell \mid \x_\ell)}{\pi_{\ell, \refe}(\a_\ell\mid \x_\ell)} \mid \x_h=x_h,\a_h =a\right] \right| & \leq \frac{4 B}{\beta} \P_{\pihat_\theta}\left[M < 4 \Ccond(\pibar_{\ell,\theta} \mid \x_\ell)^2\mid \x_h = x_h, \a_h = a\right] \nn \\
		& \quad + |\zeta_\ell(a)|.\nn
	\end{align}
	Using this with \eqref{eq:realization} and the triangle inequality, we get that
	\begin{align}
	|\E[\xi(x_{h},a_{h},\z_h(x_h,a_h))]|& \leq 4 B  \sum_{\ell = h+1}^H \P_{\pihat_\theta}\left[M < 4 \Ccond(\pibar_{\ell,\theta} \mid \x_\ell)^2 \mid \x_h= x_h, \a_h = a_h\right] \nn \\
	& \quad + 4 B  \sum_{\ell = h+1}^H \P_{\pihat_\theta}\left[M < 4 \Ccond(\pibar_{\ell,\theta} \mid \x_\ell)^2 \mid \x_h= x_h, \a_h = \afrak\right] \nn \\
	& \quad + H \beta \max_{\ell\in [h+1\ldotst H]}|\zeta_\ell(a_h)|+ H \beta \max_{\ell\in [h+1\ldotst H]}|\zeta_\ell(\afrak)|.\nn 
	\end{align}
	Substituting the bound on $\zeta_\ell$ in \eqref{eq:zeta} completes the proof.
	\end{proof}

	\begin{lemma}
        \label{lem:innerKL_first}
    Let $h\in[0\ldotst H]$ be given. Under the setting of \cref{lem:fitval_pre} and the notation in its proof, we have that for all $\ell\in [h+1\ldotst H]$ and $(x_h,a_h)\in \cX \times \cA$:
    \begin{align}
        &\E_{\pihat_{\theta}}\left[\kl{\pihat_{\ell,\theta}(\cdot \mid \x_\ell)}{\pibar_{\ell,\theta}(\cdot \mid \x_\ell)}\mid \x_h = x_h, \a_h =a_h\right]\nn \\ 
		& \leq 4\prn*{\frac{\Rmax}{\beta} +\log(4 M \log(4 \tilde\delta^{-1}))} \tilde\delta \nn \\
		& \quad  + \frac{\Rmax}{\beta} \log(4 M \log(4 \tilde\delta^{-1})) \cdot \P_{\pihat_\theta}\left[ M < 4\Ccond(\pibar_{\ell,\theta}\mid \x_\ell) \mid \x_h =x_h, \a_h =a_h\right].\nn 
    \end{align}
    \end{lemma}

	\begin{proof}[Proof of \cref{lem:innerKL_first}]
		We have for all $\ell\in[h+1\ldotst H]$ and $(x_h,a_h)\in \cX \times \cA$:
			\begin{align}
			& \E_{\pihat_\theta}\left[\kl{\pihat_{\ell,\theta}(\cdot \mid \x_\ell)}{\pibar_{\ell,\theta}(\cdot \mid \x_\ell)} \mid \x_h =x_h, \a_h =a_h\right] \nn \\
			& \leq \E_{\pihat_\theta}\left[\mathbb{I}\{M \geq 4\Ccond(\pibar_{\ell,\theta}\mid \x_\ell)\} \cdot \kl{\pihat_{\ell,\theta}(\cdot \mid \x_\ell)}{\pibar_{\ell,\theta}(\cdot \mid \x_\ell)} \mid \x_h =x_h, \a_h =a_h\right]\nn \\
			& \quad + \E_{\pihat_\theta}\left[\mathbb{I}\{M < 4\Ccond(\pibar_{\ell,\theta}\mid \x_\ell)\} \cdot \kl{\pihat_{\ell,\theta}(\cdot \mid \x_\ell)}{\pibar_{\ell,\theta}(\cdot \mid \x_\ell)} \mid \x_h =x_h, \a_h =a_h\right]. \label{eq:early0} 
			\end{align}
			Now, by \cref{lem:rejection_kl}, we have that for all $x\in \cX$ and $\ell\in [h+1 \ldotst H]$, $\frac{\pihat_{\ell,\theta}(\cdot\mid x)}{\pibar_{\ell,\theta}(\cdot\mid x)} \leq 4 M e^{\Rmax/\beta} \log (4\tilde\delta^{-1})$. Combining this with \eqref{eq:early0} and using \cref{lem:rejection_kl}, we get that for all $\ell \in[h+1 \ldotst H]$ and $(x_h,a_h)\in \cX \times \cA$:\loose
			\begin{align}
			& \E_{\pihat_\theta}\left[\kl{\pihat_{\ell,\theta}(\cdot \mid \x_\ell)}{\pibar_{\ell,\theta}(\cdot \mid \x_\ell)} \mid \x_h =x_h, \a_h =a_h\right] \nn \\
			& \leq 4\prn*{\frac{\Rmax}{\beta} +\log(4 M \log(4 \tilde\delta^{-1}))} \tilde\delta  \nn \\
			& \quad + \E_{\pihat_\theta}\left[\mathbb{I}\{M <4\Ccond(\pibar_{\ell,\theta}\mid \x_\ell)\} \cdot \kl{\pihat_{\ell,\theta}(\cdot \mid \x_\ell)}{\pibar_{\ell,\theta}(\cdot \mid \x_\ell)} \mid \x_h =x_h, \a_h =a_h\right],\nn \\
			& \leq 4\prn*{\frac{\Rmax}{\beta} +\log(4 M \log(4 \tilde\delta^{-1}))} \tilde\delta \nn \\
			& \quad  + \frac{\Rmax}{\beta} \log(4 M \log(4 \tilde\delta^{-1})) \cdot \P_{\pihat_\theta}\left[ M < 4 \Ccond(\pibar_{\ell,\theta}\mid \x_\ell) \mid \x_h =x_h, \a_h =a_h\right]. \label{eq:art}
			\end{align}
This completes the proof.
			\end{proof}

\subsection{Guarantee of \texttt{FitValue} for \mtalg}	
\label{sec:guaranteefit}
\begin{lemma}
	\label{lem:fitval}
	Let $\beta, \delta, \veps\in(0,1)$ and $\piref$ be given and suppose that \cref{assum:linear} holds with $B>0$. Consider a call to $\mtalg(\beta, \delta, \veps, \piref)$ (\cref{alg:ops-dp}) and let $(\lambda,\nu)$ and $\Trounds$ be as in \mtalg{}. Then, there is an event $\cE^\reg$ of probability at least $1-\delta/4$ under which for all $t\in [\Trounds]$ and $h\in[H]$, the variables in \cref{alg:ops-dp} satisfy:
\begin{align}
	& \|\theta\ind{t}_h -\theta^\star_{h,\beta}\|^2_{\Sigma_h\ind{t}}\nn \\
& \leq 4 \lambda B^2 + \frac{C_1}{\Nreg}  +  C_2 \deltarej  + C_3 \sum_{(x_h,a_h)\in \cC_h\ind{t}} \sum_{\ell=h+1}^H \P_{\pihat\ind{t}}\left[ \Mrej < 16\Ccond(\pibar\ind{t}_{\ell}\mid \x_\ell)^2 \mid \x_h =x_h, \a_h =a_h\right] \nn \\
& \quad + 2304  H B \beta \sum_{(x_h,a_h)\in \cC_h\ind{t}}\sum_{\ell=h+1}^H \E_{\pihat\ind{t}}\left[\min\left(1,\Ccond(\pibar\ind{t}_{\ell}\mid \x_\ell) \cdot  \sqrt{\frac{2}{\Mrej}} \right)\mid  \x_h =x_h, \a_h=a_h \right] \nn \\
& \quad + 19200  H\beta^2\sum_{(x_h,a_h)\in \cC_h\ind{t}}\sum_{\ell=h+1}^H \E_{\pihat\ind{t}}\left[\kl{\pibar\ind{t}_{\ell}(\cdot \mid \x_\ell)}{\pibar^\star_{\ell,\beta}(\cdot \mid \x_\ell)}^2 \mid \x_h=x_h,\a_h = a_h \right] \nn \\
& \quad + 7680  H \Rmax^2 \sum_{(x_h,a_h)\in \cC_h\int{t}}\sum_{\ell=h+1}^H \E_{\pihat\ind{t}}\left[ B\ind{t}_{\ell}(\x_\ell)^2\mid \x_h=x_h,\a_h=a_h \right], \label{eq:fitvalue_bound}
\end{align}
where $\pihat\ind{t}$ is as in \cref{alg:ops-dp};
\begin{align}
\pibar\ind{t}_{h}(\cdot\mid x) &\propto \pi_{h,\refe}(\cdot\mid x) \cdot e^{\varphibar\ind{t}_h(x,\cdot)^\top \theta_h\ind{t}/\beta}; \label{eq:truncatedpolicies1} \\ \pibar^{t,\star}_{h,\beta}(\cdot\mid x) &\propto \pi_{h,\refe}(\cdot\mid x) \cdot e^{\varphibar\ind{t}_h(x,\cdot)^\top \theta_{h,\beta}^\star/\beta}; \label{eq:truncatedpolicies2}\\
 \varphibar\ind{t}_{h}(x,\cdot) &\coloneqq \varphi_h(x,\cdot) \cdot \mathbb{I}\left\{ \|\varphi_h(x,\cdot)\|^2_{(\Sigma\ind{t}_h)^{-1}}  \leq \nu^2\right\};\nn \\
 \varphi_h(\cdot, \cdot)&\coloneqq \phi_h(\cdot,\cdot)- \phi_h(\cdot, \fraka),
\end{align} 
with $\fraka$ as in \cref{alg:ops-dp}; and
\arxiv{
    \begin{gather}
		C_1\coloneqq 6400 B^2   H^2 d\log(3 N/\delta')+  3840  H^2 B^2 \Trounds, \nn \\
		C_2 \coloneqq  19200 \Trounds  \big(12\Rmax^2+ 48 \beta^2 \log(4 \Mrej\log(4 \deltarej^{-1}))^2 + 85 H^2 B^2\big) + 3072 \Trounds   H^2 B\beta\log (4 \Mrej\log (4 \deltarej^{-1})),\nn \\
			C_3 \coloneqq 16 H \cdot \big(8 \Rmax^2 \log(4 \Mrej\log(4 \deltarej^{-1}))^2 + 200 H^2B^2 \big)  + 3072   H B^2, \nn\\ 
			\shortintertext{and for $x\in \cX$:}
    B\ind{t}_{\ell}(x)\coloneqq  \min \left( 
        1, \max_{\pi \in \{\pibar\ind{t}_{\ell},\pibar^{t,\star}_{\ell,\beta}, \pi^\star_{\ell,\beta}\}} \Ccond(\pi\mid x) \cdot \P_{\a \sim \pi_{\ell,\refe}(\cdot \mid x)}\left[\|\varphi_\ell(x,\a)\|^2_{(\Sigma_\ell\ind{t})^{-1}}>\nu^2\right]\right).\nn 
        \end{gather}
}
\end{lemma}
\begin{proof}[\pfref{lem:fitval}]
  Note that from \cref{line:updatePsi} of \cref{alg:ops-dp}, the set $\cC_h\ind{t}$ is a multiset of the form $\cC = \bigcup_{i\in{n}} \{(x_i,a_i),(x_i,\afrak)\}$, and thus satisfies the precondition of \cref{lem:fitval_pre}. The result of the lemma thus follows from \cref{lem:fitval_pre} with \[(\theta_{h+1:H}, \cC_h,\Sigma_{h+1:H},\tilde\delta,M,\delta',N)=(\theta_{1:H}\ind{t}, \cC_h\ind{t},\Sigma\ind{t}_{h+1:H},\deltarej,\Mrej,\delta/(4H\Trounds),\Nreg),\] and \cref{lem:unionbound} (essentially the union bound over $t\in[\Trounds]$ and $h\in[H]$). 
\end{proof}

\newpage

\section{Proof of \creftitle{thm:main}}
\label{sec:proof_main}

In this section, we provide the proof of the main guarantee for \mtalg{}. Before presenting the proof, we refine the guarantees of \uncertsa and \fitval{} by incorporating the parameter choices from \mtalg and combining the guarantees across layers $h \in [H]$. The final guarantees for \uncertsa and \fitval{} are stated in \cref{lem:test} and \cref{lem:main}, respectively, after which we proceed to prove \cref{thm:main}.

\begin{lemma}
    \label{lem:test}
    Let $\beta,\veps,\delta\in(0,1)$, $B>0$, and $\piref$ be given and consider a call to $\mtalg(\beta, \delta, \veps, B,\piref)$ (\cref{alg:ops-dp}). Let $(\nu,\Trounds)$ and $(\vepslip, \uveps, \cE^\spann)$ be as in \cref{alg:ops-dp} and \cref{lem:uncertaintraj}, respectively. Finally, let $\cJ^\spann \coloneqq \{t \in [\Trounds]:  \|\varphi_h(x\ind{t}_h,a\ind{t}_h)\|^2_{(\Sigma\ind{t}_{h})^{-1}} \leq \nu^2/4, \forall h\in[H]\}$, where $(x\ind{t}_h,a\ind{t}_h, \Sigma\ind{t}_h, \varphi_h)$ are as in \cref{alg:ops-dp} when the algorithm terminates. Then we have $\cJ^\spann \not= \emptyset$, and under the event $\cE^\spann$, we have that for all $\tauind \in \cJ^\spann$, the variables in \cref{alg:ops-dp} satisfy:
    \begin{itemize} 
    \item For all $h\in[H]$, $\ell\in[0\ldotst h-1]$, and $(x_\ell,a_\ell)\in \cC\ind{\tauind}_{\ell}$,
    \begin{align}
        \P_{\pihat\ind{ \tauind}_{\ell+1:h}} \left[\|\varphi_h(\x_h,\a_h)\|^2_{(\Sigma\ind{ \tauind}_{h})^{-1}}  > \nu^2 \mid \x_\ell=x_\ell, \a_\ell=a_\ell \right] \leq \vepslip; \label{eq:firstineq}
    \end{align}
	\item There exists $\cX\ind{\tauind}_{h,\spann}\subseteq \cX$ such that for all $\ell\in[0\ldotst h-1]$ and $(x_\ell,a_\ell)\in \cC\ind{\tauind}_\ell$, $\P_{\pihat\ind{ \tauind}}[\x_h\in \cX\ind{\tauind}_{h,\spann} \mid \x_\ell=x_\ell, \a_\ell=a_\ell]\geq 1 - \veps_\spann$, and for all $x_h\in \cX\ind{\tauind}_{h,\spann}$:
    \begin{align}
		\P_{\a\sim \pi_{h,\refe}(\cdot \mid x_h)} \left[\|\varphi_h(x_h,\a)\|^2_{(\Sigma\ind{ \tauind}_{h})^{-1}}  > \nu^2 \right] \leq \uveps.\label{eq:secondineq}
    \end{align}
\end{itemize}
    \end{lemma}
\begin{proof}[\pfref{lem:test}]
	We start by proving that $\cJ^\spann \not=\emptyset$.
Let $(x\ind{t}_h, a\ind{t}_h)$ be as in \cref{alg:ops-dp} and define 
\[
\forall s\in[\Trounds],\quad u\ind{s}_h\coloneqq \varphi_h(x\ind{s}_h,a\ind{s}_h) \quad \text{and}\quad  \forall t\in[\Trounds], \forall h \in[H],\quad  U\ind{t}_h \coloneqq \sum_{s=1}^{t-1} u\ind{s}_h (u\ind{s}_h)^\top.
\]
Note that $\Sigma_h\ind{t} = \lambda I + U\ind{t}_h$, for all $h\in[H]$ and $t\in[\Trounds]$.  By \cref{lem:elliptic_potential}, we have that: 
   \begin{align}
	\sum_{t\in[\Trounds]}  \sum_{h\in[H]}  1 \wedge \|u\ind{t}_h\|_{(\lambda I + U\ind{t}_h)^{-1}}  \leq  H \sqrt{2\Trounds d \log (1+ \Trounds/\lambda)}.  \label{eq:bodd}
   \end{align}
  Therefore, there exists an $\tauind \in [\Trounds]$ such that for all $h\in[H]$:
   \begin{align}
    1 \wedge \|u\ind{ \tauind}_h\|_{(\lambda I + U\ind{ \tauind}_h)^{-1}} &  \leq  \frac{1}{\Trounds}\sum_{t\in[\Trounds]}  \sum_{h'\in[H]} \|u\ind{t}_{h'}\|_{(\lambda I + U\ind{t}_{h'})^{-1}},\nn \\
     &\leq  \frac{H \sqrt{2\Trounds d \log (1+ \Trounds/\lambda)}}{\Trounds}, \nn \\
    &  \leq  \frac{\nu}{2}, \label{eq:lookpre}
   \end{align}
   where the last inequality follows from the fact that $\Trounds\geq 8\nu^{-4} d H^2 \log (1+\Trounds/\lambda)$ (\sssref{app:params}). Since $\nu<1$ (\sssref{app:params}), \eqref{eq:lookpre} implies that: 
   \begin{align}
   \forall h\in[H],\quad \|u\ind{ \tauind}_h\|_{(\lambda I + U\ind{ \tauind}_h)^{-1}} \leq \frac{\nu}{2}. \label{eq:look}
   \end{align}
   Using the definitions of $u\ind{j}_h$ and $U\ind{j}_h$, this shows that $ \|\varphi_h(x_h\ind{\tauind}, a_h\ind{ \tauind})\|^2_{(\Sigma\ind{ \tauind}_h)^{-1}}\leq \nu^2/4$, for all $h\in[H]$, which implies that $\cJ^\spann\not=\emptyset$.

   We now prove \eqref{eq:firstineq} and \eqref{eq:secondineq} under $\cE^\spann$. Fix $j\in \cJ^\spann$ and condition on $\cE^\spann$.
    Using \cref{lem:uncertaintraj} (and the conditioning on $\cE^{\spann}$) and the definition of $\cJ^\spann$, we have that for all $h\in[H]$, $\ell \in [0\ldotst h-1]$, and $(x_\ell,a_\ell)\in \cC\ind{ \tauind}_\ell$:
\begin{align}
&\P_{\pihat\ind{ \tauind}_{\ell+1:h}} \left[\|\varphi_h(\x_h,\a_h)\|^2_{(\Sigma\ind{ \tauind}_{h})^{-1}}  > \nu^2 \mid \x_\ell=x_\ell, \a_\ell=a_\ell \right]\nn \\
& =  \P_{\pihat\ind{ \tauind}_{\ell+1:h}} \left[\|\varphi_h(\x_h,\a_h)\|^2_{(\Sigma\ind{ \tauind}_{h})^{-1}}  > \nu^2\vee \left(2 \|u\ind{ \tauind}_h\|^2_{(\lambda I + U\ind{ \tauind}_h)^{-1}} \right)  \mid \x_\ell=x_\ell, \a_\ell=a_\ell \right],  \nn \\
& = \P_{\pihat\ind{ \tauind}_{\ell+1:h}} \left[\|\varphi_h(\x_h,\a_h)\|^2_{(\Sigma\ind{ \tauind}_{h})^{-1}}  > \nu^2\vee \left(2 \|\varphi_h(x_h\ind{ \tauind}, a_h\ind{ \tauind})\|^2_{(\Sigma\ind{\tauind}_h)^{-1}} \right)  \mid \x_\ell=x_\ell, \a_\ell=a_\ell \right], \nn \\
& \leq \vepslip.\nn 
\end{align}
Similarity, using \cref{lem:uncertaintraj} and the definition of $\cJ^\spann$ once more, we have that there exists $\cX\ind{\tauind}_{h,\spann}\subseteq \cX$ such that for all $\ell\in[0\ldotst h-1]$ and $(x_\ell,a_\ell)\in \cC\ind{\tauind}_\ell$, $\P_{\pibar\ind{\tauind}}[\x_h\in \cX\ind{\tauind}_{h,\spann} \mid \x_\ell=x_\ell, \a_\ell=a_\ell]\geq 1 - \uveps$, and for all $x_h\in \cX\ind{\tauind}_{h,\spann}$:
\begin{align}
	& \P_{\a\sim \pi_{h,\refe}(\cdot \mid x_h)} \left[\|\varphi_h(x_h,\a)\|^2_{(\Sigma\ind{ \tauind}_{h})^{-1}}  > \nu^2 \right]\nn \\
	& \leq \P_{\a\sim \pi_{h,\refe}(\cdot \mid x_h)} \left[\|\varphi_h(x_h,\a)\|^2_{(\Sigma\ind{ \tauind}_{h})^{-1}}  > \nu^2 \vee \left(2 \|\varphi_h(x_h\ind{ \tauind}, a_h\ind{ \tauind})\|^2_{(\Sigma\ind{ \tauind}_h)^{-1}} \right) \right],\nn \\ 
	&\leq \uveps.\nn 
\end{align}
This completes the proof. 
\end{proof}

\begin{lemma}[Estimation error]
    \label{lem:main}
    Let $\beta, \veps,\delta\in(0,1)$, $B>0$, and $\piref$ be such that $\veps\leq \beta^2/4$ and suppose that \cref{assum:linear} holds. Consider a call to $\mtalg(\beta, \delta, \veps,\piref)$ (\cref{alg:ops-dp}), and let $\cE^\reg$, $\cE^\spann$, and $\cJ^\spann$ be as in \cref{lem:fitval}, \cref{lem:uncertaintraj}, and \cref{lem:test}, respectively. Then, under the event $\cE^\reg \cap \cE^\spann$, we have that for all $\tauind \in \cJ^\spann$ and $h\in[H]$, the parameter vector $\theta\ind{ \tauind}_h$ in \cref{alg:ops-dp} satisfies
\begin{align}
	\|\theta\ind{ \tauind}_h-\theta^\star_{h,\beta} \|_{\Sigma_h\ind{ \tauind}}^2 \leq \veps^2_\reg,\nn 
\end{align}
where $\Sigma_h\ind{\tauind}$ is as in \cref{alg:ops-dp} at iteration $j$ and $\veps^2_\reg \coloneqq \veps$.
\end{lemma}

As we will see shortly in the proof of \cref{lem:main}, the reason the result holds and why we do not encounter compounding errors from future layers $\ell > h$ is due to the use of KL-regularization. The regularization ensures that errors from subsequent layers are raised to the fourth power, resulting in favorable error propagation across layers.

\begin{proof}[\pfref{lem:main}] Let $(\vepslip, \uveps)$ be as in \cref{lem:uncertaintraj}. In this proof, we condition on $\cE^\spann\cap \cE^\reg$. and fix $\tauind\in \cJ^\spann$. We proceed via backward induction over $\ell=H+1,\dots,1$ to show that
	\begin{align}
		\|\theta\ind{ \tauind}_\ell - \theta^\star_{\ell,\beta}\|_{\Sigma_\ell\ind{ \tauind}}^2 \leq \veps^2_\reg. \label{eq:target}
	\end{align} 
	The base case holds trivially by the convention that $\theta\ind{\tauind}_{H+1}=\theta^\star_{H+1,\beta}= 0$. Let $h\in[H]$ and suppose that \eqref{eq:target} holds for $\ell=h+1$. We show that it holds for $\ell=h$. 

	By \cref{lem:fitval} and the conditioning on $\cE^\reg$, we have that (with $\pibar\ind{\tauind}_\ell$, $\pibar^{\tauind,\star}_{\ell,\beta}, C_1,C_2,C_3$, and $B_\ell\ind{ \tauind}$ as in \cref{lem:fitval})
	\begin{align}
	& \|\theta\ind{ \tauind}_h -\theta^\star_{h,\beta}\|^2_{\Sigma_h\ind{ \tauind}}\nn \\
	& \leq 4 \lambda B^2 + \frac{C_1}{\Nreg}  +  C_2 \deltarej  + C_3 \sum_{(x_h,a_h)\in \cC_h\ind{\tauind}}\sum_{\ell=h+1}^H \P_{\pihat\ind{\tauind}}\left[ \Mrej < 16\Ccond(\pibar\ind{\tauind}_{\ell}\mid \x_\ell)^2 \mid \x_h =x_h, \a_h =a_h\right] \nn \\
& \quad + 2304  H B \beta \sum_{(x_h,a_h)\in \cC\ind{\tauind}_h}\sum_{\ell=h+1}^H \E_{\pihat\ind{\tauind}}\left[  \min \left(1,\Ccond(\pibar\ind{\tauind}_{\ell}\mid \x_\ell) \cdot  \sqrt{\frac{2}{\Mrej}} \right) \mid  \x_h =x_h, \a_h=a_h \right] \nn \\
& \quad + 19200  H\beta^2\sum_{(x_h,a_h)\in \cC\ind{\tauind}_h}\sum_{\ell=h+1}^H \E_{\pihat\ind{\tauind}}\left[\kl{\pibar\ind{\tauind}_{\ell}(\cdot \mid \x_\ell)}{\pibar^\star_{\ell,\beta}(\cdot \mid \x_\ell)}^2 \mid \x_h=x_h,\a_h = a_h \right] \nn \\
& \quad + 7680  H \Rmax^2\sum_{(x_h,a_h)\in \cC\ind{\tauind}_h}\sum_{\ell=h+1}^H \E_{\pihat\ind{\tauind}}\left[ B\ind{\tauind}_{\ell}(\x_\ell)^2\mid \x_h=x_h,\a_h=a_h \right]. \label{eq:rhs}
	\end{align}
        We start by bounding the KL term on the right-hand side of \eqref{eq:rhs}, then bound the terms involving 
        $(B\ind{\tauind}_\ell(\x_\ell))$ and $\Ccond$, which correspond to distribution shift. 
\begin{align}
    B\ind{\tauind}_{\ell}(x)&\coloneqq  
        \max_{\pi \in \{\pibar\ind{\tauind}_{\ell},\pibar^{\tauind,\star}_{\ell,\beta}, \pi^\star_{\ell,\beta}\}} \min \left(1,\Ccond(\pi\mid x) \cdot \P_{\a \sim \pi_{\ell,\refe}(\cdot \mid x)}\left[\|\varphi_\ell(x,\a)\|^2_{(\Sigma_\ell\ind{\tauind})^{-1}}>\nu^2\right]\right),\nn
\end{align} 
for all $x\in \cX$.
We call $(B\ind{j}_\ell)$ \emph{distribution shift} terms because they reflect the event in which the algorithm is surprised by a new direction in feature space when the rollout policy changes.
\paragraphi{Bounding the KL term} 
By the induction hypothesis, \cref{lem:KLbound} applied to each $\ell>h$ (the precondition $\|\thetastar_{\ell,\beta}- \theta\ind{\tauind}_\ell\|\leq \beta/\nu$ of \cref{lem:KLbound} is satisfied thanks to the induction hypothesis and $\veps_\reg=\veps^{1/2}\leq \beta/2$), and Jensen's inequality, we have that for all $(x_h,a_h)\in \cC\ind{ \tauind}_h$ and $\ell \in [h+1 \ldotst H]$: 
\begin{align}
	& \E_{\pihat\ind{ \tauind}}\left[\kl{\pibar\ind{\tauind}_{\ell}(\cdot \mid \x_\ell)}{\pibar^{\tauind,\star}_{\ell,\beta}(\cdot \mid \x_\ell)}^2 \mid \x_h=x_h, \a_h=a_h\right] \nn \\
	& \leq \E_{\pihat\ind{\tauind}}\left[ \left(\beta^{-1}\sum_{a\in \cA} \pibar\ind{\tauind}_{\ell}(a \mid \x_\ell)\cdot  \|\varphibar\ind{\tauind}_\ell(\x_\ell,a)\|^2_{(\Sigma_\ell\ind{ \tauind})^{-1}} \right)^2  \mid \x_h=x_h, \a_h=a_h\right]   \|\theta^\star_{\ell,\beta} - \theta\ind{\tauind}_\ell\|^4_{\Sigma\ind{\tauind}_\ell}, \label{eq:toplugging}
\end{align}
where $\varphibar\ind{\tauind}$ is as in \cref{lem:fitval}. Now, by \cref{lem:rejection_tv_average_new}, we have that for any $x\in \cX$ and $\ell \in [h+1 \ldotst H]$: 
\begin{align}
& \left|\sum_{a\in \cA} \pibar\ind{\tauind}_{\ell}(a \mid x)\cdot  \|\varphibar\ind{\tauind}_\ell(x,a)\|^2_{(\Sigma_\ell\ind{ \tauind})^{-1}}  - \sum_{a\in \cA} \pihat\ind{\tauind}_{\ell}(a \mid x)\cdot  \|\varphibar\ind{\tauind}_\ell(x,a)\|^2_{(\Sigma_\ell\ind{ \tauind})^{-1}} \right|\nn \\
& \leq  \frac{4}{\lambda}\deltarej + \frac{4}{\lambda}  \cdot \mathbb{I}\left\{ \Mrej < 4\Ccond(\pibar\ind{\tauind}_{\ell}\mid x)\right\}.
\end{align}
 Plugging this into \eqref{eq:toplugging} and using Jensen's inequality, we get that for all $(x_h,a_h)\in \cC\ind{ \tauind}_h$ and $\ell \in [h+1 \ldotst H]$:
    \begin{align} 
& \E_{\pihat\ind{ \tauind}}\left[\kl{\pibar\ind{\tauind}_{\ell}(\cdot \mid \x_\ell)}{\pibar^{\tauind,\star}_{\ell,\beta}(\cdot \mid \x_\ell)}^2 \mid \x_h=x_h, \a_h=a_h\right] \nn \\
	& \leq  \frac{4}{\beta^2}\cdot  \E_{\pihat\ind{\tauind}}\left[\|\varphibar\ind{\tauind}_\ell(\x_\ell,\a_\ell)\|^4_{(\Sigma_\ell\ind{ \tauind})^{-1}} \mid \x_h=x_h, \a_h=a_h\right]    \cdot \|\theta^\star_{\ell,\beta} - \theta\ind{\tauind}_\ell\|^4_{\Sigma\ind{\tauind}_\ell}\nn \\
    & \quad + \left(\frac{32 \deltarej^2}{\lambda^2\beta^2} +  \frac{32}{\beta^2 \lambda^2}\P_{\pihat\ind{\tauind}}\left[\Mrej < 4\Ccond(\pibar\ind{\tauind}_{\ell}\mid \x_\ell)\mid \x_h =x_h, \a_h =a_h\right] \right)  \cdot \|\theta^\star_{\ell,\beta} - \theta\ind{\tauind}_\ell\|^4_{\Sigma\ind{\tauind}_\ell},\nn \\
	& \leq \frac{\nu^4}{\beta^2} + \frac{16}{\lambda^{2}\beta^2} \cdot \P_{\pihat\ind{\tauind}}\left[\| \varphi_\ell(\x_\ell,\a_\ell)\|^2_{(\Sigma\ind{\tauind}_\ell)^{-1}}>\nu^2  \mid \x_h=x_h, \a_h=a_h\right] \cdot  \|\theta\ind{\tauind}_\ell - \theta^\star_{\ell,\beta}\|^4_{\Sigma\ind{ \tauind}_\ell} \nn \\
    & \quad +  \left(\frac{32 \deltarej^2}{\lambda^2\beta^2} +  \frac{32}{\beta^2 \lambda^2}\P_{\pihat\ind{\tauind}}\left[\Mrej < 4\Ccond(\pibar\ind{\tauind}_{\ell}\mid \x_\ell)\mid \x_h =x_h, \a_h =a_h\right] \right)  \cdot \|\theta^\star_{\ell,\beta} - \theta\ind{\tauind}_\ell\|^4_{\Sigma\ind{\tauind}_\ell},
\end{align}
where the last inequality follows by the fact that  $\sigma_{\min}(\Sigma\ind{ \tauind}_\ell) \geq \lambda$ and $\|\varphi_\ell(x,a)\|\leq 2$, for all $\ell\in[H],(x,a)\in \cX\times \cA$ (follows by \cref{assum:linear}). And so, by \cref{lem:test} (in particular \eqref{eq:firstineq}) and the induction hypothesis again (to bound $\|\theta^\star_{\ell,\beta} - \theta\ind{\tauind}_\ell\|^4_{\Sigma\ind{\tauind}_\ell}$), we have that for all $(x_h,a_h)\in \cC\ind{ \tauind}_h$ and $\ell \in [h+1 \ldotst H]$: 
\begin{align}
	  &\E_{\pibar\ind{ \tauind}}\left[\kl{\pibar\ind{ \tauind}_{\ell}(\cdot \mid \x_\ell)}{\pibar^{\tauind,\star}_{\ell,\beta}(\cdot \mid \x_\ell)}^2 \mid \x_h=x_h, \a_h=a_h\right] \nn \\ 
& \leq  \left(\lambda^2 \nu^4 + 16 \vepslip +32 \deltarej^2 +  32\P_{\pihat\ind{\tauind}}\left[\Mrej < 4\Ccond(\pibar\ind{\tauind}_{\ell}\mid \x_\ell)\mid \x_h =x_h, \a_h =a_h\right]\right) \cdot \frac{\veps_\reg^4}{\beta^2 \lambda^2}. \label{eq:klbound}
\end{align}
We now bound the ``distribution shift'' terms $(B\ind{\tauind}_\ell)$ appearing in \eqref{eq:klbound} and on the right-hand side of \eqref{eq:rhs}.  
\paragraphi{Bounding the distribution shift terms} Let $\cX\ind{\tauindl}_{\ell,\spann}$ and $(\vepslip, \uveps)$ be as in \cref{lem:test} and \cref{lem:uncertaintraj}, respectively. By \cref{lem:test}, we have that for all $\ell\in[h+1 \ldotst H]$ and $(x_h,a_h) \in \cC_h\ind{ \tauind}$, 
\begin{gather}
   \P_{\pihat\ind{ \tauind}}\left[\x_\ell \in \cX\ind{\tauind}_{\ell,\spann} \mid \x_h=x_h, \a_{h}=a_{h}\right] \geq 1- \veps_\spann, \label{eq:spann}
   \intertext{and for all $x\in \cX\ind{\tauind}_{\ell,\spann}$:}
   \P_{\a\sim \pi_{\ell,\refe}(\cdot \mid x)} \left[\|\varphi_\ell(x,\a)\|^2_{(\Sigma\ind{ \tauind}_{\ell})^{-1}} > \nu^2\right] \leq \uveps \leq  \frac{1}{4 \Ccond (\pistar_{\ell,\beta})}, 
   \label{eq:ellip1}
\end{gather} 
where the last inequality follows by the fact that $\Nspanb \geq 4 \Ccond(\pistar_\beta)$ (see parameter choices in \cref{alg:ops-dp}). On the other hand, by \cref{lem:helperest} (stated and proven in the sequel), we have that for all $\ell\in [h+1 \ldotst H]$ and $(x,a)\in \cX\ind{\tauind}_{\ell,\spann}\times \cA$: 
\begin{align}
    \frac{\pibar\ind{ \tauind}_\ell(a\mid x)}{\pi_{\ell,\refe}(a\mid x)} \vee \frac{\pibar^{\tauind,\star}_{\ell,\beta}(a\mid x)}{\pi_{\ell,\refe}(a\mid x)}& \leq    2 e^{\frac{4\nu }{\beta}\|\theta\ind{ \tauind}_\ell - \theta^\star_{\ell,\beta}\|_{\Sigma\ind{ \tauind}_\ell}}  \cdot  \frac{\pistar_{\ell,\beta}(a\mid x)}{\pi_{\ell,\refe}(a\mid x)}, \nn \\
& \leq 2 e  \frac{\pistar_{\ell,\beta}(a\mid x)}{\pi_{\ell,\refe}(a\mid x)},  \label{eq:soak} 
 \end{align}
 where the last step follows by the induction hypothesis and the fact that $\veps_\reg \leq \beta/(4\nu)$.
 Therefore, we have that for all $\ell\in[h+1\ldotst H]$, $x\in \cX\ind{\tauind}_{\ell,\spann}$, and $\pi \in\{\pibar\ind{ \tauind}_{\ell}, \pibar^{\tauind,\star}_{\ell,\beta}\}$:
 \begin{align} \Ccond(\pi\mid x) \leq 2 e \Ccond (\pistar_{\ell,\beta}). \label{eq:ccov}
 \end{align}
Thus, combining \eqref{eq:ccov} and \eqref{eq:ellip1}, we get that for all $x\in \cX\ind{\tauind}_{\ell,\spann}$:
\begin{align}
 B\ind{ \tauind}_\ell(x)  &= \max_{\pi \in \{\pibar\ind{\tauind}_{\ell},\pibar^{\tauind,\star}_{\ell,\beta}, \pi^\star_{\ell,\beta}\}}\min \left(1,  \Ccond(\pi\mid x) \cdot \E_{\a \sim \pi_{\ell,\refe}(\cdot \mid x)}\left[\|\varphi_\ell(x,\a)\|^2_{(\Sigma_\ell\ind{\tauind})^{-1}}>\nu\right] \right),\nn \\
& \leq \min \left(1, 2 e \Ccond(\pistar_{\ell,\beta}\mid x) \cdot \E_{\a \sim \pi_{\ell,\refe}(\cdot \mid x)}\left[\|\varphi_\ell(x,\a)\|^2_{(\Sigma_\ell\ind{\tauind})^{-1}}>\nu\right] \right),\nn \\
& \leq  2 e \Ccond(\pistar_{\ell,\beta}) \cdot \uveps,\nn 
\end{align}
where the last step follows by \eqref{eq:ellip1} and that $\uveps \leq 1$. Thus, using \eqref{eq:spann}, we have that for all $\ell\in[h+1\ldotst H]$ and $(x_h,a_h)\in \cC\ind{ \tauind}_h$:
\begin{align}
& \E_{\pihat\ind{ \tauind}} \left[ B\ind{ \tauind}_\ell(\x_\ell) \mid \x_h=x_h, \a_{h}=a_{h}\right]\nn \\
& \leq \E_{\pihat\ind{ \tauind}} \left[\mathbb{I}\{\x_\ell\in \cX\ind{\tauind}_{\ell,\spann} \} \cdot B\ind{ \tauind}_\ell(\x_\ell) \mid \x_h=x_h, \a_{h}=a_{h}\right]\nn \\
& \quad +  \E_{\pihat\ind{ \tauind}} \left[\mathbb{I}\{\x_\ell\not\in \cX\ind{\tauind}_{\ell,\spann} \} \cdot B\ind{\tauind}_\ell(\x_\ell) \mid \x_h=x_h, \a_{h}=a_{h}\right],\nn \\
& \leq (1+ 2 e \Ccond(\pistar_{\ell,\beta}))\cdot\uveps.\label{eq:lastone} 
\end{align}
Now, by \eqref{eq:soak} and the fact that $\Mrej \geq  32 e (1\vee \Ccond(\pistar_{\ell,\beta})^2)$ (see parameter choice in \sssref{app:params}), we have that for all $\ell\in [h+1\ldotst H]$: if $x \in \cX\ind{\tauind}_{\ell,\spann}$ then \[\Mrej \geq (4\Ccond(\pibar\ind{j}_\ell \mid x)) \vee (16\Ccond(\pibar\ind{j}_\ell \mid x)^2)\] and so by \eqref{eq:ellip1}, we have for all $(x_h,a_h)\in \cC\ind{\tauind}_h$ and $\ell\in [h+1\ldotst H]$:
\begin{align}
    1 - \veps_\spann  & \leq  \P_{\pihat\ind{ \tauind}}\left[\x_\ell \in \cX\ind{\tauind}_{\ell,\spann} \mid \x_h=x_h, \a_{h}=a_{h}\right], \nn \\
   &   \leq  \P_{\pihat\ind{ \tauind}}\left[\Mrej \geq  (4\Ccond(\pibar\ind{j}_\ell \mid \x_\ell)) \vee (16\Ccond(\pibar\ind{j}_\ell \mid \x_\ell)^2) \mid \x_h =x_h, \a_h =a_h\right]. \nn 
\end{align}
This implies that for all $\ell\in [h+1\ldotst H]$ and  $(x_h,a_h)\in \cC\ind{\tauind}_h$:
\begin{align}
    \P_{\pihat\ind{ \tauind}}\left[\Mrej <  4\Ccond(\pibar\ind{\tauind}_{\ell}\mid \x_\ell)\mid \x_h =x_h, \a_h =a_h\right] & \leq \veps_\spann,  \label{eq:ccov5}\\
    \shortintertext{and}
    \P_{\pihat\ind{ \tauind}}\left[\Mrej <  16\Ccond(\pibar\ind{\tauind}_{\ell}\mid \x_\ell)^2\mid \x_h =x_h, \a_h =a_h\right] & \leq \veps_\spann. \label{eq:ccov2}
\end{align}
Finally, we have that for all $\ell\in [h+1\ldotst H]$ and $(x_h,a_h)\in \cC\ind{\tauind}_h$:
\begin{align}
   &  \E_{\pihat\ind{\tauind}}\left[  \min \left(1,\Ccond(\pibar\ind{\tauind}_{\ell}\mid \x_\ell) \cdot \sqrt{\frac{2}{\Mrej}} \right) \mid  \x_h =x_h, \a_h=a_h \right]\nn \\
& \leq \E_{\pihat\ind{\tauind}}\left[ \mathbb{I}\{\x_\ell\in \cX\ind{j}_{\ell,\spann}\}\cdot  \min \left(1,\Ccond(\pibar\ind{\tauind}_{\ell}\mid \x_\ell) \cdot \sqrt{\frac{2}{\Mrej}} \right) \mid  \x_h =x_h, \a_h=a_h \right]\nn \\
& \quad + \E_{\pihat\ind{\tauind}}\left[ \mathbb{I}\{\x_\ell\not\in \cX\ind{j}_{\ell,\spann}\}\cdot  \min \left(1,\Ccond(\pibar\ind{\tauind}_{\ell}\mid \x_\ell) \cdot \sqrt{\frac{2}{\Mrej}} \right) \mid  \x_h =x_h, \a_h=a_h \right], \nn \\
\intertext{and so by \eqref{eq:soak} and \eqref{eq:spann}:}
& \leq 2 e\Ccond(\pistar_{\ell,\beta}) \cdot \sqrt{\frac{2}{\Mrej}} + \P_{\pihat\ind{\tauind}}\left[ \x_\ell\not\in \cX\ind{j}_{\ell,\spann} \mid  \x_h =x_h, \a_h=a_h \right], \nn \\
& \leq 2 e\Ccond(\pistar_{\ell,\beta}) \cdot \sqrt{\frac{2}{\Mrej}} + \veps_\spann. \label{eq:ccov1}
\end{align}
\paragraphi{Putting it all together}
Combining \eqref{eq:ccov1}, \eqref{eq:ccov5}, \eqref{eq:ccov2}, and \eqref{eq:lastone} with \eqref{eq:klbound} and \eqref{eq:rhs} and using that $|\cC\ind{\tauind}_h|\leq 2 \Trounds$, we get 
	\begin{align}
	 & \|\theta\ind{ \tauind}_h -\theta^\star_{h,\beta}\|^2_{\Sigma_h\ind{ \tauind}}\nn \\
	& \leq 4 \lambda B^2 + \frac{C_1}{\Nreg}  +  C_2 \deltarej  + C_3 H \Trounds \veps_\spann \nn \\
    & \quad + 1536 e  H^2 B \beta \Trounds \left(\max_{\ell\in [H]} \Ccond(\pistar_{\ell,\beta}) \cdot \sqrt{\frac{2}{\Mrej}} + \veps_\spann\right) \nn \\
    & \quad + 3200  H^2\beta^2\Trounds \left(\lambda^2 \nu^4 + 16 \vepslip +32 \deltarej^2 +  32\veps_\spann\right) \cdot \frac{\veps_\reg^4}{\beta^2 \lambda^2} \nn \\
    & \quad + 7680  H^2 \Rmax^2 \Trounds (1+ 2 e \Ccond(\pistar_{\ell,\beta}))^2\cdot\uveps^2,\nn\\
& \leq \veps^2_\reg,\nn 
\end{align}
where the last inequality follows by the parameter choices in \cref{alg:ops-dp}. This completes the induction and implies the desired result.
\end{proof}

\begin{lemma}[Helper lemma for estimation error]
    \label{lem:helperest}
    Let $h\in[0\ldotst H]$ be given. Consider the setting of \cref{lem:main} and let $\tauind \in \cJ\ind{\spann}$ with $\cJ^{\spann}$ as in \cref{lem:test}. Further, let $\cX\ind{\tauind}_{\ell,\spann}$ and $(\vepslip,\uveps)$ be as in \cref{lem:test} and \cref{lem:uncertaintraj}, respectively. Then, we have that for all $\ell\in [h+1 \ldotst H]$ and $(x,a)\in \cX\ind{\tauind}_{\ell,\spann}\times \cA$: 
    \begin{align}
    \frac{\pibar\ind{\tauind}_\ell(a\mid x)}{\pi_{\ell,\refe}(a\mid x)} \vee \frac{\pibar^{\tauind,\star}_{\ell,\beta}(a\mid x)}{\pi_{\ell,\refe}(a\mid x)} \leq    2 e^{\frac{4\nu }{\beta}\|\theta\ind{ \tauind}_\ell - \theta^\star_{\ell,\beta}\|_{\Sigma\ind{ \tauind}_\ell}}  \cdot  \frac{\pistar_{\ell,\beta}(a\mid x)}{\pi_{\ell,\refe}(a\mid x)}.\nn 
    \end{align}
\end{lemma} 
    \begin{proof}[Proof of \cref{lem:helperest}]
    By \cref{lem:test}, we have that for all $\ell\in[h+1 \ldotst H]$ and $(x_h,a_h) \in \cC_h\ind{\tauind}$, 
    \begin{gather}
       \P_{\pihat\ind{\tauind}}\left[\x_\ell \in \cX\ind{\tauind}_{\ell,\spann} \mid \x_h=x_h, \a_{h}=a_{h}\right] \geq 1- \veps_\spann, \label{eq:spann0}
    \intertext{and for all $x\in \cX\ind{\tauind}_{\ell,\spann}$:}
    \P_{\a\sim \pi_{\ell,\refe}(\cdot \mid x)} \left[\|\varphi_\ell(x,\a)\|^2_{(\Sigma\ind{ \tauind}_{\ell})^{-1}} > \nu^2\right] \leq \uveps \leq  \frac{1}{4 \Ccond (\pistar_{\ell,\beta})}, 
    \label{eq:ellip100}
    \end{gather} 
    where the last inequality follows by the fact that $\Nspanb \geq 4 \Ccond(\pistar_\beta)$ (see parameter choices in \cref{alg:ops-dp}). 
    Now, by \cref{lem:coverability} and \cref{lem:test}, we have that for all $\ell\in[h+1\ldotst H]$ and $(x,a)\in \cX\ind{\tauind}_{\ell,\spann} \times \cA$:
 \begin{align} 
& \frac{\pibar\ind{ \tauind}_\ell(a\mid x)}{\pi_{\ell,\refe}(a\mid x)} \vee \frac{\pibar^{\tauind,\star}_{\ell,\beta}(a\mid x)}{\pi_{\ell,\refe}(a\mid x)}\nn \\ & \leq \min \left\{\frac{e^{\frac{4\nu}{\beta}\|\theta\ind{ \tauind}_\ell - \theta^\star_{\ell,\beta}\|_{\Sigma\ind{\tauind}_\ell}}}{1- \Ccond(\pistar_{\ell,\beta}) \cdot \P_{\a\sim \pi_{\ell,\refe}(\cdot \mid x)}\left[\|\varphi_\ell(x,\a)\|^2_{(\Sigma_\ell\ind{ \tauind})^{-1}}>\nu^2\right] }  \cdot  \frac{\pistar_{\ell,\beta}(a\mid x)}{\pi_{\ell,\refe}(a\mid x)},\ e^{2 B /\beta}\right\}, \nn \\
& \leq 2 e^{\frac{4\nu }{\beta}\|\theta\ind{ \tauind}_\ell - \theta^\star_{\ell,\beta}\|_{\Sigma\ind{ \tauind}_\ell}}  \cdot  \frac{\pistar_{\ell,\beta}(a\mid x)}{\pi_{\ell,\refe}(a\mid x)},
 \end{align}
 where the last step follows by \eqref{eq:ellip100}.
\end{proof}

\begin{proof}[Proof of \cref{thm:main}]
    In this proof, we let $\cE^\reg$, $\cE^\spann$, and $\cJ^\spann$ be as in \cref{lem:fitval}, \cref{lem:uncertaintraj}, and \cref{lem:test}, respectively, and we condition on $\cE^\reg \cap \cE^\spann$. Further, let $\tauindl\in\brk{T}$ be the index \cref{alg:ops-dp} for which the algorithm returns $\pihat\ind{j}$, and note that $\tauindl\in\cJ^\spann$. 
    
    By the performance difference lemma (\cref{lem:perform}) and \cref{lem:wonky}, the policy $\pihat\ind{\tauindl}$ satisfies
    \begin{align}
       & J_\beta(\pistar_\beta) - J_\beta(\pihat\ind{\tauindl}_{1:H}) \nn \\ 
       & = \sum_{h=1}^H \E_{\pihat\ind{\tauindl}}\left[\sum_{a\in \cA} \pistar_{h,\beta}(a\mid \x_h)\cdot \left(\Qstar_{h,\beta}(\x_h,a)- \beta \cdot \log \frac{\pistar_{h,\beta}(a\mid \x_h)}{\pi_{h,\refe}(a\mid \x_h)} \right)\right]\nn \\
        & \quad  - \sum_{h=1}^H \E_{\pihat\ind{\tauindl}}\left[\sum_{a\in \cA} \pihat\ind{\tauindl}_h(a\mid \x_h)\cdot \left( \Qstar_{h,\beta}(\x_h,a) - \beta \cdot \log \frac{\pihat\ind{\tauindl}_h(a\mid \x_h)}{\pi_{h,\refe}(a\mid \x_h)} \right) \right],\nn \\
        & = \sum_{h=1}^H \E_{\pihat\ind{\tauindl}}\left[\sum_{a\in \cA} \pistar_{h,\beta}(a\mid \x_h)\cdot \left(\Qstar_{h,\beta}(\x_h,a)- \beta \cdot \log \frac{\pistar_{h,\beta}(a\mid \x_h)}{\pi_{h,\refe}(a\mid \x_h)} \right)\right]\nn \\
        & \quad  - \sum_{h=1}^H \E_{\pihat\ind{\tauindl}}\left[\sum_{a\in \cA} \pihat\ind{\tauindl}_h(a\mid \x_h)\cdot \left( \Qstar_{h,\beta}(\x_h,a) - \beta \cdot \log \frac{\pibar\ind{\tauindl}_h(a\mid \x_h)}{\pi_{h,\refe}(a\mid \x_h)} \right) \right]\nn \\
        & \quad +\sum_{\ell=1}^H  \beta\cdot \E_{\pihat\ind{\tauindl}}\left[\kl{\pihat\ind{\tauindl}_\ell(\cdot \mid \x_\ell)}{\pibar\ind{\tauindl}_\ell(\cdot \mid \x_\ell)}\right], \label{eq:goback}
    \end{align}
    where $\pibar\ind{\tauindl}$ is as in \cref{lem:fitval}. Now by \cref{lem:innerKL_first} (instantiated with $h =0$, $\pihat_\theta= \pihat\ind{j}$, and $\pibar_\theta= \pibar\ind{j}$), we can bound the KL term in \eqref{eq:goback} as follows: for all $\ell\in [H]$,
    \begin{align}
    & \E_{\pihat\ind{\tauindl}}\left[\kl{\pihat\ind{\tauindl}_\ell(\cdot \mid \x_\ell)}{\pibar\ind{\tauindl}_\ell(\cdot \mid \x_\ell)}\right]\nn \\ &\leq  4\prn*{\frac{\Rmax}{\beta} +\log(4 \Mrej \log(4 \deltarej^{-1}))} \deltarej \nn \\
        & \quad  + \frac{\Rmax}{\beta} \log(4 \Mrej \log(4 \deltarej^{-1})) \cdot \P_{\pihat\ind{j}}\left[ \Mrej < 4\Ccond(\pibar_{\ell}\ind{\tauind}\mid \x_\ell) \right]. \label{eq:klterm_pre}
        \end{align}
    Now, note that for all $\ell\in[H]$, \[\left|Q_{\ell,\beta}^{\star}(\cdot,\cdot) - \beta \cdot \log \frac{\pibar\ind{\tauindl}_\ell(\cdot\mid \cdot)}{\pi_{\ell,\refe}(\cdot \mid \cdot)}\right| \leq  BH + 2 H \beta \sup_{(x,a)\in \cX\times \cA}\left|\log \frac{\pibar_{\ell,\theta}(a \mid x)}{\pi_{\ell,\refe}(a\mid x)}\right| \leq 5BH,\] since $e^{-2 B/\beta}\leq\frac{\pibar_{\ell,\theta}(\cdot \mid \cdot)}{\pi_{\ell,\refe}(\cdot\mid \cdot)}\leq e^{2B/\beta}$ for all $\theta_\ell\in \bbB(B)$. Thus, by \cref{lem:rejection_tv_average_new}, we have that for all $\ell \in [H]$:
\begin{align}
&\left|\E_{\pihat\ind{\tauindl}}\left[\sum_{a\in \cA} \pihat\ind{\tauindl}_\ell(a\mid \x_\ell)\cdot \left( Q_{\ell,\beta}^{\star}(\x_\ell,a) - \beta \cdot \log \frac{\pibar\ind{\tauindl}_\ell(a\mid \x_\ell)}{\pi_{\ell,\refe}(a\mid \x_\ell)} \right) \right]\right.\nn \\
& \quad  \left. -  \E_{\pihat\ind{\tauindl}}\left[\sum_{a\in \cA} \pibar\ind{\tauindl}_\ell(a\mid \x_\ell)\cdot \left(Q_{\ell,\beta}^{\star}(\x_\ell,a) - \beta \cdot \log \frac{\pibar\ind{\tauindl}_\ell(a\mid \x_\ell)}{\pi_{\ell,\refe}(a\mid \x_\ell)} \right) \right] \right|\nn \\
& \leq 5 B H \deltarej + 5 BH \cdot \P_{\pihat\ind{\tauindl}}\left[ \Mrej <  4 \Ccond(\pibar\ind{\tauind}_\ell\mid x)  \right].\label{eq:offday}
\end{align}
Combining this with \eqref{eq:goback} and \eqref{eq:klterm_pre}, we get that
\begin{align}
   & J_\beta(\pistar_\beta)- J_\beta(\pihat\ind{\tauindl}_{1:H})\nn \\
  & \leq   \sum_{h=1}^H \E_{\pihat\ind{\tauindl}}\left[\sum_{a\in \cA} \pistar_{h,\beta}(a\mid \x_h)\cdot \left(\Qstar_{h,\beta}(\x_h,a)- \beta \cdot \log \frac{\pistar_{h,\beta}(a\mid \x_h)}{\pi_{h,\refe}(a\mid \x_h)} \right)\right]\nn \\
        & \quad  - \sum_{h=1}^H \E_{\pihat\ind{\tauindl}}\left[\sum_{a\in \cA} \pibar\ind{\tauindl}_h(a\mid \x_h)\cdot \left( \Qstar_{h,\beta}(\x_h,a) - \beta \cdot \log \frac{\pibar\ind{\tauindl}_h(a\mid \x_h)}{\pi_{h,\refe}(a\mid \x_h)} \right) \right] \nn \\
        & \quad +  4 H\prn*{\Rmax + 5 H B/4 + \beta \log(4 \Mrej \log(4 \deltarej^{-1}))}\cdot  \deltarej \nn \\
        & \quad  +  B \left( \log(4 \Mrej \log(4 \deltarej^{-1})) + 5 H  \right) \sum_{\ell=1}^H \P_{\pihat\ind{j}}\left[ \Mrej < 4\Ccond(\pibar_{\ell}\ind{\tauind}\mid \x_\ell)\right] ,
        \intertext{and so by \cref{lem:truncation}, we have
        }
        & \leq \beta \sum_{\ell=1}^H \E_{\pihat\ind{\tauindl}}\left[\kl{\pibar\ind{\tauindl}_{\ell}(\cdot \mid \x_\ell)}{\pibar^{\tauindl,\star}_{\ell,\beta}(\cdot \mid \x_\ell)} \right] \nn \\
        & \quad +  4 H\prn*{\Rmax + 5 H B/4 + \beta \log(4 \Mrej \log(4 \deltarej^{-1}))} \cdot \deltarej \nn \\
        & \quad  +  B \left( \log(4 \Mrej \log(4 \deltarej^{-1})) + 5 H  \right) \sum_{\ell=1}^H \P_{\pihat\ind{j}}\left[ \Mrej < 4\Ccond(\pibar_{\ell}\ind{\tauind}\mid \x_\ell)\right], \label{eq:stepone}
    \end{align} 
    where $\pibar^{\tauindl,\star}_{\ell,\beta}$ is as in \cref{lem:fitval}.
    We start with bounding the KL terms in \cref{eq:stepone}.

\paragraphi{Bounding the KL term} 
We now bound the KL term in \eqref{eq:stepone}. The argument closely follows the proof of \cref{lem:main}, where we bounded the expectation of the squared KL term. The key difference here is that \eqref{eq:stepone} does not include a squared term.
By \cref{lem:KLbound} and \cref{lem:main} (which implies that $\|\theta^\star_{\ell,\beta} - \theta\ind{\tauindl}_\ell\|_{\Sigma\ind{\tauindl}_\ell}\leq\beta/\nu$---the precondition of \cref{lem:KLbound}), we have that for all $\ell \in [H]$: 
\begin{align}
    & \E_{\pihat\ind{ \tauindl}}\left[\kl{\pibar\ind{\tauindl}_{\ell}(\cdot \mid \x_\ell)}{\pibar^{\tauindl,\star}_{\ell,\beta}(\cdot \mid \x_\ell)}\right] \nn \\
	& \leq \E_{\pihat\ind{\tauindl}}\left[ \beta^{-1}\sum_{a\in \cA} \pibar\ind{ \tauindl}_{\ell}(a \mid \x_\ell)\cdot  \|\varphibar\ind{\tauindl}_\ell(\x_\ell,a)\|^2_{(\Sigma_\ell\ind{ \tauindl})^{-1}} \right]  \cdot \|\theta^\star_{\ell,\beta} - \theta\ind{\tauindl}_\ell\|^2_{\Sigma\ind{\tauindl}_\ell}, \label{eq:toplugging0} 
\end{align}
where $\varphibar\ind{\tauindl}$ is as in \cref{lem:fitval}. Now by \cref{lem:rejection_tv_average_new}, we have that for any $x\in \cX$ and $\ell \in [H]$: 
    \begin{align}
    & \left|\sum_{a\in \cA} \pibar\ind{\tauindl}_{\ell}(a \mid x)\cdot  \|\varphibar\ind{\tauindl}_\ell(x,a)\|^2_{(\Sigma_\ell\ind{ \tauindl})^{-1}}  - \sum_{a\in \cA} \pihat\ind{\tauindl}_{\ell}(a \mid x)\cdot  \|\varphibar\ind{\tauindl}_\ell(x,a)\|^2_{(\Sigma_\ell\ind{ \tauindl})^{-1}} \right|\nn \\
    & \leq \frac{4}{\lambda} \deltarej + \frac{4}{\lambda}  \cdot \mathbb{I}\left\{ \Mrej < 4\Ccond(\pibar\ind{\tauind}_{\ell}\mid x)\right\}.\nn 
    \end{align}
    Plugging this into \eqref{eq:toplugging0} and using Jensen inequality, we get that for all $\ell \in [H]$:
        \begin{align} 
    & \E_{\pihat\ind{ \tauindl}}\left[\kl{\pibar\ind{\tauindl}_{\ell}(\cdot \mid \x_\ell)}{\pibar^{\tauindl,\star}_{\ell,\beta}(\cdot \mid \x_\ell)} \right] \nn \\
        & \leq  \frac{1}{\beta}\cdot  \E_{\pihat\ind{\tauindl}}\left[\|\varphibar\ind{\tauindl}_\ell(\x_\ell,\a_\ell)\|^2_{(\Sigma_\ell\ind{ \tauindl})^{-1}}\right]    \cdot \|\theta^\star_{\ell,\beta} - \theta\ind{\tauindl}_\ell\|^2_{\Sigma\ind{\tauindl}_\ell}\nn \\
        & \quad + \left(\frac{4 \deltarej}{\lambda\beta} +  \frac{4}{\beta \lambda}\P_{\pihat\ind{\tauindl}}\left[\Mrej < 4\Ccond(\pibar\ind{\tauind}_{\ell}\mid \x_\ell)\right] \right)  \cdot \|\theta^\star_{\ell,\beta} - \theta\ind{\tauindl}_\ell\|^2_{\Sigma\ind{\tauindl}_\ell},\nn \\
        & \leq \frac{\nu^2}{\beta} + \frac{4}{\lambda \beta} \cdot \P_{\pihat\ind{\tauind}}\left[\| \varphi_\ell(\x_\ell,\a_\ell)\|^2_{(\Sigma\ind{\tauind}_\ell)^{-1}}>\nu^2  \right] \cdot  \|\theta\ind{\tauindl}_\ell - \theta^\star_{\ell,\beta}\|^2_{\Sigma\ind{ \tauindl}_\ell} \nn \\
        & \quad + \left(\frac{4 \deltarej}{\lambda\beta} +  \frac{4}{\beta \lambda}\P_{\pihat\ind{\tauindl}}\left[\Mrej < 4\Ccond(\pibar\ind{\tauind}_{\ell}\mid \x_\ell)\right] \right)  \cdot \|\theta^\star_{\ell,\beta} - \theta\ind{\tauindl}_\ell\|^2_{\Sigma\ind{\tauindl}_\ell},\nn 
    \end{align}
    where the last inequality follows by the fact that $\sigma_{\min}(\Sigma\ind{ \tauindl}_\ell) \geq \lambda$, and $\|\varphi_\ell(x,a)\|\leq 2$, for all $\ell\in[H],(x,a)\in \cX\times \cA$ (follows by \cref{assum:linear}).
    And so, by \cref{lem:test} and \cref{lem:main}, we have that for all $\ell \in [H]$:
    \begin{align}
          &\E_{\pibar\ind{ \tauindl}}\left[\kl{\pibar\ind{ \tauindl}_{\ell}(\cdot \mid \x_\ell)}{\pibar^{\tauindl,\star}_{\ell,\beta}(\cdot \mid \x_\ell)}\right] \nn \\ 
    & \leq  \left(\lambda \nu^2 + 4 \vepslip +4 \deltarej +  4\P_{\pihat\ind{\tauindl}}\left[\Mrej < 4\Ccond(\pibar\ind{\tauind}_{\ell}\mid \x_\ell)\right]\right) \cdot \frac{\veps_\reg^2}{\beta \lambda}. \label{eq:klbound0}
    \end{align}
    We now bound the distribution shift terms $(\P_{\pihat\ind{\tauindl}}\left[\Mrej < 4\Ccond(\pibar\ind{\tauind}_{\ell}\mid \cdot)\right])_\ell$ in \eqref{eq:klbound0} and on the right-hand side of \eqref{eq:stepone}.

    \paragraphi{Bounding the distribution shift terms} Let $\cX\ind{\tauindl}_{\ell,\spann}$ and $(\vepslip, \uveps)$ be as in \cref{lem:test} and \cref{lem:uncertaintraj}, respectively. By \cref{lem:test}, we have that for all $\ell\in[H]$ and $(x_0,a_0)\in \cX \times \cA$, 
\begin{gather}
    \P_{\pihat\ind{ \tauindl}}\left[\x_\ell \in \cX\ind{\tauindl}_{\ell,\spann}\right]=  \P_{\pihat\ind{ \tauindl}}\left[\x_\ell \in \cX\ind{\tauindl}_{\ell,\spann} \mid \x_0=x_0,\a_0=a_0\right] \geq 1- \veps_\spann, \label{eq:spann0}
\shortintertext{where $\uveps$ is as in \cref{lem:uncertaintraj}, and for all $x\in \cX\ind{\tauindl}_{\ell,\spann}$:}
   \P_{\a\sim \pi_{\ell,\refe}(\cdot \mid x)} \left[\|\varphi_\ell(x,\a)\|^2_{(\Sigma\ind{ \tauindl}_{\ell})^{-1}} > \nu^2\right] \leq \uveps \leq  \frac{1}{{4\Ccond (\pistar_{\ell,\beta})}}, \label{eq:ellip10}
\end{gather} 
where the last inequality follows by the fact that $\Nspanb \geq 4 \Ccond(\pistar_\beta)$ (see parameter choices in \cref{alg:ops-dp}). On the other hand, by \cref{lem:helperest}, we have that for all $\ell\in [H]$ and $(x,a)\in \cX\ind{\tauind}_{\ell,\spann}\times \cA$: 
\begin{align}
    \frac{\pibar\ind{ \tauind}_\ell(a\mid x)}{\pi_{\ell,\refe}(a\mid x)} \vee \frac{\pibar^{\tauind,\star}_{\ell,\beta}(a\mid x)}{\pi_{\ell,\refe}(a\mid x)}& \leq    2 e^{\frac{4\nu }{\beta}\|\theta\ind{ \tauind}_\ell - \theta^\star_{\ell,\beta}\|_{\Sigma\ind{ \tauind}_\ell}}  \cdot  \frac{\pistar_{\ell,\beta}(a\mid x)}{\pi_{\ell,\refe}(a\mid x)}, \nn \\
& \leq 2 e  \frac{\pistar_{\ell,\beta}(a\mid x)}{\pi_{\ell,\refe}(a\mid x)},  \label{eq:soak2} 
 \end{align}
 where the last step follows by \cref{lem:main} and the fact that $\veps_\reg \leq \beta/(4\nu)$ (see choice of $\nu$ in \cref{alg:ops-dp}).
 Therefore, we have that for all $\ell\in[H]$, $x\in \cX\ind{\tauindl}_{\ell,\spann}$, and $\pi \in\{\pibar\ind{ \tauindl}_{\ell}, \pibar^{\tauindl,\star}_{\ell,\beta}\}$:
 \begin{align}
\Ccond(\pi\mid x) \leq 2 e \Ccond (\pistar_{\ell,\beta}). \label{eq:ccov0}
 \end{align}
Now, by \eqref{eq:soak2} and the fact that $\Mrej \geq  8 e  \Ccond(\pistar_{\ell,\beta})$ (see \sssref{app:params}), we have that for all $\ell\in [H]$: $x \in \cX\ind{\tauind}_{\ell,\spann}$ only if $\Mrej \geq 4\Ccond(\pibar\ind{j}_\ell \mid x)$ and so by \eqref{eq:ellip1}, we have for all $\ell\in [H]$:
\begin{align}
    1 - \veps_\spann  & \leq  \P_{\pihat\ind{ \tauind}}\left[\x_\ell \in \cX\ind{\tauind}_{\ell,\spann} \right]   \leq  \P_{\pihat\ind{ \tauind}}\left[\Mrej \geq  4\Ccond(\pibar\ind{j}_\ell \mid \x_\ell) \right]. \nn 
\end{align}
This implies that for all $\ell\in [H]$:
\begin{align}
    \P_{\pihat\ind{ \tauind}}\left[\Mrej <  4\Ccond(\pibar\ind{\tauind}_{\ell}\mid \x_\ell)\right] & \leq \veps_\spann.  \label{eq:ccov00}
\end{align}
\paragraph{Putting it all together}
Combining \eqref{eq:ccov00} with \eqref{eq:klbound0} and \eqref{eq:stepone}, we get that
\begin{align}
    & J_\beta(\pistar_\beta) - J_\beta(\pihat\ind{\tauindl}_{1:H}) \nn \\
    & \leq  \sum_{\ell=1}^H\left(\lambda \nu^2 + 4 \vepslip +4 \deltarej +  4\P_{\pihat\ind{\tauindl}}\left[\Mrej < 4\Ccond(\pibar\ind{\tauind}_{\ell}\mid \x_\ell)\right]\right) \cdot \frac{\veps_\reg^2}{\lambda} \nn \\
    & \quad +  4 H\prn*{\Rmax + 5 H B/4 + \beta \log(4 \Mrej \log(4 \deltarej^{-1}))} \cdot \deltarej \nn \\
    & \quad  +  B \left( \log(4 \Mrej \log(4 \deltarej^{-1})) + 5 H  \right) \sum_{\ell=1}^H \P_{\pihat\ind{j}}\left[ \Mrej < 4\Ccond(\pibar_{\ell}\ind{\tauind}\mid \x_\ell)\right], \nn  \\
    &  \leq  H \cdot \left(\lambda \nu^2 + 4 \vepslip +4 \deltarej +  4\veps_\spann \right) \cdot \frac{\veps_\reg^2}{\lambda} \nn \\
    & \quad +  4 H\prn*{\Rmax + 5 H B/4 + \beta \log(4 \Mrej \log(4 \deltarej^{-1}))} \cdot \deltarej \nn \\
    & \quad  +  B \left( \log(4 \Mrej \log(4 \deltarej^{-1})) + 5 H  \right) H \veps_\spann, \nn \\
    & \leq \veps,\label{eq:try}
\end{align}
where the last inequality follows by the choice of parameters in \cref{alg:ops-dp}.
Combining this with the fact that $\P[\cE^\spann\cap \cE^\reg]\geq 1 -\delta$ (by \cref{lem:fitval} and \cref{lem:test} and the union bound) completes the proof. 
\end{proof}

 \newpage
\section{Technical Lemmas for Multi-Turn Exploration}
\label{sec:technical_rl}
In this section, we present and prove the technical results required for the proofs in the Multi-turn setting. Some of the statements provided here are generalizations of those in \cref{app:algorithms}, originally formulated for the contextual bandit setting.
\begin{lemma}[KL bound for truncated softmax policies]
	\label{lem:KLbound}
	Let $h\in[H]$, $B,\nu>0$, $(x,\fraka)\in \cX\times \cA$, $\theta_h \in \bbB(B)$, and $\Sigma_h \in \reals^{d\times d}$ be given, and let $\varphi_h(x,\cdot)\coloneqq \phi_h(x,\cdot)- \phi_h(x,\fraka)$. Further, define $\pibar_{h,\theta}(\cdot\mid x) \propto \pi_{h,\refe}(\cdot\mid x) \cdot e^{\beta^{-1}\varphibar_h(x,\cdot)^\top \theta}$ and $\pibar^\star_{h,\beta}(\cdot\mid x) \propto \pi_{h,\refe}(\cdot\mid x) \cdot e^{\beta^{-1} \varphibar_h(x,\cdot)^\top \theta^\star_{h,\beta}}$, where $\varphibar_{h}(x,\cdot) \coloneqq \varphi_h(x,\cdot) \cdot \mathbb{I}\left\{\|\varphi_h(x,\cdot)\|^2_{\Sigma^{-1}_h} \leq \nu^2\right\}$. If $\|\theta^\star_{h,\beta}- \theta_h\|_{\Sigma_h}\leq \beta/\nu$, then we have that
	\begin{align}
		\kl{\pibar_{h,\theta}(\cdot \mid x)}{ \pibar^\star_{h,\beta} (\cdot \mid x)} \leq  \beta^{-1}\E_{\a \sim \pibar_{h,\theta}(\cdot \mid x)}\left[ \|\varphibar_h(x,\a)\|^2_{\Sigma_h^{-1}}\right] \cdot \|\theta^\star_{h,\beta} - \theta_h\|^2_{\Sigma_h}.\nn 
	\end{align}
\end{lemma}
\begin{proof}[\pfref{lem:KLbound}]
	Let $h\in[H]$, $B,\nu>0$, $(x,\fraka)\in \cX\times \cA$, $\theta_h \in \bbB(B)$, and $\Sigma_h \in \reals^{d\times d}$ be as in the lemma statement. We have for $\widebar{Z}_{h,\theta}(x) \coloneqq \E_{\a \sim \pi_{h,\refe}(\cdot \mid x)} \left[\exp(\beta^{-1} \varphibar_h(x,\a)^\top \theta) \right]$:
	\begin{align}
		&\kl{\pibar_{h,\theta}(\cdot \mid x)}{ \pibar^\star_{h,\beta} (\cdot \mid x)}\nn \\ & = \beta \log \frac{\widebar{Z}_{h,\theta^\star_{h,\beta}}(x)}{\widebar{Z}_{h,\theta_{h}}(x)} +  \E_{\a\sim \pibar_{h,\theta}(\cdot \mid x)}\left[\varphibar_h(x,\a)^\top (\theta_h -  \theta^\star_{h,\beta})\right],\nn \\
		& = \beta \log \left(\E_{\a \sim \pibar_{h,\theta}(\cdot \mid x)} \left[\exp(\beta^{-1} \varphibar_h(x,\a)^\top (\theta^\star_{h,\beta} - \theta_h)) \right]  \right) +  \E_{\a\sim \pibar_{h,\theta}(\cdot \mid x)}\left[\varphibar_h(x,\a)^\top (\theta_h -  \theta^\star_{h,\beta})\right]. \label{eq:desiredres}
	\end{align}
	Now, by H\"older's inequality, we have that 
	\begin{align}
		\left|\inner{\varphibar_h(x,\a)}{\theta^\star_{h,\beta} - \theta_h}\right| &  \leq \|\varphibar_h(x,\a)\|_{\Sigma_h^{-1}} \cdot \|\theta^\star_{h,\beta} - \theta_h\|_{\Sigma_h}, \nn \\
		& \leq \nu \cdot \|\theta^\star_{h,\beta} - \theta_h\|_{\Sigma_h}, \nn \\
		& \leq \beta, \label{eq:less}
	\end{align}
	where the last inequality follows by the assumption made on $\|\theta^\star_{h,\beta} - \theta_h\|_{\Sigma_h}$. Combining \eqref{eq:less} with the fact that $e^x \leq 1 + x + x^2$, for all $x\leq 1$, we get that 
	\begin{align}
	&\beta \log \left(\E_{\a \sim \pibar_{h,\theta}(\cdot \mid x)} \left[\exp(\beta^{-1} \varphibar_h(x,\a)^\top (\theta^\star_{h,\beta} - \theta_h)) \right] \right) \nn \\
	& \leq  \E_{\a \sim \pibar_{h,\theta}(\cdot \mid x)} \left[ \inner{\varphibar_h(x,\a)}{\theta^\star_{h,\beta} - \theta_h} \right] + \beta^{-1}\E_{\a \sim \pibar_{h,\theta}(\cdot \mid x)} \left[ \inner{\varphibar_h(x,\a)}{\theta^\star_{h,\beta} - \theta_h}^2 \right], \nn \\
	& \leq \E_{\a \sim \pibar_{h,\theta}(\cdot \mid x)} \left[ \inner{\varphibar_h(x,\a)}{\theta^\star_{h,\beta} - \theta_h} \right] + \beta^{-1}\E_{\a \sim \pibar_{h,\theta}(\cdot \mid x)}\left[ \|\varphibar_h(x,\a)\|^2_{\Sigma_h^{-1}}\right] \cdot \|\theta^\star_{h,\beta} - \theta_h\|^2_{\Sigma_h}.\nn 
	\end{align}
	Plugging this in to \cref{eq:desiredres} gives the desired result.
\end{proof}
\begin{lemma}[Density ratio bound]
	\label{lem:coverability}
	Suppose \cref{assum:linear} holds and let $h\in[H]$, $B,\nu>0$, $(x,\fraka)\in \cX\times \cA$, $\theta_h \in \bbB(B)$, and $\Sigma_h \in \reals^{d\times d}$ be given and let $\varphi_h(x,\cdot)\coloneqq \phi_h(x,\cdot)- \phi_h(x,\fraka)$. Further, define $\pibar_{h,\theta}(\cdot\mid x) \propto \pi_{h,\refe}(\cdot\mid x) \cdot e^{\varphibar_h(x,\cdot)^\top \theta/\beta},$ where $\varphibar_{h}(x,\cdot) \coloneqq \varphi_h(x,\cdot) \cdot \mathbb{I}\left\{\|\varphi_h(x,\cdot)\|^2_{\Sigma^{-1}_h} \leq \nu^2\right\}$. Then,
	\arxiv{
	\begin{align}
	\forall a\in \cA,\quad \frac{\pibar_{h,\theta}(a\mid x)}{\pi_{h,\refe}(a\mid x)} \leq \min \left(\frac{\pistar_{h,\beta}(a\mid x)}{\pi_{h,\refe}(a\mid x)}\cdot \frac{e^{\frac{4\nu}{\beta} \|\theta_h - \theta^\star_{h,\beta}\|_{\Sigma_h}}}{1 - \Ccond(\pistar_{h,\beta}\mid x) \cdot p_\texttt{surprise}(x)} ,\ e^{2B/\beta}\right),\nn 
	\end{align}
	}
	where $\pi^\star_{h,\beta}$ and $\theta^\star_{h,\beta}$ are as in \cref{def:inition} and \cref{assum:linear}, respectively, and \[p_\texttt{surprise}(x)\coloneqq \P_{\a\sim \pi_{h,\refe}(\cdot \mid x)} \left[\|\varphibar_h(x,\a)\|^2_{\Sigma_h^{-1}} > \nu^2 \right].\]
\end{lemma}
\begin{proof}[\pfref{lem:coverability}]
Let $h\in[H]$, $B,\nu>0$, $(x,\fraka)\in \cX \times \cA$, $\theta_h \in \bbB(B)$, and $\Sigma_h \in \reals^{d\times d}$ be as in the lemma statement, and fix $a\in \cA$. We have 
\begin{align} 
\frac{\pibar_{h,\theta}(a\mid x)}{\pi_{h,\refe}(a\mid x)} = \frac{\exp (\beta^{-1}\varphibar_h(x,a)^\top \theta_h) }{Z_{\theta_h}},	\nn 
\end{align}
where $Z_{\theta} \coloneqq \E_{\a \sim \pi_{h,\refe}(\cdot \mid x)} \left[\exp(\beta^{-1} \varphibar_h(x,\a)^\top \theta)\right]$, for $\theta \in \reals^d$. Therefore, we have that  
\begin{align}
	\frac{\pibar_{h,\theta}(a\mid x)}{\pi_{h,\refe}(a\mid x)}= \frac{1}{\E_{\a \sim \pi_{h,\refe}(\cdot\mid x)} \left[\exp(\beta^{-1} \inner{\varphibar_h(x,\a)- \varphibar_h(x,a)}{\theta_h})\right]}. \label{eq:checkpt}
\end{align}
On the other hand, we have
\begin{align}
&	\E_{\a \sim \pi_{h,\refe}(\cdot\mid x)} \left[\exp(\beta^{-1} \inner{\varphibar_h(x,\a)- \varphibar_h(x,a)}{\theta_h})\right]\nn \\
 &\geq  \E_{\a\sim \pi_{h,\refe}(\cdot \mid  x)} \left[\exp(\beta^{-1} \inner{\varphibar_h(x,\a)- \varphibar_h(x,a)}{\theta_h}) \cdot \mathbb{I}\left\{\|\varphibar_h(x,\a)\|^2_{\Sigma_h^{-1}} \leq \nu^2 \right\}\right], \nn
 \intertext{and so by H\"older's inequality, we have}
 \arxiv{
 & \geq    \E_{\a\sim \pi_{h,\refe}(\cdot \mid  x)} \left[e^{-\frac{2}{\beta} \|\theta_h - \theta^\star_{h,\beta}\|_{\Sigma_h}\cdot \|\varphibar_h(x,\a)- \|\varphibar_h(x,a)\|_{\Sigma_h^{-1}}} e^{\beta^{-1} \inner{\varphibar_h(x,\a)- \varphibar_h(x,a)}{\theta^\star_{h,\beta}}} \cdot \mathbb{I}\left\{\|\varphibar_h(x,\a)\|^2_{\Sigma_h^{-1}} \leq \nu^2 \right\}\right],\nn \\
 & \geq    \E_{\a\sim \pi_{h,\refe}(\cdot \mid  x)} \left[e^{-\frac{2}{\beta} \|\theta_h - \theta^\star_{h,\beta}\|_{\Sigma_h}\cdot \left(\|\varphibar_h(x,\a)\|_{\Sigma_h^{-1}}+\varphibar_h(x,a)\|_{\Sigma_h^{-1}} \right)} e^{\beta^{-1} \inner{\varphibar_h(x,\a)- \varphibar_h(x,a)}{\theta^\star_{h,\beta}}} \cdot \mathbb{I}\left\{\|\varphibar_h(x,\a)\|^2_{\Sigma_h^{-1}} \leq \nu^2 \right\}\right],\nn \\
 }
 & \geq   \frac{e^{- \frac{4\nu}{\beta} \cdot \|\theta_h - \theta^\star_{h,\beta}\|_{\Sigma_h}}}{e^{\beta^{-1} \varphibar_h(x,a)^\top \theta^\star_{h,\beta}}}
  \left( \E_{\a\sim \pi_{h,\refe}(\cdot \mid  x)} \left[e^{\beta^{-1} \inner{\varphibar_h(x,\a)}{\theta^\star_{h,\beta}}} \cdot \mathbb{I}\left\{\|\varphibar_h(x,\a)\|^2_{\Sigma_h^{-1}} \leq \nu^2 \right\}\right]\right),\nn \\
 & =  \frac{e^{- \frac{4\nu}{\beta} \cdot \|\theta_h - \theta^\star_{h,\beta}\|_{\Sigma_h}}}{e^{\beta^{-1} \varphibar_h(x,a)^\top \theta^\star_{h,\beta}}}
  \left(Z_{{\theta^\star_{h,\beta}}} - \E_{\a\sim \pi_{h,\refe}(\cdot \mid  x)} \left[e^{\beta^{-1} \inner{\varphibar_h(x,\a)}{\theta^\star_{h,\beta}}} \cdot \mathbb{I}\left\{\|\varphibar_h(x,\a)\|^2_{\Sigma_h^{-1}} > \nu^2 \right\}\right]\right),\nn \\
  & = \frac{Z_{{\theta^\star_{h,\beta}}} \cdot e^{- \frac{4\nu}{\beta} \cdot \|\theta_h - \theta^\star_{h,\beta}\|_{\Sigma_h}}}{e^{\beta^{-1} \varphibar_h(x,a)^\top \theta^\star_{h,\beta}}}
  \left(1 - \E_{\a\sim \pi_{h,\refe}(\cdot \mid  x)} \left[ \frac{e^{\beta^{-1} \inner{\varphibar_h(x,\a)}{\theta^\star_{h,\beta}}}}{Z_{{\theta^\star_{h,\beta}}}} \cdot \mathbb{I}\left\{\|\varphibar_h(x,\a)\|^2_{\Sigma_h^{-1}} > \nu^2 \right\}\right]\right),\nn \\
  & = \frac{Z_{{\theta^\star_{h,\beta}}} \cdot e^{- \frac{4\nu}{\beta} \cdot \|\theta_h - \theta^\star_{h,\beta}\|_{\Sigma_h}}}{e^{\beta^{-1} \varphibar_h(x,a)^\top \theta^\star_{h,\beta}}}
  \left(1 - \P_{\a\sim \pistar_{h,\beta}(\cdot \mid  x)} \left[\|\varphibar_h(x,\a)\|^2_{\Sigma_h^{-1}} > \nu^2 \right]\right), \nn \\
  \intertext{and we can further lower bound by}
  & \geq  \frac{Z_{{\theta^\star_{h,\beta}}} \cdot e^{- \frac{4\nu}{\beta} \cdot \|\theta_h - \theta^\star_{h,\beta}\|_{\Sigma_h}}}{e^{\beta^{-1} \varphibar_h(x,a)^\top \theta^\star_{h,\beta}}}
  \left(1 - \Ccond(\pistar_{h,\beta}\mid x) \cdot \P_{\a\sim \pi_{h,\refe}(\cdot \mid  x)} \left[\|\varphibar_h(x,\a)\|^2_{\Sigma_h^{-1}} > \nu^2 \right]  \right), \nn \\
& = \frac{\pi_{h,\refe}(a\mid x)\cdot e^{- \frac{4\nu}{\beta} \cdot \|\theta_h - \theta^\star_{h,\beta}\|_{\Sigma_h}}}{\pistar_{h,\beta}(a \mid x)}
\left(1 - \Ccond(\pistar_{h,\beta}\mid x) \cdot \P_{\a\sim \pi_{h,\refe}(\cdot \mid  x)} \left[\|\varphibar_h(x,\a)\|^2_{\Sigma_h^{-1}} > \nu^2 \right]\right),\nn 
\end{align}
where the last equality follows by \cref{assum:linear} and \cref{lem:wonky}.
Using this with \eqref{eq:checkpt} and the fact that $\pibar_{h,\theta}(a\mid x)/\pi_{h,\refe}(a\mid x)$ is at most $e^{2B/\beta}$, we get the desired result. 
\end{proof}
\begin{lemma}[Two-sided bound for truncated policies]
	\label{lem:truncation}
Suppose \cref{assum:linear} holds and let $h\in[H]$, $\nu>0$, $(x,\fraka)\in \cX\times \cA$, $\theta_h \in \reals^d$, and $\Sigma_h \in \reals^{d\times d}$ be given and let $\varphi_h(\cdot,\cdot)\coloneqq \phi_h(\cdot,\cdot)- \phi_h(\cdot,\fraka)$. Let $\pi^\star_{h,\beta}$ and $\theta^\star_{h,\beta}$ are as in \cref{def:inition} and \cref{assum:linear}, respectively. Further, define 
\begin{align}
	\pibar_{h,\theta}(\cdot\mid x) \propto \pi_{h,\refe}(\cdot\mid x) \cdot e^{\varphibar_h(x,\cdot)^\top \theta/\beta},\quad  \text{and} \quad \pibar^\star_{h,\beta}(\cdot\mid x) \propto \pi_{h,\refe}(\cdot\mid x) \cdot e^{\varphibar_h(x,\cdot)^\top \theta_{h,\beta}^\star/\beta}, \label{eq:truncatedpolicies}\end{align} 
    where $\varphibar_{h}(x,\cdot) \coloneqq \varphi_h(x,\cdot) \cdot \mathbb{I}\left\{ \|\varphi_h(x,\cdot)\|^2_{\Sigma^{-1}_h}  \leq \nu^2\right\}$. Then, we have
\arxiv{
\begin{align}
	&\left|\sum_{a\in \cA} \pistar_{h,\beta}(a\mid x)\cdot \left(Q^{\star}_{h,\beta}(x,a)- \beta \cdot \log \frac{\pistar_{h,\beta}(a\mid x)}{\pi_{h,\refe}(a\mid x)} \right)-\sum_{a\in \cA} \pibar_{h,\theta}(a\mid x)\cdot \left( Q_{h,\beta}^{\star}(x,a) - \beta \cdot \log \frac{\pibar_{h,\theta}(a\mid x)}{\pi_{h,\refe}(a\mid x)} \right) \right|\nn \\
& \leq\beta \kl{\pibar_{h,\theta}(\cdot \mid x)}{\pibar^\star_{h,\beta}(\cdot \mid x)}\nn \\
& \quad + 2 \Rmax \max_{\pi \in \{\pibar_{h,\theta},\pibar^\star_{h,\beta}, \pi^\star_{h,\beta}\}}\min\left(1, \Ccond(\pi \mid x) \cdot \P_{\a\sim \pi_{h,\refe}(\cdot \mid x)}\left[\|\varphi_h(x,\a)\|^2_{\Sigma_h^{-1}}>\nu^2\right]\right).\nn 
\end{align}
}
\end{lemma}
\begin{proof}[\pfref{lem:truncation}]
	Let $h\in[H]$, $\nu>0$, $(x,\fraka)\in \cX\times \cA $, $\theta_h \in \reals^d$, and $\Sigma_h \in \reals^{d\times d}$ be as in the lemma statement. We define $\wtilde{Q}^\star_{h,\beta}(\cdot,\cdot)\coloneqq Q^\star_{h,\beta}(\cdot,\cdot)- Q^\star_{h,\beta}(\cdot,\fraka)$. Note that by \cref{assum:linear}, we have that for all $(x',a')\in \cX \times \cA$:
	\begin{align} 
		\wtilde{Q}^\star_{h,\beta}(x',a') = \varphi_h(x',a')^\top \theta_{h,\beta}^\star,\nn 
	\end{align}
	where $\varphi_h$ is as in the lemma statement.
Now, observe that by the definition of $\varphibar_h$, we have for any $\pi\in \Pi$,
\begin{align}
 &\left|	\E_{\pi}\left[\varphi_h(\x_h,\a_h)^\top \theta_{h,\beta}^\star\mid \x_h = x \right] - \E_{\pi}\left[\varphibar_h(\x_h,\a_h)^\top \theta_{h,\beta}^\star \mid \x_h =x\right]\right|\nn \\&  \leq R_{\max} \cdot  \P_{\pi}\left[\|\varphi_h(\x_h,\a_h)\|^2_{\Sigma_h^{-1}}>\nu^2\mid \x_h = x\right]. \label{eq:key} 
\end{align}
Instantiating this with $\pi = \pibar_{h,\theta}$ and using \cref{assum:linear}, we get that 
\begin{align}
	&\left| \sum_{a\in \cA} \pibar_{h,\theta}(a\mid x)\cdot \wtilde{Q}_{h,\beta}^{\star}(x,a) - \E_{\pibar_{h,\theta}}\left[\varphibar_h(\x_h,\a_h)^\top \theta_{h,\beta}^\star \mid \x_h =x\right]\right|\nn \\& \leq R_{\max} \cdot \P_{\pibar_{h,\theta}}\left[\|\varphi_h(\x_h,\a_h)\|^2_{\Sigma_h^{-1}}>\nu^2\mid \x_h = x\right], \nn \\
	& \leq R_{\max} \cdot  \min\left(1, \Ccond(\pibar_{h,\theta}\mid x) \cdot  \P_{\pi_\refe}\left[\|\varphi_h(\x_h,\a_h)\|^2_{\Sigma_h^{-1}}>\nu^2\mid \x_h = x\right]\right).\label{eq:first}
\end{align}
Now, instantiating \eqref{eq:key} with $\pi=\pistar_{h,\beta}$ and using \cref{assum:linear}, we get that
\begin{align}
	&\sum_{a\in \cA} \pistar_{h,\beta}(a\mid x)\cdot \left(\wtilde{Q}_{h,\beta}^{\star}(x,a) - \beta \log \frac{\pistar_{h,\beta}(a\mid x)}{\pi_{h,\refe}(a\mid x)} \right)\nn \\ &\leq  \E_{\pi^\star_{h,\beta}}\left[\varphibar_h(\x_h,\a_h)^\top \theta_{h,\beta}^\star \mid \x_h = x\right] - \beta \kl{\pistar_{h,\beta}(\cdot \mid x)}{\pi_{h,\refe}(\cdot \mid x)} \nn \\
	& \qquad +   R_{\max} \cdot \P_{\pistar_{h,\beta}} \left[\|\varphi_h(\x_h,\a_h)\|^2_{\Sigma_h^{-1}}>\nu^2\mid \x_h = x\right], \nn \\
	& \leq \E_{\pi^\star_{h,\beta}}\left[\varphibar_h(\x_h,\a_h)^\top \theta_{h,\beta}^\star \mid \x_h = x\right] - \beta \kl{\pistar_{h,\beta}(\cdot \mid x)}{\pi_{h,\refe}(\cdot \mid x)}\nn \\
	& \qquad + R_{\max} \cdot \min \left(1,\Ccond(\pi_{h,\beta}^\star\mid x) \cdot  \P_{\pi_\refe}\left[\|\varphi_h(\x_h,\a_h)\|^2_{\Sigma_h^{-1}}>\nu^2\mid \x_h = x\right]\right),\nn \\
	& \leq \E_{\pibar^\star_{h,\beta}}\left[\varphibar_h(\x_h,\a_h)^\top \theta_{h,\beta}^\star \mid \x_h = x\right] - \beta \kl{\pibar^\star_{h,\beta}(\cdot \mid x)}{\pi_{h,\refe}(\cdot \mid x)} \nn \\
	& \qquad  + R_{\max} \cdot \min \left(1,\Ccond(\pi_{h,\beta}^\star\mid x) \cdot  \P_{\pi_\refe}\left[\|\varphi_h(\x_h,\a_h)\|^2_{\Sigma_h^{-1}}>\nu^2\mid \x_h = x\right]\right), \label{eq:second}
\end{align}
where in the last step we used the fact that \begin{align}\pibar^\star_{h,\beta}(\cdot \mid x)\in \argmax_{\pi\in  \Delta(\cA)} \left\{\sum_{a\in \cA}\pi(a)\cdot \varphibar_{h}(x,a)^\top \theta^\star_{h,\beta}- \beta \kl{\pi(\cdot)}{\pi_{h,\refe}(\cdot \mid x)}\right\},
\label{eq:maximizer}\end{align} which follows from the definition of $\pibar^\star_{h,\beta}$. Using \eqref{eq:key} with \cref{assum:linear} again, we have that for all $\pi\in \Pi$:
\begin{align}
	&\sum_{a\in \cA} \pi(a\mid x)\cdot \left(\wtilde{Q}_{h,\beta}^{\star}(x,a) - \beta \log \frac{\pi(a\mid x)}{\pi_{h,\refe}(a\mid x)} \right)\nn \\  &\geq  \E_{\pi}\left[\varphibar_h(\x_h,\a_h)^\top \theta_{h,\beta}^\star \mid \x_h = x\right]- \beta \kl{\pi(\cdot \mid x)}{\pi_{h,\refe}(\cdot \mid x)} \nn \\ & \quad  - R_{\max} \cdot \P_{\pi} \left[\|\varphi_h(\x_h,\a_h)\|^2_{\Sigma_h^{-1}}>\nu^2\mid \x_h = x\right].\nn 
\end{align}
Thus, taking the maximum over $\pi$ on both sides and using the definition of $\pistar_{h,\beta}$ we get 
\begin{align}
	&\sum_{a\in \cA} \pistar_{h,\beta}(a\mid x)\cdot \left(\wtilde{Q}_{h,\beta}^{\star}(x,a) - \beta \log \frac{\pistar_{h,\beta}(a\mid x)}{\pi_{h,\refe}(a\mid x)} \right)\nn \\ & \geq \max_{\pi} \left\{ \E_{\pi}\left[\varphibar_h(\x_h,\a_h)^\top \theta_{h,\beta}^\star \mid \x_h = x\right]- \beta \kl{\pi(\cdot \mid x)}{\pi_{h,\refe}(\cdot \mid x)} \right. \nn \\
	&\qquad\qquad  \left. - R_{\max} \cdot \P_{\pi} \left[\|\varphi_h(\x_h,\a_h)\|^2_{\Sigma_h^{-1}}>\nu^2\mid \x_h = x\right]\right\}, \nn \\
	& \geq \E_{\pibar^\star_{h,\beta}}\left[\varphibar_h(\x_h,\a_h)^\top \theta_{h,\beta}^\star \mid \x_h = x\right]- \beta \kl{\pibar^\star_{h,\beta}(\cdot \mid x)}{\pi_{h,\refe}(\cdot \mid x)} \nn \\
	&\qquad \qquad  - R_{\max} \cdot \P_{\pibar^\star_{h,\beta}} \left[\|\varphi_h(\x_h,\a_h)\|^2_{\Sigma_h^{-1}}>\nu^2\mid \x_h = x\right], \nn \\
	& \geq \E_{\pibar^\star_{h,\beta}}\left[\varphibar_h(\x_h,\a_h)^\top \theta_{h,\beta}^\star \mid \x_h = x\right]- \beta \kl{\pibar^\star_{h,\beta}(\cdot \mid x)}{\pi_{h,\refe}(\cdot \mid x)} \nn \\
	& \quad  - R_{\max} \cdot \min \left(1,\Ccond(\pibar_{h,\beta}^\star\mid x) \cdot  \P_{\pi_\refe}\left[\|\varphi_h(\x_h,\a_h)\|^2_{\Sigma_h^{-1}}>\nu^2\mid \x_h = x\right]\right). \label{eq:third}
\end{align}
By combining \eqref{eq:first}, \eqref{eq:second}, and \eqref{eq:third}, we get that
\begin{align}
	&\left|\begin{array}{l} \sum_{a\in \cA} \pistar_{h,\beta}(a\mid x)\cdot \left(Q^{\star}_{h,\beta}(x,a)- \beta \cdot \log \frac{\pistar_{h,\beta}(a\mid x)}{\pi_{h,\refe}(a\mid x)} \right)\nn \\ -\sum_{a\in \cA} \pibar_{h,\theta}(a\mid x)\cdot \left( Q_{h,\beta}^{\star}(x,a) - \beta \cdot \log \frac{\pibar_{h,\theta}(a\mid x)}{\pi_{h,\refe}(a\mid x)} \right)\end{array} \right|\nn \\
	&=\left|\begin{array}{l}\sum_{a\in \cA} \pistar_{h,\beta}(a\mid x)\cdot \left(\wtilde{Q}^{\star}_{h,\beta}(x,a)- \beta \cdot \log \frac{\pistar_{h,\beta}(a\mid x)}{\pi_{h,\refe}(a\mid x)} \right) \nn \\
        -\sum_{a\in \cA} \pibar_{h,\theta}(a\mid x)\cdot \left( \wtilde{Q}_{h,\beta}^{\star}(x,a) - \beta \cdot \log \frac{\pibar_{h,\theta}(a\mid x)}{\pi_{h,\refe}(a\mid x)} \right)\end{array}\right|\nn \\
	& \leq \Big|\E_{\pibar^\star_{h,\beta}}\left[\varphibar_h(\x_h,\a_h)^\top \theta_{h,\beta}^\star \mid \x_h = x\right]- \beta \kl{\pibar^\star_{h,\beta}(\cdot \mid x)}{\pi_{h,\refe}(\cdot \mid x)} \nn \\
	& \quad - \E_{\pibar_{h,\theta}}\left[\varphibar_h(\x_t,\a_h)^\top \theta_{h,\beta}^\star\mid \x_h =x \right]+  \beta \kl{\pibar_{h,\theta}(\cdot \mid x)}{\pi_{h,\refe}(\cdot \mid x)}\Big| \nn \\
	& \quad  + 2 \max_{\pi \in \{\pibar_{h,\theta},\pibar^\star_{h,\beta}, \pi^\star_{h,\beta}\}}R_{\max} \min \left(1, \Ccond(\pi \mid x) \cdot \P_{\pi_\refe}\left[\|\varphi_h(\x_h,\a_h)\|^2_{\Sigma_h^{-1}}>\nu^2\mid \x_h = x\right]\right),\nn \\
        & = \E_{\pibar^\star_{h,\beta}}\left[\varphibar_h(\x_h,\a_h)^\top \theta_{h,\beta}^\star \mid \x_h = x\right]- \beta \kl{\pibar^\star_{h,\beta}(\cdot \mid x)}{\pi_{h,\refe}(\cdot \mid x)} \nn \\
	& \quad - \E_{\pibar_{h,\theta}}\left[\varphibar_h(\x_t,\a_h)^\top \theta_{h,\beta}^\star\mid \x_h =x \right]+  \beta \kl{\pibar_{h,\theta}(\cdot \mid x)}{\pi_{h,\refe}(\cdot \mid x)} \nn \\
	& \quad  + 2 \max_{\pi \in \{\pibar_{h,\theta},\pibar^\star_{h,\beta}, \pi^\star_{h,\beta}\}}R_{\max} \min \left(1, \Ccond(\pi \mid x) \cdot \P_{\pi_\refe}\left[\|\varphi_h(\x_h,\a_h)\|^2_{\Sigma_h^{-1}}>\nu^2\mid \x_h = x\right]\right),\nn \\
	& = \beta \kl{\pibar_{h,\theta}(\cdot \mid x)}{\pibar^\star_{h,\beta}(\cdot \mid x)} \nn \\
    & \quad + 2 \max_{\pi \in \{\pibar_{h,\theta},\pibar^\star_{h,\beta}, \pi^\star_{h,\beta}\}}R_{\max}\min \left(1, \Ccond(\pi \mid x) \cdot \P_{\pi_\refe}\left[\|\varphi_h(\x_h,\a_h)\|^2_{\Sigma_h^{-1}}>\nu^2\mid \x_h = x\right]\right),\nn 
\end{align}
where the second-to-last step follows by \eqref{eq:maximizer}, and the last step is a standard manipulation of KL-divergence. %
\end{proof}
\begin{lemma}
	\label{lem:wonky}
The state-action value functions $(Q^\star_{h,\beta})$ and policies $(\pistar_{h,\beta})$ in \cref{def:inition} satisfy: for all $h\in[H]$ and $(x,a)\in \cX\times \cA$,
\begin{gather}\sup_{\pi_{h+1:H}\colon \cX \rightarrow \Delta(\cA)} Q^{\pi}_{h,\beta}(x,a) =Q^{\pistar_{\beta}}_{h,\beta}(x,a)= Q^{\star}_{h,\beta}(x,a);\nn 
	\shortintertext{and} 
\pistar_{h,\beta}(\cdot \mid x)\in \argmax_{\pi \in  \Delta(\cA)} \sum_{a'\in \cA} \pi(a') \cdot \left( Q^{\star}_{h,\beta}(x,a') - \beta \log \frac{\pi(a')}{\pi_{h,\refe}(a'\mid x)} \right), \label{eq:thanks}
\end{gather}
where $Q^\pi_h$ is as in \cref{def:qfunc}.
\end{lemma} 
\begin{proof}[\pfref{lem:wonky}]
We prove the result via backward induction over $\ell = H+1, \dots, 1$ that 
and $(x,a)\in \cX\times \cA$,
\begin{gather}\sup_{\pi_{\ell+1:H}\colon \cX \rightarrow \Delta(\cA)} Q^{\pi}_{\ell,\beta}(x,a) =Q^{\pistar_{\beta}}_{\ell,\beta}(x,a)= Q^{\star}_{\ell,\beta}(x,a); \label{eq:result1}
	\shortintertext{and} 
\pistar_{\ell,\beta}(\cdot \mid x)\in \argmax_{\pi \in  \Delta(\cA)} \sum_{a'\in \cA} \pi(a') \cdot \left( Q^{\star}_{\ell,\beta}(x,a') - \beta \log \frac{\pi(a')}{\pi_{\ell,\refe}(a'\mid x)} \right). \label{eq:result2}
\end{gather}
The base case $\ell=H+1$ is immediate. Now, suppose that \eqref{eq:result1} and \eqref{eq:result2} hold for $\ell=h\in[2\ldotst H+1]$. We show that they hold for $\ell=h-1$. Fix $(x,a)\in \cX\times \cA$. For any $\pi:\cX \rightarrow \Delta(\cA)$, we have \loose
\begin{align}
&Q^{\pi}_{h-1,\beta}(x,a)\nn \\ & = \rstar_{h-1}(x,a) + \E_{\pi}\left[\sum_{\ell=h}^H \rstar_\ell(\x_\ell,\a_\ell)-\beta \sum_{\ell=h}^H \log \frac{\pi_\ell(\a_\ell \mid \x_\ell)}{\pi_{\ell,\refe}(\a_\ell \mid \x_\ell)} \mid \x_{h-1} = x,\a_{h-1} = a\right], \nn \\
& = \rstar_{h-1}(x,a) + \E_{\pi}\left[Q^{\pi}_{h,\beta}(\x_h,\a_h) - \beta \log \frac{\pi_h(\a_h\mid \x_h)}{\pi_{h,\refe}(\a_h\mid \x_h)} \mid \x_{h-1} = x,\a_{h-1} = a\right], \nn \\
& =  \rstar_{h-1}(x,a)\nn\\
& \quad  +\E\left[ \sum_{a'\in \cA} \pi_h(a'\mid \x_h) \left( Q^{\pi}_{h,\beta}(\x_h,a') - \beta \log \frac{\pi_h(a'\mid \x_h)}{\pi_{h,\refe}(a'\mid \x_h)}\right) \mid \x_{h-1} = x,\a_{h-1} = a\right].\nn 
\end{align}
Therefore, we have 
\begin{align}
& \max_{\pi_{h:H}: \cX \rightarrow \Delta(\cA)} Q^{\pi}_{h-1,\beta}(x,a) = 
\rstar_{h-1}(x,a)\nn \\
&\quad +\E\left[ \max_{\pi_h:\cX\rightarrow \Delta(\cA)} \sum_{a'\in \cA} \pi_h(a'\mid \x_h) \left(  \max_{\pi_{h+1:H}: \cX \rightarrow \Delta(\cA)} Q^{\pi}_{h,\beta}(\x_h,a') - \beta \log \frac{\pi_h(a'\mid \x_h)}{\pi_{h,\refe}(a'\mid \x_h)}\right) \mid \x_{h-1} = x,\a_{h-1} = a\right],\nn 
\intertext{and so using the induction assumption (in particular \eqref{eq:result1} for $\ell=h$), we get}
&= 
\rstar_{h-1}(x,a)\nn \\
& \quad  +\E\left[ \max_{\pi_h:\cX\rightarrow \Delta(\cA)} \sum_{a'\in \cA} \pi_h(a'\mid \x_h) \left( Q^{\star}_{h,\beta}(\x_h,a') - \beta \log \frac{\pi_h(a'\mid \x_h)}{\pi_{h,\refe}(a'\mid \x_h)}\right) \mid \x_{h-1} = x,\a_{h-1} = a\right], \nn \\
& = \rstar_{h-1}(x,a) +\E\left[\sum_{a'\in \cA} \pi_{h,\refe}(a'\mid \x_h)\cdot e^{Q^{\star}_{h,\beta}(\x_h,a')/\beta} \mid \x_{h-1} = x,\a_{h-1} = a\right], \nn \\
& = \rstar_{h-1}(x,a) + \cT_{h,\beta}[Q^\star_{h,\beta}](x,a),\nn \\
& = Q^{\star}_{h,\beta}(x,a).
\end{align}
This shows \eqref{eq:result1} for $\ell =h-1$. Finally, \eqref{eq:result2} for $\ell=h-1$ follows from a direct calculation.
\end{proof}

\begin{lemma}[Performance difference lemma for KL-regularized regret]
	\label{lem:perform}
For all $\pi_{1:H}, \pi'_{1:H}\subset \{\pi \colon \cX \rightarrow \Delta(\cA)\}$, we have 
\begin{align}
	J_\beta(\pi) - J_\beta(\pi') & = \sum_{h=1}^H \E_{\pi'}\left[\sum_{a\in \cA} \pi_h(a\mid \x_h)\cdot \left(Q^{\pi}_{h,\beta}(\x_h,a)- \beta \cdot \log \frac{\pi_h(a\mid \x_h)}{\pi_{h,\refe}(a\mid \x_h)} \right)\right]\nn \\
	& \quad  - \sum_{h=1}^H \E_{\pi'}\left[\sum_{a\in \cA} \pi'_h(a\mid \x_h)\cdot \left( Q_{h,\beta}^{\pi}(\x_h,a) - \beta \cdot \log \frac{\pi'_h(a\mid \x_h)}{\pi_{h,\refe}(a\mid \x_h)} \right) \right], \label{eq:performed}
\end{align}
where $J_\beta (\pi) \coloneqq \sum_{h=1}^H\E_{\pi}\left[\rstar_h(\x_h,\a_h)\right] -  \beta \kl{\pi}{\pi_\refe}$.
\end{lemma}

\begin{proof}[\pfref{lem:perform}]
  Corollary of \cref{rem:general} with $h=1$, taking expectation over $\x_1\sim\rho$.\loose

\end{proof}

\begin{lemma}[Performance difference lemma (generalized version)]
	\label{rem:general}
For all $h\in[H]$, $x\in \cX$, and $\pi_{h:H}, \pi'_{h:H}\subset \{\pi \colon \cX \rightarrow \Delta(\cA)\}$:
\begin{align}
	& \E_{\pi}\left[\sum_{\ell=h}^H \rstar_\ell(\x_\ell,\a_\ell)-\beta \sum_{\ell=h}^H \log \frac{\pi_\ell(\a_\ell \mid \x_\ell)}{\pi_{\ell,\refe}(\a_\ell \mid \x_\ell)} \mid \x_h = x\right] \nn \\
	& \quad - \E_{\pi'}\left[\sum_{\ell=h}^H \rstar_\ell(\x_\ell,\a_\ell)-\beta \sum_{\ell=h}^H \log \frac{\pi'_\ell(\a_\ell \mid \x_\ell)}{\pi_{\ell,\refe}(\a_\ell \mid \x_\ell)} \mid \x_h = x\right]\nn \\
	& = \sum_{\ell=h}^H \E_{\pi'}\left[\sum_{a\in \cA} \pi_\ell(a\mid \x_\ell)\cdot \left(Q^{\pi}_{\ell,\beta}(\x_\ell,a)- \beta \cdot \log \frac{\pi_\ell(a\mid \x_\ell)}{\pi_{\ell,\refe}(a\mid \x_\ell)} \right) \mid \x_h = x\right]\nn \\
	& \quad  - \sum_{\ell=h}^H \E_{\pi'}\left[\sum_{a\in \cA} \pi'_\ell(a\mid \x_\ell)\cdot \left( Q_{\ell,\beta}^{\pi}(\x_\ell,a) - \beta \cdot \log \frac{\pi'_\ell(a\mid \x_\ell)}{\pi_{\ell,\refe}(a\mid \x_\ell)} \right) \mid \x_h = x\right].\nn
\end{align}
\end{lemma}
\begin{proof}[\pfref{rem:general}] First, for any $\pi_{1:H}\subset \{\pi \colon \cX \rightarrow \Delta(\cA)\}$, $(x,a)\in \cX \times \cA$, $h\in[H]$, define 
	\begin{gather}
		\arxiv{
		r^{\pi}_h(x,a)  \coloneqq 
		 \rstar_h(x,a) - \beta \log \frac{\pi_h(a \mid x)}{\pi_{h,\refe}(a \mid x)},  \quad \text{and} \quad \widebar{Q}_{h,\beta}^{\pi}(x,a) = \E_{\pi}\left[\sum_{\ell=h}^H r^\pi_\ell(\x_\ell,\a_\ell)\mid  \x_h = x,\a_h =a \right]} 
		 \nn 
		\shortintertext{and note that}
		Q^{\pi}_{h,\beta}(x,a) =  \widebar{Q}_{h,\beta}^\pi(x,a) + \beta \log \frac{\pi_h(a\mid x)}{\pi_{h,\refe}(a\mid x)}. \label{eq:Qfun0}
	\end{gather}
	We need to show that for all $h\in[H]$, $x\in \cX$, and $\pi_{h:H}, \pi'_{h:H}\subset \{\pi \colon \cX \rightarrow \Delta(\cA)\}$:
	\begin{align}
		& \E_{\pi'\circ_h \pi}\left[\sum_{\ell=h}^H r^\pi_\ell(\x_\ell,\a_\ell)\mid \x_h = x\right] - \E_{\pi'}\left[\sum_{\ell=h}^H r^{\pi'}_\ell(\x_\ell,\a_\ell) \mid \x_h = x\right]\nn \\
		& 	 = \sum_{\ell=h}^H \E_{\pi'}\left[\sum_{a\in \cA} \pi_\ell(a\mid \x_\ell)\cdot \left(Q^{\pi}_{\ell,\beta}(\x_\ell,a)- \beta \cdot \log \frac{\pi_\ell(a\mid \x_\ell)}{\pi_{\ell,\refe}(a\mid \x_\ell)} \right) \mid \x_h = x\right]\nn \\
		& \quad  - \sum_{\ell=h}^H \E_{\pi'}\left[\sum_{a\in \cA} \pi'_\ell(a\mid \x_\ell)\cdot \left( Q_{\ell,\beta}^{\pi}(\x_\ell,a) - \beta \cdot \log \frac{\pi'_\ell(a\mid \x_\ell)}{\pi_{\ell,\refe}(a\mid \x_\ell)} \right) \mid \x_h = x\right].\nn
	\end{align}
	Fix $h\in [H]$, $x\in \cX$, and $\pi_{1:H}, \pi'_{1:H}\subset \{\pi \colon \cX \rightarrow \Delta(\cA)\}$.
	We now show via induction over $j=h,\dots,H+1$ that
	\begin{align}
		&\E_{\pi'\circ_h \pi}\left[\sum_{\ell=h}^H r^\pi_\ell(\x_\ell,\a_\ell)\mid \x_h = x\right] - \E_{\pi'}\left[\sum_{\ell=h}^H r^{\pi'}_\ell(\x_\ell,\a_\ell) \mid \x_h = x\right]\nn \\
 & = \sum_{\ell = h}^{j-1} \E_{\pi'}\left[\sum_{a\in \cA} \pi_\ell(a\mid \x_\ell)\cdot \widebar{Q}^{\pi}_{\ell,\beta}(\x_\ell,a) - \sum_{a\in \cA} \pi'_\ell(a\mid \x_\ell)\cdot \widebar{Q}_{\ell,\beta}^{\pi}(\x_\ell,a)\mid \x_h = x\right]\nn \\
		& \quad + \beta \sum_{\ell = h}^{j-1} \E_{\pi'}\left[\log \frac{\pi'_\ell(\a_\ell \mid \x_\ell)}{\pi_\ell(\a_\ell \mid \x_\ell)}\mid \x_h = x\right] \nn \\
		& \quad  + \E_{\pi'\circ_j \pi}\left[\sum_{\ell=j}^H r^{\pi}_\ell(\x_\ell,\a_\ell)\mid \x_h = x\right]- \E_{\pi'}\left[\sum_{\ell=j}^H r^{\pi'}_\ell(\x_\ell,\a_\ell)\mid \x_h = x\right]. \label{eq:toinduct0}
	\end{align}
	The base case $j=h$ is trivial. Now assume that \eqref{eq:toinduct0} holds for $j\in\crl*{h,\ldots,H+1}$; we show\arxiv{ that} it holds for $j+1$:\loose
	\begin{align}
		& \E_{\pi'\circ_j \pi}\left[\sum_{\ell=j}^H r^\pi_\ell(\x_\ell,\a_\ell)\mid \x_h = x\right] - \E_{\pi'}\left[\sum_{\ell=j}^H r^{\pi'}_\ell(\x_\ell,\a_\ell)\mid \x_h = x\right]\nn \\
		& =  \E_{\pi' \circ_j \pi}\left[\sum_{\ell=j}^H r^\pi_\ell(\x_\ell,\a_\ell)\mid \x_h = x\right] - \E_{\pi'\circ_{j+1} \pi}\left[r^{\pi'}_\ell(\x_\ell,\a_\ell)+ \sum_{\ell=j+1}^H r^{\pi}_\ell(\x_\ell,\a_\ell)\mid \x_h = x\right] \nn \\
		& \quad + \E_{\pi'\circ_{j+1} \pi}\left[r^{\pi'}_\ell(\x_\ell,\a_\ell)+ \sum_{\ell=j+1}^H r^{\pi}_\ell(\x_\ell,\a_\ell)\mid \x_h = x\right] -  \E_{\pi'}\left[\sum_{\ell=j}^H r^{\pi'}_\ell(\x_\ell,\a_\ell)\mid \x_h = x\right],\nn \\
		& =  \E_{\pi' \circ_j \pi}\left[\sum_{\ell=j}^H r^\pi_\ell(\x_\ell,\a_\ell)\mid \x_h = x\right] - \E_{\pi'\circ_{j+1} \pi}\left[r^{\pi'}_\ell(\x_\ell,\a_\ell)+ \sum_{\ell=j+1}^H r^{\pi}_\ell(\x_\ell,\a_\ell)\mid \x_h = x\right] \nn \\
		& \quad + \E_{\pi'\circ_j \pi}\left[\sum_{\ell=j+1}^H r^{\pi}_\ell(\x_\ell,\a_\ell)\mid \x_h = x\right] -  \E_{\pi'}\left[\sum_{\ell=j+1}^H r^{\pi'}_\ell(\x_\ell,\a_\ell)\mid \x_h = x\right],\nn \\
		& =  \E_{\pi' \circ_j \pi}\left[\sum_{\ell=j}^H r^\pi_\ell(\x_\ell,\a_\ell)\mid \x_h = x\right] - \E_{\pi'\circ_{j+1} \pi}\left[ \sum_{\ell=j}^H r^{\pi}_\ell(\x_\ell,\a_\ell)\mid \x_h = x\right] \nn \\
		& \quad + \E_{\pi'}\left[r^{\pi}_\ell(\x_\ell,\a_\ell)-r^{\pi'}_\ell(\x_\ell,\a_\ell) \mid \x_h = x\right] \nn\\
		& \quad + \E_{\pi'\circ_{j+1} \pi}\left[\sum_{\ell=j+1}^H r^{\pi}_\ell(\x_\ell,\a_\ell)\mid \x_h = x\right] -  \E_{\pi'}\left[\sum_{\ell=j+1}^H r^{\pi'}_\ell(\x_\ell,\a_\ell)\mid \x_h = x\right],\nn \\
		& =  \E_{\pi' \circ_j \pi}\left[\sum_{a\in \cA} \pi_j(a\mid \x_\ell)\cdot \widebar{Q}^{\pi}_{j,\beta}(\x_j,a) - \sum_{a\in \cA} \pi'_j(a\mid \x_j)\cdot \widebar{Q}_{j,\beta}^{\pi}(\x_j,a)\mid \x_h = x\right] \nn \\
		& \quad + \beta \E_{\pi'}\left[\log \frac{\pi'_j(\a_j \mid \x_j)}{\pi_j(\a_j \mid \x_j)}\mid \x_h = x\right] \nn\\
		& \quad + \E_{\pi'\circ_j \pi}\left[\sum_{\ell=j+1}^H r^{\pi}_\ell(\x_\ell,\a_\ell)\mid \x_h = x\right] -  \E_{\pi'}\left[\sum_{\ell=j+1}^H r^{\pi'}_\ell(\x_\ell,\a_\ell)\mid \x_h = x\right].
	\end{align}
	This shows that \eqref{eq:toinduct0} holds with $j$ replaced by $j+1$ and completes the induction. Instantiating \eqref{eq:toinduct0} with $\ell=H+1$ shows that
	\begin{align}
		& \E_{\pi'\circ_h \pi}\left[\sum_{\ell=h}^H r^\pi_\ell(\x_\ell,\a_\ell)\mid \x_h = x\right] - \E_{\pi'}\left[\sum_{\ell=h}^H r^{\pi'}_\ell(\x_\ell,\a_\ell) \mid \x_h = x\right]\nn\\ & = \sum_{\ell=h}^{H} \E_{\pi'}\left[\sum_{a\in \cA} \pi_\ell(a\mid \x_\ell)\cdot \widebar{Q}^{\pi}_{\ell,\beta}(\x_\ell,a) - \sum_{a\in \cA} \pi'_\ell(a\mid \x_\ell)\cdot \widebar{Q}_{\ell,\beta}^{\pi}(\x_\ell,a)+ \beta \log \frac{\pi'_\ell(\a_\ell \mid \x_\ell)}{\pi_\ell(\a_\ell \mid \x_\ell)}\mid \x_h = x\right].\nn 
	\end{align}
	Combining this with \eqref{eq:Qfun0} implies the desired result. 
	\end{proof}

	\begin{lemma}
		\label{lem:multiset}
		Let $\cC\subset \cX \times \cA$ be a multiset of the form \begin{align}\cC = \bigcup_{i\in [N]} \{(x_i,a_i), (x_i,\afrak)\}, \label{eq:multiset}
		\end{align} for $N\geq 1$. 
		Then, for any non-negative $f\colon \cX\times \cA \rightarrow \reals$, we have
		\begin{align}
		\sum_{(x ,a )\in  \cC} f(x ,a ) + \sum_{(x ,a )\in  \cC} f(x ,\mathfrak{a}) \leq 3  \sum_{(x ,a )\in  \cC} f(x ,a ). \label{eq:multiset_lemma}
		\end{align}
	\end{lemma}
\begin{proof}[\pfref{lem:multiset}]
Because $\cC$ is a multiset satisfying \eqref{eq:multiset} and $f$ is non-negative, we have
\begin{align}
\sum_{(x ,a )\in  \cC} f(x ,a ) &  \geq \sum_{(x ,a )\in  \cC: a \neq \mathfrak{a}} f(x ,a ) + \sum_{(x ,a )\in  \cC: a \neq \mathfrak{a}} f(x ,\mathfrak{a}), \nonumber\\
& \geq \sum_{(x ,a )\in  \cC: a \neq \mathfrak{a}} f(x ,\mathfrak{a}).
\end{align} 
On the other hand, we also have that 
\begin{align}
\sum_{(x ,a )\in  \cC} f(x ,a ) &  = \sum_{(x ,a )\in  \cC: a \neq \mathfrak{a}} f(x ,a ) + \sum_{(x ,a )\in  \cC: a = \mathfrak{a}} f(x ,\mathfrak{a}), \nonumber\\
& \geq \sum_{(x ,a )\in  \cC: a = \mathfrak{a}} f(x ,\mathfrak{a}).
\end{align}
Combining (1) and (2) implies that 
\begin{align}
\sum_{(x ,a )\in  \cC} f(x ,\mathfrak{a})&=\sum_{(x ,a )\in  \cC: a = \mathfrak{a}} f(x ,\mathfrak{a}) + \sum_{(x ,a )\in  \cC: a \neq  \mathfrak{a}} f(x ,\mathfrak{a}),\nonumber \\
& \leq 2\sum_{(x ,a )\in  \cC} f(x ,a ),\nonumber
\end{align}
which implies \eqref{eq:multiset_lemma} after adding $\sum_{(x ,a )\in  \cC} f(x ,a )$ on both sides.
\end{proof}

\end{document}